\documentclass[11pt,a4paper,twoside]{book}
\pdfoutput=1

\usepackage[utf8]{inputenc}
\usepackage[T1]{fontenc}
\usepackage{microtype}
\usepackage{amsfonts}
\usepackage{amsmath}
\usepackage{amssymb}
\usepackage{amsthm}
\usepackage{mathpazo}
\usepackage{nicefrac}
\usepackage{bm}
\usepackage{mathtools}
\mathtoolsset{showonlyrefs}

\usepackage[paper=a4paper,hmarginratio=1:1,vmarginratio=1:1,scale=0.75,bindingoffset=5mm]{geometry}
\usepackage[dvipsnames]{xcolor}
\usepackage[parfill]{parskip}
\usepackage[inline]{enumitem}

\usepackage{setspace}
\usepackage{ifthen}

\usepackage[nottoc]{tocbibind}
\usepackage[title,titletoc]{appendix}

\usepackage{hyperref}
\hypersetup{
  final=true,
  plainpages=false,
  pdfstartview=FitV,
  pdftoolbar=true,
  pdfmenubar=true,
  bookmarksopen=true,
  bookmarksnumbered=true,
  breaklinks=true,
  linktocpage,
  colorlinks=true,
  linkcolor=NavyBlue,
  urlcolor=ForestGreen,
  citecolor=NavyBlue,
  anchorcolor=ForestGreen,
  pdfpagelayout=TwoPageRight,
  pdftitle={Beyond traditional assumptions in fair machine learning},
  pdfauthor={Niki Kilbertus},
}
\usepackage{url}
\usepackage{fancyhdr}
\fancypagestyle{plain}{
  \fancyhead{}
  
}

\renewcommand{\cleardoublepage}{\clearpage\if@twoside \ifodd\c@page\else%
  \hbox{}%
  \thispagestyle{empty}
  \newpage%
  \if@twocolumn\hbox{}\newpage\fi\fi\fi}

\usepackage{algorithm}
\usepackage[noend]{algpseudocode}

\algnewcommand\algorithmicinput{\textbf{Input:}}
\algnewcommand\Input{\item[\algorithmicinput]}

\algnewcommand\algorithmicinputx{\textbf{Input}}
\algnewcommand\Inputx{\item[\algorithmicinputx]}

\usepackage{xpatch}
\makeatletter
\xpatchcmd{\algorithmic}{\itemsep\z@}{\itemsep=0.7ex plus1pt}{}{}
\makeatother

\usepackage[pdftex]{graphicx}
\usepackage[labelsep=space,tableposition=top]{caption}
\usepackage{wrapfig}
\usepackage[position=top]{subfig}
\usepackage{booktabs}
\usepackage{multirow}
\usepackage{tabularx}
\usepackage{ragged2e}
\newcolumntype{Y}{>{\RaggedRight\arraybackslash}X}
\usepackage{siunitx}

\RequirePackage{natbib}
\setcitestyle{authoryear,round,citesep={;},aysep={,},yysep={;}}

\usepackage{titlesec}

\titleformat{name=\chapter}[display]
{\filleft\bfseries}
{\hfill\textcolor{gray!75}{{\fontsize{50}{60}\selectfont \thechapter}}}
{-12pt}
{\Huge}

\titleformat{name=\chapter,numberless}[display]
{\filleft\bfseries}
{\filleft}
{-12pt}
{\Huge}

\usepackage{tikz}
\usepackage{pgfplots}
\usepackage{pgfplotstable}
\usepgfplotslibrary{groupplots,colorbrewer}
\pgfplotsset{compat=newest}
\pgfplotstableset{col sep=comma}
\usetikzlibrary{external,calc,shapes,fadings,decorations,arrows,positioning,cd,decorations.pathmorphing,decorations.shapes}

\definecolor{py-3-1}{HTML}{1F77B4}
\definecolor{py-3-2}{HTML}{FF7F0E}
\definecolor{py-3-3}{HTML}{2CA02C}
\definecolor{py-4-1}{HTML}{D62728}
\definecolor{py-4-2}{HTML}{9467BD}
\definecolor{py-4-3}{HTML}{BCBD22}
\definecolor{py-4-4}{HTML}{17BECF}
\definecolor{DarkRed}{RGB}{165, 0, 38}

\renewcommand{\maketitle}{

\thispagestyle{empty}
\begin{singlespace}
\begin{center}

\begin{minipage}[c]{\textwidth}
  \centering \Huge \bfseries{Beyond traditional assumptions\texorpdfstring{\\}{} in fair machine learning}
\end{minipage}

\vspace*{2.5cm}

{\includegraphics[width=0.2\textwidth]{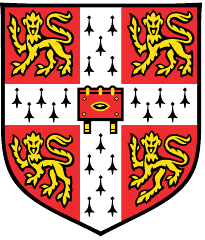} \par}

\vspace*{1.5cm}

\begin{minipage}[c]{\textwidth}
  \centering \Large \bfseries{Niki Kilbertus}
\end{minipage}

\vspace*{1cm}

\begin{minipage}[c]{\textwidth}
    \centering \Large Supervisor: Prof.~Dr.~Carl~E.~Rasmussen
\end{minipage}

\vspace*{0.5cm}

\begin{minipage}[c]{\textwidth}
    \centering \Large Advisor: Dr.~Adrian~Weller
\end{minipage}

\vspace*{1cm}

\begin{minipage}[c]{\textwidth}
  \centering 
  {\large Department of Engineering \par}
  {\large University of Cambridge \par}
\end{minipage}

\vspace{3cm}

\begin{minipage}[c]{\textwidth}
  \centering
  {\large This thesis is submitted for the degree of  \par}
  {\large \emph{Doctor of Philosophy} \par}
\end{minipage}

\vfill

\begin{minipage}[c]{\textwidth}
  \large
    \begin{minipage}[b]{0.49\textwidth}
      \flushleft Pembroke College
    \end{minipage}
    \begin{minipage}[b]{0.49\textwidth}
      \flushright October 2020
    \end{minipage}
\end{minipage}

\end{center}

\end{singlespace}
}

\newcommand\B[1]{{\bm{#1}}}

\newcommand{\share}[2]{\langle {#1}\rangle_{#2}}
\DeclareMathOperator*{\argmin}{arg\,min}

\DeclareMathOperator*{\diag}{diag}
\DeclareMathOperator*{\sign}{sign}
\DeclareMathOperator*{\maximize}{maximize}

\newcommand{\usub}[2]{\underset{#1}{\underbrace{#2}}}
\newcommand*{\given}{{\,|\,}}
\DeclareMathOperator{\indep}{\perp\!\!\!\perp}
\DeclareMathOperator{\dep}{\not\! \perp\!\!\!\perp}

\newcommand{\bN}{\mathbb{N}}

\newcommand{\bR}{\mathbb{R}}
\newcommand{\bZ}{\mathbb{Z}}
\DeclareMathOperator{\E}{\mathbb{E}}

\DeclareMathOperator{\dP}{P}
\DeclareMathOperator{\prob}{P}
\DeclareMathOperator{\distr}{P}
\DeclareMathOperator{\dens}{p}
\DeclareMathOperator{\var}{\mathrm{var}}
\DeclareMathOperator{\cov}{\mathrm{cov}}
\DeclareMathOperator{\pa}{\mathrm{pa}}

\newcommand{\cD}{\mathcal{D}}

\newcommand{\cG}{\mathcal{G}}

\newcommand{\cN}{\mathcal{N}}

\newcommand{\cX}{\mathcal{X}}
\newcommand{\cY}{\mathcal{Y}}
\newcommand{\cZ}{\mathcal{Z}}

\newcommand{\sspace}{\ensuremath \mathcal{Z}}
\newcommand{\xspace}{\ensuremath \mathcal{X}}
\newcommand{\outspace}{\ensuremath \mathcal{Y}}
\newcommand{\Qcal}{\mathcal{Q}}

\newcommand{\F}{\bm{F}}
\newcommand{\const}{\B{c}}
\newcommand{\avec}{\B{a}}
\newcommand{\x}{\B{x}}
\newcommand{\y}{\B{y}}
\newcommand{\z}{\B{z}}
\newcommand{\A}{\B{A}}
\newcommand{\X}{\B{X}}
\newcommand{\Z}{\B{Z}}
\newcommand{\bphi}{\ensuremath \B{\phi}}
\newcommand{\btheta}{{\B{\theta}}}
\newcommand{\blambda}{{\B{\lambda}}}

\newcommand{\tcG}{\tilde{\mathcal{G}}}
\newcommand{\hY}{\hat{Y}}
\newcommand{\hy}{\hat{y}}
\newcommand{\pitil}{\ensuremath \tilde{\pi}}

\DeclareMathOperator{\Ber}{\mathrm{Ber}}
\DeclareMathOperator{\CFU}{\ensuremath \mathrm{CFU}}
\newcommand{\mul}{\texttt{mul}}
\newcommand{\user}{{\textsc U}}
\newcommand{\dc}{\textsc{M}}
\newcommand{\reg}{\textsc{Reg}}

\def\eqp{\: .}
\def\eqc{\: ,}

\newcommand{\qtxtq}[1]{\quad\text{#1}\quad}

\newcommand{\for}{\quad \text{for }}
\newcommand{\fora}{\quad \text{for }}

\newcommand{\secref}[1]{Section~\ref{#1}}
\newcommand{\secsref}[1]{Sections~\ref{#1}}

\newcommand{\chapref}[1]{Chapter~\ref{#1}}
\newcommand{\chapsref}[1]{Chapters~\ref{#1}}
\newcommand{\Chapref}[1]{Chapter~\ref{#1}}

\newcommand\blfootnote[1]{%
  \begingroup
  \renewcommand\thefootnote{}\footnote{#1}%
  \addtocounter{footnote}{-1}%
  \endgroup
}

\newcommand{\xhdr}[1]{\noindent{\textbf{#1.}}}

\newcommand{\newdef}[1]{\textbf{#1}}

\newcommand{\directquote}[1]{``\emph{#1}''}

\newcommand{\match}[1]{{\color{py-3-1}\text{#1}} = {\color{py-3-2}\text{#1}'}}

\newtheorem{theorem}{Theorem}
\newtheorem*{theorem*}{Theorem}

\newtheorem*{lemma*}{Lemma}

\newtheorem{corollary}{Corollary}
\newtheorem*{corollary*}{Corollary}

\newtheorem{proposition}{Proposition}
\newtheorem*{proposition*}{Proposition}

\theoremstyle{definition}
\newtheorem{definition}{Definition}

\tikzstyle{obs} = [circle,fill=white,draw=black,inner sep=1pt,minimum size=20pt,font=\fontsize{10}{10}\selectfont,node distance=1,thick]
\tikzstyle{latent} = [obs,dotted]
\tikzstyle{protected} = [obs,text=Orange,draw=Orange]
\tikzstyle{unfair} = [obs,text=BrickRed,draw=BrickRed]
\tikzstyle{target} = [obs,text=MidnightBlue,draw=MidnightBlue]
\tikzstyle{feature} = [obs,text=ForestGreen,draw=ForestGreen]
\tikzstyle{simple} = [thick,circle,draw,minimum height=1.7em,inner sep=0pt,text centered]

\newcommand{\edge}[3][]{ %
  \foreach \x in {#2} { %
    \foreach \y in {#3} { %
      \path (\x) edge [->, >=latex, #1,thick] (\y) ;%
    } ;
  } ;
}

\pgfplotsset{hline/.style={%
  before end axis/.append code={%
    \draw[densely dashed, black, opacity=1]
    ({rel axis cs:0,0} |- {axis cs:0,#1}) -- ({rel axis cs:1,0} |- {axis cs:0,#1});
  }
}}

\pgfplotsset{vline/.style={%
  before end axis/.append code={%
    \draw[densely dashed, black, opacity=1]
    ({rel axis cs:0,0} |- {axis cs:#1,0}) -- ({rel axis cs:0,1} |- {axis cs:#1,0});
  }
}}

\newcommand{\pubitem}[7]{
  \ifthenelse{\equal{#6}{}}{}{$\triangleright\triangleright$} \ifthenelse{\equal{#4}{}}{\textsc{#1}}{%
  \href{#4}{\textsc{#1}}}\\
  {\small #2.}
  \ifthenelse{\equal{#5}{}}{}{\\{\small \href{#5}{[\url{#5}]}}}
  \\
  \emph{#3}%
}

\begin{document}

\frontmatter

\maketitle

\cleardoublepage
\chapter*{\centering \Large Acknowledgments}
\thispagestyle{empty}

The four years of my PhD have been filled with enriching and fun experiences.
I owe all of them to interactions with exceptional people.
Carl Rasmussen and Bernhard Schölkopf have put trust and confidence in me from day one.
They enabled me to grow as a researcher and as a human being.
They also taught me not to take things too seriously in moments of despair.
Thank you!
Adrian Weller's contagious positive energy gave me enthusiasm and drive.
He showed me how to be considerate and relentless at the same time.
I thank him for his constant support and sharing his extensive network.
Among others, he introduced me to Matt~J.~Kusner, Ricardo Silva, and Adri\`a Gasc\'on who have been amazing role models and collaborators.
I hope for many more joint projects with them!
Moritz Hardt continues to be an inexhaustible source of inspiration and I want to thank him for his mentorship and guidance during my first paper.
I was extremely fortunate to collaborate with outstanding people during my PhD beyond the ones already mentioned.
I have learned a lot from every single one of them, thank you:
Philip Ball,
Stefan Bauer,
Silvia Chiappa,
Elliot Creager,
Timothy Gebhard,
Manuel Gomez Rodriguez,
Anirudh Goyal,
Krishna~P.~Gummadi,
Ian Harry,
Dominik Janzing,
Francesco Locatello,
Krikamol Muandet,
Frederik Träuble,
Isabel Valera, and
Michael Veale.

I cannot begin to describe how grateful I am for all the hours I spent with the puppies:
GB woof (Giambattista Parascandolo),
Mateo (Chan) Rojas-Carulla,
Poru (Paul Rubenstein),
Alessandro Labongo,
John Bradshaw,
Axel Neitz, and
Ilya (Sensei) Tolstikhin.
Thank you for everything!
I always thoroughly enjoyed the two-person reading groups and collaborations with Jiri Hron as well as our Itsu dates with Matej Balog.
Timothy Gebhard showed me how productive and satisfying close collaboration can be, thank you for that.
The (rest of the) CamTue crew was a great source of joy and support as well:
Adam \'Scibior,
Shane Gu,
Matthias Bauer,
Chaochao Lu,
Julius von Kügelgen,
Vincent Stimper, and
Erik Daxberger.
There are countless people I have not collaborated with, who have contributed (knowingly or unknowingly) to  making my PhD such an amazing time.
At the risk of missing some (sorry!):
Borja Balle,
Sebastian Gomez-Gonzalez,
Dieter Büchler,
Diego Agudelo-Espa\~na,
Jonas Kübler,
Sebastian Weichwald,
Luigi Gresele,
Okan Koc,
Carl-Johann Simon-Gabriel,
Behzad Tabibian,
Justus Winkelmann,
Danilo Brajovic,
Tor Fjelde,
Karl Krauth,
Mark Rowland,
Amir Hossein Karimi,
Tameem Adel,
John Bronskill,
Jonathon Gordon,
Adri\`a Garriga-Alonso,
Robert Pinsler,
Siddarth Swaroop,
and many more.

I thank Mama, Papa, Hanna, Pauli and the rest of my family for their unconditional love and support.
Finally, none of this would have been possible without my extraordinary wife Jani, who put up with me every day of my PhD and still loves me.

\cleardoublepage
\chapter*{\centering \Large Abstract}
\thispagestyle{empty}

This thesis scrutinizes common assumptions underlying traditional machine learning approaches to fairness in consequential decision making.
After challenging the validity of these assumptions in real-world applications, we propose ways to move forward when they are violated.
First, we show that group fairness criteria purely based on statistical properties of observed data are fundamentally limited.
Revisiting this limitation from a causal viewpoint we develop a more versatile conceptual framework, causal fairness criteria, and first algorithms to achieve them.
We also provide tools to analyze how sensitive a believed-to-be causally fair algorithm is to misspecifications of the causal graph.
Second, we overcome the assumption that sensitive data is readily available in practice.
To this end we devise protocols based on secure multi-party computation to train, validate, and contest fair decision algorithms without requiring users to disclose their sensitive data or decision makers to disclose their models.
Finally, we also accommodate the fact that outcome labels are often only observed when a certain decision has been made.
We suggest a paradigm shift away from training predictive models towards directly learning decisions to relax the traditional assumption that labels can always be recorded.
The main contribution of this thesis is the development of theoretically substantiated and practically feasible methods to move research on fair machine learning closer to real-world applications.

\tableofcontents
\listoffigures
\listoftables

\mainmatter

\chapter{Introduction and overview}
\label{chap:intro}

\graphicspath{{figs/chap1/}}

\section{Motivation}
\label{sec:motivation}

As machine learning penetrates all aspects of our everyday lives and the push for automation moves beyond industrial settings to decisions about human beings, we are facing a multitude of new challenges.
While the performance goals for narrowly confined industrial automation tasks are often easy to express, automated decisions affecting people'{}s livelihood and well-being must meet more complex requirements.
We want to verify that these systems make ethically agreeable decisions and that they do so for the right reasons.
These considerations include a broad range of concerns about fairness, trust, accountability, transparency, and privacy.
Similar issues are also faced by human decision makers and have been the subject of various fields of research long before the advent of computers and large scale data analytics.
However, software is a human creation and thereby, in principle, we expect to have full control over its actions.
Hence, we may also hope to exert full control over the consequences of its deployment.
However, modern data-driven machine learning systems can be notoriously hard to interpret and their downstream impact unpredictable.
Moreover, because of their scalability, the potential for harm due to misspecifications of an automated system often by far exceeds what any individual human decision maker can incur.
Unsurprisingly, undesirable effects of machine learning systems have already been observed in practice.

What do we mean by ``undesirable effects''?
There is an ongoing debate whether morality is innate or a socially constructed ideal among ethicists and philosophers.
Arguably, the nuances of our present-day understanding still depend on subjective human judgment both on the level of individuals as well as societies and cultures.
Therefore, throughout this thesis, we will avoid attempts of an objective definition or even opinion of ``right or wrong'' and ``fair or unfair''.
This is not to say that there will not be definitions of fairness.
In fact, we will encounter whole families of formally defined fairness criteria.
However, we explicitly consider those to be mere candidates that may possibly cover some reasonable dimensions of what is considered fair by some people within a given context.
Even though our language will not always be as cautious as in the previous sentence, we ask the reader to consider what follows from this viewpoint.

Now, to illustrate the dangers of deploying machine learning systems in the social context, we briefly highlight a small collection of recent examples of harmful bias.
In many cases we do not know the specifics of the underlying systems, in particular, whether they were powered by machine learning.
They are meant to serve as clarifying examples that could have plausibly been generated by machine learning algorithms and share a vast scale of impact enabled by digital technologies.
The wide deployment of such systems as well as their opaqueness have also been identified by Cathy O'Neil as key characteristics of ``weapons of math destruction'' \citep{o2016weapons}.
For a popular introduction containing plentiful examples of how algorithms can scale and accelerate inequality and injustice, we also point the reader to \citet{noble2018algorithms,benjamin2019race,broussard2018artificial,eubanks2018automating,kearns2019ethical}.
For further reading on the risks and dangers of data-driven decision support systems and the sources of bias or discrimination from a more technical perspective, we refer the reader to the excellent online book by \citet{barocas-hardt-narayanan} as well as reports by \citet{Barocas2016,House2016}.

\paragraph{Criminal justice.}
A decision support tool called \emph{Correctional Offender Management Profiling for Alternative Sanction} (COMPAS) was used in U.S.\ courts to estimate the recidivism risk of defendants for decisions on pre-trial detention.
Reporters at ProPublica claimed that COMPAS is biased against blacks, because among convicts that did not go on to re-offend within two years of their release, blacks had received systematically worse scores than whites \citep{Angwin2016}.

\paragraph{Bias in search and recommendation.}
In web searches for names, Google'{}s advertisement delivery algorithm AdSense was found to deliver ads suggestive of arrest more often for names primarily assigned to black people than for those assigned to white people \citep{sweeney2013discrimination}.
Furthermore, stereotype exaggeration and systematic underrepresentation of women were found in image search results for a variety of occupations \citep{kay2015unequal}.

\paragraph{Exam grading.}
In a randomized study of discrimination in grading in India, it was found that teachers assigned worse scores to exams seemingly written by children from lower castes consistent with statistical discrimination \citep{hanna2012discrimination}.

\paragraph{Prime free same-day delivery.}
When Amazon rolled out free same-day delivery, the region of availability excluded predominantly black ZIP codes in some cities \citep{Ingold2016}.

\paragraph{Image recognition.}
After a public media outcry over Google Photos labeling pictures of black people as gorillas in 2015 \citep{Dougherty2015}, Amazon received bad publicity in 2018, because their commercial facial recognition system ``Rekognition'' falsely matched 28 members of congress with mugshots.
Black members were disproportionately affected \citep{Snow2018}.
An academic study of commercially available gender classification software from images confirms substantial disparities in accuracy for different demographic groups formed by gender and skin tone with darker-skinned females being the most misclassified group \citep{buolamwini2018gender}.

\paragraph{Bias in semantics derived from human language.}
Automatically deriving features and semantics from large corpora of human language is an important tool in natural language processing.
Such systems have been shown to contain human-like biases, often reflecting problematic associations regarding race and gender \citep{caliskan2017semantics}.
For example, popular word embeddings have been shown to establish a similar relation between ``man'' and ``computer programmer'' as between ``woman'' and ``homemaker'' among similar examples \citep{bolukbasi2016man}.
These imprints of historical bias can also be found in machine learning powered translation software such as Google translate, see Figure~\ref{fig:google_translate}.
As yet another example, a chatbot named ``Tay'' launched by Microsoft to engage in conversations with people on Twitter in 2016 soon posted wildly racist tweets and reprehensible images as it learned from its conversations \citep{Lee2016}.

\begin{figure}
\centering
\includegraphics[width=0.6\textwidth]{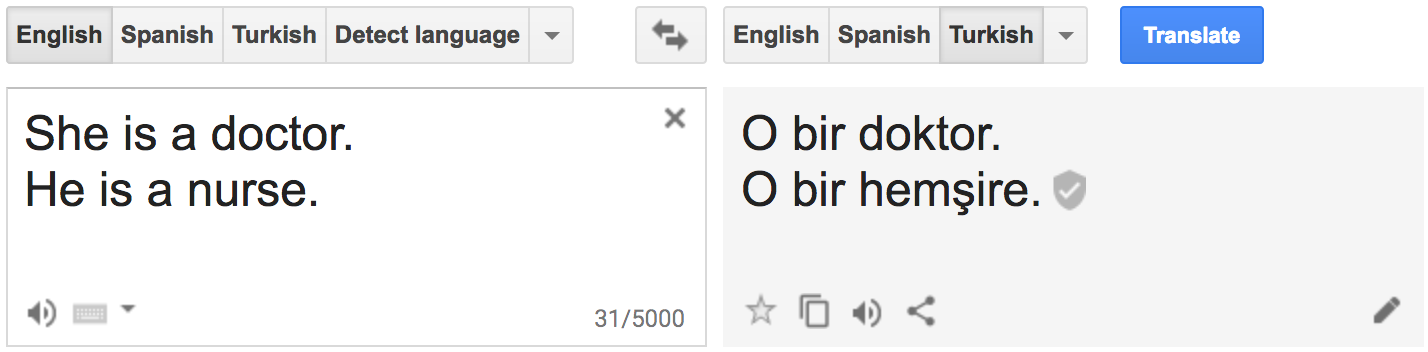}\\
\includegraphics[width=0.6\textwidth]{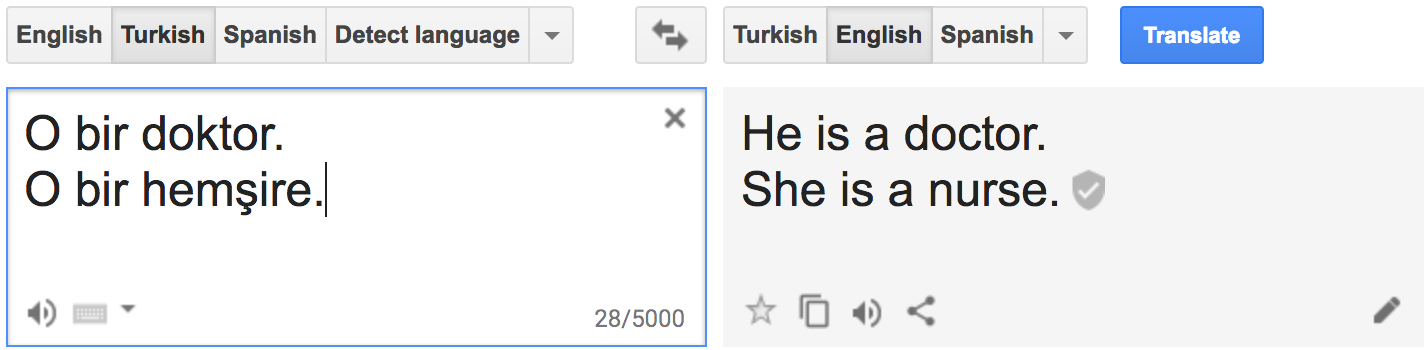}
\caption[An example of stereotypes in algorithmic language translation]{These translation results on \url{https://translate.google.com} were obtained on August 11, 2018 and show stereotypical gender association of certain occupations.}
\label{fig:google_translate}
\end{figure}

\paragraph{Focus of this thesis.}
These examples underline the breadth of the potential for algorithmic discrimination.
In this thesis we will focus on scenarios in which machine learning systems are involved in taking concrete decisions (typically one out of a finite set of possible actions) that directly affect individuals well-being and livelihoods.
The criminal justice and exam grading examples fit well into this category.
Other applications could be lending, hiring, or insurance decisions and treatment choices in health-care.
Furthermore, we typically assume the algorithm to directly trigger decisions deterministically.
In practice, the more likely scenario may be one in which the algorithm only supports or informs human decision makers in various ways.
Ultimately, the core issues we discuss remain regardless of whether the algorithm takes, supports, or informs decisions.

\section{Outline and contributions}
\label{sec:outline}

Albeit not immediately relevant for the technical contributions of this thesis, we will briefly dive a little bit deeper into fundamental questions of fairness and justice from a philosophical viewpoint in \chapref{chap:background}.
Readers only interested in the methodological contributions of the thesis may skip this chapter.

\Chapref{chap:existing} contains an introduction to existing approaches to fair machine learning for decision making.
After an overview of the general setting and some notation, we focus on notions of fairness for classification or decision tasks with a finite set of actions.
We specifically emphasize the assumptions required for the respective approaches and the resulting limitations.
The remaining chapters each propose solutions that allow us to relaxing some of these restrictive assumptions.

In \chapref{chap:causal}, we take a critical look at common statistical fairness criteria.
Through the lens of causality, we crisply articulate why and when specific group fairness notions fail.
These insights suggest to shift attention from the elusive question ``What is the right fairness criterion?'' to ``What do we want to assume about our model of the causal data generating process?''
We propose the language of causality as a useful technical tool to scrutinize and talk about such assumptions.
Within that framework, we further expose previously ignored subtleties fundamental to the problem, e.g., how to interpret protected attributes in a causal sense.
For example, in most of the real-world examples at the end of \secref{sec:motivation}, we use terms such as ``perceived'' or ``seemingly'' in conjunction with group membership indicators such as ``black'' or ``disadvantaged''.
We propose two complementary viewpoints for interpreting causal influences that enable practitioners to better navigate the conceptual fuzziness around these terms.
Finally, we put forward natural causal non-discrimination criteria and develop predictive algorithms that satisfy them.

Subsequently, in \chapref{chap:sensitivity}, we also scrutinize one of the major assumptions of causal reasoning, namely that the true causal graph is known.
Potential misspecifications of the assumed causal model can invalidate conclusions about the fairness of learned models.
One common way for misspecification to occur is via \emph{unmeasured confounding}: the true causal effect between variables is partially due to unobserved quantities.
Causally fair classifiers critically rely on assumptions about the causal structure.
In case the real world does not satisfy these assumptions, such a proposed classifier will actually be unfair during deployment.
\Chapref{chap:sensitivity} develops computationally efficient tools to assess the sensitivity of fairness measures to unobserved confounding.
These allow decision makers to measure how unfair their proposed classifier becomes in case the real world deviates from the causal assumptions within quantitatively measurable limits.
Hence, the main contribution of this chapter is to bring conceptually useful causal fairness notions closer to application by quantifying how they break, when the required assumptions do not hold perfectly and instead may be slightly violated.
We demonstrate how to perform sensitivity analysis with our tools on two real-world datasets.

\Chapref{chap:blindjustice} focuses on the pressing question of how training of fair machine learning systems may compromise user privacy.
Typically, sensitive data is required to train fair models.
However, users may be wary of providing this data to decision makers.
We suggest to tackle this predicament via cryptographic tools that allow us to train fair models with access only to encrypted sensitive data.
The main contribution of this chapter is to extend existing \emph{secure multi-party computation} protocols to also incorporate linear constraints in the optimization of machine learning models.
Thereby, we bring fair model training into the realm of computing on encrypted inputs---in our case encrypting sensitive features for the training data.
Moreover, we introduce a notion of accountability by allowing external entities to check the fairness properties of decision models and verify their outputs without revealing sensitive user data or the model itself, which may be considered secretive intellectual property.

Finally, in \chapref{chap:decisions}, we highlight that most traditional approaches to fair machine learning are based on predictive models learned from static datasets in a supervised fashion, implicitly assuming training data to be an i.i.d.\ sample from the distribution expected during deployment.
However, as we derive decisions from these predictions, the data distribution observed at test time may depend on these decisions.
Specifically, we consider a setting where labels only come into existence when a positive decision is made---if a loan is denied, there is not even an option for the individual to pay it back.
Our first contribution is to formally show that in such a setting, if an imperfect predictive model is used for data collection, the observed data will not suffice to train an optimal predictive model via constrained risk minimization---the most common technique in existing works on statistical fairness.
Instead we propose a paradigm shift away from minimizing a predictive loss with observed labels towards maximizing a utility that directly measures the quality of decisions in terms of accuracy and fairness.
We show that this requires non-deterministic decision rules that ``explore over all inputs'', i.e., give positive decisions with non-zero probability for any input.
Finally, we develop an algorithm to solve the utility maximization and demonstrate its efficacy in terms of accuracy and fairness on synthetic and real-world data.
Beyond enabling fair model training under more realistic circumstances, namely when labels only exist for certain decisions, this chapter also formalizes why it is crucial to distinguish between predictions and decisions in a social context.

\Chapref{chap:conclusion} first summarizes the contributions of the thesis.
We then draw our conclusions and highlight directions for future work before eventually circling back to the broader context of injustice and discrimination.

\section{Publications}
\label{sec:more_pubs}

This thesis is based on the following publications.

\Chapref{chap:causal} is based on \citet{kilbertus2017avoiding}:

\fbox{\parbox{\textwidth}{
  \pubitem{Avoiding Discrimination through Causal Reasoning}
  {Niki Kilbertus, Mateo Rojas-Carulla, Giambattista Parascandolo, Moritz Hardt, Dominik Janzing, Bernhard Schölkopf}
  {Neural Information Processing Systems (NeurIPS), 2017}
  {https://arxiv.org/abs/1706.02744}
  {}
  {}
}}

\Chapref{chap:sensitivity} is based on \citet{kilbertus19sensitivity}:

\fbox{\parbox{\textwidth}{
  \pubitem{The sensitivity of counterfactual fairness to unmeasured confounding}
  {Niki Kilbertus, Philip~J.~Ball, Matt Kusner, Adrian Weller, Ricardo Silva}
  {Uncertainty in Artificial Intelligence (UAI), 2019}
  {https://arxiv.org/abs/1907.01040}
  {https://github.com/nikikilbertus/cf-fairness-sensitivity}
  {}
}}

\Chapref{chap:blindjustice} is based on \citet{kilbertus18a}:

\fbox{\parbox{\textwidth}{
  \pubitem{Blind Justice: Fairness with Encrypted Sensitive Attributes}
  {Niki Kilbertus, Adri\`{a} Gasc\'{o}n, Matt Kusner, Michael Veale, Krishna~P.~Gummadi, Adrian Weller}
  {International Conference on Machine Learning (ICML), 2018}
  {https://arxiv.org/abs/1806.03281}
  {https://github.com/nikikilbertus/blind-justice}
  {}
}}

\Chapref{chap:decisions} is based on \citet{kilbertus2018fair}:

\fbox{\parbox{\textwidth}{
  \pubitem{Fair Decisions Despite Imperfect Predictions}
  {Niki Kilbertus, Manuel Gomez-Rodriguez, Bernhard Schölkopf, Krikamol Muandet, Isabel Valera}
  {International Conference on Artificial Intelligence and Statistics (AISTATS), 2020}
  {https://arxiv.org/abs/1902.02979}
  {https://github.com/nikikilbertus/fair-decisions}
  {}
}}

\newpage
During my PhD I also contributed to the following papers, which will not be described in this thesis \citep{hron2020exploration,kilbertus2020class,trauble2020independence,PhysRevD.100.063015,kilbertus2018generalization,parascandolo2018learning,GebKilParHarSch17}:\blfootnote{${}^*$ equal contribution}

\fbox{\parbox{\textwidth}{
\pubitem{Exploration in two-stage recommender systems}
{Jiri Hron${}^*$, Karl Krauth${}^*$, Michael~I.~Jordan, Niki Kilbertus}
{ACM RecSys {\normalfont workshop} REVEAL: Bandit and Reinforcement Learning from User Interactions, 2020}
{https://arxiv.org/abs/2009.08956}
{}
{}
}}

\fbox{\parbox{\textwidth}{
\pubitem{A class of algorithms for general instrumental variable models}
{Niki Kilbertus, Matt J. Kusner, Ricardo Silva}
{Neural Information Processing Systems (NeurIPS), 2020}
{https://arxiv.org/abs/2006.06366}
{}
{}
}}

\fbox{\parbox{\textwidth}{
\pubitem{Is Independence all you need? On the Generalization of Representations Learned from Correlated Data}
{Frederik Träuble, Elliot Creager, Niki Kilbertus, Anirudh Goyal, Francesco Locatello, Bernhard Schölkopf, Stefan Bauer}
{under review}
{https://arxiv.org/abs/2006.07886}
{}
{}
}}

\fbox{\parbox{\textwidth}{
\pubitem{Convolutional neural nets: a magic bullet for gravitational-wave detection?}
{Timothy Gebhard${}^*$, Niki Kilbertus${}^*$, Ian Harry, Bernhard Schölkopf}
{Physical Review D, 2019}
{https://arxiv.org/abs/1904.08693}
{https://github.com/timothygebhard/magic-bullet}
{}
}}

\fbox{\parbox{\textwidth}{
\pubitem{Generalization in anti-causal learning}
{Niki Kilbertus${}^*$, Giambattista Parascandolo${}^*$, Bernhard Schölkopf}
{Neural Information Processing Systems (NeurIPS) {\normalfont workshop on} Critiquing and Correcting Trends in Machine Learning, 2018}
{https://arxiv.org/abs/1812.00524}
{}
{}
}}

\fbox{\parbox{\textwidth}{
\pubitem{Learning Independent Causal Mechanisms}
{Giambattista Parascandolo, Niki Kilbertus, Mateo Rojas-Carulla, Bernhard Schölkopf}
{International Conference on Machine Learning (ICML), 2018}
{https://arxiv.org/abs/1712.00961}
{}
{}
}}

\fbox{\parbox{\textwidth}{
\pubitem{Searching for Gravitational Waves with Fully Convolutional Neural Nets}
{Timothy Gebhard${}^*$, Niki Kilbertus${}^*$, Giambattista Parascandolo, Ian Harry, Bernhard Schölkopf}
{Neural Information Processing Systems (NeurIPS) {\normalfont workshop on} Deep Learning for Physical Sciences, 2017}
{https://dl4physicalsciences.github.io/files/nips_dlps_2017_13.pdf}
{https://github.com/nikikilbertus/convwave}
{}
}}

\chapter{Background}
\label{chap:background}

In this chapter we provide some background about the fundamental challenges around fairness and justice from various viewpoints.
We take a top down approach, starting from abstract philosophical notions of justice, to concrete social challenges related to discrimination, to the specifics of machine learning.
While we continuously narrow down the scope, we hope that the big picture context provided in this chapter helps the reader not to lose sight of the core issues once we dive into technical details in the following chapters.
Readers only interested in the technical contributions can skip this chapter.
It starts with terminology and connections to other disciplines in \secref{sec:disciplines}, illuminating how terminology is used throughout the thesis and providing broader context for how to think about fairness in machine learning.
This discussion naturally leads to the question, whether we can reasonably expect to be able to practically mitigate unfairness in machine learning systems.
In \secref{sec:goals}, we explore on a high level which challenges we are facing when trying to achieve lasting positive impact on society via automated decisions.
There, we will also introduce the main theme of the thesis: \emph{scrutinizing common assumptions underlying traditional approaches to fairness in machine learning and proposing solutions for when these assumptions are violated.}
After a brief overview of where and how bias can enter the machine learning pipeline in \secref{sec:how}, we then summarize and conclude this introductory chapter in \secref{sec:takeaway}.

\section{Terminology and insights from other disciplines}
\label{sec:disciplines}

While some sort of immoral bias is present in all the examples in \secref{sec:motivation}, it can be hard to pinpoint---let alone quantify---what exactly is ``wrong''.
There are multiple dimensions along which one can characterize these difficulties to break down the field of algorithmic bias into more specific subproblems.

In this section we take a step back and assay the nature of such problems from a philosophical and ethical viewpoint.
Among other disciplines, scholars in ethics, philosophy, and law have been thinking about issues of justice, fairness, discrimination, and bias from different perspectives for much longer than the computer science community.
If we want to mitigate ethically objectionable impact of machine learning models, we need to acquire at least a basic understanding of why and what is considered ``unfair'' and to what extent we can hope this understanding to be universal.
We cannot provide a faithful account of centuries of research in philosophy, ethics, and law.
Thus this section is neither exhaustive nor unbiased.
It serves to illustrate some relevant viewpoints and provide useful terminology.
As a starting point for further reading on fairness from the perspective of political philosophy for computer scientists, we refer to \citet{pmlr-v81-binns18a}, who elucidates some connections between recent machine learning research and philosophy.
For a purely philosophical treatise, we suggest readers start with a modern overview of philosophical theories of justice by \citet{sandel2010justice} and then dive deeper into more focused works \citep{bentham1780introduction,young1995equity,roemer1998theories,moulin2004fair,roemer2009equality,rawls2009theory}.

\paragraph{Terminology and definitions.}
Within the machine learning community, \emph{fairness} (as in \emph{fairness in machine learning}, \emph{fair learning}, or \emph{algorithmic fairness}) developed as the prevailing umbrella term for a range of concerns about unethical behavior or impact of automated systems.
The word ``fairness'' is multi-faceted and can have different technical meanings across disciplines.
Just to name a few examples, \emph{fair value} in economics often refers to a rational estimate of the potential market price of a good.
In other fields of computer science, fairness is discussed both as a property of concurrency in the context of unbounded nondeterminism, as well as in the context of network engineering.
There, it describes how resources are shared and allocated among applications.
\emph{Fair division} in a game theoretic context typically describes the division of goods among players with subjective valuations subject to various constraints.
Within economics, the term \emph{fairness} was used to describe a market anomaly in which firms do not strictly maximize profits, because it may be considered unfair to raise prices to exploit shifts in demand \citep{kahneman1986fairness}.
This notion of fairness was an important concept in the emergence of behavioral economics.
Within each of these domains, fairness describes a rather narrow concept and there is typically no universally accepted definition.
Occasional confusion is expected when narrow mathematical constructs are named after real life problems.

Within machine learning, fairness has primarily discussed distributive aspects of social justice.
Since empirical evidence suggests that people do not reach a consensus in defining fairness intuitively \citep{yaari1984dividing}, distributive justice in the social context typically refers to \emph{subjective fairness} rather than \emph{objective fairness}.
\citet{yaari1984dividing} conclude that
\begin{quote}
\directquote{[\ldots] a satisfactory theory of distributive justice would have to be endowed with considerable detail and finesse.
Sweeping solutions and world-embracing theories are not likely to be adequate for dealing with the intricacies inherent in the problem of How to Distribute.}
\end{quote}
Note that the problem of \emph{how to distribute} considers a wide variety of goods and services, including decisions and outcomes in applications such as criminal justice, hiring, insurance, or lending.
This conclusion should remind us that---as a relatively young community---we may be prone to misunderstandings.
Currently, we do not expect a single technical definition to emerge as the ``right'' one.
Therefore, we must be transparent about our assumptions, be specific about context, and refrain from relying on an intuitive understanding of terminology used in our definitions.

We have already used the words \emph{fairness} and \emph{justice} almost synonymously without further explanation.
Together with \emph{discrimination} and \emph{bias} they are ubiquitous in the literature on fairness in machine learning.
There is no general agreement on their meaning, neither in a technical nor necessarily an intuitive sense.
One usually has to rely on the context for proper interpretation.
In the following, we provide clarifying remarks on the differences and connections between those terms, starting with popular dictionary definitions of fairness and justice.
\begin{description}
  \item[Fairness] is the quality of making judgments that are free from discrimination.
  \item[Justice] is an action that is morally right and fair.
\end{description}
From this perspective, fairness may be considered an idealized concept, whereas justice is about taking the right actions (potentially when unfairness has already occurred).

John Rawls even argues that fairness is more fundamental than justice \citep{rawls1958justice}:
\begin{quote}
\directquote{It might seem at first sight that the concepts of justice and fairness are the same, and that there is no reason to distinguish them, or to say that one is more fundamental than the other. I think that this impression is mistaken. In this paper I wish to show that the fundamental idea in the concept of justice is fairness;}
\end{quote}
He further considers fairness to imply mutual agreement among parties that no one is taken advantage of, or forced to give in to claims that anyone of them considers illegitimate.
Justice on the other hand, does not necessarily require such mutual agreement, as is usually the case in law enforcement.

From a more applied perspective, \citet{friedman1996bias} define \emph{bias} in one of the earliest works on fairness in computer systems as follows:
\begin{quote}
\directquote{In its most general sense, the term bias means simply ``slant.''
Given this undifferentiated usage, at times the term is applied with relatively neutral content.
[\ldots]
At other times, the term bias is applied with significant moral meaning.
[\ldots]
we use the term bias to refer to computer systems that systematically and unfairly discriminate against certain individuals or groups of individuals in favor of others.
A system discriminates unfairly if it denies an opportunity or a good or if it assigns an undesirable outcome to an individual or group of individuals on grounds that are unreasonable or inappropriate.}
\end{quote}

By now, the reader has certainly noted the inherent circularity or ambiguity of these definitions.
While there may be an innate mutual understanding of certain moral values \citep{turiel2002culture,blair1995cognitive}, other aspects of morality are inherently subjective.
This is reflected for example by considerably different legislation across nations.
Hence, the circularity of the provided definitions as well as their reliance on similarly fuzzy terms such as ``right'', ``unfairly'', ``unreasonable'', and ``inappropriate'' merely reflect the elusive nature of the subject.

As mentioned above, we will consider systems that take, trigger, or inform actions in the physical world.
Following the above dictionary definitions, we are less concerned with fairness (as an idealized concept), than with justice (taking the right action).
Even though \emph{non-discriminatory} or \emph{just} may often be more appropriate terms than \emph{fair}, the community has converged to the term \emph{fairness}.
It will be instructive for categorizing existing work on fairness, to borrow some more insights from philosophers' viewpoints on justice.
Instead of describing general frameworks such as consequentialist, deontological, contractualist, or virtue ethics, we will right away dive into distributive justice following the focus area of the thesis.

\paragraph{Philosophical theories of justice.}
Philosophical theories of distributive justice need to specify which goods are to be distributed among which entities, and what constitutes a proper distribution.
They do not answer questions about whether a particular distribution is objectively correct, or who should have the right to choose and enforce a certain distribution.
Opposed to this \emph{distributive} approach---defining a favorable distribution in a consequentialist fashion---we can also follow a certain procedure to assign outcomes to individuals, called \emph{procedural approach} to fairness, which is closer to a deontological approach.
By adhering to such a (defined to be) fair procedure, we consequently accept the arising distribution as fair.

We can further divide fairness notions along another dimension into \emph{normative} (or \emph{prescriptive}) and \emph{descriptive} (or \emph{comparative}) approaches.
Crudely, the former start from a fixed distribution or procedure, which is assumed to be fair by general agreement and thus ought to be reached or followed.
Contrary, \emph{descriptive} approaches seek to empirically crowd source society'{}s opinion on what is perceived---and thus accepted to be---fair.

Most democracies today deploy a somewhat procedural approach by accepting the distribution that arises naturally from adhering to the law and social norms.
Of course, there is substantial feedback in that the procedure, i.e., the legislation, is heavily influenced by critically reflecting on and reacting to the resulting distribution.
Similarly, while the legal system has a predominantly normative character, the prescription is largely determined by a comparative assessment of society'{}s opinion at large---at least in most democracies.
Consequently, we argue that our reality of justice hinges on a procedural approach that is normative in its execution, but heavily influenced by a comparative democratic process and regularized by theoretic distributive ideals.
A large body of work on fairness in machine learning seeks to derive technical fairness criteria from normative, distributive theories.
While this can be an instructive endeavor, one should not lose sight of the current procedural legal practices and descriptive influences \citep{green2018the}.

\section{Fundamental challenges of fair decision making}
\label{sec:goals}

Let us now dissect specific challenges of fair decision making and point out important concepts to factor into the process.

\paragraph{The role of historicity.}
In a fictitious world where our moral understanding (and thus our legal system) had been stationary over time, and all humans had always acted morally benign (lawfully), it is conceivable that nobody would voice concerns about unfairness or discrimination.
This fictitious world is the ideal we strive for when fighting discrimination.
The thought experiment also highlights that virtually all situations, in which fairness arises as an issue, are invariably tied to historicity.
Unfairness occurs when past or current practices conflict with today'{}s moral or legal standards.
Considering these historical events and practices to leave behind unfavorable disparities as an imprint on our society, we have to consider multiple goals for fair machine learning systems.

One may choose to accept existing disparities to be an unalterable truth about the state of our world that can fairly be exploited in a decision making process.
Such a viewpoint could be advocated for by not wanting to take blame for historic wrongdoing.
Following this idea of ``the data are what the data are'', or ``what you see is what you get'' \citep{Friedler2016}, the goal could be to not amplify, perpetuate, or create new harmful biases.
Decision makers may be drawn to such arguments when adhering to stricter correctional fairness measures would negatively impact their utility.
On the other hand, we may seek to correct for past misdoing by taking affirmative action to dynamically transition the current, unfair situation to a favorable state that is better aligned with today'{}s moral standards.
This often requires an understanding of how society evolves dynamically as a consequence of decisions, i.e., how such interventions feed back into the state of society and social dynamics.

Let us give two examples showing that both modes of reasoning are not entirely clear-cut.
Is it fair if a warehouse owner only employs individuals that can lift 40 pounds even though this may be heavily correlated with gender and age?
Or should they be required to work towards equalizing such differences, perhaps by providing free access to strength training?
Another example in which such questions become particularly relevant is retributive or corrective justice: to what extent does punishment contribute to a better society?
Is it fair to imprison people with a higher potential to commit more severe crimes even if that correlates with race?
Or should we instead focus more resources on supporting these communities to fight the root cause of their criminal activity?

While we have seen in the previous section how theories of justice may describe desirable states or actions, there are yet more choices to be made when setting out to operationalize these criteria.
Beyond the generic issue that there is no principled way to measure, evaluate, or compare different fairness notions---especially for normative approaches---we also need to discuss appropriate fairness goals as well as their time horizon in any given context.

\paragraph{Context sensitivity and long- versus short-term interventions.}
\citet{pmlr-v81-binns18a} provides three compelling examples to highlight the context dependence of fairness judgments.

\begin{enumerate}
  \item When it comes to \textbf{voting}, we tend to favor an egalitarian outcome distribution:
  every person has one and only one vote regardless of skill or effort.
  \item However, the same does not hold true for \textbf{job applications}.
  While we still generally seek to provide equal opportunity, it is morally acceptable to condition on skill and effort.
  In this scenario, we may find that existing disparities in the distribution of skill and effort can be explained by systematic historical discrimination and its perpetuation, leading to a debate whether we should take affirmative action or accept the status quo as a property of society.
  After all, the distributions simply differ and it may be hard to determine whether the root cause is past discrimination, or a legitimate alternative explanation.
  \citet{rawls2009theory} argues that
  \begin{quote}
  \directquote{The natural distribution is neither just nor unjust; nor is it unjust that persons are born into society at some particular position.
  These are simply natural facts. What is just and unjust is the way that
  institutions deal with these facts.}
  \end{quote}
  There is some ambiguity in whether the \emph{natural distribution} is the one observed today (potentially reflecting historic bias), or some idealistic, unobservable distribution.
  Rawls' statement represents the viewpoint that justice is about the actions we take in a given situation.
  However, it does not answer how precisely to deal with disparities in the natural or observed distribution.

  This begs naive questions such as ``How much of the past should we aim to correct for?'' or ``How far into the future should we aim to reach a desirable state?''
  Partially motivated by these questions, \citet{hucheng2018} study a simplified dynamic model of the labor market, and argue that \directquote{A dynamic model recognizes the powerful ripple effect of the past and calls for a fairness intervention that carries momentum into the future.}
  Similarly, \citet{liu18c} analyze a one-step feedback model of how decisions according to some fairness criteria affect the well-being of different groups.
  They show that enforcing common fairness criteria (see \chapref{chap:existing}) may actually cause harm in the long run.
  This indicates that there is a trade-off between optimizing for short-term (doing the right thing now) and long-term (achieving the right state in the future) fairness goals even when measured with the same criterion.
  \item In his third example, \citet{pmlr-v81-binns18a} considers \textbf{airport security screenings}.
  One may find it appropriate, perhaps due to social solidarity, to distribute scrutiny equally among all individuals in an egalitarian fashion.
  However, the data indicate real differences in the base rates, i.e., some visually identifiable subgroups of the population are statistically at higher risk of attempting a terroristic attack.
  Is it fair to use a predictive system to subject this group to more rigorous examination?
  After all, statistically speaking, focusing more attention to these groups can prevent unnecessary deaths.
\end{enumerate}

He concludes that \directquote{We therefore can’t assume that fairness metrics which are appropriate in one context will be appropriate in another.}
Thus, the appropriateness of different fairness criteria can only be judged in a narrow context and with a good understanding of the specific historical, sociological and legal perspective.

\paragraph{Subjectivity and non-stationarity.}
We have argued before that fairness is not only context dependent, but also subjective.
Moral values vary greatly across nations and cultures, which is also reflected in the respective legal systems.
Even within specific cultures, individuals often hold wildly different opinions on questions of justice and morality.
Again, this contributes to the difficulty of nailing down precise notions of fairness, even after fixing the domain and specific goals.
Arguably, what matters to the individual is the comparative hardship suffered by those at risk of being discriminated against and the mechanisms or circumstances leading to such disparate hardship.
Hence, beyond domain knowledge, a deep empathic understanding of experiences and emotions is crucial to devise effective fairness interventions.

Looking back, our understanding of fairness and egalitarianism has changed drastically over time.
Today, one may find the late onset and slow progress of worldwide emancipation for gender equality, continuing racial discrimination, or the idea of capital punishment shocking---if not repelling.
We speculate that past societies on average felt roughly as assertive about their contemporary moral values as we do now.
They may similarly have found preceding practices repugnant.
This indicates that even if we could agree on specific fair machine learning methods today, following generations may well despair over our current moral wrongdoing and struggle to deal with the biased imprints we have left behind.

\paragraph{Individual perception and the comparative nature of fairness.}
Specific personal complaints about unfairness are typically based on comparative arguments.
For example, when being denied a loan, a black person may complain that their white neighbor, who has an otherwise similar socio-economic background, received a loan.
On an individual level we tend to care less about being treated according to some abstract fairness maxim than about how we are treated relative to others.
In such comparisons, the similarity with respect to all relevant aspects except for group membership is crucial.
We regularly even ponder counterfactual questions like ``Would I have gotten the loan if I were white?'' or ``Would I get paid more if I were a man?''.
Such counterfactuals can be thought of as comparisons with a version of oneself, where only the group membership in question has been changed.
The contrast between this comparative nature and the distributive theories discussed in the previous sections leads to the distinction between \emph{group} and \emph{individual fairness} common in the machine learning literature.

Most of our discussion so far was concerned with fairness on a group level.
In contrast, individual fairness demands that, colloquially, ``similar people are treated similarly''.
This notion goes back to Aristotle'{}s principle that, paraphrased, ``equals should be treated equally''.
A formalization of this idea requires a measure of similarity for both individuals and outcomes \citep{Dwork2012}.
For now, we note that any concrete distributive ideal of what would constitute a perfectly fair state of society may still inevitably lead to perceived unfairness on an individual level, especially when trying to move from a non-ideal state to the ideal one.
We will come back to technical definitions of both group and individual fairness notions in \chapref{chap:existing}.

\paragraph{Moving to machine learning.}
From the picture we have drawn so far, it becomes apparent that fair decision making can only be successfully tackled on a systemic level, taking into account all factors from historic and current societal circumstances to design decisions by individual engineers.
Narrowing down the scope further, we now describe the specific context for the machine learning systems that we will later analyze in isolation.

We assume a fixed set of demographic groups, at least one of which face disadvantages when it comes to access to certain opportunities, goods, or services.
Such groups are often declared by the legal code.
For example, the Equality Act (2010) in the United Kingdom specifies the following \emph{protected characteristics}\footnote{\url{https://www.gov.uk/discrimination-your-rights}; \url{http://www.legislation.gov.uk/ukpga/2010/15/contents}}:
age,
gender reassignment,
being married or in a civil partnership,
being pregnant or on maternity leave,
disability,
race including color, nationality, ethnic or national origin,
religion or belief,
sex, and
sexual orientation.
Discrimination with respect to these characteristics is legally prohibited at work, in education, as a consumer, when using public services, when buying or renting property, and as a member or guest of a private club or association.
In Germany, according to the ``Allgemeines Gleichbehandlungsgesetz'',\footnote{\url{https://www.bmas.de/DE/Service/Gesetze/allgemeines-gleichbehandlungsgesetz.html}} it is  illegal to discriminate based on
age,
sex,
sexual identity,
disability,
race or ethnic origin, and
religion or belief
within a similar scope.
As a final example, in the United States legally recognized protected classes include
race,
color,
sex,
religion,
national origin,
citizenship, age,
pregnancy,
familial status,
disability status,
veteran status, and even
genetic information
by a number of legal texts such as the Civil Rights Act (1964, 1968), the Equal Pay Act (1963), the Immigration Reform and Control Act, the Age Discrimination in Employment Act (1967), the Pregnancy Discrimination Act, the Rehabilitation Act (1973), the Americans with Disabilities Act (1990), the Vietnam Era Veterans' Readjustment Assistance Act (1974), or the Genetic Information Nondiscrimination Act.\footnote{We thank Moritz Hardt for summarizing these legally protected classes and their corresponding legal texts in public talks.}

Against the backdrop of these legal regulations, we commonly refer to \emph{protected} or \emph{sensitive attributes} in the following chapters without further explanation.
Moreover, we assume that individuals can be unambiguously assigned to one of a finite set of such socially salient groups relevant to the given decision scenario.
We highlight that this is a crude simplification, which can by itself introduce new issues.
These challenges around group membership on an algorithmic level are touched upon in the next section and again briefly in \chapref{chap:causal}.
We note that legislation does not provide a moral framework, let alone concrete non-discrimination definitions that could be formalized mathematically or operationalized algorithmically.

Moreover, we restrict ourselves to scenarios in which machine learning systems are used to either fully take over or otherwise inform and support decisions that directly affect the well-being and livelihood of individuals.
We often refer to such situations as making \emph{consequential decisions}.
In this context, the goal of group fair machine learning is to make consequential decisions such that no protected group is disadvantaged or discriminated against.
Such discrimination could manifest itself in different ways.
For example, different groups could be treated differently---there are different decision processes in place for different groups---which may be unethical.
On the other hand, some group may experience undesirable downstream impacts even though---or even because---inputs are processed in the exact same way for everyone.
We will return to ways of quantifying specific dimensions of disadvantage or discrimination in \chapref{chap:existing}.

Since the focus of this thesis is on data-driven systems, data itself plays a crucial role.
What we have previously called ``the state of society'' or ``status quo'', to a machine learning system is typically represented by the data, which is also its main input.
The first essential step towards fairness, even before analyzing the algorithms themselves, is to understand whether data adequately represents the ``status quo'' in sufficient detail with respect to the given task and context.

Even more broadly, one must take into account the entire socioalgorithmic system they are embedded in, including data collection and downstream impact of decisions, which may alter the data we get to collect in the future.
In the next section, we will highlight some ways in which bias can creep into the different stages of what \citet{barocas-hardt-narayanan} call the \emph{machine learning loop}.
For more details and examples we refer the reader to the work of \citet[Chapter 1]{barocas-hardt-narayanan} as well as popular books on the topic \citep{o2016weapons,noble2018algorithms,broussard2018artificial,eubanks2018automating,benjamin2019race,kearns2019ethical}.

\section{From data to decisions and back}
\label{sec:how}

\paragraph{Measurement.}
Machine learning algorithms see the world through data.
While architecture choices, learning algorithms and inductive biases also play an important role, final decisions are most affected by the training data.
Hence, we must not assume that data are an unemotional and unopinionated mirror of the ``true state of the world'' that we can blindly trust in.
Instead they must be scrutinized, distrusted, and continuously analyzed as part of a larger data-driven system.
Unfortunately, machine learning engineers often seem to believe that their job only starts after the measurement phase.
Data are readily available and rarely questioned, but instead treated like mere facts about the state of the world.
However, measuring data about humans is a different process from, say, a physical measurement using a well-understood and calibrated measurement device.
Often, we seek to measure social constructs, such as the race or gender a person identifies with, for which we do not have well-defined scales.
Even the available options frequently change.
Which properties to measure as well as the scales on which to measure them are chosen by a small set of individuals.
In addition, the measurements themselves may be self-reported, introducing a new host of potential difficulties.

The lack of objective ground truth about gender or race assignments poses great conceptual challenges when we discuss causal notions of fairness.
Broadly speaking, in causal modeling we care about how certain quantities causally affect each other, i.e., for which there are invariable mechanisms that ensure that changing one variable will inevitably change another in a specific way.
In fairness, naturally arising questions are how a protected attribute, say race or gender, causally influenced a consequential decision.
For such statements to make sense, we need to clearly define what any given node, especially nodes marked as protected, references.
Most work on causal fairness notions consider protected attributes such as race or sex to take on a fixed value at birth.
Thus all other measurable features come later and are therefore causal descendants of the protected attribute.
We will challenge this viewpoint to some extent in \chapref{chap:causal}.
However, there are even more fundamental ontological and epistemic questions about the validity of causal models, which are described with great clarity in \citet[Chapter 4]{barocas-hardt-narayanan}.
Recently, scholars started to challenge the stability of the typically proposed ontologies of causal statements \citet{kohler2018eddie,hu2020whats,dembroff2020taylor}.
We believe this debate to become of crucial importance not only for causal modeling of fairness in the machine learning realm, but for a deeper understanding of quantifying discrimination and injustice in general.

Even when we do not collect data about humans, the measurement process may capture social disparities.
For example, there has been an effort to map out potholes in the city of Boston with a smartphone app that automatically reports the location of potholes recognized by its sensors.
\citet{crawford2013hidden} points out that the collected data reflect different levels of smartphone penetration for different demographics.
As a consequence, disproportionately fewer potholes have been recorded in lower-income areas and regions with predominantly elderly inhabitants.
\citet{barocas-hardt-narayanan} describe the inherent messiness of measurement as a \directquote{manifestation of the limitations of data-driven techniques}.

While data are important, they are far from the only issue.
In particular, we cannot expect to fix all concerns about discrimination by only working on datasets.
While one may hope to carefully design and collect an ``unbiased dataset'' that precludes any systematic bias downstream the machine learning pipeline, such efforts are futile.

\paragraph{Existing disparities.}
Most machine learning techniques in the social context still cling to the traditional objective of maximizing accuracy (or some related performance measure) with respect to some \emph{target variable} or \emph{label}.
We will see in \chapref{chap:existing} that many attempts at fair machine learning merely add constraints to this eager optimization paradigm to incorporate fairness.
As the addition of constraints can at most reduce the achieved accuracy, this gives rise to a tension between fairness and accuracy.
In the literature this has often been referred to as ``the cost of fair classification'' or the ``fairness accuracy trade-off''.

However, the constrained optimization approach to fairness appears to be inconsistent in the implicit assumptions about the recorded target labels.
On the one hand, suggesting fairness constraints acknowledges perpetuated stereotypes, a history of explicit discrimination, and other factors that lead to statistical disparities in today'{}s society, and thus in the recorded target labels.
On the other hand, accuracy with respect to these recorded target labels is still considered a meaningful goal to optimize for as the main objective.
Issues surrounding the validity of the recorded labels as an optimization target are referred to as forms of \emph{label bias}.
Recently \citet{wick2019unlocking} describe this dualism well in that \directquote{[\ldots] phenomena such as label bias are appropriately acknowledged as a source of unfairness when designing fair models, only to be tacitly abandoned when evaluating them.}
While making assumptions and especially also ethical judgments about the available data (and the distribution it comes from) explicit can help on a conceptual level \citep{Friedler2016}, there still exists no widely accepted way of incorporating and dealing with such assumptions algorithmically.
In other words, despite several attempts, ethical frameworks do not directly allow for operationalization in algorithmic terms.

As a consequence, the definition of the target variable plays a prominent role and its choice is heavily influenced by the parties involved in a given decision scenario.
Moreover, the chosen measurable target is often a proxy for the ``true goal'', an elusive, abstract (social) construct.
For instance, we use credit score as a proxy for creditworthiness, likelihood of recidivism for the influence of incarceration on character and behavior, or perhaps even generated revenue for the overall value of an employee.
Similar considerations hold true for the integrity of the features used as inputs to a machine learning pipeline.
Together with the messiness of measurement and the subjective choice of (proxies for) relevant variables, existing disparities become a complex source of bias that is hard to model and account for in observed datasets.

\paragraph{Dynamics, adaptivity and feedback.}
Just like the machine learning loop begins before the modeling part, namely with data collection, it does not end after a prediction has been made either.
Consequential actions affect individuals and thus the communities they are embedded in.
Again, this is in stark contrast with natural sciences, where the laws of physics do not react to our actions and we can justifiably consider ourselves passive observers of the outcomes of our experiments (barring measurements in quantum mechanics).
This interpretation breaks for algorithmic predictive systems, as soon as their predictions are revealed and influence actions.
In consequential decision settings, traditional assessment of machine learning on static datasets is rarely a good measure.
This observation is rooted in the difficulty of controlling deployment (or test-time) conditions for consequential decisions.
In particular, there is a host of plausible mechanisms through which the deployment conditions may depend on the choice of the decision policy.

For example, a traffic congestion prediction may cause drivers to congest routes shown as clear and clear up routes predicted to be congested \citep{barocas-hardt-narayanan}.
In general, widely accessible predictions by automated systems may influence people'{}s actions, thereby breaking the assumptions that went into the prediction, which consequently invalidates itself.
This is related to self fulfilling prophecies, where a prediction of a raise in stock prices can cause growing demand, which in turn actually leads to an increase in price.
Similar feedback loops have been observed in predictive policing \citep{lum2016to,ensign2017runaway}, as well as in the effects of pretrial detention on conviction, future crime and employment \citep{dobbie2018effects}.

In another form of adaptivity, individuals may strategically and deliberately invest effort to change their features (or disclosure thereof) adaptively to receive favorable decisions \citep{hardt2016strategic,dong2018strategic}.
These efforts can either be an attempt to \emph{game the system} or to \emph{legitimately self-improve}, warranting interpretations as either a nuisance to be countered by robust classification, or an opportunity for mechanism design \citep{miller2020strategic}.
Recently, it has been shown that the different costs of strategic behavior for different groups can give rise to further disparities \citep{milli2019social}.
Similarly, the mere act of forming a prediction or an immediate consequence thereof can influence the distribution of the prediction target, which is referred to as \emph{performativity} \citep{perdomo2020performative}.
Finally, the data we get to observe to train a decision system may depend on the decisions taken.
For example, in a \emph{selective labels} setting we only get to observe true outcomes when we take a positive decision (we only get to observe repayment of a loan if one was granted in the first place).
We will return to this specific setting in \chapref{chap:decisions}.
Differences in the ability to adapt and the environmental dynamics for members of different demographic groups can result in perpetuated or even amplified disparities.
Since anticipating and modeling such complex societal interactions is generically error prone, they pose a serious challenge for fair machine learning.

\paragraph{Closing the loop.}
We have highlighted some intricately interrelated challenges when moving discussions about justice and fairness from the philosophical to the technical realm of machine learning algorithms.
Our moral valuation of the current state of affairs is hard to formalize mathematically, measurement and data collection is a messy, error-prone process, and our decision systems affect both in non-obvious ways which are hard to model.
In short, our decision algorithms are part of an evolving socioalgorithmic system with numerous feedback loops and interactions.
Hence, their immediate and long-term consequences are hard to predict, let alone control.

\section{The role of algorithms and their developers}
\label{sec:takeaway}

To conclude this chapter, we raise the question to what extent machine learning is the right tool to tackle issues of fairness and discrimination.
One of the key goals is to automate decisions at scale.
At the same time, we argued in this chapter that fairness requires a case-by-case analysis of the context and goals.
Can machine learning still be a suitable tool to increase justice and diversity?
Which goals are we equipped and entitled to take on as machine learning researchers and engineers?
Should all algorithmically supported decisions be closely monitored by humans and judged by our current moral values, or should we even entrust them with shaping our society on their own terms, guiding the way to a more equitable future?
Let us exemplify the importance of these questions.\footnote{We thank Moritz Hardt for pointing us to this example.}

Pretrial risk assessment is usually framed as a predictive task in trying to estimate the risk of a person not appearing to court or recidivating in the meantime.
A high risk of not appearing to court or of recidivism is then converted into a decision to detain the defendant.
Once we accept this viewpoint, machine learning indeed has a lot to offer in terms of making accurate predictions and thus seemingly contributing to the larger social good.
However, in this example it may be more productive to challenge the basic assumption rather than specifics of the predictive algorithms and to recognize that people fail to appear because they do not have access to child care or transportation, cannot afford not to go to work, or even have other overlapping court appointments.
Hence, appropriate countermeasures would include enabling functional two-way communication and provide child care and transportation support to defendants.
Indeed some of these measures have been part of a Harris County lawsuit settlement in 2019 \citep{harris2019}.
However, whenever algorithms are used within a larger system, they may also be at fault.
To still make progress in those settings, we will typically require strict simplifying assumptions about which effects we consider in a given problem formulation.
We will clearly state these assumptions near the beginning of each of the \chapsref{chap:causal}-\ref{chap:decisions}.

We will not engage in a broader discussion of unequal power and access to machine learning and the resulting responsibilities.
We only note that while there are democratic mechanisms in place to prevent individual policy makers from instilling non-representative, subjective opinions into our laws, no such mechanisms currently exist for machine learning engineers, software developers, and data collectors in both industry and government.
To date, the demographic makeup of these professions considerably misrepresents society as a whole, which leaves us worrying about a lack of diversity of viewpoints when designing and building potentially impactful algorithms.
Harms in deploying algorithmic systems can only be anticipated and detected if a diverse team with a broad set of viewpoints, experiences and conceptual frameworks dedicates conscious effort into it.
While diversity and inclusion are perhaps more important for designing equitable systems than algorithmic considerations, a thorough discussion thereof goes beyond the scope of this thesis.
We refer the interested reader to recent work by \citet{mohamed2020decolonial}, who explore connections between society and machine learning from the viewpoint of critical science and decolonial theories.

After this rather pessimistic perspective of fair machine learning, we must not disguise that we are already facing very real concrete issues with automated decisions.
Even though we may not hope to solve them entirely, some are understood well enough to be readily improved.
Since large scale machine learning applications are still commonly deployed without any consideration for fairness throughout the process, even imperfect methods can have a substantial positive impact.
Moreover, algorithmic methods may also help to uncover and quantify existing discrimination in the first place \citep{kleinberg2018discrimination,abebe2020roles}.
At the least, machine learning and statistics provide valuable tools to quantify biases, potentially providing the smoking gun evidence required to trigger introspection and positive change.

\chapter{Existing approaches}
\label{chap:existing}

In this section we provide a survey of fair machine learning for consequential decision making.
We categorize the existing body of work by characteristics of the employed fairness criteria.
Other possible segmentations would be \emph{group} versus \emph{individual} fairness (see \secref{sec:individualfairness}), or according to implementation aspects such as the distinction of \emph{pre-processing}, \emph{in-processing}, and \emph{post-processing} methods.
We will adhere to ``fairness notion'' as the primary partition and mention such orthogonal properties for relevant works as we discuss them.
As a general warning, because fairness is a fairly young field within the machine learning community and for reasons we discussed in \secref{sec:disciplines}, there is no agreement on the naming of fairness criteria.
Some of them have been developed concurrently by different teams of researchers who gave them different names.
We try to use common terminology and clearly reference the corresponding works.

As becomes apparent from the examples in the motivation in \secref{sec:motivation}, there are plenty applications of machine learning outside of automated decision making for which fairness can be a major concern.
There is also an extensive literature on fairness in settings such as
ranking, 
recommender systems, 
computer vision,
natural language processing,
bandit learning, 
clustering,
dimensionality reduction, 
as well as many others which we will not discuss in this thesis.
\citet{chouldechova2020snapshot} provide a concise summary of the current research frontier including pointers to the relevant literature.

\section{Notation}
\label{sec:notation}

Most works on fairness in machine learning for consequential decisions have focused on what we call \emph{outcome-based} notions.
Outcome-based fairness encompasses scenarios in which a system determines (or supports the decision) whether to grant a loan to an applicant based on their financial history, to release a defendant for parole based on a questionnaire and their criminal history, or to invite a job applicant to an interview based on their written application.
The mentioned inputs are only examples among a great variety of possible features to take into account.

Throughout this thesis, we use the following notation.
Typically, we denote random variables as well as matrices by upper case letters and their domains (and other sets) as the corresponding calligraphic upper case letters.
The elements of such sets, i.e., also specific values of random variables, are denoted by lower case letters.
We sometimes highlight that values are multi-dimensional (i.e., vectors) using bold font.
Distributions of random variables are denoted by $\distr(\cdot)$ and we overload notation by also using $\prob(\cdot)$ for probabilities.
For example, if $X$ is a continuous real-valued random variable and $Y$ is a binary random variable taking values in $\{0,1\}$, we commonly denote their joint and marginal distributions by $\distr(X, Y)$, and $\distr(X), \distr(Y)$ respectively (instead of, e.g., $\distr_{X,Y}, \distr_X, \distr_Y$).
Similarly, we use $\prob(Y = 1 \given X=x)$ and $\prob(Y = 1 \given x)$ equivalently for the probability of $Y = 1$ given $X = x$.
Expectations are denoted by $\E[\cdot]$ where the source of randomness is either clear from the context, or explicitly added as a subscript of $\E$.
Unless explicitly stated otherwise, we assume densities exist for all continuous distributions and thus sometimes use $\distr$ even for densities whenever there is no ambiguity.
When explicitly highlighting that we are referring to densities, we use $\dens(\cdot)$.
As an example, for a random variable $X$, we write $\dens(x)$ for the density, dropping the subscript as in $\dens_X(x)$ for simplicity.
We denote statistical independence between two random variables by $\cdot \indep \cdot$ and conditional independence by $\cdot \indep \cdot \given \cdot$ accordingly.
Our notation, especially overloading the symbol $\distr$, glances over measure theoretic subtleties.
Since those will not be relevant for the content of this thesis, we opted for a notation that is easy to parse and understand.
We now describe common concepts and variable names that we will use throughout the thesis.

\begin{description}
  \item[$Z$] is (are) the \newdef{sensitive} or \newdef{protected attribute(s)} we want to protect for, e.g., gender, race, age, disability, see \secref{sec:goals}.
  \item[$X$] are the \newdef{non-sensitive features}, e.g., salary, SAT scores, etc.
  We will often refer to $X$ just as \newdef{features}, implying that they are non-sensitive.
  \item[$Y$] is the \newdef{true outcome}, i.e., the observed quantity of interest---for example whether a person defaulted on a loan or re-offended when let out on parole.
  Note that $Y$ is typically only available in historical data, i.e., a potential training set for supervised machine learning algorithms.
  Moreover, the term ``true'' does not imply that this outcome is fair or agreeable, merely that it has been observed or measured, see \secref{sec:how} for a more detailed discussion.
  \item[$\hY$] is the \newdef{predicted outcome} for $Y$ from an automated (machine learning) system.
  This notation implies that predicting the chosen target $Y$ has been accepted to be a worthwhile optimization goal.
  In a slight abuse of notation, we will consider $\hY$ as a function mapping $\hY: \cX \times \cZ \to \cY$, where $\hY(X, Z)$ can still be considered a random variable.
  The randomness can come from both, the randomness of $X$, $Z$ and because $\hY$ may be a randomized mapping.
\end{description}

Note that predicted outcomes may or may not be directly translated into decisions.
While most of the literature does not explicitly distinguish between predictions and decisions, we elaborate on the importance of this difference in \chapref{chap:decisions}.
Until then, we assume that predictions are closely linked to decisions and even use prediction, outcome, and decisions synonymously.

Outcome-based fairness is concerned with how we can define, measure, and mitigate (un)fairness of the predictions $\hY$ (taking values in $\cY$) with respect to the features $X$, protected attributes $Z$, and the true outcome $Y$ (taking values in given domains $\cX$, $\cY$, and $\cZ$ respectively).
The most common and highly relevant scenario considers $p \in \bN$ categorical (mostly binary) sensitive attributes $\cZ$, binary outcomes $\cY = \{0, 1\}$, and $d$-dimensional features (real-valued or categorical), $\cX \subset \bR^d$.
We typically assume $\hY = 1$ to be the more desirable decision or prediction and $Y = 1$ to be the more desirable true outcome or label.
Also, in settings with binary sensitive attributes we consider $Z = 1$ to indicate membership in a minority or disadvantaged group.

As described in \secref{sec:motivation}, this framework only captures a fraction of situations in which fairness may be a concern.
However, it entails common settings in which automated systems arguably have the most decisive direct impact on people'{}s lives.
Indeed, our excursion into philosophical aspects of fairness in \secref{sec:disciplines} almost exclusively dealt with notions of justice that are best described in the outcome-based framework.
In addition, it allows for the rather simple setup outlined above, thus facilitating in-depth technical discussion of concrete challenges, both theoretical and empirical.

\section{Fairness through unawareness}
\label{sec:fairness_through_unawareness}

The first idea to ensure fairness that comes to mind in the outcome-based framework is to simply omit the sensitive attributes $Z$ from the inputs to the algorithm.
This approach is called \newdef{fairness through unawareness} and can be formally written as $\hY: \cX \to \cY$, i.e., we omit $Z$ as an input for $\hY$.
Due to possible correlations between sensitive attributes and other features, which may readily be exploited by a machine learning algorithm, this approach cannot reliably exclude systematically biased outcomes.
For example, while the name of a person is typically not considered a protected attribute, it can correlate strongly with gender or ethnicity.
Thus an algorithm could internally ``recover'' the gender of an individual from their name and exploit existing correlations in the training data in its decisions \citep{Dwork2012}.
While fairness through unawareness is generally considered insufficient in most applications, in certain scenarios there have also been arguments in favor of this approach \citep{grgic2016case}.

Fairness through unawareness is often said to avoid \textbf{disparate treatment}, a US-centric legal term denoting decisions that are explicitly based on protected attributes, as well as intentional discrimination.
To first order, disparate treatment occurs whenever a model uses membership in a protected group as input and differentiates based on it.
Fore more details about the legal interpretation of disparate treatment in the context of data-driven decisions we refer the reader to work by \citet{barocas2016big}.

\section{Observational group matching criteria}
\label{sec:observational_group_matching}

\newdef{Observational criteria} are fairness measures that only depend on the joint distribution $\distr(X, Y, Z, \hY)$.
All probabilities and distributions in this section are with respect to this joint distribution and we only consider notions of group fairness, see \secref{sec:goals}.
For ease of notation, in the examples here we will assume $\cZ = \{0, 1\}$.
Most criteria easily extend to the case of multiple, categorical protected attributes.

\paragraph{Demographic parity and disparate impact.}
Proactively adapting to the shortcomings of fairness through unawareness, \newdef{demographic parity} (DP) (also called \newdef{statistical parity}) suggests to ensure statistical independence of the protected attribute and the predicted outcome: $\hY \indep Z$.
We can write this as
\begin{equation*}
  \distr(\hY \given Z) = \distr(\hY) \eqc
\end{equation*}
which for the binary setting is equivalent to
\begin{equation*}
  \prob(\hY = 1 \given Z = 0) = \prob(\hY = 1 \given Z = 1) \eqp
\end{equation*}
In words, the overall probability for a member of group $Z=0$ to receive outcome $\hY=1$ (or $\hY=0$), is the same as for a member of group $Z=1$.
While this equality is intended to describe a state with zero unfairness, in practice we often wish to work with a measure of various degrees of unfairness.
Two possibilities immediately come to mind:
the difference \citep{Calders2010}
\begin{equation}\label{eq:cv}
  \prob(\hY = 1 \given Z = 0) - \prob(\hY = 1 \given Z = 1) \in [-1, 1]
\end{equation}
(or its absolute value), and the ratio \citep{Feldman2015,Zafar2017}
\begin{equation}\label{eq:di}
  \frac{\prob(\hY = 1 \given Z = 0)}{\prob(\hY = 1 \given Z = 1)} \in [0, \infty) \eqp
\end{equation}
In eq.~\eqref{eq:cv} demographic parity is achieved for a value of~$0$, whereas in eq.~\eqref{eq:di} for the value~$1$.
The ratio definition in eq.~\eqref{eq:di} is inspired by the legal notion of \newdef{disparate impact} (DI) \citep{Barocas2016}.
Therefore, the term is also sometimes used for eq.~\eqref{eq:di} and due to the similarity even for eq.~\eqref{eq:cv}.
Note that due to different interpretations of the legal term \emph{disparate impact}, it is also used as an umbrella term for various kinds of outcome disparity that arise despite the decision not being actively based on the protected attribute.
In \chapref{chap:blindjustice}, we will use disparate impact as such an umbrella term.
In \citet{Zafar2017a}, disparate impact is used to mimic the p\%-rule introduced in the legal literature on employment \citep{biddle2006adverse} by formalizing it as
\begin{equation}\label{eq:ppercent_first}
  \min\left\{
  \frac{\prob(\hY = 1 \given Z=1)}{\prob(\hY = 1 \given Z=0)} \eqc
  \frac{\prob(\hY = 1 \given Z=0)}{\prob(\hY = 1 \given Z=1)}
  \right\} \ge \frac{p}{100} \eqp
\end{equation}

Demographic parity has been criticized for two specific reasons in particular \citep{Hardt2016}.
Both of them are due to the fact that it is based on too little information, caring only about group membership and the outcome.
First, it allows $\hY$ to choose equal fractions differently in the two groups.
In the hiring example, it could select the least qualified 10\% of applicants in the disadvantaged group and the top qualified 10\% of applicants in the privileged group for a phone interview based on the written application.
As long as the same fraction of each group is given the opportunity, demographic parity is satisfied.
Because it formally allows for such unfair practice, demographic parity can be too weak.

On the other hand, it completely disregards potentially legitimate correlation between $Z$ and $\hY$.
This limitation can be exemplified in the Berkeley college admission case \citep{Bickel1975}.
Bickel had shown that a lower college-wide admission rate for women than for men could be explained by the fact that women applied for more competitive departments.
When adjusted for department choice, women experienced a slightly higher acceptance rate compared with men in each individual department.
In this case, the department choice correlated with the applicant'{}s gender as well as with the chance of admission due to differences in competitiveness.
One may believe to understand the nature of both mechanisms reasonably well and tend to conclude that the subsequent correlation between gender (protected attribute) and the admission decision is legitimate and need not be corrected for by the admission committee.
In this case ($\cov(Y, A) \ne 0$), demographic parity does not allow the optimal and arguably fair predictor $\hY = Y$.
In this sense demographic parity is too strong.
We will revisit the Berkeley college admission example from a causal perspective in \secref{sec:causalcriteria}.

After reading \chapref{chap:intro}, we hope our interpretation of the situation made the reader feel uncomfortable or at least a bit skeptical.
The observed differences may be rooted much deeper.
Perhaps certain departments are known for a hostile and poisonous environment towards women or existing gender imbalance and a lack of (visibility of) role models further discourage women from applying to certain departments?
Again, real social change will require larger-scale systemic reforms and are unlikely to emerge from the specific choice of the admission decision system ignoring other aspects.
However, we will continue to try to isolate the effects of algorithms within such a system to try to better understand their role in it.

\paragraph{Equal odds and equal opportunity.}
We now describe attempts to remedy the shortcomings of demographic parity.
Ignoring the true outcome $Y$ seemed to leave too little information for a meaningful fairness criterion.
\newdef{Equalized odds} therefore requires the predicted outcome to be independent of the protected attribute conditioned on the true outcome: $\hY \indep Z \given Y$ \citep{Hardt2016}.
The intuition behind the criterion can be roughly described as: \emph{The predicted outcome should not be informed by group membership except for information that comes from the true outcome and thus legitimizes differences.}
Under our assumptions, it can be written equivalently as
\begin{equation*}
  \prob(\hY = 1 \given Y = y, Z = 0) = \prob(\hY = 1 \given Y = y, Z = 1) \for y \in \{0,1\} \eqp
\end{equation*}
Note that this already implies an analogous statement for $\hY = 0$.
This definition is also sometimes interpreted as: \emph{Equalized odds seeks to balance false positive rates (FPR) and false negative rates (FNR) between different demographic groups.}
The intuitive idea is that we should not make wrong predictions at different rates for the two groups.

As a more specific variant of equalized odds, \newdef{equal opportunity} recognizes that in many scenarios we care more about not falsely denying a desirable opportunity, than about providing an opportunity undeservedly \citep{Hardt2016}.
If the desired outcome is encoded by $\hY = 1$, equal opportunity amounts to balancing the FNR
\begin{equation*}
  \prob(\hY = 0 \given Y = 1, Z = 0) = \prob(\hY = 0 \given Y = 1, Z = 1) \eqc
\end{equation*}
whereas if the desired outcome is encoded by $\hY = 0$, it corresponds to balancing the FPR
\begin{equation*}
  \prob(\hY = 1 \given Y = 0, Z = 0) = \prob(\hY = 1 \given Y = 0, Z = 1) \eqp
\end{equation*}
A similar criterion was concurrently proposed by \citet{Zafar2017a}, who refer to it as avoiding \newdef{disparate mistreatment}.

\paragraph{Calibration and predictive parity.}
Calibration is a concept that was not originally associated with fairness.
It describes how well reality follows the predictions, i.e., an algorithm is well-calibrated if an event that it predicts to occur with some probability actually happens with this probability.
In other words, calibration is sometimes described as \emph{the predictions mean what the are supposed to mean}.
In the fairness literature, \newdef{calibration} amounts to $Y \indep Z \given \hY$ and can be written as \citep{Chouldechova2017,Kleinberg2016}
\begin{equation*}
  \prob(Y = \hY \given \hY = \hy, Z = 0) = \prob(Y = \hY \given \hY = \hy, Z = 1) \for \hy \in \cY \eqp
\end{equation*}
In words, a given prediction should actually come true with the same probability for members in both groups.
We note that this definition differs slightly from non-fairness related notions of calibrated score functions.
Perhaps more appropriately, it has also been called \newdef{predictive parity} (PP), and further specialized to \newdef{positive predictive parity} (PPP) and \newdef{negative predictive parity} (NPP) for
\begin{align*}
  \prob(Y = 1 \given \hY = 1, Z = 0) &= \prob(Y = 1 \given \hY = 1, Z = 1) \; \text{and}
  \\
  \prob(Y = 0 \given \hY = 0, Z = 0) &= \prob(Y = 0 \given \hY = 0, Z = 1)
\end{align*}
respectively.
Closely related criteria matching so called \newdef{false omission rates} (FOR) and \newdef{false discovery rates} (FDR) read \citep{Zafar2017a}
\begin{align*}
  \prob(Y = 0 \given \hY = 1, Z = 0) &= \prob(Y = 0 \given \hY = 1, Z = 1) \; \text{and}
  \\
  \prob(Y = 1 \given \hY = 0, Z = 0) &= \prob(Y = 1 \given \hY = 0, Z = 1) \eqp
\end{align*}

\paragraph{Generic matching of conditional probabilities.}
At this point a clear pattern emerges.
Let us draw the confusion tables of true outcomes $Y$ and predictors $\hY$ for two groups denoted by \textcolor{py-3-1}{blue ($Z=0$)} and \textcolor{py-3-2}{orange ($Z=1$)} as in Table~\ref{tab:confusiontable}.
One can then formulate a multitude of observational \newdef{matching criteria} by equating different quantities made up of $a$, $b$, $c$, and $d$ across protected groups \citep{Berk2017,Zafar2017a}.
We provide a structured overview of some criteria and the corresponding expressions to be matched in Table~\ref{tab:expressions} and Table~\ref{tab:criteria}.
Note that in the binary case, there is some redundancy.
For example, the matching criteria for calibration and predictive parity are identical.

\begin{table}
\centering
\caption[Confusion tables for two different demographic groups]{Confusion tables for two different demographic groups.}
\label{tab:confusiontable}
{\color{py-3-1}
\begin{tabular}{c|cc}
\multicolumn{3}{c}{group $Z=0$}\\
\toprule
& $\hY=0$ & $\hY=1$ \\
\midrule
 $Y=0$ & \multicolumn{1}{c}{$a$} & $b$ \\
 $Y=1$ & \multicolumn{1}{c}{$c$} & $d$ \\
\bottomrule
\end{tabular}
}
\hspace{2cm}
{\color{py-3-2}
\begin{tabular}{c|cc}
\multicolumn{3}{c}{group $Z=1$}\\
\toprule
& $\hY=0$ & $\hY=1$ \\
\midrule
 $Y=0$ & \multicolumn{1}{c}{$a'$} & $b'$ \\
 $Y=1$ & \multicolumn{1}{c}{$c'$} & $d'$ \\
\bottomrule
\end{tabular}
}
\end{table}

\def\expressiontableextraheight{5pt}
\begin{table}
\centering
\caption[Expressions that may be required to be matched in different fairness criteria in terms of the entries of the confusion table]{Expressions that may be required to be matched in different fairness criteria in terms of the entries of the confusion tables in Table~\ref{tab:confusiontable}.}
\label{tab:expressions}
\begin{tabular}{cl}
\toprule
\textbf{expression} & \textbf{name} \\
\midrule
$\frac{b+d}{a+b+c+d}$ & acceptance rate (AR) \\[\expressiontableextraheight]
$\frac{a+d}{a+b+c+d}$ & accuracy (AC) \\[\expressiontableextraheight]
$\frac{b+c}{a+b+c+d}$ & error rate (ER) \\[\expressiontableextraheight]
$\frac{a}{a+b}$ & true negative rate (TNR) \\[\expressiontableextraheight]
$\frac{d}{c+d}$ & true positive rate (TPR) \\[\expressiontableextraheight]
$\frac{b}{a+b}$ & false positive rate (FPR) \\[\expressiontableextraheight]
$\frac{c}{c+d}$ & false negative rate (FNR) \\[\expressiontableextraheight]
$\frac{a}{a+c}$ & negative predictive value (NPV) \\[\expressiontableextraheight]
$\frac{d}{b+d}$ & positive predictive value (PPV) \\[\expressiontableextraheight]
$\frac{c}{a+c}$ & false discovery rate (FDR) \\[\expressiontableextraheight]
$\frac{b}{b+d}$ & false omission rate (FOR) \\
\bottomrule
\end{tabular}
\end{table}

\def\criteriatableextraheight{0.5cm}
\begin{table}
\centering
\caption[The requirements of some common observational matching criteria in terms of the expressions in Table~\ref{tab:expressions}]{The requirements of some common observational matching criteria in terms of the expressions in Table~\ref{tab:expressions}.}
\label{tab:criteria}
\begin{tabularx}{\textwidth}{Yp{4cm}p{4.5cm}}
\toprule
\textbf{name(s)} & \textbf{matching criteria} & \textbf{references}\\
\midrule
demographic parity
& $\match{AR}$
& \citep{Calders2010}\newline \citep{kamishima2012fairness}\newline \citep{Zemel2013}\newline \citep{Edwards2015}\newline \citep{Feldman2015}\newline \citep{Zafar2017}
\\[\criteriatableextraheight]
balanced classification rate
& ${\color{py-3-1}\frac{\text{TPR} + \text{FPR}}{2}} = {\color{py-3-2}\frac{\text{TPR}' + \text{FPR}'}{2}}$
& \citep{friedler2018comparative}
\\[\criteriatableextraheight]
equalized odds
& $\match{FPR}$ and\newline $\match{FNR}$
& \citep{Hardt2016}\newline
\\[\criteriatableextraheight]
equality of opportunity
& $\match{FNR}$ \newline (if $\hY = 1$ is desirable)
& \citep{Hardt2016}\newline
\\[\criteriatableextraheight]
calibration
& $\match{FDR}$ and\newline $\match{FOR}$
& \citep{Chouldechova2017}\newline \citep{Kleinberg2016}
\\[\criteriatableextraheight]
predictive parity
& $\match{PPV}$ and\newline $\match{NPV}$
& \citep{Zafar2017a} \\
\bottomrule
\end{tabularx}
\end{table}

The fairness criteria outlined in Table~\ref{tab:criteria} can be categorized into three generic groups described by (conditional) independences between $Y, Z, \hY$.
\citet{barocas-hardt-narayanan} refer to
\begin{align*}
\text{\newdef{independence:}} \qquad & \hY \indep Z \eqc \\
\text{\newdef{separation:}} \qquad & \hY \indep Z \given Y \eqc \\ \text{\newdef{sufficiency:}} \qquad & Y \indep Z \given \hY \eqp
\end{align*}
All observational group matching criteria can be assigned in spirit to one of these three categories.

We highlight at this point, that all the criteria mentioned so far are observational, i.e., they can be formulated with reference only to the joint distribution $\distr(X, Y, \hY, Z)$ (where we have not yet made use of $X$).
Given a dataset for features, protected attributes, true outcomes, and predicted outcomes from some system, no further assumptions are needed to directly verify each notion on the empirical distribution.
Moreover, for each of them there is at least one scenario in which it can be framed as a desirable property in terms of fairness.
A principled comparison of their respective utility for fair decision making is virtually impossible without context and a specific application.
Unfortunately, one also can not have them all.
\paragraph{Impossibility results.}
Discussions about fairness often arise in situations, where the observed \newdef{base rates} differ across protected groups.
The base rate is the fraction of people with positive true outcome $Y = 1$.
In our notation, this means ${\color{py-3-1} p := \nicefrac{a+b}{a+b+c+d}} \ne {\color{py-3-2}\nicefrac{a'+b'}{a'+b'+c'+d'} =: p'}$, where we denote the base rates for the two groups by $p$ and $p'$.
From these expressions, one can directly verify that within each protected group \citep{Chouldechova2017}
\begin{equation}\label{eq:impossibility}
  {\color{py-3-1}\text{FPR}} = \frac{{\color{py-3-1}p}}{1 - {\color{py-3-1}p}} \, \frac{1 - {\color{py-3-1}\text{PPV}}}{{\color{py-3-1}\text{PPV}}} \, {\color{py-3-1}\text{FNR}}
  \qtxtq{and}
  {\color{py-3-2}\text{FPR}'} = \frac{{\color{py-3-2}p'}}{1 - {\color{py-3-2}p'}} \, \frac{1 - {\color{py-3-2}\text{PPV}'}}{{\color{py-3-2}\text{PPV}'}} \, {\color{py-3-2}\text{FNR}'}
  \eqp
\end{equation}
Therefore, if we wish to achieve equalized odds ($\match{FPR}$, $\match{FNR}$) and calibration (in particular $\match{PPV}$) at the same time, eq.~\eqref{eq:impossibility} implies that ${\color{py-3-1}p} = {\color{py-3-2}p'}$ or $\match{FPR} = \match{FNR} = 0$.

As a consequence, except for the degenerate cases of equal base rates or a perfect predictor, no system can satisfy calibration and equalized odds at the same time \citep{Kleinberg2016,Chouldechova2017}.
In general, unless the confusion matrices in Table~\ref{tab:confusiontable} are scalar multiples of each other, or all off-diagonal entries are $0$, many pairs of observational fairness criteria (e.g., Table~\ref{tab:criteria}) cannot be satisfied jointly.
For example, \citet{pleiss2017fairness} studied which fairness criteria are compatible with calibration.
Similarly, impossibility results can be phrased elegantly in terms of separation and sufficiency \citep{barocas-hardt-narayanan}.
\citet{kim2020model} provide a model agnostic characterization of the (in)compatibility of many combinations of criteria listed in Table~\ref{tab:criteria}.
Comparing different outcome-based (un)fair predictors to one another remains a challenge that has only recently seen some first advances \citep{speicher2018unified}.
The impossibility of ``playing it safe'', by satisfying all potentially desirable group fairness criteria simultaneously, highlights the necessity of comparing and discussing their usage in the context of a specific application.

Even before the first mention of these results \citep{Kleinberg2016,Chouldechova2017}, a closely related debate was carried out in the COMPAS example which we encountered in \secref{sec:motivation}.
\citet{Angwin2016} from ProPublica argued that COMPAS does not satisfy equality of opportunity, because among the people who did not go on to re-offend ($Y = 1$), black defendants ($Z = 1$) got worse scores than white defendants ($Z = 0$), i.e., $\prob(\hY = 1 \given Y=1, Z=0) \ne \prob(\hY = 1 \given Y=1, Z=1)$.
In a rebuttal to the ProPublica article, the company behind COMPAS (at the time it was called \emph{Northpointe}, in the meantime they have been renamed to \emph{Equivant}) argued for calibration as a more appropriate fairness criterion in their application and demonstrated that COMPAS indeed satisfies calibration \citep{Dieterich2016}.
This specific real-world example of incompatibility and the resulting trade-offs has also be studied in more detail from an algorithmic and ethical perspective \citep{corbett-davies2016,flores2016false,corbett2017algorithmic}.

\section{Sub-group and individual fairness}
\label{sec:individualfairness}

An interesting non-observational fairness definition is \newdef{individual fairness} \citep{Dwork2012}, which assumes the existence of a similarity measure for individuals, and requires that any two similar individuals should receive a similar distribution over outcomes.
Hence, we here consider a mapping $\hY: \cX \to \Delta(\cY)$, where $\Delta(\cY)$ denotes the set of all probability distributions on $\cY$.
Then the resulting criterion for individual fairness takes the form
\begin{equation}\label{eq:lipschitz}
  D(\hY(x), \hY(x')) \le d(x, x') \fora x, x' \in \cX \eqc
\end{equation}
where $D: \Delta(\cY) \times \Delta(\cY) \to [0,\infty)$ and $d: \cX \times \cX \to [0, \infty)$ are metrics.
More recent work lends additional conceptual and theoretical support to such a definition \citep{Friedler2016}.

However, \citet{Dwork2012} also acknowledge the main difficulty of the approach, namely choosing metrics $D$ and $d$ in practice.
This issue has been approached in an online learning setting, where a regulator ``knows unfairness when they see it'' to circumvent a normative definition of similarity \citep{gillen2018online}.
Along similar lines, \citet{kim2018fairness} assume that instead of having access to the metric in analytical form, we can query it a bounded number of times, where the queries also may be answered by experts.
Their fairness notion aligns with ideas from previous work aiming to relax individual fairness by attempting to interpolate between group and individual fairness \citep{hebert2017calibration,pmlr-v80-kearns18a}.
\citet{lahoti2019operationalizing} model side-information, such as fairness judgments from a variety of sources (including human judgments) as a fairness graph and then learn a unified presentation capturing pairwise fairness to tackle the issue of learning a similarity metric.
Learning such a metric from human annotated data has also been explored in detail specifically for the criminal recidivism prediction task on COMPAS data \citep{wang2019empirical}.

These attempts rely on ensuring parity not between group statistics, but between many (possibly overlapping) subgroups of the population with identical (or similar) features.
Besides operationalizing individual fairness as proposed by \citet{Dwork2012}, a key motivation for such approaches to subgroup fairness is to avoid \emph{fairness gerrymandering} \citep{pmlr-v80-kearns18a,hebert2017calibration}.
In concurrent work, \citet{yona2018probably} arrive at the conclusion that ensuring the strict definition eq.~\eqref{eq:lipschitz} exactly is generally computationally intractable.
However, these intractabilities can be overcome for a relaxed notion they call \emph{approximate metric-fairness}.
For the remainder of this thesis we will focus on group fairness notions.

\section{Causal fairness criteria}
\label{sec:causalcriteria}

In this section, we motivate the study of fairness through the lens of causality and describe some existing causal notions of fairness.
An in-depth discussion of our own contributions in going beyond observational criteria follows in \chapref{chap:causal}.

\paragraph{Why causality?}
To motivate the importance of understanding cause-effect relations in fairness, let us come back to the Berkeley college admission example \citep{Bickel1975}.
The potential allegation in this study was based on the observation that the overall acceptance rate for females was lower.
If one does not believe this to be a coincidence, the conclusion that the acceptance decisions are based on gender stands to reason.
After all, it appears unlikely that the admission decision influenced the gender.
However, this seemingly plausible explanation neglects possible alternatives due to confounding.
When presented with two correlated quantities, we all too commonly conclude that either the first influences the second or vice versa.
Thereby, we overlook the possibility of a third variable influencing both quantities in question.
In the Berkeley college admission case, only a closer analysis revealed the underlying Simpson'{}s paradox and opened up the possibility that gender may not have directly influenced the admission decision.
As soon as one conditions on department choice, the effect reverses and females experience higher acceptance rates in each department.

Crucially, when having access only to the joint distribution of admission decision $\hY$, gender $Z$, and some features $X$, e.g., SAT scores, we cannot tell, whether $Z$ had a direct influence on $\hY$.
If this were the case, it would raise serious concerns about blatantly direct gender discrimination.
However, the same joint distribution could arise under a different generative model, where gender has a direct causal effect only on the department choice, which in turn affects the outcome.
In short, different causal structures can entail the same joint distribution over all variables.
As a consequence, all observational criteria---based only on the joint distribution---cannot distinguish between causally different scenarios.
Since the example shows that whether a decision is perceived as fair or not depends heavily on the causal structure of the process, this is bad news for observational criteria.

The importance of \emph{unidentifiability} of causal structure from observational data for fairness can be traced back to \citet[Section 6]{Pearl2009}.
\citet{Hardt2016} explicitly mention that observational criteria are too weak to distinguish two intuitively different scenarios.
However, the work does not provide a formal mechanism to articulate why and how these scenarios should be considered different.
Proposing a causality-based approach to cope with this issue is one of our contributions described in \chapref{chap:causal}.

To further motivate a causal framework to analyze fairness, recall the intuition behind some of the observational criteria.
For example, demographic parity asks the outcome to be statistically independent of the protected attribute.
However, we conjecture that the way the general public intuitively thinks about demographic parity is that \emph{membership in a protected group should not have any effect on the predicted outcome}.
Similarly, equalized odds informally asks that the predicted outcome should not be informed by group membership except for information that comes from the true outcome.
This is likely understood as \emph{the predicted outcome should only be affected by the group membership to the extent to which this is justified or mediated by the true outcome}.
Both informal explanations hinge on causal statements rather than mere statistical dependence.
Thus, a causal framework may be more appropriate to formalize them.

The arguments in favor of causality so far were based on issues of group fairness criteria.
However, causality offers a similar conceptual benefit to dealing with individual fairness.
Initial work on individual fairness focused on distance measures of the features $X$ of individuals in the context of a given application to assess similarity.
As outlined in \secref{sec:goals}, the most immediate comparison conceptually is with a different version of ourselves, where everything is equal except for the protected attribute.
Causality allows us to formally reason about such counterfactuals, rendering it an appealing tool to investigate individual notions of fairness.
The idea of a similarity measure of individuals in observed data can then be addressed by the method of \emph{matching} used in counterfactual reasoning \citep{Rosenbaum1983,Qureshi2016}.
That is, evaluating approximate counterfactuals by comparing individuals with similar values of covariates $X$ but different protected attribute $Z$.

\paragraph{Structural equation models and causal graphs.}

We assume a basic understanding of the key ideas of causality, in particular \emph{interventions} and \emph{counterfactuals} and predominantly work with the graphical approach to causality and the \emph{do-calculus}, see \citet{Pearl2009,peters2017elements}.
We also borrow notation from these books.
Let us now recall the definition of structural equation models and provide a subjective opinion on how to interpret them.

For our purposes, a \newdef{causal graph} is a directed, acyclic graph whose nodes reference random variables.
The directed edges represent causal influences between those variables.
In that way, causal graphs are a convenient tool to organize assumptions about the data generating process.
More formally, we work with Structural Equation Models.
For a rigorous introduction of structural equation models with more attention to measure-theoretic considerations, we refer the reader to \citet{bongers2016foundations}.
A \newdef{Structural Equation Model (SEM)} is a set of equations of the form
\begin{equation*}
  V_i = f_i(\pa(V_i), N_i) \for i \in \{1, \dots, n\} \eqc
\end{equation*}
where
\begin{itemize}
  \item the $V_i$ are real-valued random variables,
  \item $\pa(V_i) \subset \{V_1, \ldots, V_n\}$ are called the \newdef{parents} of~$V_i$,
  \item the real-valued random variables $N_i$ are independently distributed according to $\distr(N_i)$,
  \item the \newdef{causal graph} induced by the structural equations $f_i$, i.e., the directed graph with nodes $\{V_i\}_{i=1}^n$ and edges $V_i \to V_j$ if and only if $V_i \in \pa(V_j)$, is acyclic.
\end{itemize}
We typically denote the resulting directed acyclic graph (DAG) by $\cG$ and call $\pa(V_i)$ the \newdef{direct causes} of $V_i$.
The random variables $N_i$ are also referred to as \newdef{noise variables} or \newdef{exogenous variables}.

The structural equations of an SEM are interpreted as assignments.
Because we assume acyclicity, following the nodes of the graph in topological order, we can recursively compute the values of all variables, given the noise variables.
In most applications, we assume the $\{f_i\}_{i=1}^n$ to be deterministic, i.e., all the randomness comes from the noise variables.
Hence, given specific values for the $N_i$, the values of all $V_i$ are fixed.
A structural equation model entails a unique joint distribution over all variables, leading us to view structural equation models as \newdef{data generating models}.
As mentioned above, the same joint distribution can usually be entailed by multiple structural equation models with distinct causal graphs.

Some graph-theoretical terminology will come in handy in the following two chapters.
A \newdef{path} in a DAG $\cG$ is a sequence of distinct nodes~$V_1, \dots, V_k$, for~$k \ge 2$, such that~$V_i \to V_{i+1}$ or $V_{i+1} \to V_i$ for all~$i \in \{1, \dots,k-1\}$.
A path is \newdef{directed}, if $V_i \to V_{i+1}$ for all~$i \in \{1, \dots,k-1\}$, i.e., all arrows ``point in the same direction''.
We say a path between $V_1$ and $V_k$ in a DAG $\cG$ is \newdef{blocked by a set of nodes~$S$}, where~$V_1, V_k \notin S$ when one of the following two conditions is met:
(a) There exists a $V_i \in S$ such that $V_{i-1} \to V_i \to V_{i+1}$ or $V_{i+1} \to V_i \to V_{i-1}$.
(b) Or there exists a $V_i$ such that $V_{i-1} \to V_i \gets V_{i+1}$ and neither $V_i$ nor any of its descendants is in $S$.
The notion of blocked paths finally allows us to define \newdef{d-separation}: 
Two disjoint sets of nodes $V$ and $W$ in a DAG $\cG$ are \newdef{d-separated} by a third pairwise disjoint set of nodes $S$ if every path between nodes in $V$ and $W$ is blocked by $S$.

Any structural equation model by definition entails a DAG $\cG$ and a unique distribution over all its variables.
This raises a general question about the compatibility of graphs with joint distributions over the same variables in terms of d-separation on the one hand and conditional independences on the other hand.
In particular, the following definitions will be relevant for \chapref{chap:causal}.
Let $\cG$ be a DAG and $\distr$ a joint distribution over all variables in $\cG$ that allows for a density with respect to the product measure.
Then we say $\distr$ \newdef{is Markov} with respect to $\cG$ if for all disjoint sets of nodes $V, W$ and $S$ we have that
\begin{equation*}
V \text{ and } W \text{ are d-separated by } S \Rightarrow V \indep W \given S \eqp
\end{equation*}
Two DAGs $\cG_1$ and $\cG_2$ are called \newdef{Markov equivalent} if the sets of all distributions (with a density) that are Markov with respect to a $\cG_1$ and $\cG_2$ respectively, are equal.
Finally, from this follows the \newdef{Markov equivalence class of a DAG $\cG$}, which we define as the set of all DAGs that are Markov equivalent to $\cG$.

Since two graphs are Markov equivalent if and only if they satisfy the same set of d-separations, by the Markov property, we can infer that they satisfy the exact same set of conditional independence relations.
As we have seen, the Markov property talks about when graphical d-separation criteria imply the corresponding conditional independence relation.
There is an analogous definition for the other direction:
A joint distribution $\distr$ is \newdef{faithful to the DAG $\cG$} if for all disjoint sets of nodes $V, W$ and $S$ we have that
\begin{equation*}
V \indep W \given S V \text{ and } \Rightarrow W \text{ are d-separated by } S \eqp
\end{equation*}
Faithfulness can be violated when different paths in a given causal graph cancel each other out such that there is an independence relation satisfied by the joint distribution, for which the corresponding d-separation is violated.
In \chapref{chap:causal} we will argue that such a situation may indeed occur in practice.

\paragraph{Interpretation.} We consider the causal influences captured by the $f_i$ to be mechanisms that are intrinsic to how the universe works.
Prototypical examples would be laws of physics, however, we will also assume that there exist invariant mechanisms between fuzzier concepts that can still be captured in a functional form.
For example, in causal inference for health-care, we may assume that one can meaningfully model concepts such as ``smoker'' and ``has lung cancer'' as random variables as well as the causal influence between the two as a mathematical function.
Similarly, economists may seek to compute the causal effect of the ``health of institutions'' on the gross domestic product.
None of these concepts are unambiguously defined properties of the physical world and---as somewhat constructed notions---are themselves not ``causing'' anything in the physical sense.
However, we still believe that there exist causal mechanisms relating them and that we can meaningfully represent those by mathematical equations.

In these examples, we essentially claim that the ontology of ``smoker'', ``lung cancer'', ``health of institutions'', and ``gross domestic product'' as well as the assumed interactions between them are stable enough.
Intuitively, this means that we can pinpoint and agree upon what these terms reference with a sufficient degree of certainty and precision.
The mass or electric charge of a physical object have a more stable ontology than the health of institutions.
However, we may still agree on what is referenced by the health of institutions and can agree on the methods that led us to believe in that ontology.
As we have discussed in \secref{sec:how}, moving to concepts such as race and gender, ontological stability becomes a serious concern.
How to practically address these challenges is the subject of ongoing research \citep{kohler2018eddie,hu2020whats}.

In the context of this thesis, the predictor~$\hY$ maps inputs, e.g., the features~$X$ and the protected attribute~$Z$, to a predicted outcome.
Hence, we model $\hY$ as a childless node, whose parents are its input variables.
This predictor node has a special role, because its structural equation does not correspond to universally true mechanism of the universe, but can be chosen by the decision-maker, for example by training a predictive model.
While we commonly display $\hY$ as part of a causal graph like all other variables, it is important to keep in mind that its incoming edges have a different meaning from the others.

We note that a key difference between causal SEMs and probabilistic graphical models is that in SEMs we only care about causal factorizations of the joint probability.
This means that by writing down a causal DAG, we explicitly formulate the assumption that the parents of any given variable are precisely the variables that actually have a direct causal influence on that variable.
In contrast, probabilistic graphical models allow for any kind of factorization of the joint distribution and thus typically do not require parents to be causes of their descendants.
Arguably the two most important tools the language of causal SEMs offers beyond the standard statistical toolkit of probabilistic graphical models are formal notions of interventions and counterfactuals.
Hence, it is little surprising that most existing causal fairness criteria are based on one or the other.
We now provide a short overview of works at the intersection of causality and fairness in machine learning.

\paragraph{Causal fairness criteria.}
\citet{Kusner2017} were among the first to put forward one possible causal definition, namely the notion of \newdef{counterfactual fairness}.
It ensures that the outcome any individual receives in the real world is the same as the one she would receive in a \emph{counterfactual world} in which only the protected attribute of the individual has changed, everything else remaining equal.
This can be formalized as
\begin{equation}\label{eq:counterfactual}
  \prob(\hY_{Z=z} = 1 \given X=x, Z=z) = \prob(\hY_{Z=z'} = 1 \given X=x, Z=z) \fora x \in \cX, \; z, z' \in \cZ \eqp
\end{equation}
We interpret the right hand side of this equation as \emph{the probability that $\hY$ predicts $1$ for a given individual for which we observe features $x$ and protected attribute $z$, had the protected attribute been $z'$ instead of $z$}.
An analogous interpretation of the left hand side of eq.~\eqref{eq:counterfactual} asserts that it is equal to $\prob(\hY = 1 \given X=x, Z=z)$.

Note that this definition requires modeling counterfactuals on a per individual level, which is a delicate task.
Even determining the effect of \emph{race} at the group level is difficult as discussed by \citet{VanderWeele2014,hu2020whats}.
Interventions on (often ill-defined) protected attributes such as gender or race are generically hard to conceive.
For example, imagine a pregnant woman'{}s job application gets rejected.
How should we conceive of the world in which she has been a man?
Should we think of her as being born male, or being perceived as male during the hiring procedure?
Is she a pregnant man now?
The counterfactual where somebody was ``born a different person'', may lead to the comparison of vastly different (fictitious) individuals, undermining the key motivation for counterfactual fairness.
Such challenges are aggravated when moving beyond inappropriately simplified binary gender assignments.
In \chapref{chap:causal} we propose a proxy-based approach to formalize a continuum of possible interventions that may be harder or easier to conceive and---perhaps even more important---to perform in practice.

In contrast to counterfactuals, for a pure intervention there is no abduction step, i.e., we do not update the distribution of the exogenous variables conditioned on the observations.
Instead, we simply replace the structural equation for the variable we intervene on by setting it to a constant (or a fixed distribution).
Intuitively, this notion is most often compared to randomized trials.
For example, before having recruiters screen written applications, we could intervene by randomizing the names---perhaps one of the strongest salient indicators for gender---of the applicants.
Note that in this case we only intervene on a proxy for gender.
The intervention does not account for potential gender discrimination applicants have experienced before this specific application.
However, it can easily be carried out in practice.
Indeed, such a randomized study has shown that applications with typical White-sounding names receive considerably higher callback rates for interviews than African-American-sounding names \citep{bertrand2004emily}.

Formally, \newdef{interventional fairness} can be stated in the form
\begin{equation*}
  \prob(\hY = 1 \given do(Z=0) {\color{gray}, X=x}) = \prob(\hY = 1 \given do(Z=1) {\color{gray}, X=x}) \eqc
\end{equation*}
where we may or may not condition on $x \in \cX$.
Note that even when we condition on the features, we still match probabilities within groups of potentially different individuals who happen to have the same features.
This notion is different from the counterfactual statement in eq.~\eqref{eq:counterfactual}, where we condition on having observed a specific individual to update the exogenous variables before intervening on the protected attribute.
We further remark that we still used the letter $Z$, even though in the example we spoke of a proxy of the protected attribute.
This distinction will be made clear in \chapref{chap:causal}.

\paragraph{Related work on causality-based fairness.}
More generally, causality has already been employed for the discovery of discrimination in existing datasets \citep{Bonchi2015, Qureshi2016}.
Causal graphical conditions to identify \emph{meaningful partitions} have been proposed for the discovery and prevention of certain types of discrimination by
preprocessing the data \citep{Zhang2017b}.
These conditions rely on the evaluation of \emph{path specific effects}, which relates closely to earlier work by \citet[Section 4.5.3]{Pearl2009}.
\citet{Nabi2017} picked up this notion and generalized Pearl'{}s approach by a constraint based prevention of discriminatory path specific effects arising from counterfactual reasoning.
The idea of path specific effects to distinguish fair and unfair causal pathways, and further to suppress influences along the unfair ones, have also been combined with deep learning and variational inference approaches to render it applicable in complex, non-linear scenarios \citep{chiappa2018path}.
Further, \citet{zhang2018fairness} propose three fine grained types of causal influence, based on which they attempt to aid the practitioner to jointly model the data generating mechanisms and an appropriate fairness criterion.
For more details, we refer the reader to a survey on causal reasoning for algorithmic fairness by \citet{loftus2018causal}.

\paragraph{Technical challenges of causal reasoning for fairness.}
At the beginning of this section, we argued that causality is conceptually useful to tackle some of the subtle issues when it comes to fairness criteria.
However, there is a price to pay.
First, most works start from the usual assumption that the causal graph is known.
Since causal discovery is generically difficult \citep{peters2017elements}, this is not an assumption to make lightly.
\citet{russell2017worlds} present some ideas on how to achieve approximate counterfactual fairness for multiple competing causal models simultaneously to account for some degree of uncertainty about the correct one.
Furthermore, causal inference often comes with a set of hard-to-verify or even untestable standard assumptions such as identifiability, no unobserved confounding, or faithfulness \citep{Pearl2009}.
Finally, in particular for notions based on counterfactuals, there is usually no data available for training and verification, because by their very definition, we never observe the counterfactual outcomes.
We develop tools to deal with some of these issues.
In \chapref{chap:causal}, we address potential misspecifications of the causal graph.
The issue of labels missing depending on our decisions are discussed in \chapref{chap:decisions}.

\section{Fair representations}
\label{sec:fairrepresentations}

Supervised representation learning offers a natural approach to demographic parity.
The basic idea is to transform the input features via a mapping $t: \cX \to \widetilde{\cX}, x \mapsto \widetilde{x}$ into a new space $\widetilde{\cX}$ such that the following conditions hold:
\begin{enumerate}
  \item We can train a predictor $\hY: \widetilde{\cX} \to \cY$ with high accuracy. Informally, the new representation $\widetilde{x}$ of the features $x$ should still contain all relevant information about the true outcome $y$.
  \item It is hard to learn a predictor $\hat{Z}: \widetilde{\cX} \to \cZ$, i.e., the new representation $\widetilde{x}$ of the features $x$ contains no information about the corresponding protected attribute $z$.
\end{enumerate}
The aspiration of learned representations is usually their usefulness for potentially unknown downstream tasks.
In the fairness context, the hope is that once we have learned such a fair transformation $t$, one can freely train any unconstrained predictor $\hY: \widetilde{\cX} \to \cY$ without having to worry about introducing unfairness.

\citet{Zemel2013,calmon2017optimized} formulate this pre-processing task as an optimization problem and directly learn $t$.
A large body of work in this direction is based on \emph{adversarial training}, i.e., learning $t$ jointly with $\hY$ and $\hat{Z}$ by optimizing for the two objectives---minimizing the loss for $\hY$ and maximizing the loss for $\hat{Z}$---simultaneously by alternating gradient descent \citep{Edwards2015,sokolic2017learning,madras18a}.
Another approach exploits variational auto-encoders \citep{kingma2013auto} with an additional loss terms to render the latent representation a bad predictor for the protected attribute \citep{louizos2015variational}.

What we described here as \emph{fair representation learning} is still narrowly concerned with outcome-based fairness.
Other issues with biased representations, some of which were mentioned in \secref{sec:motivation}, are beyond the scope of this thesis.

\section{Procedural approaches and human perception}
\label{sec:procedural_approaches}

Almost all existing work discussed in this chapter falls into the broad category of distributive fairness.
In a series of intriguing papers, \citet{grgic2016case,grgic2018,grgic2018human,grgic2020dimensions} study procedural approaches, or \emph{process fairness}, based on the collective moral judgment of humans.
For example, the surveys ask which features people find fair to use, e.g., in the COMPAS recidivism risk assessment setting, and further analyze how perceptions of fairness relate to demographics and personal experiences of individuals.
Such empirical approaches may help describe the procedural component of fairness and justice.
For example, in a related setting, the landmark study on moral decisions in trolley problems related to autonomous driving called \emph{the moral machine experiment} provided insights that may inform regulatory decisions regarding autonomous vehicles \citep{awad2018moral}.

\chapter{Causality and fairness}
\label{chap:causal}

In this chapter, we go beyond the assumption that only observational data without any additional information is available.
By framing the problem of discrimination based on protected attributes in the language of causal reasoning, we can circumvent the limitations of observational criteria.
This viewpoint shifts attention from ``What is the right fairness criterion?'' to ``What do we want to assume about our model of the causal data generating process?''
Through the lens of causality, we make several contributions.
First, we crisply articulate why and when observational criteria fail, thus formalizing what was before a matter of opinion.
Second, our approach exposes previously ignored subtleties and why they are fundamental to the problem.
For example, we do not assume we can evaluate individual causal effects and meaningfully conceive interventions on protected attributes.
Finally, we put forward natural causal non-discrimination criteria and develop algorithms that satisfy them.

The main content of this chapter has been published in the following paper:

\fbox{\parbox{\textwidth}{
\pubitem{Avoiding Discrimination through Causal Reasoning}
{Niki Kilbertus, Mateo Rojas-Carulla, Giambattista Parascandolo, Moritz Hardt, Dominik Janzing, Bernhard Schölkopf}
{Neural Information Processing Systems (NeurIPS), 2017}
{https://arxiv.org/abs/1706.02744}
{}
{}
}}

\section{Introduction}
\label{sec:neurips:intro}

We start from the fact that observational criteria are insufficient to conclusively capture fairness simply because there exist scenarios with \emph{intuitively} different social interpretations that admit identical joint distributions over $(\hY, Z, Y, X)$ \citep{Hardt2016}.
Thus, no observational criterion can distinguish them.
Inspired by Pearl'{}s causal interpretation of Simpson'{}s paradox \citep[Section 6]{Pearl2009}, we propose causality as a way of coping with this unidentifiability result.
Carefully using the language of causal reasoning supports several contributions:
\begin{itemize}
  \item Revisiting the two scenarios proposed by \citet{Hardt2016} that cannot be differentiated by observational criteria, we articulate a natural causal criterion that formally distinguishes them.
  \item We point out subtleties in fair decision making that arise naturally from a causal perspective, but have gone widely overlooked in the past.
  Specifically, we formally argue for the need to distinguish between the underlying concept behind a protected attribute, such as race or gender, and its \emph{proxies} available to the algorithm, such as visual features or name.
  \item We introduce and discuss two natural causal criteria centered around the notion of \emph{interventions} (relative to a causal graph) to formally describe specific forms of discrimination.
  \item Finally, we initiate the study of algorithms that avoid these forms of discrimination.
  Under certain linearity assumptions about the underlying causal model generating the data, an algorithm to remove a specific kind of discrimination leads to a simple and natural heuristic.
\end{itemize}

At a higher level, our work proposes a shift from trying to find a single statistical fairness criterion to arguing about properties of the data and which assumptions about the generating process are justified.
Causality provides a flexible framework for organizing such assumptions.
In particular, we will introduce two complementary thought frameworks, which we call the \emph{benevolent} and the \emph{skeptic} viewpoint.
The benevolent viewpoint is characterized by generally not assuming causal influences of the sensitive attribute on the decision to be unfair.
Modelers may then mark specific variables along such paths that they deem should be causally irrelevant for the decision as \emph{proxy variables}.
In this scenario our goal is to make decisions that are not causally influenced by such proxy variables.
In the skeptic viewpoint, we take the stance that all causal influences from the sensitive attribute on the decision are by default considered inappropriate.
However, as modelers and stakeholder we may declare some variables along such paths as \emph{resolving variables} if we believe that they can fairly be used for making decisions despite being influenced by the sensitive attribute.
For example, while an employer may be prohibited from discriminating applicants based on their country of origin, it may still be appropriate to demand certain language skills, which are often influenced by the country of origin.
This process of closely analysing an assumed causal model underlying the observed data and making deliberate judgment calls about the fairness of causal influences, i.e., marking variables as proxies or resolving variables, is at the heart of our proposal and calls for a more situational and context dependent assessment of fairness.

In addition to our introduction of structural equation models in \secref{sec:causalcriteria}, one more concept that will be useful in our exposition is that of \emph{terminal ancestors}.
We will refer to the \emph{terminal ancestors of a node~$V$ in a causal
graph~$\cG$}, denoted by~$ta^{\cG}(V)$, which are those ancestors of~$V$ that are also root nodes of~$\cG$.

\section{Unresolved discrimination and limitations of observational criteria}
\label{sec:neurips:limitations}

\begin{figure}
  \centering
  \begin{tikzpicture}
    \pgfmathsetmacro{\d}{1}
    \node[simple] (Z) at (0,0) {$Z$};
    \node[simple] (X) at (-0.6,-1) {$X$};
    \node[simple] (hY) at (0.6,-1) {$\hY$};
    \edge {Z} {X,hY};
    \edge {X} {hY}
  \end{tikzpicture}%
  \caption[A simple causal graph for a college admission setting]{The admission decision~$\hY$ does not only directly depend on gender~$Z$, but also on department choice~$X$, which in turn is also affected by gender~$Z$.}
  \label{fig:pearl}
\end{figure}
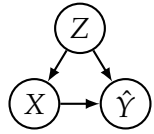

To bear out the limitations of observational criteria, we again turn to the commentary on claimed gender discrimination in Berkeley college admissions in Section 4.5.3 of \citet{Pearl2009}.
To recap, \citet{Bickel1975} had shown earlier that a lower college-wide admission rate for women than for men was explained by the fact that women applied in more competitive departments.
When adjusted for department choice, women experienced a slightly higher acceptance rate compared with men.
From the causal point of view, arguably what matters is the \emph{direct effect} of the protected attribute (here, gender~$Z$) on the decision (here, college admission~$\hY$) that cannot be ascribed to a mediator that has been judged to be admissible in the decision process without fairness concerns.
We shall use the term \emph{resolving variable} for any such variable in the causal graph that is influenced by~$Z$ in a manner that we accept as non-discriminatory, such as department choice~$X$ in our example, see Figure~\ref{fig:pearl}.\footnote{We remark again that disarming any discrimination concerns by only considering direct effects to be discriminatory is a strong normative statement.
The root cause may be much closer to different departments following discriminatory policies or nurturing a hostile environment towards women, see \secref{sec:causalcriteria}.}
With this convention, the criterion can be stated as follows.

\begin{definition}[Unresolved discrimination]
  A variable~$V$ in a causal graph exhibits \emph{unresolved discrimination} if there exists a directed path from~$Z$ to~$V$ that is not blocked by a resolving variable and~$V$ itself is non-resolving.
\end{definition}

Pearl'{}s commentary is consistent with what we call the \emph{skeptic viewpoint}.
All paths from the protected attribute~$Z$ to $\hY$ are problematic, unless they are justified by a resolving variable.
The presence of unresolved discrimination in the predictor~$\hY$ is worrisome and demands further scrutiny.
In practice,~$\hY$ is not a priori part of a given graph.
Instead it is our objective to construct it as a function of the features~$X$, some of which might be resolving.
Hence we should first look for unresolved discrimination in the features.
A canonical way to avoid unresolved discrimination in~$\hY$ is to only input the set of features that do not exhibit unresolved discrimination.
However, the remaining features might be affected by non-resolving \emph{and} resolving variables.
In \secref{sec:neurips:criteria} we investigate whether one can exclusively remove unresolved discrimination from such features.
A related notion of ``explanatory features'' in a non-causal setting was introduced by \citet{Kamiran2013}.

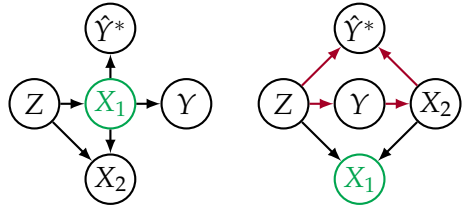
\begin{figure}
  \centering
\begin{tikzpicture}
  \pgfmathsetmacro{\d}{1}
  \node[simple] (Z) at (0,0) {$Z$};
  \node[simple, Green] (X1) at (\d,0) {$X_1$};
  \node[simple] (Y) at (2*\d,0) {$Y$};
  \node[simple] (X2) at (\d,-\d) {$X_2$};
  \node[simple] (hY) at (\d,\d) {$\hY^*$};
  \edge {Z} {X1,X2};
  \edge {X1} {Y,X2,hY};
\end{tikzpicture}%
  \hspace{0.5cm}
\begin{tikzpicture}
  \pgfmathsetmacro{\d}{1}
  \node[simple] (Z) at (0,0) {$Z$};
  \node[simple] (Y) at (\d,0) {$Y$};
  \node[simple, Green] (X1) at (\d,-\d) {$X_1$};
  \node[simple] (X2) at (2*\d,0) {$X_2$};
  \node[simple] (hY) at (\d,\d) {$\hY^*$};
  \edge[DarkRed] {Z} {Y,hY};
  \edge[DarkRed] {Y} {X2};
  \edge[DarkRed] {X2} {hY};
  \edge {Z} {X1};
  \edge {X2} {X1};
\end{tikzpicture}%
  \caption[Example for two graphs that generate the same joint distribution, but entail different fairness judgments]{Two graphs that may generate the same joint distribution for the Bayes optimal unconstrained predictor~$\hY^*$.
  If~$X_1$ is a resolving variable,~$\hY^*$ exhibits unresolved discrimination in the right graph (along the red paths), but not in the left one.}
  \label{fig:unidentifiability}
\end{figure}

The definition of unresolved discrimination in a predictor has some interesting special cases worth highlighting.
If we take the set of resolving variables to be empty, we intuitively get a causal analog of demographic parity.
No directed paths from~$Z$ to~$\hY$ are allowed, but~$Z$ and~$\hY$ can still be statistically dependent.
Similarly, if we choose the set of resolving variables to be the singleton set~$\{Y\}$ containing the true outcome, we obtain a causal analog of equalized odds where strict independence is not necessary.
The causal intuition implied by ``the protected attribute should not affect the prediction'', and ``the protected attribute can only affect the prediction when the information comes through the true label'', is neglected by (conditional) statistical independences~$Z \indep \hY$, and~$Z \indep \hY \given Y$, but well captured by only considering dependences mitigated along directed causal paths.

We will next show that observational criteria are fundamentally unable to determine whether a predictor exhibits unresolved discrimination or not.
This is true even if the predictor is \emph{Bayes optimal}.
In passing, we also note that fairness criteria such as equalized odds may or may not exhibit unresolved discrimination, but this is again something an observational criterion cannot determine.

\begin{theorem}\label{thm:unidentifiability}
  Given a joint distribution over the protected attribute~$Z$, the true label~$Y$, and some features~$X_1, \dots, X_n$, in which we have already specified the resolving variables, no observational criterion can generally determine whether the Bayes optimal unconstrained predictor or the Bayes optimal equal odds predictor exhibit unresolved discrimination.\footnote{The Bayes optimal equal odds predictor is the Bayes optimal predictor among all predictors that satisfy the equal odds fairness criterion.}
\end{theorem}

\begin{proof}
  Let us consider the two graphs in Figure~\ref{fig:unidentifiability}.
  First, we show that these graphs can generate the same joint distribution~$\distr(Z, Y, X_1, X_2, \hY^*)$ for the Bayes optimal unconstrained predictor~$\hY^*$.

  We choose the following structural equations for the graph on the left\footnote{$\sigma(x) = 1 / (1 + e^{-x})$}:
  $Z = \Ber(\nicefrac{1}{2})$,
  $X_1$ is a mixture of Gaussians $\cN(Z+1, 1)$
  with weight~$\sigma(2 Z)$ and
  $\cN(Z-1, 1)$ with weight~$\sigma(-2 Z)$,
  $Y = \Ber( \sigma(2 X_1) )$,
  $X_2 = X_1 - Z$,
  and $\hY^* = X_1$.
  Throughout this proof we assume that the Bernoulli distribution~$\Ber$ has support~$\{-1, 1\}$ instead of $\{0, 1\}$.

  For the graph on the right, we define the structural equations
  $Z = \Ber(\nicefrac{1}{2})$,
  $Y = \Ber(\sigma(2 Z))$,
  $X_2 = \cN(Y,1)$,
  $X_1 = Z + X_2$, and
  $\hY^* = X_1$.

  First we show that in both scenarios~$\hY^*$ is actually an optimal score.
  In the first scenario~$Y \indep Z \given X_1$ and~$Y \indep X_2 \given X_1$ thus the optimal predictor is only based on~$X_1$.
  We find
  \begin{equation}\label{eq:YgivenX1}
    \prob(Y = y \given X_1 = x_1) = \sigma(2 x_1 y) \eqc
  \end{equation}
  which is monotonic in~$x_1$.
  Hence optimal classification is obtained by thresholding a score based only on~$\hY^* = X_1$.

  In the second scenario, because~$Y \indep X_1 \given \{Z, X2\}$ the optimal predictor only depends on~$Z, X_2$.
  We compute for the densities (which all exist and are positive)
  \begin{subequations}\label{eq:YgivenX2A}
    \begin{align}
      \distr(Y \given X_2, Z) & = \frac{\distr(Y, X_2, Z)}{\distr(X_2, Z)} \\
      &= \frac{\distr(X_2, Z \given Y) \distr(Y)}{\distr(X_2, Z)} \\
      &= \frac{\distr(X_2 \given Y) \distr(Z \given Y)
        \distr(Y)}{\distr(X_2, Z)} \\
      &= \frac{\distr(X_2 \given Y) \frac{\distr(Y \given Z)
        \distr(Z)}{\distr(Y)} \distr(Y)}{\distr(X_2, Z)} \\
      &= \frac{\distr(X_2 \given Y) \distr(Y \given Z) \distr(Z)}
        {\distr(X_2, Z)} \eqc
    \end{align}
  \end{subequations}
  where for the third equal sign we use~$Z \indep X_2 \given Y$.
  In the numerator we have
  \begin{equation}\label{eq:density}
    \dens(x_2 \given Y=y) \dens(y \given Z=z) \dens(z)
    = f_{\cN(y,1)}(x_2) f_{\Ber(\sigma(2 z))}(y) f_{\Ber(\nicefrac{1}{2})}(z)\eqc
  \end{equation}
  where~$f_{D}$ is the probability density function of the distribution~$D$.
  The denominator can be computed by summing up eq.~\eqref{eq:density} for~$y\in \{-1,1\}$.
  Overall this results in
  \begin{equation*}
    \prob(Y = y \given X_2 = x_2, Z=z) = \sigma(2 y (z + x_2)) \eqp
  \end{equation*}

  Since by construction~$X_1 = Z + X_2$, the optimal predictor is again~$\hY^* = X_1$.
  If the joint distribution~$\distr(Z, Y, \hY^*)$ is identical in the two scenarios, so are the joint distributions~$\distr(Z, Y, X_1, X_2, \hY^*)$, because of~$X_1 = \hY^*$ and~$X_2 = X_1 - Z$.
  To show that the joint distributions~$\distr(Z, Y, \hY^*) = \distr(Y \given Z, \hY^*) \distr(\hY^* \given Z) \distr(Z)$ are the same, we compare the conditional distributions in the factorization.
  Let us start with~$\distr(Y \given Z, \hY^*)$.
  Since~$\hY^* = X_1$ and in the first graph~$Y \indep Z \given X_1$, we already found the distribution in eq.~\eqref{eq:YgivenX1}.
  In the right graph,~$\distr(Y \given \hY^*, Z) = \distr(Y \given X_2 + Z, Z) = \distr(Y \given X_2, Z)$ which we have found in eq.~\eqref{eq:YgivenX2A} and coincides with the conditional in the left graph because of~$X_1 = Z + X_2$.

  Now consider~$\hY^* \given Z$. In the left graph we have~$\distr(\hY^* \given Z) = \distr(X_1 \given Z)$ and the distribution~$\distr(X_1 \given Z)$ is just the mixture of Gaussians defined in the structural equation model.
  In the right graph~$\hY^* = Z + X_2 = Y + \cN(Z,1)$ and thus~$\distr(\hY^* \given Z) = \cN(Z \pm 1)$ for~$Y = \pm 1$.
  Because of the definition of~$Y$ in the structural equations of the right graph, following a Bernoulli distribution with probability~$\sigma(2 Z)$, this is the same mixture of Gaussians as the one we found for the left graph.
  The distribution of~$Z$ is identical in both cases.
  Consequently the joint distributions agree.
  When~$X_1$ is a resolving variable, the optimal predictor in the left graph does not exhibit unresolved discrimination, whereas the graph on the right does.

  The proof for the equal odds predictor~$\widetilde{\hY}$ is immediate once we show~$\widetilde{\hY} = X_2$.
  This can be seen from the graph on the right, because here~$X_2 \indep Z \given Y$ and both using~$Z$ or~$X_1$ would violate the equal odds condition.
  Because the joint distribution in the left graph is the same,~$\widetilde{\hY} = X_2$ is also the optimal equal odds score.\qedhere
\end{proof}

The two graphs in Figure~\ref{fig:unidentifiability} are taken from \citet{Hardt2016}, which we here reinterpret in the causal context to prove Theorem~\ref{thm:unidentifiability}.
We point out that there is an established set of conditions under which unresolved discrimination can, in fact, be determined from observational data.
Note that the two graphs are not Markov equivalent.
Therefore, to obtain the same joint distribution we must violate faithfulness.\footnote{See \secref{sec:causalcriteria} for definitions of \emph{faithfulness}, \emph{Markov property}, and \emph{Markov equivalence class}.}
If we do assume the Markov condition and faithfulness, then conditional independences determine the graph up to its Markov equivalence class.
We later argue that violation of faithfulness is by no means pathological, but emerges naturally when designing predictors~$\hY$.
In any case, interpreting conditional dependences can be difficult in practice \citep{Cornia2014}.

\section{Proxy discrimination and interventions}
\label{sec:neurips:proxies}

We now turn to an important aspect of our framework.
Determining causal effects in general requires modeling interventions.
Interventions on deeply rooted individual properties such as \emph{gender} or \emph{race} are notoriously difficult to conceptualize---especially at an individual level, and impossible to perform in a randomized trial.
\citet{VanderWeele2014} discuss the problem comprehensively in an epidemiological setting and \citet{hu2020whats} analyze it in the context of fair machine learning using the example of what constitutes sex as a sensitive attribute.
From a machine learning perspective, it thus makes sense to separate the protected attribute~$Z$ from its potential \emph{proxies}, such as name, visual features, languages spoken at home, etc.
Intervention based on proxy variables poses a more manageable problem.
By deciding on a suitable proxy we can find an adequate mounting point for determining and removing its influence on the prediction.
Moreover, in practice we are often limited to imperfect measurements of~$Z$ in any case, making the distinction between root concept and proxy prudent.

As was the case with resolving variables, a \emph{proxy} is a priori nothing more than a descendant of~$Z$ in the causal graph that we choose to label as a proxy.
Nevertheless in reality we envision the proxy to be a clearly defined observable quantity that is significantly correlated with~$Z,$ yet in our view should not affect the prediction.

\begin{definition}[Potential proxy discrimination]\label{def:proxy}
  A variable~$V$ in a causal graph exhibits \emph{potential proxy discrimination}, if there exists a directed path from~$Z$ to~$V$ that is blocked by a proxy variable and~$V$ itself is not a proxy.
\end{definition}

Potential proxy discrimination articulates a causal criterion that is in a sense dual to unresolved discrimination.
From the \emph{benevolent viewpoint}, we \emph{allow} any path from~$Z$ to~$\hY$ unless it passes through a proxy variable, which we consider worrisome.
This viewpoint acknowledges the fact that the influence of~$Z$ on other variables may be complex and it can be too restraining to rule out all but a few designated features.
In practice, as with unresolved discrimination, we can naively build an unconstrained predictor based only on those features that do not exhibit potential proxy discrimination.
Then we must not provide proxy~$P$ as input to~$\hY;$\footnote{We note that the symbols $\prob$ (used for the distributions of random variables) and $P$ (denoting proxy variables) may be hard to distinguish.
However, the respective meaning will usually be clear from the context.}
unawareness, i.e., excluding~$P$ from the inputs of~$\hY$, suffices.
However, by granting~$\hY$ access to~$P$, we can carefully tune the function~$\hY(P, X)$ to cancel the implicit influence of~$P$ on features~$X$ that exhibit potential proxy discrimination by the explicit dependence on $P$.
Due to this possible cancellation of paths, we called the path based criterion \emph{potential} proxy discrimination.
When building predictors that exhibit no \emph{overall proxy discrimination}, we precisely aim for such a cancellation.

Fortunately, this idea can be conveniently expressed by an \emph{intervention} on~$P$, which is denoted by~$do(P=p)$ \citep{Pearl2009}.
Visually, intervening on~$P$ amounts to removing all incoming arrows of~$P$ in the graph;
algebraically, it consists of replacing the structural equation of $P$ by $P=p$, i.e., we put point mass on the value~$p$.

\begin{definition}[Proxy discrimination]\label{def:overalldisc}
  A predictor~$\hY$ exhibits no \emph{proxy discrimination} based on a proxy~$P$ if for all~$p, p'$
  \begin{equation}\label{eq:proxy-eq}
    \distr(\hY \given do(P=p) ) = \distr(\hY \given do(P=p') ) \eqp
  \end{equation}
\end{definition}

The interventional characterization of proxy discrimination leads to a simple procedure to remove it in causal graphs that we will turn to in the next section.
It also leads to several natural variants of the definition that we discuss in \secref{sec:neurips:variants}.
We remark that Equation~\eqref{eq:proxy-eq} is an equality of probabilities in the ``do-calculus'' that cannot in general be inferred by an observational method, because it depends on an underlying causal graph \citep{Pearl2009}.
However, in some cases, we do not need to resort to interventions to avoid proxy discrimination.

\begin{proposition}\label{pro:unawareness}
  If there is no directed path from a proxy to a feature, unawareness as in \secref{sec:fairness_through_unawareness} avoids proxy discrimination.
\end{proposition}

\begin{proof}
  An unaware predictor~${\hY}$ is given by~${\hY} = r (X)$ for some function~$r$ and features~$X$.
  If there is no directed path from proxies~$P$ to~$X$, i.e.,~$P \notin ta^{\cG}(X)$, then~${\hY} = r(X) = r(ta^{\cG}(X)) = r(ta^{\cG}(X) \setminus \{P\})$.
  Thus~$\distr({\hY} \given do(P=p)) = \distr({\hY})$ for all~$p$, which avoids proxy discrimination.\qedhere
\end{proof}

\section{Procedures for avoiding discrimination}
\label{sec:neurips:criteria}

Having motivated the two types of discrimination that we distinguish, we now turn to building predictors that avoid them in a given causal model.
First, we remark that a more comprehensive treatment requires individual judgment of not only variables, but the legitimacy of every existing path that ends in~$\hY$, i.e., evaluation of \emph{path-specific effects} \citep{Zhang2017b, Nabi2017}, which is tedious in practice.
The natural concept of proxies and resolving variables covers most relevant scenarios and allows for natural removal procedures.

\begin{figure}
\begin{center}
  \begin{minipage}{0.5\textwidth}
    \begin{tikzpicture}
      \pgfmathsetmacro{\dy}{0.9}
      \node[simple] (NP) at (-1,0) {$N_P$};
      \node[simple, Green] (Z) at (0,0) {$Z$};
      \node[simple] (NX) at (1,0) {$N_X$};
      \node[simple, DarkRed] (P) at (-0.5,-\dy) {$P$};
      \node[simple] (X) at (0.5,-\dy) {$X$};
      \node[simple] (hY) at (0,-2*\dy) {$\hY$};
      \node[draw=none] at (-0.7,-2*\dy) {$\tcG$};
      \edge {NP} {P};
      \edge {Z} {P,X};
      \edge {NX} {X};
      \edge {P} {X,hY};
      \edge {X} {hY};
    \end{tikzpicture}%
    \hspace{0.5cm}
    \begin{tikzpicture}
      \pgfmathsetmacro{\dy}{0.9}
      \node[simple] (NP) at (-1,0) {$N_P$};
      \node[simple, Green] (Z) at (0,0) {$Z$};
      \node[simple] (NX) at (1,0) {$N_X$};
      \node[simple, DarkRed, double] (P) at (-0.5,-\dy) {$P$};
      \node[simple] (X) at (0.5,-\dy) {$X$};
      \node[simple] (hY) at (0,-2*\dy) {$\hY$};
      \node[draw=none] at (-0.7,-2*\dy) {$\cG$};
      \edge[DarkRed] {P} {X,hY};
      \edge[DarkRed] {X} {hY};
      \edge {Z} {X};
      \edge {NX} {X};
    \end{tikzpicture}%
    \caption[A template graph to reason about proxy discrimination and its intervened version]{A template graph~$\tcG$ for proxy discrimination (left) with its intervened version~$\cG$ (right).
    While from the benevolent viewpoint we do not generically prohibit any influence from~$Z$ on~$\hY$, we want to guarantee that the proxy~$P$ has no overall influence on the prediction, by adjusting~$P \to \hY$ to cancel the influence along~$P \to X \to \hY$ in the intervened graph.}
  \label{fig:proxytemplate}
  \end{minipage}%
  \hfill
  \begin{minipage}{0.45\textwidth}
    \begin{tikzpicture}
      \pgfmathsetmacro{\dy}{0.9}
      \node[simple] (NE) at (-1,0) {$N_E$};
      \node[simple, DarkRed] (Z) at (0,0) {$Z$};
      \node[simple] (NX) at (1,0) {$N_X$};
      \node[simple, Green] (E) at (-0.5,-\dy) {$E$};
      \node[simple] (X) at (0.5,-\dy) {$X$};
      \node[simple] (hY) at (0,-2*\dy) {$\hY$};
      \node[draw=none] at (-0.7,-2*\dy) {$\tcG$};
      \edge {NE} {E};
      \edge {Z} {E,X};
      \edge {NX} {X};
      \edge {E} {X,hY};
      \edge {X} {hY};
    \end{tikzpicture}%
    \hspace{0.5cm}
    \begin{tikzpicture}
      \pgfmathsetmacro{\dy}{0.9}
      \node[simple] (NE) at (-1,0) {$N_E$};
      \node[simple, DarkRed] (Z) at (0,0) {$Z$};
      \node[simple] (NX) at (1,0) {$N_X$};
      \node[simple, Green, double] (E) at (-0.5,-\dy) {$E$};
      \node[simple] (X) at (0.5,-\dy) {$X$};
      \node[simple] (hY) at (0,-2*\dy) {$\hY$};
      \node[draw=none] at (-0.7,-2*\dy) {$\cG$};
      \edge[DarkRed] {Z} {X};
      \edge[DarkRed] {X} {hY};
      \edge {NX} {X};
      \edge {E} {X,hY};
    \end{tikzpicture}%
    \caption[A template graph to reason about unresolved discrimination and its intervened version]{A template graph~$\tcG$ for unresolved discrimination (left) with its intervened version~$\cG$ (right).
    While from the skeptical viewpoint we generically do not want~$Z$ to influence~$\hY$, we first intervene on~$E$ interrupting all paths through~$E$ and only cancel the remaining influence from~$Z$ to~$\hY$.}
    \label{fig:explanatorytemplate}
  \end{minipage}
\end{center}
\end{figure}

\paragraph{Avoiding proxy discrimination.} \label{subsec:neurips:proxydiscrimination}
While presenting the general procedure, we illustrate each step in the example shown in Figure~\ref{fig:proxytemplate}.
A protected attribute~$Z$ affects a proxy~$P$ as well as a feature~$X$.
Both~$P$ and~$X$ have additional unobserved causes~$N_P$ and~$N_X$, where~$N_P, N_X, Z$ are pairwise independent.
Finally, the proxy also has an effect on the features~$X$ and the predictor~$\hY$ is a function of~$P$ and~$X$.
Given labeled training data, our task is to find a good predictor that exhibits no proxy discrimination within a hypothesis class of functions~$\hY_{\theta}(P, X)$ parameterized by a real valued vector~$\theta$.

We now work out a formal procedure to solve this task under specific assumptions.
While the general procedure in principle works for arbitrarily large graphs and potentially non-linear structural equations, we will simultaneously illustrate the algorithm in a small fully linear example for clarity.
\textcolor{gray}{We will use gray font for the specific example, in which the structural equations are given by
\begin{equation*}
  P = \alpha_P Z + N_P, \qquad
  X = \alpha_X Z + \beta P + N_X, \qquad
  \hY_{\theta} = \lambda_P P + \lambda_X X \eqp
\end{equation*}
Note that we choose linear functions parameterized by~$\theta = (\lambda_P, \lambda_X)$ as the hypothesis class for~$\hY_{\theta}(P, X)$ in this example.}

In the procedure we clarify the notion of \emph{expressibility}, which is an assumption about the relation of the given structural equations and the hypothesis class we choose for~$\hY_{\theta}$.

\begin{proposition}\label{pro:proxydisc}
  If there is a choice of parameters~$\theta_0$ such that~$\hY_{\theta_0}(P,X)$ is constant with respect to its first argument and the structural equations are \emph{expressible}, the following procedure returns a predictor from the given hypothesis class that exhibits no proxy discrimination and is non-trivial in the sense that it can make use of features that exhibit potential proxy discrimination.
\end{proposition}

\begin{enumerate}
  \item Intervene on~$P$ by removing all incoming arrows and replacing the structural equation for~$P$ by~$P=p$.
    \textcolor{gray}{%
    For the example in Figure~\ref{fig:proxytemplate},
    \begin{equation} \label{eq:proxyR}
      P = p, \qquad
      X = \alpha_X Z + \beta P + N_X, \qquad
      \hY_{\theta} = \lambda_P P + \lambda_X X \eqp
    \end{equation}
    }
  \item Iteratively substitute variables in the equation for~$\hY_{\theta}$ from their structural equations until only root nodes of the intervened graph are left, i.e., write~$\hY_{\theta}(P, X)$ as~$\hY_{\theta}(P, g(ta^{\cG}(X)))$ for some function~$g$.
    Since the causal graph is acyclic, we can write any variable as a function of only root nodes by iteratively substituting parent nodes with their structural equations.
    \textcolor{gray}{%
    In the example,~$ta(X) = \{Z, P, N_X\}$ and
    \begin{equation}\label{eq:proxyRroots}
      \hY_{\theta} = g(ta(X)) = (\lambda_P + \lambda_X \beta) p + \lambda_X (\alpha_X Z + N_X) \eqp
    \end{equation}
    }
  \item We now require the distribution of~$\hY_{\theta}$ in eq.~\eqref{eq:proxyRroots} to be independent of~$p$.
    \textcolor{gray}{%
    In the example, we require for all~$p, p'$ that
    \begin{equation}\label{eq:proxyconstraint}
      \distr((\lambda_P + \lambda_X \beta) p + \lambda_X (\alpha_X Z + N_X)) =
      \distr((\lambda_P + \lambda_X \beta) p' + \lambda_X (\alpha_X Z + N_X)) \eqp
    \end{equation}
    }
    In general, the goal is to write the predictor as a function of~$P$ and all the other roots of~$\cG$ separately.
    If our hypothesis class is such that there exists a parameter vector~$\tilde{\theta}$ and a function $\tilde{g}$ such that~$\hY_{\theta}(P, g(ta(X))) = \hY_{\tilde{\theta}}(P, \tilde{g}(ta(X) \setminus \{P\}))$, we call the structural equation model and hypothesis class specified in eq.~\eqref{eq:proxyR} \emph{expressible}.
    Equation~\eqref{eq:proxyconstraint} then yields the \emph{non-discrimination constraint}~$\tilde{\theta} = \theta_0$.
    \textcolor{gray}{%
    Our example model is expressible with~$\tilde{\theta} = (\lambda_P + \lambda_X \beta, \lambda_X)$ and~$\tilde{g} = \alpha_X Z + N_X$.
    A possible~$\theta_0$ is~$\theta_0 = (0, \lambda_X)$, which simply yields~$\lambda_P = -\lambda_X \beta$.
    }
  \item Given labeled training data, we can optimize the predictor~$\hY_{\theta}$ within the hypothesis class as given in eq.~\eqref{eq:proxyR}, subject to the non-discrimination constraint.
    \textcolor{gray}{%
    In our linear example, we obtain
    \begin{equation*}
      \hY_{\theta} = -\lambda_X \beta P + \lambda_X X = \lambda_X (X - \beta P) \eqc
    \end{equation*}
    with the free parameter~$\lambda_X \in \bR$ that we can optimize for accuracy.}
\end{enumerate}

In general, the non-discrimination constraint eq.~\eqref{eq:proxyconstraint} is by construction just~$\distr(\hY \given do(P=p)) = \distr(\hY \given do(P=p'))$, coinciding with Definition~\ref{def:overalldisc}.
Thus Proposition~\ref{pro:proxydisc} holds by construction of the procedure.
The choice of~$\theta_0$ strongly influences the non-discrimination constraint.
However, as the example shows, it allows~$\hY_{\theta}$ to exploit features that exhibit potential proxy discrimination.

\paragraph{Avoiding unresolved discrimination.}\label{subsec:neurips:unexplaineddiscrimination}
We proceed analogously to the previous section using the example graph in Figure~\ref{fig:explanatorytemplate}.
Instead of the proxy, we consider a resolving variable~$E$.
\textcolor{gray}{%
For the running example, the causal dependences are equivalent to the ones in Figure~\ref{fig:proxytemplate} and we again assume linear structural equations for a running example
\begin{equation*}
  E = \alpha_E Z + N_E, \qquad
  X = \alpha_X Z + \beta E + N_X, \qquad
  \hY_{\theta} = \lambda_E E + \lambda_X X \eqp
\end{equation*}}

Let us now try to adjust the previous procedure to the context of avoiding unresolved discrimination.

\begin{enumerate}
  \item Intervene on~$E$ by fixing it to a random variable~$\eta$ with~$\distr(\eta) = \distr(E)$, the marginal distribution of~$E$ in~$\tcG$, see Figure~\ref{fig:explanatorytemplate}.
    \textcolor{gray}{%
    In the example we find
    \begin{align} \label{eq:explanatoryR}
      E = \eta, \qquad
      X = \alpha_X Z + \beta E + N_X, \qquad
      \hY_{\theta} = \lambda_E E + \lambda_X X \eqp
    \end{align}
    }
  \item By iterative substitution of structural equations, write $\hY_{\theta}(E, X)$ as~$\hY_{\theta}(E, g(ta^{\cG}(X)))$ for some function~$g$.
    \textcolor{gray}{%
    In the example
    \begin{equation}\label{eq:explanatoryRroots}
      \hY_{\theta} = g(ta^{\cG}(X)) = (\lambda_E + \lambda_X \beta) \eta + \lambda_X \alpha_X Z + \lambda_X N_X \eqp
    \end{equation}
    }
  \item We now demand the distribution of~$\hY_{\theta}$ in eq.~\eqref{eq:explanatoryRroots} be invariant under interventions on~$Z$, which coincides with conditioning on~$Z$ whenever~$Z$ is a root of~$\tcG$.
    \textcolor{gray}{%
    Hence, in the example, for all~$z, z'$
    \begin{equation}\label{eq:explanatoryconstraint}
      \distr((\lambda_E + \lambda_X \beta) \eta + \lambda_X \alpha_X z + \lambda_X N_X) =
      \distr((\lambda_E + \lambda_X \beta) \eta + \lambda_X \alpha_X z' + \lambda_X N_X) \eqp
    \end{equation}
    }
\end{enumerate}

Here, the subtle asymmetry between proxy discrimination and unresolved discrimination becomes apparent.
Our approach relies on ``bundling'' directed paths from the sensitive variable $Z$ to the predictor $\hY_{\theta}$ into the ones that pass through $P$ and $E$ respectively.
In the benevolent setting, when talking about proxy discrimination, the goal is to correct for the influence of all paths passing through $P$.
Since our predictor $\hY_{\theta}$ can use the proxy $P$ directly as input, under the expressibility assumption, we can tune it such that it precisely counteracts the aggregate influence $P$ has on $\hY_{\theta}$ mediated by all other possible features.

On the other hand, in the skeptic setting, when talking about resolved discrimination, the goal is to only allow for the influences mediated by $E$.
Hence we would similarly seek to cancel out all paths from $Z$ on $\hY_{\theta}$ that do not pass through $E$.
However, in this setting $\hY_{\theta}$ is not explicitly a function of~$Z$.
Therefore, we cannot tune the predictor to cancel implicit influences of~$Z$ through~$X$.
There might still be a~$\theta_0$ such that~$\hY_{\theta_0}$ indeed fulfills eq.~\eqref{eq:explanatoryconstraint}, but without being able to use $Z$ as a direct input for the predictor there is no principled way for us to construct it.

In the example, eq.~\eqref{eq:explanatoryconstraint} suggests the obvious \emph{non-discrimination constraint}~$\lambda_X = 0$.
We can then proceed as before and, given labeled training data, optimize~$\hY_{\theta} =
\lambda_E E$ by varying~$\lambda_E$.
However, by setting~$\lambda_X = 0$, we also cancel the path~$Z \to E \to X \to \hY$, even though it is blocked by a
resolving variable.
In general, if~$\hY_{\theta}$ does not have access to~$Z$, we can not adjust for unresolved discrimination without also removing resolved influences from~$Z$ on~$\hY_{\theta}$.
If, however,~$\hY_{\theta}$ is a function of~$Z$, i.e., we add the term~$\lambda_Z Z$ to~$\hY_{\theta}$ in eq.~\eqref{eq:explanatoryR}, the non-discrimination constraint is~$\lambda_Z = - \lambda_X \alpha_X$ and we can proceed analogously to the procedure for avoiding proxy discrimination.

We want to end on a final remark about the intuition underlying the expressibility assumption.
The key point in removing proxy discrimination is to ``separate out'' the effect of $P$ on $\hY_{\theta}$ via features $X$ such that it can be ``neutralized'' by using $P$ as a separate input.
For example, in a fully linear system, $P$ can only enter as a single additive linear term via features $X$.
The influence of $P$ can thus easily be eliminated or neutralized by subtracting precisely that term, which we can do via direct access to $P$.
A non-linear example of functional forms with a similarly meaningful notion of ``neutralization'' is when $P$ enters as an overall multiplicative term that only depends on $P$ and is non-zero.
In this case, we can divide the entire expression by that term.
The expressibility assumption captures all the scenarios in which the influence of $P$ can be meaningfully eliminated.
We note that there are also non-linear scenarios in which this is not possible.
For example, if $P$ enters $\hY_{\theta}$ through $X$ in a functional form such as $\sin(e^{X_1 \cdot P})$ (where $X_1$ is one of the observed non-sensitive features), it is not clear how one should ``neutralize'' the effect of $P$ by providing it as a separate argument to $\hY_{\theta}$.

\paragraph{Relating proxy discrimination to other notions of fairness.}\label{sec:neurips:variants}
Motivated by the algorithm to avoid proxy discrimination, we discuss some natural variants of the notion that connect our interventional approach to individual fairness and other proposed criteria.
Given the multitude of proposed fairness criteria outlined in \chapref{chap:existing} together with their incompatibilities and inconsistent naming conventions, it is useful to compare criteria and find specific assumptions under which some of them may become mathematically equivalent.
In this part, we explore some such connections for a generic graph structure as shown on the left in Figure~\ref{fig:generalgraph}.
The proxy~$P$ and the features~$X$ could be multidimensional.
The empty circle in the middle represents any number of variables forming a DAG that respects the drawn arrows.
This general DAG may also be empty.
That means that, e.g., arrows between $X$ and $P$ directly are also allowed.
Figure~\ref{fig:proxytemplate} is an example thereof.
All dashed arrows are optional depending on the specifics of the situation.

\begin{figure}
  \centering
  \begin{tikzpicture}
  \pgfmathsetmacro{\d}{1.5}
  \pgfmathsetmacro{\dy}{1.2}
    \node[simple, minimum height=1cm] (C) at (0,0) {DAG};
    \node[simple] (Z) at (-1.5*\d, \dy) {$Z$};
    \node[simple] (P) at (-0.5*\d, \dy) {$P$};
    \node[simple] (hY) at ( 0.5*\d, \dy) {$\hY$};
    \node[simple] (X) at ( 1.5*\d, \dy) {$X$};
    \edge {Z} {P};
    \edge[dashed] {Z} {C};
    \edge[<->, dashed] {P} {C};
    \edge[<->, dashed] {X} {C};
    \edge[dashed] {P} {hY};
    \edge {X} {hY};
    \node[draw=none] (L) at (-\d,0) {$\tcG$};
  \end{tikzpicture}%
  \hspace{2cm}
  \begin{tikzpicture}
  \pgfmathsetmacro{\d}{1.5}
  \pgfmathsetmacro{\dy}{1.2}
    \node[simple, minimum height=1cm] (C) at (0,0) {DAG};
    \node[simple] (Z) at (-1.5*\d, \dy) {$Z$};
    \node[simple] (P) at (-0.5*\d, \dy) {$P$};
    \node[simple] (hY) at ( 0.5*\d, \dy) {$\hY$};
    \node[simple] (X) at ( 1.5*\d, \dy) {$X$};
    \edge[dashed] {Z} {C};
    \edge[dashed] {P} {C};
    \edge[<->, dashed] {X} {C};
    \edge[dashed] {P} {hY};
    \edge {X} {hY};
    \node[draw=none] (L) at (-\d,0) {$\cG$};
  \end{tikzpicture}%
  \caption[A generic graph to reason about proxy discrimination and its intervened version]{{\em Left:} A generic graph~$\tcG$ to describe proxy discrimination.
  {\em Right:} The graph corresponding to an intervention on~$P$.
  The circle labeled ``DAG'' represents any sub-DAG of~$\tcG$ and~$\cG$ containing an arbitrary number of variables that is compatible with the shown arrows.
  Dashed arrows can, but do not have to be present in a given scenario.}
\label{fig:generalgraph}
\end{figure}
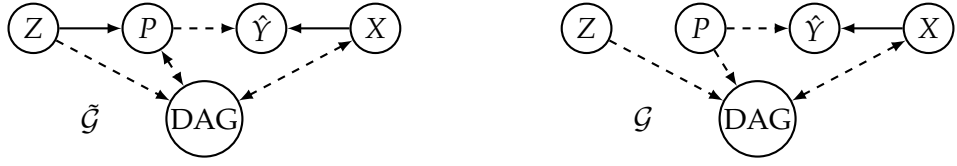

\begin{definition}\label{def:nondisc}
  A predictor~$\hY$ exhibits no \emph{individual proxy discrimination}, if for all~$x$ and all~$p, p'$
  \begin{equation*}
    \distr(\hY \given do(P=p), X=x) = \distr(\hY \given do(P=p'), X=x) \eqp
  \end{equation*}
  A predictor~$\hY$ exhibits no \emph{proxy discrimination in expectation}, if for all~$p, p'$
  \begin{equation*}
    \E[\hY \given do(P=p)] = \E[\hY \given do(P=p')] \eqp \label{eq:inexp}
  \end{equation*}
\end{definition}

Individual proxy discrimination aims at comparing examples with the same features~$X$, for different values of~$P$. Note that this can be individuals with different values for the unobserved non-feature variables.
A true individual-level comparison of the form ``What would have happened to me, if I had always belonged to another group'' is captured by counterfactual fairness \citep{Kusner2017, Nabi2017}.

For an analysis of proxy discrimination, we need the structural equations for~$P, X, \hY$ in Figure~\ref{fig:generalgraph}
\begin{subequations}
  \begin{align*}
    P &= \hat{f}_P (\pa(P)) \eqc \\
    X &= \hat{f}_X (\pa(X)) = f_X (P, ta^{\cG} (X) \setminus \{P\}) \eqc \\
    \hY &= \hat{f}_{\hY} (P, X) = f_{\hY} (P, ta^{\cG} (\hY) \setminus \{P\}) \eqp
  \end{align*}
\end{subequations}
For convenience, we will use the notation~$ta^{\cG}_{P}(X) := ta^{\cG} (X) \setminus \{P\}$.
We can find~$f_X, f_{\hY}$ from~$\hat{f}_X, \hat{f}_{\hY}$ by first rewriting the functions in terms of root nodes of the \emph{intervened graph}, shown on the right side of Figure~\ref{fig:generalgraph}, and then assigning the \emph{overall} dependence on $P$ to the first argument.

We now compare proxy discrimination to other existing notions.

\begin{theorem}\label{thm:rvtodistr}
  Let the influence of~$P$ on~$X$ be additive and linear, i.e.,
  \begin{equation*}
    X = f_X (P, ta^{\cG}_P(X)) = g_X(ta^{\cG}_P(X)) + \mu_X P
  \end{equation*}
  for some function~$g_X$ and~$\mu_X \in \bR$. Then any predictor of the form
  \begin{equation*}
    \hY = r(X - \E [X \given do(P)])
  \end{equation*}
  for some function~$r$ exhibits no proxy discrimination.
\end{theorem}

\begin{proof}
  It suffices to show that the argument of~$r$ is constant with respect to~$P$, because then~$\hY$ and thus~$\distr(\hY)$ are invariant under changes of~$P$.
  We compute
  \begin{align*}
    \E[X \given do(P)] &= \E[ g_X(ta^{\cG}_P (X)) + \mu_X P \given do(P)] \\
    &= \usub{=0}{\E[ g_X(ta^{\cG}_P (X)) \given do(P)]} + \E[\mu_X P \given do(P)]\\
    &= \mu_X P \eqp
  \end{align*}
  Hence,
  \begin{equation*}
    X - \E[X \given do(P)] = g_X(ta^{\cG}_P (X))
  \end{equation*}
  is constant with respect to~$P$.\qedhere
\end{proof}

Note that in general~$\E[X \given do(P)] \neq \E[X \given P]$.
Since in practice we only have observational data from~$\tcG$, one cannot simply build a predictor based on the ``regressed out features''~$\tilde{X} := X - \E[X \given P]$ to avoid proxy discrimination.
In the scenario of Figure~\ref{fig:proxytemplate}, the direct effect of~$P$ on~$X$ along the arrow~$P \to X$ in the left graph cannot be estimated by~$\E[X \given P]$, because of the common confounder~$Z$.
The desired interventional expectation~$\E[X \given do(P)]$ coincides with~$\E[X \given P]$ only if one of the arrows~$Z \to P$ or~$Z \to X$ is not present.
Estimating direct causal effects is a hard problem, well studied by the causality community and often involves instrumental variables \citep{Angrist2001}.

In \secref{sec:fairrepresentations} we provided a short overview of learning fair representation.
A common theme in this field is to learn a representation $\tilde{X}$ that is statistically independent of the sensitive data, but still contains information to predict $Y$.
A natural idea is to use $X - \E[X \given Z]$ as input to a downstream classifier instead of $X$ to ``regress out'' linear correlations.
Within our setting, where we try to avoid proxy discrimination, we typically want to remove the effect of proxies and not the protected attribute directly.
Therefore, we can ask the question when using the representation $\tilde{X} := X - \E[X \given P]$ suffices to remove proxy discrimination, i.e., when the difference between $\E[X \given P]$ and $\E[X \given do(P)]$ becomes irrelevant for our purposes.

\begin{corollary}\label{cor:rvtodistr}
  Under the assumptions of Theorem~\ref{thm:rvtodistr}, if all directed paths from any ancestor of~$P$ to~$X$ in the graph~$\cG$ are blocked by~$P$, then any predictor based on the \emph{adjusted features}~$\tilde{X} := X - \E[X \given P]$ exhibits no proxy discrimination and can be learned from the observational distribution~$\distr(P, X, Y)$ when target labels~$Y$ are available.
\end{corollary}

\begin{proof}
  Let~$A$ denote the set of ancestors of~$P$. Under the given assumptions~$A \cap ta^{\cG} (X) = \emptyset$, because in~$\cG$ all arrows into~$P$ are removed, which breaks all directed paths from any variable in~$A$ to~$X$ by assumption.
  Hence the distribution of~$X$ under an intervention on~$P$ in~$\tcG$, where the influence of potential ancestors of~$P$ on~$X$ that does not go through~$P$ would not be affected, is the same as simply conditioning on~$P$.
  Therefore~$\E[X \given do(P)] = \E[X \given P]$, which can be computed from the joint observational distribution, since we observe~$X$ and~$P$ as generated by~$\tcG$.\qedhere
\end{proof}

Our definition of proxy discrimination in expectation in eq.~\eqref{eq:inexp} is motivated by a weaker notion proposed by \citet{Calders2010}.
It asks for the expected outcome to be the same across the different populations~$\E[\hY \given P=p] = \E[\hY \given P=p'].$
Again, when talking about proxies, we must be careful to distinguish conditional and interventional expectations, which is captured by the following proposition and its corollary.

\begin{proposition}\label{pro:inexpectation}
  Any predictor of the form~$\hY = \lambda (X - \E[X \given do(P)]) + c$ for~$\lambda, c \in \bR$ exhibits no proxy discrimination in expectation.
\end{proposition}

\begin{proof}
  We directly test the definition of proxy discrimination in expectation using linearity of expectation
  \begin{align*}
    \E[\hY \given do(P=p)] &= \E[\lambda (X - \E[X \given do(P)])
      + c \given do(P=p)] \\
    &= \lambda (\E[ X \given do(P=p)] - \E[X \given do(P=p)]) + c \\
    &= c \eqp
  \end{align*}
  This holds for any~$p$, hence proxy discrimination in expectation is achieved.\qedhere
\end{proof}

From this and the proof of Corollary~\ref{cor:rvtodistr} we conclude the following Corollary.

\begin{corollary}
  If all directed paths from any ancestor of~$P$ to~$X$ are blocked by~$P$, any predictor of the form~$\hY = r(X - \E[X \given P])$ for linear~$r$ exhibits no proxy discrimination in expectation and can be learned from the observational distribution~$\distr(P, X, Y)$ when target labels~$Y$ are available.
\end{corollary}

Finally, we provide an additional statement that is a first step towards the ``opposite direction'' of Theorem~\ref{thm:rvtodistr}, i.e., whether we can infer information about the structural equations, when we are given a predictor of a special form that does not exhibit proxy discrimination.

\begin{theorem}
  Let the influence of~$P$ on~$X$ be additive and linear and let the influence of~$P$ on the argument of~$\hY$ be additive linear, i.e.,
  \begin{align*}
    f_X (ta^{\cG}(X)) &= g_X(ta^{\cG}_P(X)) + \mu_X P \\
    f_{\hY} (P, ta^{\cG}_P(X)) &= h(g_{\hY}(ta^{\cG}_P(X)) + \mu_{\hY} P)
  \end{align*}
  for some functions~$g_X, g_{\hY}$, real numbers~$\mu_X, \mu_{\hY} \in \bR$ and a smooth, strictly monotonic function~$h$.
  Then any predictor that avoids proxy discrimination is of the form
  \begin{equation*}
    {\hY} = r(X - \E[X \given do(P)])
  \end{equation*}
  for some function~$r$.
\end{theorem}

\begin{proof}
  From the linearity assumptions we conclude that
  \begin{equation*}
    \hat{f}_{\hY}(P, X) = h(g_X(ta^{\cG}_P(X)) + \mu_X P + \hat{\mu}_{\hY} P) \eqc
  \end{equation*}
  with~$\hat{\mu}_{\hY} = \mu_{\hY} - \mu_X$ and thus~$g_X = g_{\hY}$.
  That means that both the dependence of~$X$ on~$P$ along the path~$P \to \dots \to X$ as well as the direct dependence of~${\hY}$ on~$P$ along~$P \to {\hY}$ are additive and linear.

  To avoid proxy discrimination, we need
  \begin{subequations}\label{eq:distreq}
  \begin{align}
    \distr({\hY} \given do(P=p))
    &= \distr(h(g_{\hY}(ta^{\cG}_P(X)) + \mu_{\hY} p)) \\
    &\overset{!}{=} \distr(h(g_{\hY}(ta^{\cG}_P(X)) + \mu_{\hY} p'))
    = \distr({\hY} \given do(P=p')) \eqp
  \end{align}
  \end{subequations}

  Because~$h$ is smooth an strictly monotonic, we can conclude that already the distributions of the argument of~$h$ must be equal, otherwise the transformation of random variables could not result in equal distributions, i.e.,
  \begin{equation*}
    \distr(g_{\hY}(ta^{\cG}_P(X)) + \mu_{\hY} p)
    \overset{!}{=} \distr(g_{\hY}(ta^{\cG}_P(X)) + \mu_{\hY} p') \eqp
  \end{equation*}
  Since, up to an additive constant, we are comparing the distributions of the \emph{same} random variable~$g_{\hY}(ta^{\cG}_P(X))$ and not merely identically distributed ones, the following condition is not only sufficient, but also necessary for eq.~\eqref{eq:distreq}
  \begin{equation*}
    g_{\hY}(ta^{\cG}_P(X)) + \mu_{\hY} p
    \overset{!}{=} g_{\hY}(ta^{\cG}_P(X)) + \mu_{\hY} p' \eqp
  \end{equation*}
  This holds true for all~$p, p'$ only if~$\mu_{\hY} = 0$, which is equivalent to~$\hat{\mu}_{\hY} = - \mu_X$.

  Because as in the proof of Theorem~\ref{thm:rvtodistr}
  \begin{equation*}
    \E[X \given do(P)] = \mu_X P \eqc
  \end{equation*}
  under the given assumptions any predictor that avoids proxy discrimination is simply
  \begin{equation*}
    {\hY} = X + \hat{\mu}_{\hY} P = X - \E[X \given do(P)] \eqp
  \end{equation*}\qedhere
\end{proof}

\section{Conclusion}
\label{sec:neurips:conclusion}

The perspective developed in this chapter so far naturally addresses shortcomings of earlier statistical approaches.
Causal fairness criteria are suitable whenever we are willing to make assumptions about the (causal) generating process governing the data.
Whilst not always feasible, the causal approach naturally creates an incentive to scrutinize the data more closely and work out plausible assumptions to be discussed alongside any conclusions regarding fairness.
Key concepts of our conceptual framework are \emph{resolving variables} and \emph{proxy variables} that play a dual role in defining causal discrimination criteria.
We develop a practical procedure to remove proxy discrimination given the structural equation model and analyze a similar approach for unresolved discrimination.
In the case of proxy discrimination for linear structural equations, the procedure has an intuitive form that is similar to heuristics already used in the regression literature.
Our framework is limited by the assumption that we can construct a valid causal graph.
The removal of proxy discrimination moreover depends on the functional form of the causal dependencies.
This dependence is captured by the notion of expressibility.
The causal perspective suggests a number of interesting new directions at the technical, empirical, and conceptual level.

\chapter{Sensitivity of causal fairness}
\label{chap:sensitivity}

\graphicspath{{figs/chap5/}}

In this chapter, we go one step beyond our findings in \chapref{chap:causal} and also scrutinize one of the primary assumptions in causal modeling, namely that the causal graph is known.
Potential misspecifications of the causal model introduce new opportunities for bias.
One common way for misspecification to occur is via \emph{unmeasured confounding}: the true causal effect between variables is partially described by unobserved quantities.
We develop tools to assess the sensitivity of fairness measures to this confounding for the popular class of non-linear additive noise models (ANMs).
Specifically, we give a procedure for computing the maximum difference between two counterfactually fair predictors, where one has become biased due to confounding.
For the case of bivariate confounding our technique can be swiftly computed via a sequence of closed-form updates.
For multivariate confounding we give an algorithm that can be efficiently solved via automatic differentiation.
We demonstrate our new sensitivity analysis tools in real-world fairness scenarios to assess the bias arising from confounding.

The main content of this chapter has been published in the following paper:

\fbox{\parbox{\textwidth}{
\pubitem{The sensitivity of counterfactual fairness to unmeasured confounding}
{Niki Kilbertus, Philip~J.~Ball, Matt Kusner, Adrian Weller, Ricardo Silva}
{Uncertainty in Artificial Intelligence (UAI), 2019}
{https://arxiv.org/abs/1907.01040}
{https://github.com/nikikilbertus/cf-fairness-sensitivity}
{}
}}

\section{Introduction}
\label{sec:uai:intro}

In this section, we will focus on \emph{counterfactual fairness} (CF) as introduced by \citet{Kusner2017}, an individual-specific criterion aimed at answering the counterfactual question: ``What would have been my prediction if---all else held causally equal---I was a member of another protected group?''.
Despite the utility of such causal criteria, they are often contested, because they are based on strong assumptions that are hard to verify in practice.
First and foremost, all causal fairness criteria proposed in the literature assume that the causal structure of the problem is known.
Typically, one relies on domain experts and methods for causal discovery from data to construct a plausible causal graph.
While it is often possible with few variables to get the causal graph approximately right, one often needs untestable assumptions to construct the full graph.
The most common untestable assumption is that there is no unmeasured confounding between some variables in the causal graph.
Because we cannot measure it, this confounding can introduce bias that is unaccounted for by causal fairness criteria.

As a solution to this problem, we will introduce tools to measure the sensitivity of the popular \emph{counterfactual fairness} criterion to unmeasured confounding.
Our tools are designed for the commonly used class of non-linear additive noise models \citep[ANMs,][]{hoyer2009nonlinear}.
Specifically, they describe how counterfactual fairness changes under a given amount of confounding.
The core ideas here described can be adapted for sensitivity analysis of other measures of causal effect, such as the average treatment effect (ATE), itself a topic not commonly approached in the context of graphical causal models.
Note that counterfactual fairness poses extra challenges compared to the ATE, as it requires the computation of counterfactuals in the sense of \citet{Pearl2009}.
Concretely, in the remainder of this chapter we will develop the following tools:
\begin{itemize}
  \item For confounding between two variables, we design a fast procedure for estimating the worst-case change in counterfactual fairness due to confounding.
  It consists of a series of closed-form updates assuming linear models with non-linear basis functions.
  This family of models is particularly useful in graphical causal models where any given node has only few parents.
  \item For more than two variables, we fashion an efficient procedure that leverages automatic differentiation to estimate worst-case counterfactual fairness.
  In particular, compared to standard sensitivity analysis \citep[typically applied to ATE problems, see e.g.][]{dorie2016flexible}, we formulate the problem in a multivariate setting as opposed to the typical bivariate case.
  The presence of other modeling constraints brings new challenges not found in the standard literature.
  \item We demonstrate that our method allows us to understand how fairness guarantees degrade based on different confounding levels.
  We also show that even under high levels of confounding, learning counterfactually fair predictors has lower fairness degradation than standard predictors using all features or using all features save for the protected attributes.
\end{itemize}

\section{Background}
\label{sec:uai:background}

We focus on a subclass of SEMs called \emph{additive noise models} (ANMs) \citep{hoyer2009nonlinear}.
This means that the structural equation for a node $X$ of the causal graph is given by $X = f_X(\pa_{\cG}(X)) + \epsilon$ for a non-linear function $f_X$.
Here, we use $\epsilon$ for independent noise variables instead of $N$ for notational convenience.
To make model fitting efficient, we will consider (a) functions $f_X$ that derive all their non-linearity from an embedding function $\B{\phi}$ of their direct parents, and are linear in this embedding;
and (b) Gaussian noise (error) $\epsilon$ so that:
\begin{align*}
X = \B{\phi}(\pa_{\cG}(X))^{\top} \B{w}_X + \epsilon \qtxtq{with} \epsilon \sim \mathcal{N}(0, \sigma_X) \eqc
\end{align*}
where $\B{w}_X$ are weights.
Later on, we will consider ANMs over observed variables, where the noises may be correlated.
Note that this class of ANMs is not closed under marginalization.
For a more detailed analysis of the testable implications of the ANM assumption, see \citep{peters2017elements}.
Neither of our choices (a) and (b) are a fundamental limitation of our framework: the framework can easily be extended to general non-linear, or even non-parametric functions $f_X$, as well as non-Gaussian noises.
We make this choice to balance flexibility and computational cost.

For convenience, we restate the definition of counterfactual fairness, where we allow for continuous target values $y \in \mathcal{Y} \subset \bR$:
\begin{equation}\label{eq:cf}
    \distr(\hY_{Z = z'} = y \given X= \B{x}, Z=z) = \distr(\hY_{Z = z} = y \given X=\B{x}, Z=z) \eqc
\end{equation}
where $\hY_{Z = z'}$ is the counterfactual prediction, imagining $Z=z'$ (note that, because in reality $Z=z$, we have that $\hY_{Z = z} = \hY$), and $\B{x}$ is a realization of other variables in the causal system.
In ANMs $\hY_{Z = z'}$ can be computed in four steps:
\begin{enumerate}
  \item Fit the parameters of the assumed causal model using the observed data: $\mathcal{D}=\{\B{x}_i, z_i\}_{i=1}^n$;
  \item Using the fitted model and data $\mathcal{D}$, estimate all noise variables $\B{\epsilon}$;
  \item Replace $Z$ with counterfactual value $z'$ in all causal model equations;
  \item Using the fitted parameters, estimated noise variables, and $z'$, recompute all variables affected (directly or indirectly) by $Z$, and recompute the prediction $\hY$.
\end{enumerate}
To learn a CF predictor satisfying eq.~\eqref{eq:cf} it is sufficient to use any variables that are non-descendants of $Z$, such as the noise variables $\B{\epsilon}$ \citep{Kusner2017}.
It may appear as if any predictor using only the noise variables as input is going to have low accuracy.
however, we note that in our setting the so called noise variables are not to be thought of as mere measurement noise that contains little information, but summarizes all other influences on their corresponding variables that have not been explicitly observed.
In this sense, the noise variables may be highly predictive of their respective variables and other quantities in the causal graph.
Thus it is not unreasonable to build counterfactually fair predictors using only the noise variables as inputs, which are non-descendants of the sensitive attribute by assumption.

One key assumption on which CF relies is that there is no \emph{unmeasured confounding} relationship missing in the causal model.
For our purposes, we formalize unmeasured confounding as non-zero correlations between any two noise variables in $\B{\epsilon}$ which are assumed to follow a multivariate Gaussian distribution.
Without accounting for this, the above counterfactual procedure will compute noise variables that are not guaranteed to be independent of $Z$.
Thus any predictor trained on these exogenous variables is not guaranteed to satisfy counterfactual fairness eq.~\eqref{eq:cf}.
This setup captures the idea that often we have a decent understanding of the causal structure, but might overlook confounding effects, here in the form of pairwise correlations of noise variables.
At the same time, such confounding is often unidentifiable (save for specific parameterizations).
Thus assessing confounding is not a model selection problem but a sensitivity analysis problem.
To perform such analysis we propose tools to measure the worst-case deviation in CF due to unmeasured confounding.
Before describing these tools, we first place them in the context of the long tradition of sensitivity analysis in causal modeling.

\paragraph{Traditional sensitivity analysis}
Sensitivity analysis for quantities such as the average treatment effect can be traced back at least to the work by Jerome Cornfield on the General Surgeon study concerning the smoking and lung cancer link \citep{rosenbaum:02}.
Rosenbaum cast the problem in a more explicit statistical framework, addressing the question on how the ATE would vary if some degree of association between a treatment and a outcome was due to unmeasured confounding.
The logic of sensitivity analysis can be described in a simplified way as follows:
i) choose a level of ``strength'' for the contribution of a latent variable to the structural equation(s) of the treatment and/or outcome;
ii) by fixing this confounder contribution, estimate the corresponding ATE;
iii) vary steps i) and ii) through a range of ``confounding effects'' to report the level of unmeasured confounding required to make the estimate ATE be statistically indistinguishable from zero;
iv) consult an expert to declare whether the level of confounding required for that to happen is too strong to be plausible, and if so, conclude that the effect is real to the best of one'{}s knowledge.
This basic idea has led to a large literature, see \citep{dorie2016flexible,Robins2000} for two noteworthy examples.

Note the crucial difference between sensitivity analysis and just fitting a latent variable model: we are not learning a latent variable distribution, as the confounding effect for a single cause-effect pair is \emph{unidentifiable}.
By holding the contribution of the confounder as constant and known, the remaining parameters become identifiable.
We can vary the sensitivity parameter without assuming a probability measure on the confounding effect.
The hypothesis test mentioned in the example above can be substituted by other criteria of practical significance.

Much of the work in the statistics literature on sensitivity analysis addresses pairs of cause-effects as opposed to a causal system with intermediate outcomes, and focuses on the binary question on when an effect is non-zero.
The \emph{grid search} idea of attempting different levels of the confounding level does not necessarily translate well to a full SEM: grid search grows exponentially with the number of pairs of variables.
In our problem formulation described in the sequel, we are interested in bounding the maximum magnitude $p_{\max}$ of the noise correlation matrix entries, while maximizing a measure of counterfactual unfairness to understand how it varies by the presence of unmeasured confounding.
The solution is not always to set all entries to $p_{\max}$, since among other things we may be interested in keeping a subset of noise correlations to be zero.
In this case, a sparse correlation matrix with all off-diagonal values set to either 0 or $p_{\max}$ is not necessarily positive-definite.
A multidimensional search for the entries of the confounding correlation matrix is then necessary, which we will do in \secref{sec:uai:multivariate} by encoding everything as a fully differentiable and unconstrained optimization problem.

Finally, we note that unmeasured confounding is but one of many possible misspecifications of causal assumptions. the true causal effect between variables is partially described by unobserved quantities.
For example, we still assume that all edges between observed variables are correctly identified.
\citet{russell2017worlds} previously addressed how to enforce counterfactual fairness across a small enumeration of different competing---but identifiable---models.
Instead, we only focus on misspecification in the form of unmeasured confounding.

\section{Tool \#1: grid-based analysis}
\label{sec:uai:bivariate}

The notion of sensitivity analysis in an SEM can be complex, particularly when the estimated quantity involves counterfactuals.
Therefore, we first describe a tool that estimates the effect of confounding on counterfactual fairness, when the confounding is limited to two variables (i.e., \emph{bivariate confounding}).
This procedure is computationally efficient for this setting.
For the general setting of confounding between any number of variables (\emph{multivariate confounding}) we will introduce a separate tool in \secref{sec:uai:multivariate}.
We now describe our fast two-variable tool using a real-world example.

\paragraph{A motivating example.}
To motivate our approach, let us revisit the example about law school success analyzed by \citet{Kusner2017}.
In this task, we want to predict the first year average grade ($Y$) of incoming law school students from their grade-point average ($G$) before entering law school and their law school admission test scores ($L$).
In the original work, the goal was to train a predictor $\hY$ that is counterfactually fair with respect to race.

To evaluate any causal notion of fairness, we need to first specify the causal graph.
Here we assume $G \to L$ with noises $\epsilon_G, \epsilon_L$, where $G$ and $L$ are both influenced by the sensitive attribute $Z$, see \textbf{Model A} in Figure~\ref{fig:bivariate}.
Given this specification, the standard way to train a counterfactually fair classifier is using $\epsilon_G, \epsilon_L$---the non-descendants of $Z$.
To do so, we first learn them from data as the residuals in predicting $G$ and $L$ from their parents.

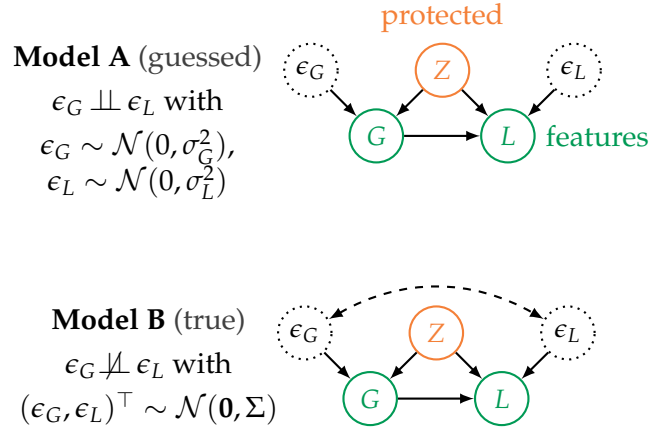
\begin{figure}
  \centering
    \begin{tikzpicture}
    \node[protected, label=above:{\color{Orange}protected}] (Z) {$Z$};
    \node[latent, left=of Z] (EG) {$\epsilon_G$};
    \node[latent, right=of Z] (EL) {$\epsilon_L$};
    \node[feature, below left=0.5 of Z] (G) {$G$};
    \node[feature, right=of G, label=right:{\color{ForestGreen}features}] (L) {$L$};
    \edge {Z} {G,L};
    \edge {G,EL} {L};
    \edge {EG} {G};
    \node[left=of G, align=center] (label) {\textbf{Model A} {\color{black!75}(guessed)}\\[3pt]
    $\epsilon_G \indep \epsilon_L$ with\\[3pt]
        $\epsilon_G \sim \mathcal{N}(0,\sigma_G^2)$,\\ $\epsilon_L \sim \mathcal{N}(0,\sigma_L^2)$\\%
    };
    \end{tikzpicture}\\[0.5cm]
    \begin{tikzpicture}[thick]
    \node[protected] (Z) {$Z$};
    \node[latent, left=of Z] (EG) {$\epsilon_G$};
    \node[latent, right=of Z] (EL) {$\epsilon_L$};
    \node[feature, below left=0.5 of Z] (G) {$G$};
    \node[feature, right=of G, label=right:{\color{white}features}] (L) {$L$};
    \edge {Z} {G,L};
    \edge {G,EL} {L};
    \edge {EG} {G};
    \draw[dashed, <->, >=latex, thick] (EG) to[bend left] (EL);
    \node[left=0.7 of G, align=center] (label) {\textbf{Model B} {\color{black!75}(true)}\\[3pt]
    $\epsilon_G \dep \epsilon_L$ with\\[3pt]
    $(\epsilon_G, \epsilon_L)^{\top} \sim \mathcal{N}(\B{0}, \Sigma)$\\\mbox{ }\\%
    };
    \end{tikzpicture}
    \caption[Causal graphs for the law school example with and without unobserved confounding]{Causal models for the law school example.
    Model A is the guessed model that has no unobserved confounding.
    Model B includes confounding via the covariance matrix $\Sigma$, which is captured by a bidirected edge using the standard acyclic directed mixed graph notation \citep[ADMG,][]{richardson:03}.
    Our techniques will estimate the worst case difference in the estimation of counterfactual fairness due to such confounding (we will consider a more complicated setup in \secref{sec:uai:experiments}).}
    \label{fig:bivariate}
\end{figure}

The validity of causal estimates rely on the assumption that the constructed causal model and its respective graph (here Model A) captures the true data-generating mechanism.
While previous work addressed how to enforce counterfactual fairness across a small enumeration of identifiable competing models \citep{russell2017worlds}, in this work we consider misspecification in the lack of \emph{unidentifiable} unmeasured confounding.
In our example, this means violation of the assumed independence of the noise variables $\epsilon_G$ and $\epsilon_L$.

To capture such confounding, we introduce \textbf{Model B} in Figure~\ref{fig:bivariate}.
Here the noise variables are not independent, they co-vary: $(\epsilon_G, \epsilon_L)^{\top}\sim\mathcal{N}(\B{0}, \Sigma)$ where,
\begin{equation*}
    \Sigma =
    \begin{pmatrix}
        \sigma_G^2 & p\, \sigma_G\, \sigma_L \\
        p\, \sigma_G \, \sigma_L & \sigma_L^2
    \end{pmatrix} \eqp
\end{equation*}
Here, $\sigma_{\bullet}$ is the standard deviation of $\bullet$ and $p \in [-1, 1]$ is the correlation, such that the overall covariance matrix $\Sigma$ is positive semidefinite.
Before going into the detailed procedure of our sensitivity analysis, let us give a general description of what we mean by Model A and Model B throughout this work.

\textbf{Model A} is the ``guessed'' causal graph model used to build a counterfactually fair predictor. The assumption we will scrutinize is that this guess at the true underlying causal model is perfectly correct.

\textbf{Model B} is a version of Model A that allows for further unobserved confounding between pairs of noise variables not originally featured in A.
Model B will play the role of a hypothetical ground truth that simulates ``true'' counterfactual versions of the predictions made within Model A.

Our tool allows us to answer the following question: how does a predictor that is counterfactually fair under Model A perform in terms of counterfactual unfairness under the confounded Model B?
Our goal is to quantify how sensitive counterfactual unfairness is to misspecifications of the causal model, in particular to unobserved confounding.
To do so, we will introduce a measure which we will call \emph{counterfactual unfairness} (CFU).
Given this, we describe how to compute the worst-case violation of counterfactual fairness within a certain confounding budget, which we characterize by the correlation $-1 \leq p_{\max} \leq 1$ in Model B.
By varying the confounding budget, we can assess how robust Model A is to different degrees of model misspecification.
Like in classical sensitivity analysis, we can alternatively start from a level of unacceptable CFU, search for the minimum $p_{\max}$ whose worst-case CFU reaches this level, and leave it to domain experts to judge the plausibility of such a degree of unmeasured confounding $p_{\max}$.

\paragraph{Notation and problem setup.}
For both Model A and B the structural equations are:
\begin{equation}\label{eq:law}
    G = \B{\phi}_G(Z)^{\top} \B{w}_{G} + \epsilon_G \eqc \quad
    L = \B{\phi}_L(Z,G)^{\top} \B{w}_{L} + \epsilon_L \eqc
\end{equation}
where $\B{\phi}_G: \mathcal{Z} \to \bR^{d_G}$ and $\B{\phi}_L: \mathcal{Z} \times \bR \to \bR^{d_L}$ denote \emph{fixed} embedding functions for $Z$ and $Z, G$ respectively, $Z \in \mathcal{Z}$ indicates the membership in a protected group (where $\mathcal{Z}$ is the set of possible groups), and $\B{w}_{G} \in \bR^{d_G}$, $\B{w}_{L} \in \bR^{d_L}$ are the weights of the model.

In order to simplify notation, for observed data $\{(z_i, g_i, l_i)\}_{i=1}^n$, we define
\begin{equation}\label{eq:notationbivariate}
    \B{x}_i =
    \begin{pmatrix}
        g_i\\
        l_i
    \end{pmatrix} \in \bR^{2}
    \eqc \quad
    \B{w} =
    \begin{pmatrix}
        \B{w}_{G} \\
        \B{w}_{L}
    \end{pmatrix} \in \bR^{d_G + d_L}
    \eqc \quad
    \Phi_i =
    \begin{pmatrix}
        \B{\phi}_{G_i}^{\top} &\B{0}^{\top} \\
        \B{0}^{\top} & \B{\phi}_{L_i}^{\top}
    \end{pmatrix} \in \bR^{2 \times (d_G + d_L)} \eqc
\end{equation}
where we write $\B{\phi}_{G_i} = \B{\phi}_G(z_i)$ and $\B{\phi}_{L_i} = \B{\phi}_L(z_i, g_i)$ for brevity.
In eq.~\eqref{eq:notationbivariate} as well as the remainder of this chapter, equations and assignments with subscripts $i$ on both sides hold for all $i \in \{1, \ldots, n\}$.\footnote{Note that $Z$ need not be exogenous.
Since we would need to include additional---standard but occluding---steps in the algorithm to handle discrete variables, this assumption is solely to simplify the presentation.}

\paragraph{Model A: fit a counterfactually fair predictor.}\label{subsec:uai:model_a}
First, we build a counterfactually fair predictor with our guessed unconfounded Model A via the following steps.
\begin{enumerate}[label=\textbf{\arabic*.},wide,labelindent=0pt]
    \item Fit Model A via regularized maximum likelihood:
    \begin{equation}\label{eq:modela_optimization}
        \min_{\B{w}, \sigma_G, \sigma_L} \sum_{i=1}^n (\B{x}_i - \Phi_i \B{w})^{\top} \Sigma^{-1} (\B{x}_i - \Phi_i \B{w})
        + \lambda \|\B{w}\|_2^2 + n \log \det (\Sigma) \eqc
    \end{equation}
    where
    \begin{equation*}
        \Sigma = \begin{pmatrix} \sigma_G^2 & 0 \\ 0 & \sigma_L^2 \end{pmatrix} \eqp
    \end{equation*}
    Note that we can alternately solve for $\B{w}$ and $\sigma_G, \sigma_L$ as follows. First fix $\sigma_G=\sigma_L=1$ and compute
    \begin{equation*}
        \tilde{\B{w}}^{\dagger} = \left(\sum_{i=1}^n \Phi_i^{\top} \Phi_i + \lambda\, \B{I} \right)^{-1} \left(\sum_{i=1}^n \Phi_i^{\top} \B{x}_i\right) \eqp
    \end{equation*}
    The optimal standard deviations $\sigma_G, \sigma_L$ are then simply given by the empirical standard deviations of the residuals under $\tilde{\B{w}}^{\dagger}$.
    Thus, the optimum of eq.~\eqref{eq:modela_optimization} is
    \begin{equation*}
        \B{w}^{\dagger} = \left(\sum_{i=1}^n \Phi_i^{\top} \Sigma^{-1} \Phi_i + \lambda\, \B{I} \right)^{-1} \left(\sum_{i=1}^n \Phi_i^{\top} \Sigma^{-1} \B{x}_i\right) \eqc
    \end{equation*}
    where $\Sigma = \diag(\sigma_G^2, \sigma_L^2)$.
    \item Given fitted weights $\B{w}^{\dagger}$, estimate the noises $\epsilon_G$, $\epsilon_L$,
    \begin{equation*}
        \hat{\B{\epsilon}}_i \equiv (\hat{\epsilon}_{g_i}, \hat{\epsilon}_{l_i})^{\top} \equiv \B{x}_i - \Phi_i \B{w}^{\dagger} \eqp
    \end{equation*}
    \item Fit a counterfactually fair predictor $\hy_i \equiv f_{\B{\theta}}(\hat{\B{\epsilon}_i})$ with parameters $\B{\theta}$ to predict outcomes $y_i$ via
    \begin{equation*}
        \B{\theta}^{\dagger} = \argmin_{\B{\theta}} \sum_{i=1}^n \mathcal{L}(f_{\B{\theta}}(\hat{\B{\epsilon}}_i), y_i) \eqc
    \end{equation*}
    for some loss function $\mathcal{L}$.
    Note that because we are only using non-descendants of $Z$, namely the noise variables, as input, the predictor is counterfactually fair by design.
    While virtually any predictive model can be used in the two-variable case, in the general case we require the counterfactually fair predictor to be differentiable, such that it is amenable to gradient-based optimization.
    The definition of counterfactual fairness constrains the optimization for any loss function.
    Here, we use the sufficient condition for counterfactual fairness that the predictor $\hY$ depends only on the noise terms, which are non-descendants of $Z$ \citep{Kusner2017}.
\end{enumerate}

\paragraph{Model B: evaluate counterfactual unfairness.}
Next, we evaluate how the predictor $f_{\B{\theta}^{\dagger}}$ obtained in the previous section breaks down in the presence of unobserved confounding, i.e., in Model B.
To do so, we fit Model B and generate ``true'' counterfactuals $\B{x}'$.
If we were handed these counterfactuals and we wanted to make predictions using $f_{\B{\theta}^{\dagger}}$ we would compute their noise terms $\hat{\B{\epsilon}}'$ using step 2 above.
If Model A was in fact the model that generated the counterfactuals $\B{x}'$ then the predictions on the noise terms for the real data and the counterfactuals would be \emph{identical}: $f_{\B{\theta}^{\dagger}}(\hat{\B{\epsilon}}) = f_{\B{\theta}^{\dagger}}(\hat{\B{\epsilon}}')$.

However, because the counterfactuals were generated by the ``true'' weights $\B{w}^{*}$ of Model B, not the weights $\B{w}^{\dagger}$ of Model A, there will be a difference between the real data and counterfactual predictions $f_{\B{\theta}^{\dagger}}(\hat{\B{\epsilon}}) \neq f_{\B{\theta}^{\dagger}}(\hat{\B{\epsilon}}')$.
It is this discrepancy we will quantify with our measure of counterfactual unfairness (CFU).
Here is how we compute it for a given confounding budget $p_{\max}$.
\begin{enumerate}[label=\textbf{\arabic*.},wide,labelindent=0pt]
    \item Fit Model B via regularized maximum likelihood:
    \begin{equation}\label{eq:modelb_optimization}
        \min_{\B{w}, \sigma_G, \sigma_L}
        \sum_{i=1}^{n}
        (\B{x}_i - \Phi_i \B{w})^{\top} \Sigma^{-1} (\B{x}_i - \Phi_i \B{w})
        + \lambda^{\dagger} \|\B{w} \|_2^2 + n\,\log \det(\Sigma) \eqc
    \end{equation}
    where
    \begin{align*}
        \Sigma &\equiv \begin{pmatrix} \sigma_G & 0 \\ 0 & \sigma_L \end{pmatrix}
        \underbrace{\begin{pmatrix} 1 & p_{\max} \\ p_{\max} & 1 \end{pmatrix}}_{=:P}
        \begin{pmatrix} \sigma_G & 0 \\ 0 & \sigma_L \end{pmatrix} \eqp
    \end{align*}
    As before, we can alternately solve for $\B{w}$ (closed-form) and $\sigma_G, \sigma_L$ (via coordinate descent).\footnote{In fact we optimize $\log(\sigma_G), \log(\sigma_L)$ to ensure the standard deviations are positive.} Let $\B{w}^*$ be the final weights after optimization.
    \item Given weights $\B{w}^*$, estimate the noises of Model B,
    \begin{equation*}
        \hat{\B{\delta}}_i = (\hat{\delta}_{g_i}, \hat{\delta}_{l_i})^{\top}
        = \B{x}_i - \Phi_i \B{w}^* \eqp
    \end{equation*}
    \item For a fixed counterfactual value $z' \in \mathcal{Z}$, compute the Model B counterfactuals of $G$ and $L$ for all $i \in \{1, \ldots, n\}$,
    \begin{equation*}
        g'_i = \B{\phi}_{G}(z'_i)^{\top} \B{w}^*_{G} + \hat{\delta}_{g_i}\eqc \quad
        l'_i = \B{\phi}_{L}(z'_i, g'_i)^{\top} \B{w}^*_{L} + \hat{\delta}_{l_i} \eqc
    \end{equation*}
    where $\B{w}^* = (\B{w}^*_{G}, \B{w}^*_{L})^{\top}$.
    If $\B{x}'_i \equiv (g_i', l_i')^{\top}$, we can write the above equation as
    \begin{equation*}
       \B{x}'_i = \Phi'_i \B{w}^* + \hat{\B{\delta}}_i \eqc
    \end{equation*}
    and $\Phi'_i \equiv \mbox{diag}(\B{\phi}_{G}(z'_i), \B{\phi}_{L}(z'_i, g'_i))$ is defined in general by sequential propagation of counterfactual values according to the ancestral ordering of the SEM.
    \item Compute the ``incorrect'' noise terms of the counterfactuals using the same procedure as in step 2 of \secref{subsec:uai:model_a}, using weights $\B{w}^{\dagger}$ of Model A:
    \begin{equation*}
        \hat{\B{\epsilon}}'_i = (\hat{\epsilon}_{g'_i}, \hat{\epsilon}_{l'_i})^{\top} = \B{x}'_i - \Phi'_i \B{w}^{\dagger} \eqp
    \end{equation*}
    Again, the predictions on the above quantity $f_{\theta^{\dagger}}(\hat{\B{\epsilon}}'_i)$ will differ from those made on the real-data noise terms $f_{\theta^{\dagger}}(\hat{\B{\epsilon}}_i)$ (unless the counterfactuals were also generated according to Model A).
    \item To measure the discrepancy, we propose to quantify counterfactual unfairness as the squared difference between the above two quantities:
    \begin{equation*}
        \CFU_i = (
        f_{\B{\theta}^{\dagger}}(\hat{\B{\epsilon}}_i) -
        f_{\B{\theta}^{\dagger}}(\hat{\B{\epsilon}}'_i)
        )^2 \eqp
    \end{equation*}
    Ultimately, to summarize the aggregate unfairness, we will compute the average counterfactual unfairness:
    \begin{align}
        \CFU = \frac{1}{n} \sum_{i=1}^n \CFU_i \label{eq:cfu} \eqp
    \end{align}
\end{enumerate}

A quick note: in the two-variable setting, given a confounding budget $p_{\max}$, the worst-case CFU occurs precisely at $p_{\max}$, which need not be the case for multivariate confounding as we show at the end of the current section.

Thus, the above procedure computes the maximum CFU with bivariate confounding budget equal to $p_{\max}$.
CFU measures how the counterfactual responses $\hY(z)$ and $\hY(z')$, defined using Model A, would differ ``in reality'', i.e., if Model B were ``true''.
What qualifies as bad CFU is problem dependent and requires interaction with domain experts, who can make judgment calls about the plausibility of the misspecification $p_{\max}$ that is required to reach a breaking point.
Here, a breaking point could be the CFU of a predictor that completely ignores the causal graph.

To summarize: we learn $\hY \equiv f_{\B{\theta}^{\dagger}}$ as a function of $X$ and $Z$, where $X$ and $Z$ are implicit in the expression of the (estimated) noise terms $\B{\epsilon}$ that are computed using the assumptions of the working Model A.
We assess how ``unfair'' $\hY$ is by comparing for each data point the two counterfactual values $\hY(z) \equiv f_{\B{\theta}^{\dagger}}(\hat{\B{\epsilon}}_i)$ and $\hY(z') \equiv f_{\B{\theta}^{\dagger}}(\hat{\B{\epsilon}}'_i)$ where the ``true'' counterfactual is generated according to the world assumed by Model B.
The space of models to which Model B belongs is a continuum indexed by $p_{\max}$, which will allow us to visualize the sensitivity of Model A by a one-dimensional curve.
We will do this by finding the best fitting model (in terms of structural equation coefficients and noise variances) at different values of $p_{\max}$, so that the corresponding CFU measure is determined by $p_{\max}$ only (results on the above law school model are shown in \secref{sec:uai:experiments}).
We assume that the free confounding parameter is not identifiable from data (as it would be the case if the model was linear and the edge $Z \rightarrow L$ was missing, the standard instrumental variable scenario).

As we have mentioned previously, in the multivariate setting, the worst-case counterfactual unfairness with a confounding budget of $p_{\max}$ is not necessarily obtained when all non-zero entries of the correlation matrix are set to $p_{\max}$.
Even though intuitively that would lead to the largest allowed correlation between any pair of variables and thus in some sense to the largest deviation from the assumed Model A, we will now show that this setting does not necessarily lead to the biggest deviation in predictions.
In particular, we will show that such a matrix with all non-zero correlation entries set to $p_{\max}$ is not a valid correlation matrix.
To this end, it suffices to find a symmetric matrix $P$ with $1$s on the diagonal that is not positive semidefinite when all its non-zero off-diagonal entries are set to the same value, which we define to be the considered confounding budget $p_{\max}$.
Since each valid correlation matrix must be positive semidefinite, the correlation matrix for the worst-case counterfactual unfairness must be different from $P$ (while maintaining the zero entries).
Because all off-diagonal entries are upper bounded by $p_{\max}$, at least one of them must be smaller than the corresponding value in $P$.

For example, consider
\begin{equation*}
P = \begin{pmatrix}
1 & p_{\max} & p_{\max} \\
p_{\max} & 1 & 0 \\
p_{\max} & 0 & 1
\end{pmatrix} \eqp
\end{equation*}
Since the eigenvalues of $P$ are $1$, $1 - \sqrt{2}\, p_{\max}$, and $1 + \sqrt{2}\, p_{\max}$, we see that $P$ is not positive semidefinite for $p_{\max} > 1/\sqrt{2}$.

In general, for $m \in \bN$ with $m > 2$, the matrix $P \in \bR^{m \times m}$ with $P_{ii} = 1$ for $i \in \{1, \ldots, m \}$, $P_{1i} = P_{i1} = p_{\max}$ for $i \in \{2, \ldots, m\}$ and $P_{ij} = 0$ for all remaining entries, has the eigenvalues (without multiplicity) $1$, $1 - \sqrt{m-1}\, p$, and $1 + \sqrt{m-1}\, p$.
Therefore, $P$ is not positive semidefinite for $p_{\max} > 1 / \sqrt{m-1}$.
We conclude that as the dimensionality of the problem increases, we may encounter such situations for ever smaller confounding budget.
This leads us to the second tool for sensitivity analysis in the multi-variate setting.

\section{Tool \#2: optimization-based analysis}
\label{sec:uai:multivariate}

In this section, we generalize the procedure outlined for the two-variable case in \secref{sec:uai:bivariate} to the general multivariate case.

\paragraph{Notation and problem setup.}
Besides the protected attribute $Z$ and the target variable $Y$, let there be $m$ additional observed feature variables $X_j$ in the causal graph $\cG$ each of which comes with an unobserved noise variable $\epsilon_j$.

As before, we express the assignment of the structural equations for a specific realization of observed features $\B{x} = (x_1,\ldots,x_m)^{\top}$ and noise terms $\B{\epsilon} = (\epsilon_1,\ldots,\epsilon_m)^{\top}$ as the following operation, i.e., $\B{x} = \Phi \B{w} + \B{\epsilon}$.
Here $\Phi$ has $m$ rows and $d = \sum_{V \in \text{has-parents}(\cG)} d_V$ columns, where $d_V$ is the dimensionality of embedding $\phi_{V}:\bR^{|\pa_{\cG}(V)|} \to \bR^{d_V}$ for each node $V \in \cG$ that has parent nodes.
Without loss of generality, we assume the nodes $\{Z\} \cup \{X_j\}_{j=1}^m$ to be topologically sorted with respect to $\cG$ with $Z$ always being first.
We combine the individual weights as
$\B{w} = (\B{w}_{X_1}, \B{w}_{X_2}, \ldots, \B{w}_{X_m}) \in \bR^d$
and represent $\Phi$ once evaluated on a specific sample $z, \B{x}$ of the variables $Z, X_1,\ldots,X_m$ as
\begin{equation*}
    \Phi =
    \begin{pmatrix}
        \B{\phi}_{X_1}^{\top} & & \B{0}^{\top}\\
        & \ddots & \\
        \B{0}^{\top} & & \B{\phi}_{X_m}^{\top}
    \end{pmatrix} \eqc
\end{equation*}
where $\B{\phi}_{X_j}$ is based on the parents of $X_j$.
The covariance matrix of the noise terms is given by
\begin{equation*}
\Sigma \equiv \diag(\sigma_1, \dots, \sigma_m)\, P\, \diag(\sigma_1, \dots, \sigma_m) \eqc
\end{equation*}
\noindent where $\sigma_1, \dots, \sigma_m$ are the standard deviations of each variable and $P$ is a correlation matrix.

\paragraph{The optimization problem.}
In the general case our goal is to find a correlation matrix $P$ that satisfies a ``confounding budget'' $p_{\max}$.
In particular we would like to constrain the correlation $P_{jk}$ between any two different variables $X_j$ and $X_k$ for $j \neq k$, while allowing $P_{jj} = 1$ for all $j$.
Additionally, we want to take into account any prior knowledge that certain variable pairs should have no correlation, if available.
The most intuitive way to budget the amount of confounding is to limit the absolute size of any correlation by $p_{\max}$ as: $|P_{jk}| \leq p_{\max}$ for all $j \ne k$.
This captures a notion of ``restricted unobserved confounding'' and leads
to the following optimization problem
\begin{align}
\begin{split}
    \maximize_P &\quad\sum_{i=1}^n \CFU_i \label{eq:multiopt_orig} \\
    \text{subject to}
    &\quad P_{jj} = 1 \quad\text{for } j \in \{1, \ldots, m\} \eqc \\
    &\quad |P_{jk}| \le p_{\max} < 1 \quad\text{for all } (j, k)
      \in \mathcal{C} \eqc \\
    &\quad P_{jk}= 0 \quad\text{for all } (j, k)
      \not\in \mathcal{C} \text{ with } j \neq k \eqp
\end{split}
\end{align}
$\mathcal{C}$ is the set of correlations that are allowed to be non-zero.
We remark that the setting where there are zero correlations can be captured using the standard acyclic directed mixed graph notation \citep[ADMG,][]{richardson:03}.
Specifically, this can be represented with ADMGs by removing bidirected edges between any two noise terms whose correlation is fixed to zero.

As in the bivariate case, CFU is a direct function of $P$ only: all other parameters will be determined given the choice of correlation matrix by maximizing likelihood.
Note that eq.~\eqref{eq:multiopt_orig} contains multiple nested optimization problems:
The outer optimization over correlation matrices $P$ and the inner optimization for the counterfactually fair model weights $\B{\theta}^{\dagger}$ as well as the weights and standard deviations of Models A and B.
To solve it efficiently, we will parameterize $P$ in a way that facilitates optimization via off-the-shelf, unconstrained, automatic differentiation tools.

\paragraph{Algorithm.}
We use the following approach to accommodate the constraints in eq.~\eqref{eq:multiopt_orig} in a way such that our algorithm does not require a constrained optimization subroutine.
Assume first that $P$ has no correlations that should be zero.
We compute $L L^{\top}$ for a matrix $L \in \bR^{m \times m}$, whose entries are the parameters we eventually optimize.
To constrain the off-diagonals to a given range and ensure that $P$ has $1$s on the diagonal, we define $P$ as,
\begin{equation*}
    P := \tanh_{p_{\max}}(L L^{\top}) := \B{I} + p_{\max} \, (\B{J} - \B{I}) \odot \tanh(L L^{\top}) \eqc
\end{equation*}
where $\odot$ denotes element-wise multiplication of matrices, $\B{J}$ is a matrix of all ones, and $\B{I}$ is the identity matrix.
This way $P$ is symmetric, differentiable with respect to the entries of $L$, has $1$s on the diagonal, and its off-diagonal values are squashed to lie within $(-p_{\max}, p_{\max})$.
While it is natural to directly mask and clamp the diagonal, there are various ways to squash the off-diagonals to a fixed range in a smooth way, which bears close resemblance to barrier methods in optimization.
We choose $\tanh()$ because of its abundance in the machine learning literature and availability in computational frameworks (including gradient implementations), but other forms of $P$ may work better for specific applications.
Note that this formulation does not guarantee $P$ to be positive semidefinite.

\begin{algorithm}[t!]
\caption{\textsc{MaxCFU}: Maximize counterfactual unfairness under a certain confounding budget constraint.}\label{algo:mult_cfu}
\begin{algorithmic}[1]
    \Input
    data $\{\B{x}_i, y_i, z_i, \}_{i=1}^n$,
    confounding budget $p_{\max}$, learning rate $\alpha$, minibatch size $B$
    \Statex
  \State $\{\hat{\B{\epsilon}}_i\}_{i=1}^n, \B{w}^{\dagger}, \B{\theta}^{\dagger} \gets \Call{FitModelA}{\{\B{x}_i, y_i, z_i, \}_{i=1}^n}$
  \State $\cD \gets \{\B{x}_i, y_i, z_i, \Phi_i, \hat{\B{\epsilon}}_i \}_{i=1}^n$\Comment{full dataset}
  \State $L \gets \Call{InitializeParameters}{{}}$
  \For{$t = 1 \ldots T$}\Comment{iterations}
  \State $\cD^{(t)} \gets \Call{SampleMinibatch}{\cD, B}$
  \State $\Delta \gets \nabla_L \Call{CFU}{\cD^{(t)}, \B{w}^{\dagger}, \frac{B}{n}\lambda^{\dagger}, \B{\theta}^{\dagger}, L}$ \Comment{autodiff}
  \State $L \gets L + \alpha \Delta$ \Comment{gradient ascent step}
  \EndFor
  \State \Return $\Call{CFU}{\cD, \B{w}^{\dagger}, \lambda^{\dagger}, \B{\theta}^{\dagger}, L}$
  \Statex
  \Function{FitModelA}{$\{\B{x}_i, y_i, z_i, \}_{i=1}^n$}
      \State $\B{w}^{+} \gets \Big(\sum_{i=1}^n \Phi_i^{\top} \Phi_i + \lambda^{\dagger} \B{I} \Big)^{-1} \Big(\sum_{i=1}^n \Phi_i^{\top} \B{x}_i\Big)$
      \State $\Sigma \gets \diag(\var(\{\B{x}_i - \Phi_i \B{w}^{+}\}_{i=1}^n))$
      \State $\B{w}^{\dagger} \gets \Bigl(\sum_{i=1}^n \Phi_i^{\top} \Sigma^{-1} \Phi_i + \lambda^{\dagger} \B{I} \Bigr)^{-1}
      \Bigl(\sum_{i=1}^n \Phi_i^{\top} \Sigma^{-1} \B{x}_i\Bigr)$
      \State $\hat{\B{\epsilon}}_i \gets \B{x}_i - \Phi_i \B{w}^{\dagger}$
      \State $\B{\theta}^{\dagger} \gets \argmin_{\B{\theta}} \sum_{i=1}^n \mathcal{L}(f_{\B{\theta}}(\hat{\B{\epsilon}}_i), y_i)$
      \State \Return $\{\hat{\B{\epsilon}}_i\}_{i=1}^n, \B{w}^{\dagger}, \B{\theta}^{\dagger}$
  \EndFunction
  \Statex
  \Function{CFU}{{$\cD$, $\B{w}^{\dagger}$, $\lambda^{\dagger}$, $\B{\theta}^{\dagger}$, $L$}}
    \State $\B{w}^*, \B{\sigma}^* \gets
      \min_{\B{w}, \B{\sigma}}
      \sum_{i=1}^{n}
      (\B{x}_i - \Phi_i \B{w})^{\top} \Sigma^{-1}
      (\B{x}_i - \Phi_i \B{w})
      + \lambda^{\dagger} \|\B{w} \|_2^2 + n\,\log \det(\Sigma)$
    \Statex \hspace*{0.5cm}where
      $\Sigma = \diag(\B{\sigma}) \tanh_{p_{\max}}(L L^{\top}) \diag(\B{\sigma})$
    \State $\hat{\B{\delta}}_i \gets \B{x}_i - \Phi_i \B{w}^*$
    \State $z'_i \gets 1 - z_i$ and $\B{x}'_i \gets \Phi'_i \B{w}^* +\hat{\B{\delta}}_i$
    \Statex \hspace*{0.5cm}where $\Phi'_i$ is computed via iterative assignment
    \State $\hat{\B{\epsilon}}'_i \gets \B{x}'_i - \Phi'_i \B{w}^{\dagger}$
    \State $\CFU \gets \frac{1}{n} \sum_{i=1}^n (
          f_{\B{\theta}^{\dagger}}(\hat{\B{\epsilon}}_i) -
          f_{\B{\theta}^{\dagger}}(\hat{\B{\epsilon}}'_i)
          )^2$
    \State \Return $\CFU$
  \EndFunction
\end{algorithmic}
\end{algorithm}

In Algorithm~\ref{algo:mult_cfu}, we describe our procedure to maximize counterfactual unfairness given a confounding budget $p_{\max}$ and observational data $\{\B{x}_i, y_i, z_i, \}_{i=1}^n \subset \bR^m \times \{0,1\}^2$.
The algorithm closely follows the procedure described in \secref{sec:uai:bivariate} for the bivariate case.
Since we use automatic differentiation provided by PyTorch \citep{paszke2017automatic} to obtain gradients, we only show the forward pass in Algorithm~\ref{algo:mult_cfu}.
For the initialization \textsc{InitializeParameters()}, we simply populate $L$ as a lower triangular matrix with small random values for the off-diagonals and $1$s on the diagonal.

The main place where code optimization can take place is step 17 of Algorithm~\ref{algo:mult_cfu}.
Alternatives to the (local) penalized maximum likelihood taking place there could be suggested (perhaps using spectral methods).
It is hard though to say much in general about Step 20, as counterfactual fairness allows for a large variety of loss functions usable in supervised learning.
In the case of linear predictors, it is still a non-convex problem due to the complex structure of the correlation matrix, and for now we leave as an open problem whether non-gradient based optimization may find better local minima.

If $\mathcal{C}$ indicates some correlations should be zero, we suggest the following standard ``clique parameterization'': $L$ is a $m \times c$ matrix where $c$ is the number of cliques in $\mathcal{C}$, with $L_{ik}$ being a non-zero parameter if and only if vertex $i$ is in clique $k$.
$L_{ik} \equiv 0$ otherwise.
It follows that such a matrix will have zeros at precisely the locations not in $\mathcal{C}$.\footnote{Barring unstable parameter cancellations that have measure zero under continuous measures on $\{L_{ik}\}$.}
See \citet{silva_nips:07,barber:09} for examples of applications of this idea.
For large cliques, further refinements are necessary to avoid unnecessary constraints, such as creating more than one row per clique of size four or larger.
Our experiments will not make use of sparse $P$ (note that this parameterization also assumes that the number of cliques is tractable).
Note that individual parameters $L_{ik}$ may not be identifiable.
However, identifiability is not necessary here, all we care about is the objective function: CFU.
As a matter of fact, multiple globally optimal solutions are to be expected even in the space of $P$ transformations.
A more direct parameterization of sparse $P$, with exactly one parameter per non-zero entry of the upper covariance matrix, is discussed by \citet{drton:04}.
Computationally, this minimal parameterization does not easily lead to unconstrained gradient-based algorithms for optimizing sparse correlation matrices with bounded entries.
We suggest the clique parameterization as a pragmatic alternative.
Special cases may be treated with more efficient specialized approaches.
\citet{cinelli19a} provide a thorough discussion of fully linear models.

In \secref{sec:uai:experiments} we will demonstrate this approach on a real-world 3-variable-confounding scenario to showcase our approach.
Before moving to empirical results and the implementation of our tools, we briefly describe a methodological extension of our approach to path-specific effects \citep{shpitser:13}.
Specifically, we describe an example that illustrates how notions of path-specific effects can be easily pipelined with our sensitivity analysis framework.

Consider Figure~\ref{fig:path}, where the path from $Z \rightarrow U$ is considered unfair and $Z \rightarrow F$ is considered fair, in the sense that we do not want a non-zero path-specific effect of $Z$ on $\hY$ that is comprised by a possible path $Z \rightarrow U \rightarrow \hY$ in the causal graph implied by the chosen construction of $\hY$. Then a path-specific counterfactually fair predictor is one that uses $\{\epsilon_U,F\}$ as input. Note that the only difference this makes in our grid-based tool is that we only estimate the noise $\epsilon_U$ for the unfair path in Model A (step 2) and fit a predictor on $\{\epsilon_U,F\}$ (step 3). Additionally, we only compute the incorrect noise terms of the counterfactuals in Model B, using the weights of Model A (step 4). For the optimization-based tool we would change lines 13, 14, and 20 in the same way.

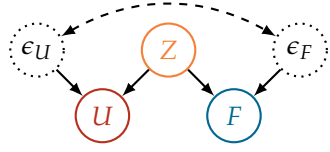
\begin{figure}[t]
    \centering
       \begin{tikzpicture}
    \node[protected] (Z) {$Z$};
    \node[latent, left=of Z] (EG) {$\epsilon_U$};
    \node[latent, right=of Z] (EL) {$\epsilon_F$};
    \node[unfair, below left=0.5 of Z] (G) {$U$};
    \node[target, right=of G, label=right:{\color{white}features}] (L) {$F$};
  \edge {Z} {G,L};
    \edge {EL} {L};
    \edge {EG} {G};
   \draw[dashed, <->, >=latex, thick] (EG) to[bend left] (EL);
    \end{tikzpicture}
    \caption[Example for a path-specific fairness setting]{A path-specific model where the path from protected attribute $Z$ to feature $U$ is unfair and the path from $Z$ to feature $F$ is fair.}
    \label{fig:path}
\end{figure}

\section{Experiments}
\label{sec:uai:experiments}

We compare the grid-based and the optimization-based tools introduced in \secsref{sec:uai:bivariate} and~\ref{sec:uai:multivariate} on two real datasets.

In all experiments our embedding $\B{\phi}$ is a polynomial basis up to a fixed degree.
The degree is determined via cross validation (5-fold) jointly with the regularization parameter $\lambda^{\dagger}$.
Our counterfactually fair predictor is regularized linear regression on the noise terms $\B{\epsilon}$:
\begin{equation*}
    \min_{\B{\theta}} \sum_{i=1}^n (y_i - \B{\phi}(\hat{\B{\epsilon}}_i)^{\top} \B{\theta})^2 + \lambda \|\B{\theta} \|_2^2 \eqp
\end{equation*}
For this model, counterfactual unfairness is:
\begin{equation*}
    \CFU_i = \Bigl(\bigl(\B{\phi}(\hat{\B{\epsilon}}_i) - \B{\phi}(\hat{\B{\epsilon}}'_i)\bigr)^{\top} \B{\theta}^{\dagger} \Bigr)^2 \eqp
\end{equation*}

For comparison, we train two baselines that also use regularized ridge regression (degree and regularization are again selected by 5-fold cross-validation):
\begin{description}[leftmargin=*,topsep=0pt,noitemsep,labelwidth=*]
    \item[unconstrained:] an unconstrained predictor using all observed variables as input $f_{\mathrm{uc}}: (Z, X_1, \ldots, X_m) \mapsto Y$.
    \item[blind unconstrained:] an unconstrained predictor using all features, but not the protected attribute, as input $f_{\mathrm{buc}}:(X_1, \ldots, X_m) \mapsto Y$.
\end{description}
Analogous to our definition of $\CFU$ in eq.~\eqref{eq:cfu}, we compute the unfairness of these baselines as the mean squared difference between their predictions on the observed data and the predictions of the counterfactually fair predictor on the observed data: $\frac{1}{n}\sum_{i=1}^n (f_{\B{\theta}^{\dagger}}(\hat{\B{\epsilon}}_i) - \hy_i^{\mathrm{(b)uc}})^2$, where $\hy_i^{\mathrm{uc}} = f_{\mathrm{uc}}(z_i, \B{x}_i)$ and $\hy_i^{\mathrm{buc}} = f_{\mathrm{buc}}(\B{x}_i)$.
This choice is motivated by the fact that in practice we care about how much potential predictions deviate from predictions satisfying a fairness measure. For our grid-based approach we repeatedly fix $p_{\max} \in (-1, 1)$ to a particular value and then use the procedure in \secref{sec:uai:bivariate} to compute $\CFU$.
For the optimization approach we similarly fix $p_{\max} \in [0, 1)$ in the constraint of eq.~\eqref{eq:multiopt_orig}.
For efficiency we use the previously found correlation matrix $P$ as initialization for the next setting of $p_{\max}$.

\begin{figure}
    \centering
    \hspace{1cm}\textbf{Law School}\\
    \includegraphics[width=0.6\textwidth]{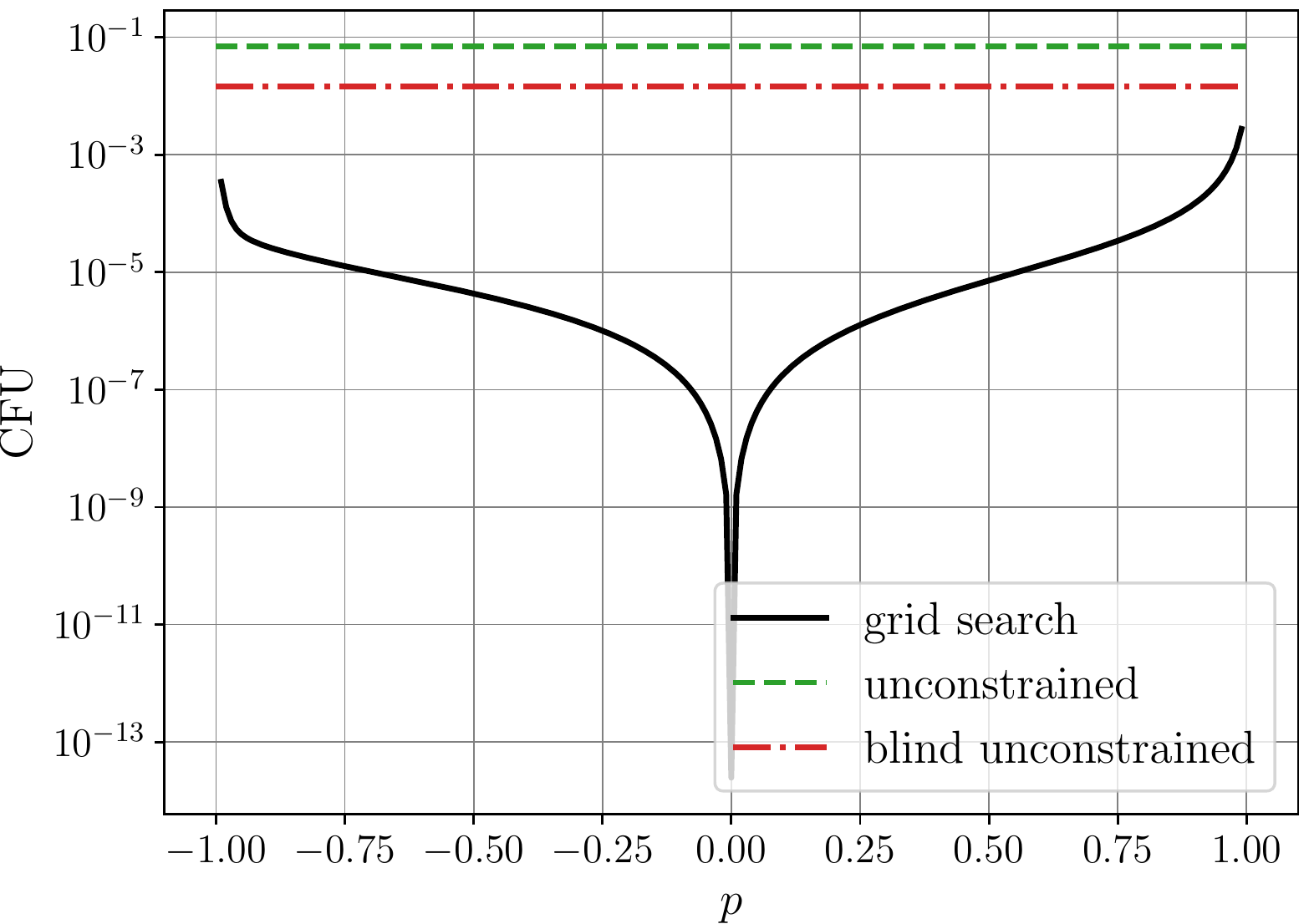}
    \caption[Results for the sensitivity of counterfactual unfairness in the law school example]{Counterfactual unfairness for the law school dataset. See text for details.}
    \label{fig:results_bivariate}
\end{figure}

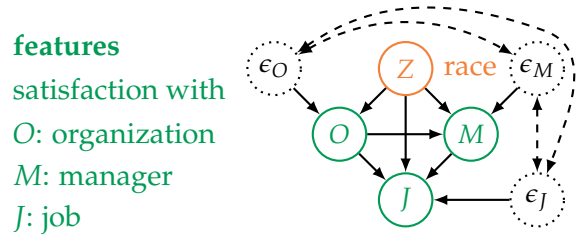
\begin{figure}[t]
    \centering
    \begin{tikzpicture}
    \node[protected, label=right:{\color{Orange}race}] (Z) {$Z$};
    \node[latent, left=of Z] (EO) {$\epsilon_O$};
    \node[latent, right=of Z] (EM) {$\epsilon_M$};
    \node[feature, below left=0.5 of Z] (O) {$O$};
    \node[feature, right=of O] (M) {$M$};
    \node[feature, below right=0.5 of O] (J) {$J$};
    \node[latent, right=of J] (EJ) {$\epsilon_J$};
    \edge {Z} {O,M,J};
    \edge {O,EM} {M};
    \edge {O,M} {J};
    \edge {EO} {O};
    \edge {EJ} {J};
    \draw[dashed, <->, >=latex, thick] (EO) to[bend left] (EM);
    \draw[dashed, <->, >=latex, thick] (EO) to[out=35,in=160] ($(EM) + (0.1,0.5)$) to[out=340,in=70] (EJ);
    \draw[dashed, <->, >=latex, thick] (EM) to (EJ);
    \node[left=1.0 of O, ForestGreen, align=left] (label) {%
    \textbf{features}\\[3pt]
    satisfaction with\\[3pt]
    $O$: organization\\[2pt]
    $M$: manager\\[2pt]
    $J$: job};
    \end{tikzpicture}%
    \caption[The assumed to be true causal graph for the NHS Staff Survey data]{The assumed to be true causal graph for the NHS Staff Survey dataset.}
    \label{fig:nhs}
\end{figure}

\begin{figure}
    \centering
    \hspace{1cm}\textbf{NHS}\\
    \includegraphics[width=0.6\textwidth]{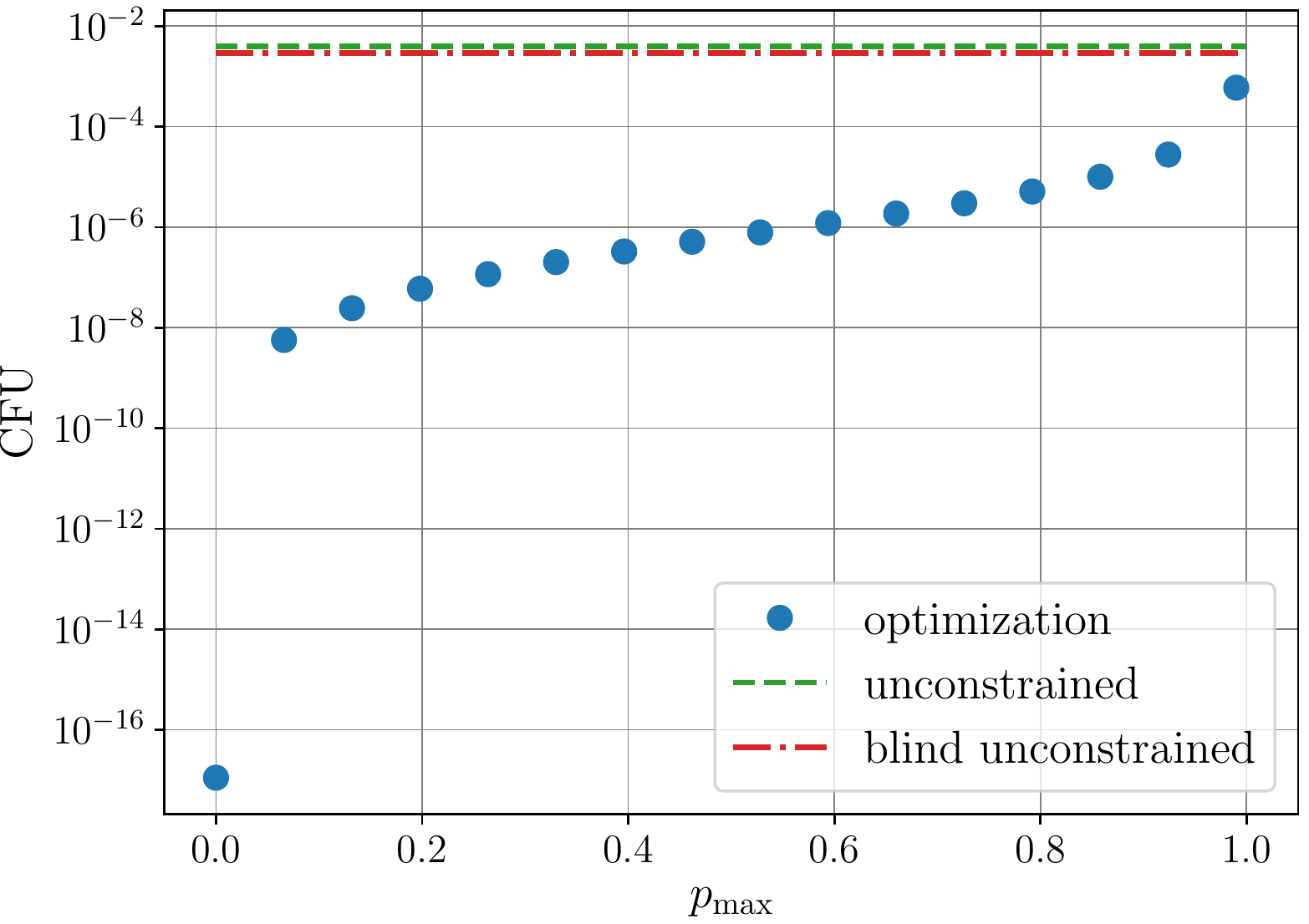}
    \caption[Results for the sensitivity of counterfactual unfairness in the NHS example]{Counterfactual unfairness as a function of $p_{\max}$ for the multivariate NHS dataset.}
    \label{fig:results_multivariate}
\end{figure}

\paragraph{Law School data.}
As our first experiment we consider the motivating example introduced in \secref{sec:uai:bivariate} on law school success (recall eq.~\eqref{eq:law} and Figure~\ref{fig:bivariate} for details on the causal models). Our data comes from law school students in the US \citep{wightman1998lsac}. As our causal model investigates confounding between two variables, we will use the grid-based approach introduced in \secref{sec:uai:bivariate} to calculate the maximum $\CFU$. Recall that for the bivariate approach we fix a confounding level $p = p_{\max}$ and then compare predictions between real data based on Model A versus counterfactuals generated from Model B. Figure~\ref{fig:results_bivariate} shows the $\CFU$ for the grid-based approach (black), alongside the baselines (green/red), as the correlation $p$ varies. We first note that the confounding is not symmetric around $p = 0$. For the law school data, negative correlations have smaller $\CFU$. In general, this is a data-specific property.

Additionally, we notice that as $p_{\max}$ moves away from $0$ (where we expect $\CFU = 0$ up to numerical errors) it increases noticeably, then plateaus in roughly $[0.1, 0.9]$ and finally increases again.
For large $p_{\max}\geq 0.9$ we cannot exclude numeric instability as the covariance matrix becomes nearly negative definite.
Finally, we note that both baseline approaches have higher $\CFU$ than found with any grid-based setting.

\paragraph{NHS Staff Survey.}
Our second experiment is based on the 2014 UK National Health Service (NHS) Survey \citep{nhs2014staffsurvey}. The goal of the survey was to \directquote{gather information that will help to improve the working lives of staff in the NHS}. Answers to the survey questions ranged from `strongly disagree' (1) to `strongly agree' (5). We averaged survey answers for related questions to create a dataset of continuous indices for: \emph{job satisfaction} $(J)$, \emph{manager satisfaction} $(M)$, \emph{organization satisfaction} $(O)$, and \emph{overall health} $(Y)$. The goal is to predict health $Y$ based on the remaining information. Additionally, we collected the race $(Z)$ of the survey respondents. Using this data, we formulate a ground-truth causal graph shown in Figure~\ref{fig:nhs} (equivalent to Model B in Figure~\ref{fig:bivariate}). This causal graph includes correlations between all noise terms $\epsilon_J, \epsilon_M, \epsilon_O$ and we assume the following structural equations
\begin{align}
    O &= \B{\phi}_O(Z)^{\top} \B{w}_{O} + \epsilon_O \label{eq:nhs} \\
    M &= \B{\phi}_M(Z,O)^{\top} \B{w}_{M} + \epsilon_M \nonumber \\
    J &= \B{\phi}_J(Z,O,M)^{\top} \B{w}_{J} + \epsilon_J \eqp \nonumber
\end{align}

Just as in the law school example, we measure the impact of unobserved confounding by comparing this model to the unconfounded model (i.e., all noise terms are jointly independent). As there is no general efficient way to grid-search for positive definite matrices that maximize $\CFU$ for a given $p_{\max}$, we make use of our optimization-based procedure for calculating maximum $\CFU$, as described in Algorithm~\ref{algo:mult_cfu}. Figure~\ref{fig:results_multivariate} shows the results of our method on the NHS dataset. Note that we only show positive $p_{\max}$ because our optimization problem eq.~\eqref{eq:multiopt_orig} only constrains the absolute value of the off-diagonal correlations. This allows the procedure to learn whether positive or negative correlations result in greater $\CFU$. As in the law school dataset we see an initial increase in CFU for small $p_{\max}$, followed by a plateau, ending with another small increase. Just like before, all settings result in lower $\CFU$ than the two baseline techniques.

\section{Conclusion}
\label{sec:uai:conclusion}

We presented two techniques to assess the impact of unmeasured confounding in causal additive noise models. We formulated unmeasured confounding as covariance between noise terms. We then introduced a grid-based approach for confounding between two variables, and an optimization-based approach for confounding in the general case. We demonstrated our approach on two real-world fairness datasets.
Our techniques can also be applied for sensitivity analysis of other causal queries, which is an interesting direction for future research.
Currently, our approach is limited to a specific type of unobserved confounding in ANMs assuming a certain functional form of the structural equations.
However, we believe the developed tools are an important step towards making causal models suitable to address discrimination in real-world prediction problems.

\chapter{Fairness and privacy}
\label{chap:blindjustice}

\graphicspath{{figs/chap6/}}

In this chapter, we go beyond the assumption that sensitive attributes are available in the clear---i.e., as clear text without encryption---to train or evaluate fair models.
To avoid disparate treatment, sensitive attributes should not be considered.
On the other hand, in order to avoid disparate impact, sensitive attributes must be examined---e.g., in order to learn a fair model, or to check if a given model is fair.
We introduce methods from secure multi-party computation which allow us to avoid both.
By encrypting sensitive attributes, we show how an outcome-based fair model may be learned, checked, or have its outputs verified and held to account, \emph{without users revealing their sensitive attributes and without modelers having to disclose their models}.

The main content of this chapter has been published in the following paper:

\fbox{\parbox{\textwidth}{
\pubitem{Blind Justice: Fairness with Encrypted Sensitive Attributes}
{Niki Kilbertus, Adri\`{a} Gasc\'{o}n, Matt Kusner, Michael Veale, Krishna~P.~Gummadi, Adrian Weller}
{International Conference on Machine Learning (ICML), 2018}
{https://arxiv.org/abs/1806.03281}
{https://github.com/nikikilbertus/blind-justice}
{}
}}

\section{Introduction}
\label{sec:icml:intro}

To motivate the main challenge in this chapter, we recall two notions of discrimination briefly mentioned in \chapref{chap:existing}.
The first type, \emph{disparate treatment} (or \emph{direct discrimination}), occurs if individuals are treated differently according to their sensitive attributes (with all others equal).
Intuitively, to avoid disparate treatment, one should not inquire about individuals' sensitive attributes, i.e., apply fairness through unawareness.
\emph{Disparate impact} (or \emph{indirect discrimination}) occurs when the \emph{outcomes} of decisions disproportionately benefit or harm individuals from subgroups with particular sensitive attribute settings without appropriate justification.
For example, firms deploying car insurance telematics devices \citep{handel2014insurance} build up high dimensional pictures of driving behavior which might easily proxy for sensitive attributes even when they are omitted.
As we have discussed in \chapref{chap:existing}, most work on observational group matching criteria has thus focused on avoiding notions of disparate impact.

In order to check and enforce such requirements, the modeler is usually assumed to have access to the sensitive attributes for individuals in the training data---however, this may be undesirable for several reasons \citep{Zliobaite2016}.
First, individuals are unlikely to want to entrust sensitive attributes to modelers in all application domains.
Where applications have clear discriminatory potential, it is understandable that individuals may be wary of providing sensitive attributes to modelers who might exploit them to negative effect, especially with no guarantee that a fair model will indeed be learned and deployed.
Even if certain modelers themselves were trusted, the wide provision of sensitive data creates heightened privacy risks in the event of a data breach.
Further, legal barriers may limit collection and processing of sensitive personal data.
A timely example is the EU'{}s General Data Protection Regulation (GDPR), which contains heightened prerequisites for the collection and processing of some sensitive attributes or the California Consumer Privacy Act. 
Unlike other data, modelers cannot justify using sensitive characteristics in fair learning with their ``legitimate interests''---and instead will often need explicit, freely given consent \citep{vealeedwardsa29}.

One way to address these concerns was recently proposed by \citet{VealeBinns2017}.
The idea is to involve a highly trusted third party, and may work well in some cases.
However, there are significant potential difficulties: individuals must disclose their sensitive attributes to the third party (even if an individual trusts the party, she may have concerns that the data may somehow be obtained or hacked by others, e.g., \citealp{Graham17});
and the modeler must disclose their model to the third party, which may be incompatible with their intellectual property or other business concerns.

To overcome these seemingly conflicting interests, we propose an approach to detect and mitigate disparate impact without disclosing readable access to sensitive attributes.
This reflects the notion that decisions should be blind to an individual'{}s status---depicted in courtrooms by a blindfolded Lady Justice holding balanced scales \citep{capers2012blind}.
We assume the existence of a regulator with fairness aims (such as a data protection authority or anti-discrimination agency).
With recent methods from \emph{secure multi-party computation} (MPC), we enable auditable fair learning while ensuring that both individuals' sensitive attributes and the modeler'{}s model remain private from all other parties---including the regulator.
Desirable fairness and accountability applications we enable include:
\begin{enumerate}
\item \textbf{Fairness certification.} Given a model and a dataset of individuals, check that the model satisfies a given observational group fairness constraint; if yes, generate a certificate.
\item \textbf{Fair model training.} Given a dataset of individuals, learn a model guaranteed and certified to be fair.
\item \textbf{Decision verification.} A malicious modeler might go through fair model training, but then use a different model in practice.
To address such accountability concerns \citep{kroll2016accountable}, we efficiently provide for an individual to challenge a received outcome, verifying that it matches the outcome from the previously certified model.
\end{enumerate}
We rely on recent theoretical developments in MPC, specifically the protocols for training linear and logistic models in secure MPC developed by \citet{mohassel2017secureml}.
In this work, we further extend these protocols to also admit linear constraints in order to enforce fairness requirements.
These extensions may be of independent interest.
Finally, we demonstrate the efficacy of our methods by evaluating them on synthetic and real-world datasets.

We note that the privacy or secrecy constraints considered in our approach are separate from other theorized, setup-dependent attacks, e.g., model extraction \citep{tramer2016stealing} or inversion \citep{fredrikson2015model}. If relevant, modelers may need to consider these separately.
In particular, our approach aims at fixing the issue of providing, transmitting, and storing sensitive data ``in the clear''.
The notion of privacy that is achieved by encrypting these data is distinct from, e.g., the guarantees provided by differential privacy.
In principle, our protocols do not prevent a malicious modeler from statistically inferring information about the protected attribute given non-sensitive data.
The challenge of training fair models in a differentially private fashion and potential tensions between fairness and differential privacy have been explored \citep{pmlr-v97-jagielski19a,cummings2019compatibility,bagdasaryan2019differential,ding2020differentially,xu2020removing}.

Besides work on differentially private fair model training, our original work has also been followed up by a line of research on fair machine learning without or only limited demographic information \citep{pmlr-v80-hashimoto18a,lahoti2020fairness,coston2019fair,chen2019fairness,rastegarpanah2020fair}.

\section{Fairness and privacy requirements}
\label{sec:icml:setting}

\paragraph{Assumptions and incentives.}
We assume three categories of participants: a \emph{modeler}~$\dc$, a {\em regulator}~$\reg$, and \emph{users} $\user_1, \ldots, \user_n$.
For each user $\user_i$, we consider a vector of sensitive features $\z_i\in\cZ = \{0, 1\}^p$ (e.g., ethnicity or gender) which might be a source of discrimination, and a vector of non-sensitive features $\x_i\in\cX = \bR^d$, again assumed to be discrete or real.
Here, we deviate from the previous restriction of only dealing with a single binary sensitive attributes and instead allow for $p \in \bN$ binary sensitive features.
Additionally, for each user there is a non-sensitive label or target $y_i\in\cY = \{0, 1\}$ which the modeler~$\dc$ would like to predict. Again, we assume binary for simplicity labels though our MPC approach could be extended to multi-label settings.

Modeler~$\dc$ wishes to train a parametric model $f_{\btheta}: \cX \to \cY$, which accurately maps features~$\x_i$ to labels~$y_i$, in a supervised fashion.
We assume $\dc$ needs to keep the model private for intellectual property or other business reasons.
The model $f_\btheta$ does not use sensitive information $\z_i$ as input to prevent disparate treatment (direct discrimination).

For each user $\user_i$, the modeler $\dc$ observes or is provided $\x_i, y_i$.
The sensitive information in~$\z_i$ is required to ensure~$f_{\btheta}$ meets a given observational group fairness condition~$\F$.
While each user~$\user_i$ wants~$f_{\btheta}$ to meet~$\F$, they also wish to keep~$\z_i$ private from all other parties. The regulator~$\reg$ aims to ensure that~$\dc$ deploys only models that meet fairness condition~$\F$. It has no incentive to collude with~$\dc$.
If collusion were a concern, more sophisticated cryptographic protocols would be required, which we briefly touch upon in \secref{sec:icml:setting}.
Further, the modeler~$\dc$ might be legally obliged to demonstrate to the regulator~$\reg$ that their model meets fairness condition~$\F$ before it can be publicly deployed.
As part of this,~$\reg$ also has a positive duty to enable the training of fair models.

Later in this section, we define and address three fundamental problems in our setup: certification, training, and verification.
For each problem, we present its functional goal and its privacy requirements.
We refer to $\B{D} = \{(\x_i, y_i)\}_{i=1}^{n}$ and $\Z = (\z_i)_{i=1}^{n} \in \bR^{n \times p}$ as the non-sensitive and sensitive data, respectively.
We will reference the overview Figure~\ref{figure.MPC} throughout the description of various components of the protocols.
A concise description of all three protocols is then given right before \secref{sec:icml:tc}.

\paragraph{Fairness considerations.}
In this chapter we focus on the p\%-rule from eq.~\eqref{eq:ppercent_first} as the fairness criterion $\F$, which we restate here as
\begin{equation}\label{eq:p_percent}
\frac{\prob(\hY = 1 \given Z=z)}{\prob(\hY = 1 \given Z=z')}
\ge \frac{p}{100} \fora z,z' \in \cZ \eqp
\end{equation}
A similar MPC approach could also be used for other observational group matching criteria mentioned in Table~\ref{tab:criteria}, which have been addressed with efficient standard (non-private) methods.

\begin{figure}[t!]
  \centerline{\includegraphics[width=\textwidth]{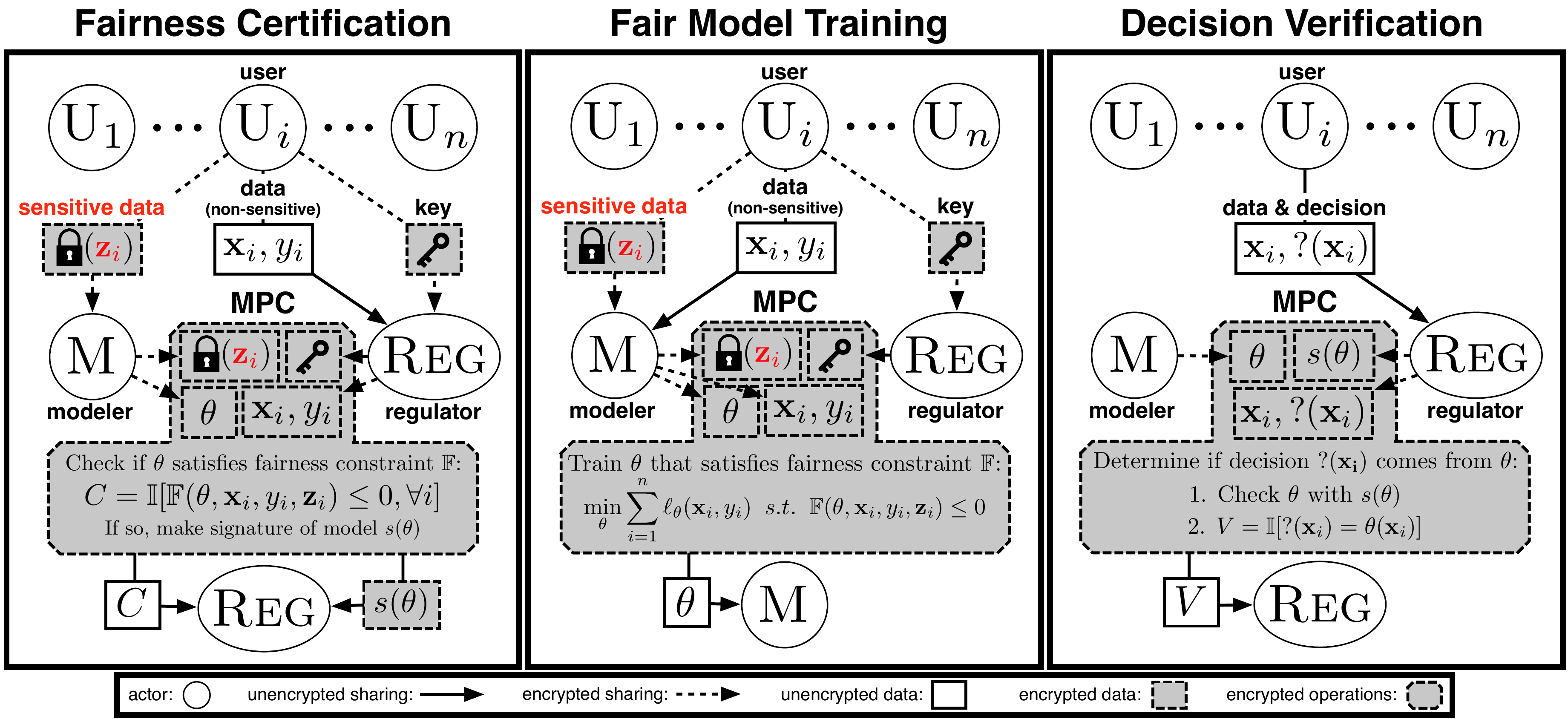}}
  \caption[Protocols for fairness certification, fair model training, and decision verification]{Our setup for \emph{Fairness certification} (\emph{Left}), \emph{Fair model training} (\emph{Center}), and \emph{Decision verification} (\emph{Right}).}
  \label{figure.MPC}
\end{figure}

\paragraph{Fairness certification.}
Given a notion of fairness $\F$, the modeler $\dc$ would like to work with the regulator $\reg$ to obtain a certificate that model $f_{\btheta}$ is fair.
To do so, we propose that users send their non-sensitive data $\B{D}$ to $\reg$;
and send \emph{encrypted} versions of their sensitive data $\Z$ to both $\dc$ and $\reg$.
Neither $\dc$ nor $\reg$ can read the sensitive data.
However, we can design a secure protocol between $\dc$ and $\reg$ (described in \secref{sec:icml:mpc}) to certify if the model is fair.
This setup is shown in Figure~\ref{figure.MPC} (\emph{Left}).

While both~$\reg$ and~$\dc$ learn the outcome of the certification, we require the following \emph{privacy constraints}:
(C1) \emph{privacy of sensitive user data}: no one other than $\user_i$ ever learns $\z_i$ in the clear,
(C2) \emph{model secrecy}: only $\dc$ learns $f_{\btheta}$ in the clear, and
(C3) \emph{minimal disclosure of $\B{D}$ to} $\reg$: only $\reg$ learns $\B{D}$ in the clear.

\paragraph{Fair model training.}
How can a modeler $\dc$ learn a fair model without access to users' sensitive data $\Z$? We propose to solve this by having users send their non-sensitive data $\B{D}$ to $\dc$ and to distribute encryptions of their sensitive data to $\dc$ and $\reg$ as in certification.
We shall describe a secure MPC protocol between $\dc$ and $\reg$ to train a fair model $f_{\btheta}$ privately. This setup is shown in Figure~\ref{figure.MPC} (\emph{Center}).

\emph{Privacy constraints}:
(C1) privacy of sensitive user data,
(C2) model secrecy, and
(C3) minimal disclosure of $\B{D}$ to $\dc$.

\paragraph{Decision verification.}
Assume that a malicious $\dc$ has had model~$f_{\btheta}$ successfully certified by $\reg$ as above.
It then swaps $f_{\btheta}$ for another potentially unfair model $f_{\btheta'}$ in the real world.
When a user receives a decision $\hy$, e.g., her mortgage is denied, she can then challenge that decision by asking~$\reg$ for a verification $V$.
The verification involves $\dc$ and $\reg$, and consists of verifying that $f_{\btheta'}(\x) = f_{\btheta}(\x)$, where $\x$ is the user'{}s non-sensitive data.
This ensures that the user would have been subject to the same result with the certified model $f_{\btheta}$, even if $f_{\btheta'} \neq f_{\btheta}$ and $f_{\btheta'}$ is not fair.
Hence, while there is no simple technical way to prevent a malicious $\dc$ from deploying an unfair model, it will get caught if a user challenges a decision that would differ under~$f_{\btheta}$. This setup is shown in Figure~\ref{figure.MPC} (\emph{Right}).

\emph{Privacy constraint}: While $\reg$ and the user learn the outcome of the verification, we require
(C1) privacy of sensitive user data, and
(C2) model secrecy.

\paragraph{Design choices.}
We include a regulator party in our setup for several reasons. Given fair learning is of most benefit to vulnerable individuals, we do not wish to deter adoption with high individual burdens.
While MPC could be carried out without the involvement of a regulator, using all users as parties, this comes at a greater computational cost.
With current methods, taking that approach would be unrealistic given the size of the user-base in many domains of concern, and would furthermore require all users to be online simultaneously.
Introducing a regulator removes these barriers and leaves users' computational burden at a minimum level, with envisaged applications practical with only their web browsers.

In cases where users are uncomfortable sharing $\B{D}$ with either $\reg$ or $\dc$, it is trivial to extend all three tasks such that all of $\x_i, y_i,\z_i$ remain private throughout, with the computation cost increasing only by a factor of 2.
This extension would sometimes be desirable as it restricts the view of $\dc$ to the final model, prohibiting inferences about $\Z$ when $\B{D}$ is known.
However, this setup hinders exploratory data analysis by the modeler which might promote robust model-building, and, in the case of verification, validation by the regulator that user-provided data is correct.

\section{Our solution}
\label{sec:icml:mpc}

Our proposed solution to these three problems is to use secure Multi-Party Computation (MPC).
Before we describe how it can be applied to fair learning, we first present the basic principles of MPC, as well as its limitations particularly in the context of machine learning applications.

\paragraph{MPC for machine learning.}
Multi-party computation protocols allow two parties~$P_1$ and~$P_2$ holding secret values~$x_1$ and~$x_2$ to evaluate an agreed-upon function $f$, via $y = f(x_1, x_2)$ in a way in which the parties (either both or one of them) learn \emph{only}~$y$.
For example, if $f(x_1, x_2) = \B{1}[x_1 < x_2]$, then the parties would learn which of their values is bigger, but nothing else.
Here, the indicator $\B{1}[\cdot]$ is $1$ if its argument is true and $0$ otherwise.
This corresponds to the well-known \emph{Yao'{}s millionaires problem}: two millionaires want to conclude who is richer without disclosing their wealth to each other.
The problem was introduced by Andrew Yao in 1982, and kicked off the area of multi-party computation in cryptography.

In our setting---instead of a simple comparison as in the millionaires problem---$f$ will be either
(i) a procedure to check the fairness of a model and certify it,
(ii) a machine learning training procedure with fairness constraints, or
(iii) a model evaluation to verify a decision.
The two parties involved in our computation are the modeler $\dc$ and the regulator $\reg$.
The inputs depend on the case (see Figure~\ref{figure.MPC}).

As generic solutions do not yet scale to real-world data analysis tasks, one typically has to tailor custom protocols to the desired functionality.
This approach has been followed successfully for a variety of machine learning tasks such as logistic and linear regression \citep{nikolaenko_privacy-preserving_2013, gascon_privacy-preserving_2017, mohassel2017secureml}, neural network training \citep{mohassel2017secureml} and evaluation \citep{gazelle, minionn}, matrix factorization \citep{nikolaenko_matrix-factorisation}, and principal component analysis \citep{pca}.
In the next section we review challenges beyond scalability issues that arise when implementing machine learning algorithms in MPC.

\paragraph{Challenges in multi-party machine learning.}
MPC protocols can be classified into two groups depending on whether the target function is represented as either a Boolean or arithmetic circuit.
All protocols proceed by having the parties jointly evaluate the circuit, processing it gate by gate while keeping intermediate values hidden from both parties by means of a secret sharing scheme.
While representing functions as circuits can be done without losing expressiveness, it means certain operations are impractical.
In particular, algorithms that execute different branches depending on the input data will explode in size when implemented as circuits, and in some cases lose their run time guarantees (e.g., consider binary search).

Crucially, this applies to \emph{floating-point arithmetic}.
While this is work in progress, state-of-the-art MPC floating-point arithmetic implementations take more than~$15$ milliseconds to multiply two $64$ bit numbers \citep[Table~4]{DDKSSZ15}, which is prohibitive for our applications.
Hence, machine learning MPC protocols are limited to \emph{fixed-point} arithmetic.
Overcoming this limitation is a key challenge for the field.
Another necessity for the feasibility of MPC is to approximate non-linear functions such as the sigmoid, ideally by (piecewise) linear functions.

\paragraph{Input sharing.}
To implement the functionality from Figure~\ref{figure.MPC}, we first need a secure procedure for the users to \emph{secret share} a sensitive value, for example her race, with the modeler~$\dc$ and the regulator~$\reg$.
We use \emph{additive secret sharing}.
A value~$z$ is represented in a finite domain~$\bZ_q$---we use~$q=2^{64}$.
To share~$z$, the user samples a value~$r$ from~$\bZ_q$ uniformly at random, and sends $z - r$ to~$\dc$ and~$r$ to~$\reg$.
While $z$ can be reconstructed (and subsequently operated on) inside the MPC computation by means of a simple addition, each share on its own does not reveal anything about~$z$ (other than that it is in~$\bZ_q$).
One can think of arithmetic sharing as a ``distributed one-time pad''.

In Figure~\ref{figure.MPC}, we now reinterpret the key held by $\reg$ and the encrypted $z$ by $\dc$, as their corresponding shares of the sensitive attributes and denote them by $\share{z}{1}$ and $\share{z}{2}$ respectively.
The idea of privately outsourcing computation to two non-colluding parties in this way is recurrent in MPC, and often referred to as the two-server model \citep{mohassel2017secureml, gascon_privacy-preserving_2017, nikolaenko_privacy-preserving_2013, pca}.

\paragraph{Signing and checking a model.}
We will see that \emph{certification} and \emph{verification} partly correspond to sub-procedures of the \emph{fair training} task:
during training we check the fairness constraint $\F$, and repeatedly evaluate partial models on the training dataset (using gradient descent).
Hence, \emph{certification} and \emph{verification} do not add technical difficulties over training, which is described in detail in \secref{sec:icml:tc}.
However, for verification, we still need to ``sign'' the model, i.e., $\reg$ should obtain a signature $s(\btheta)$ as a result of model certification, see Figure~\ref{figure.MPC} (\emph{Left}).
This signature is used to check in the verification phase, whether a given model $\btheta'$ from $\dc$ satisfies $s(\btheta') = s(\btheta)$ for a certified fair model $\btheta$ (in which case $\btheta = \btheta'$
with high probability).
Moreover, we need to preserve the secrecy of the model, i.e., $\reg$ should not be able to recover $\btheta$ from $s(\btheta)$.
These properties, given that the space of models is large, calls for a cryptographic hash function, such as SHA-256.

Additionally, in our functionality, the hash of $\btheta$ should be computed inside MPC, to hide $\btheta$ from $\reg$.
Fortunately, cryptographic hashes such as SHA-256 are a common benchmark functionality in MPC, and their execution is highly optimized.
More concretely, the overhead of computing $s(\btheta)$, which needs to be done both for certification and verification is of the order of fractions of a second \citep[Figure~14]{sha-mpc}.
While cryptographic hash functions have various applications in MPC, we believe the application to machine learning model certification is novel.

Ultimately, certification is implemented in MPC as a check that $\btheta$ satisfies the criterion $\F$, followed by the computation of $s(\btheta)$.
The regulator $\reg$ keeps the signature $s(\btheta)$ of the fair model parameters for later verification.
On the other hand, for verification, the MPC protocol first computes the signature of the model provided by $\dc$, and then proceeds with a prediction as long as the computed signature matches the one obtained by $\reg$ in the verification phase.
An alternative solution is possible based on symmetric encryption under a shared key, as highly efficient MPC implementations of block ciphers such as AES are available \citep{KellerORSSV17}.

\paragraph{Fair training.}
To realize the \emph{fair training} functionality, we closely follow the techniques recently introduced by \citet{mohassel2017secureml}.
Specifically, we extend their custom MPC protocol for logistic regression to additionally handle linear constraints.
This extension may be of independent interest and has applications for privacy-preserving machine learning beyond fairness.
The concrete technical difficulties in achieving this goal, and how to overcome them, are presented in the next section.
The formal privacy guarantees of our fair training protocol are stated in the following proposition.

\begin{proposition}
For non-colluding~$\dc$ and~$\reg$, our protocol implements the fair model training functionality satisfying constraints (C1)-(C3) in \secref{sec:icml:setting} in the presence of a semi-honest adversary.
\end{proposition}
The proof holds in the random oracle model, as a standard simulation argument combining several MPC primitives \citep{mohassel2017secureml,gascon_privacy-preserving_2017}.
It leverages security of arithmetic sharing, garbled circuits, and oblivious transfer protocols in the semi-honest model \citep{DBLP:conf/stoc/GoldreichMW87}.
Before going into details about the specific technical challenges of fair model training, we now provide a short, high-level introduction to MPC, including a description of the relevant techniques from \citet{mohassel2017secureml} used in our protocol.

\paragraph{Secret sharing.}
A secret sharing scheme allows one to split a value $x$ (the secret) among two parties, so that no party has unilateral access to $x$.
In our setting, a user Alice will secret share a sensitive value, for example her race, among a modeler~$\dc$ and a regulator~$\reg$. Among prominent secret sharing schemes are \emph{Shamir secret sharing}, \emph{xor sharing}, \emph{Yao sharing}, or \emph{arithmetic multiplicative/additive sharing}.
In our protocols we alternate between Yao sharing and additive sharing for efficiency. We have already describe the latter. To recap: the value~$x$ is represented in a finite domain~$\bZ_q$.
To share her race, Alice samples a value $r$ from~$\bZ_q$ uniformly at random, and sends $x - r$ to~$\dc$ and~$r$ to $\reg$.
We call each of $x-r$ and~$r$ a \emph{share}, and denote them as~$\share{x}{1}$ and~$\share{x}{2}$.
Now~$\dc$ and~$\reg$ can recover~$x$ by adding their shares, but each share on its own does not reveal anything about the value of~$x$ (other than that it is smaller than~$q$). The case where $q=2$ corresponds to xor sharing.

\paragraph{Function evaluation.}
MPC can be classified in two groups depending on how~$f$ is represented: either as a Boolean or arithmetic circuit.
All protocols proceed by having the parties jointly evaluate the circuit, processing it gate by gate.
For each gate~$g$ for which the value for the input wires~$x, y$ is shared among the parties, the parties run a subprotocol to produce the value $z = g(x,y)$ of the output wire, again shared, without revealing any information in the process.
In the setting where we use arithmetic additive sharing, the two parties $\dc$ and $\reg$ hold shares, $\share{x}{1}$,$\share{y}{1}$ and $\share{x}{2}$,$\share{y}{2}$, respectively.
In this case, $f$ is represented as an arithmetic circuit, and hence each gate $g$ in the circuit is either an addition or a multiplication.
Note that if $g$ is an addition gate, then a sharing of $z = g(x, y)$ can be obtained by having each party simply compute locally, i.e., without any interaction, $\share{z}{i} = \share{x}{i} + \share{y}{i}$, for $i\in\{1, 2\}$.
If $g$ is a multiplication, the subprotocol to compute shares of $z$ is much more costly.
Fortunately, it can be divided into an offline and an online phase.

\paragraph{The preprocessing model in MPC.}
In this model, two parties $P_1, P_2$ engage in an offline phase, which is data independent,
and compute (and store) \emph{shared multiplication triples} of the form $(a, b, c)$, with $c = ab$.
Here, $a,b \in \bZ_q$ are drawn uniformly at random, and each value $a, b, c$ is shared
among the parties as explained above. In the online phase, a multiplication gate
$z = \mul(x, y)$ on shared values $x, y$ can be evaluated as follows:
(1) each $P_i$ sets $\share{e}{i} = \share{x}{i} - \share{a}{i}$ and
$\share{f}{i} = \share{y}{i} - \share{b}{i}$,
(2) the parties exchange their shares of $e$ and $f$ and reconstruct these values locally by simply adding the shares, and
(3) each $P_i$ computes $\share{z}{i} = ef + f\share{a}{i} + e\share{b}{i} + \share{c}{i}$.
The correctness of this protocol can be checked by multiplying out all terms. Privacy relies on the uniform randomness of
$a, b$, and hence $\share{e}{i}$ and $\share{f}{i}$ completely
mask the values of $\share{x}{i}$ and $\share{y}{i}$, respectively.
For a formal proof see \citet{demmler_aby_2015}.

Hence, for each multiplication in the function to be evaluated, the
parties need to jointly generate a multiplication triple in advance.
For computations with many multiplications (like in our case) this can be a costly
process. However, this constraint is easy to accommodate in our architecture
for private fair model training, as $\dc$ and $\reg$ can run the offline phase
once ``overnight''.
Arithmetic multiplication via precomputed triples is a common technique exploited by
popular MPC frameworks \citep{demmler_aby_2015, damgard_multiparty_2012}.
In this setting, a number of protocols for triple generation (which we did not describe)
are available \citep{DBLP:conf/eurocrypt/KellerPR18} and under continuous improvement.
These protocols are often based on either Oblivious Transfer or Homomorphic Encryption.

\paragraph{The two-server model for multi-party learning.}
Due to a sequence of theoretical breakthroughs mentioned above as well as the general speedup enabled by faster hardware,
in the last three decades MPC has gone from being a mathematical curiosity
to a technology of practical interest with commercial applications.
There exist a number of openly available implementations \citep{demmler_aby_2015,zahur2015obliv} for generic MPC protocols such as the
ones based on arithmetic sharing \citep{damgard_multiparty_2012},
garbled circuits \citep{yao_how_1986}, or GMW \citep{DBLP:conf/stoc/GoldreichMW87}.
These protocols have different trade-offs in terms
of the number of parties they support, network requirements,
and scalability for different kinds of computations.
In our work, we focus on the $2$-party case, as the MPC computation is done by
$\dc$ and $\reg$.

Since generic protocols do not yet scale to input sizes typically encountered in machine learning applications like ours,
we extend the SGD protocol from \citet{mohassel2017secureml}, in which the following useful accelerating techniques are presented.
\begin{itemize}
\item Efficient rescaling: As our arithmetic shares represent fixed-point numbers, we need to rescale by the precision $p$ after every multiplication. This involves dividing by $2^p$, an expensive operation to do in MPC, and in particular in arithmetic sharing. \citet{mohassel2017secureml} show an elegant solution to this problem: the parties can rescale locally by dropping $p$ bits of their shares. It is not hard to see that this might produce the wrong result. However, the parameters of the arithmetic secret sharing scheme can be set such that with a tunable arbitrarily large probability the error is at most $\pm 1$. This trick can be used for any division by a power of two.
\item Alternating sharing types: As already pointed out in previous work \citep{demmler_aby_2015}, alternating between secret sharing schemes can provide significant acceleration for some applications. Intuitively, arithmetic operations are fast in arithmetic shares, while comparisons are fast in schemes that represent functions as Boolean circuits. Examples of the latter are the GMW protocol and Yao'{}s garbled circuits. In our implementation, we follow this recipe and implement matrix-vector multiplication using arithmetic sharing, while for evaluating our variant of sigmoid, we rely on the protocol from \citet{mohassel2017secureml} implemented with garbled circuits using the Obliv-C framework \citep{zahur2015obliv}.
\item Matrix multiplication triples: Another observation is that the idea described above for preprocessing multiplications over arithmetic shares can be reinterpreted at the level of matrices. This results in a faster online and offline phases, see \citet{mohassel2017secureml} for details.
\end{itemize}

\paragraph{How to prove that a protocol is secure.}
We did not provide a formal definition of security so far and instead referred the reader to \citet{mohassel2017secureml}. In MPC, privacy in the case of semi-honest adversaries is argued in the simulation paradigm, see \citet{goldreichbook} or \citet{lindell_how_2016} for formal definitions and detailed proofs. Intuitively, in this paradigm one proves that every inference that a party---in our case either $\reg$ or $\dc$---could draw from observing the execution trace of the protocol could also be drawn from the output of the execution and the party'{}s input. This is done by proving the existence of a {\em simulator} that can produce an execution trace that is indistinguishable from the actual execution trace of the protocol. A crucial point is that the simulator only has access to the input and output of the party being simulated.

\paragraph{Step-by-step description of protocols.}
Before moving on to specific technical implementation challenges of the building blocks above, we now put the all together into a concise description of the protocols for fairness certification, fair model training, and decision verification.

\begin{description}[wide,labelindent=0pt]
  \item[Fairness certification] \phantom{}
  \begin{enumerate}
  	\item Users send their non-sensitive data to $\reg$ in the clear and send one share of their sensitive data to $\dc$ and $\reg$ respectively.
  	\item $\dc$ and $\reg$ engage in a secure two-party computation to which $\dc$ provides its shares of sensitive attributes $\B{\share{Z}{1}}$ as well as model parameters $\btheta$ and $\reg$ provides its shares of sensitive attributes $\B{\share{Z}{2}}$ as well as all non-sensitive data $\X, \B{y}$ as inputs.
  	\item Within the secure computation, the agreed upon protocol reconstructs $\Z$ and evaluates the (violation of) the fairness criterion $\F(\btheta, \X, \B{y}, \Z)$.
  	\item If the model $\btheta$ satisfies the fairness criterion $\F$ on the provided dataset, still within the joint computation a hash of the model parameters $s(\btheta)$ is computed.
  	\item As a result of the joint computation, $\reg$ receives and indicator whether the fairness criterion is satisfied and if so, the hash of the provided model parameters. $\dc$ receives no output from the computation.
  \end{enumerate}
  \item[Fair model training] \phantom{}
  \begin{enumerate}
  	\item Users send their non-sensitive data to $\dc$ in the clear and send one share of their sensitive data to $\dc$ and $\reg$ respectively.
  	\item $\dc$ and $\reg$ engage in a secure two-party computation to which $\dc$ provides its shares of sensitive attributes $\B{\share{Z}{1}}$ as well as all non-sensitive data $\X, \B{y}$ and $\reg$ provides its shares of sensitive attributes $\B{\share{Z}{2}}$ as inputs.
  	\item Within the secure computation, the agreed upon protocol reconstructs $\Z$ and runs through a fixed number of stochastic gradient descent epochs to minimize the following constrained optimization problem
  	\begin{equation*}
  	  \min_{\btheta} \sum_{i=1}^n \ell_{\btheta} (\x_i, y_i) \quad \text{s.t.} \quad \F(\btheta, \x_i, y_i, \z_i) \le 0 \eqp
  	\end{equation*}
  	\item $\dc$ receives the resulting model parameters $\btheta$ of this optimization as output of the joint computation.
  	$\reg$ receives no output from the computation.
  \end{enumerate}
  \item[Decision verification] \phantom{}
  \begin{enumerate}
  	\item A specific user who has used the service of $\dc$ with inputs $\x_i$ has received $y$ as an outcome from $\dc$.
  	$\dc$ may have arrived at decision $y$ by some unknown model denoted by ?, i.e., $y = ?(\x_i)$.
  	\item The user sends $\x_i, y$ to $\reg$.
  	\item $\dc$ and $\reg$ engage in a secure two-party computation to which $\dc$ provides model parameters $\btheta$ and $\reg$ provides $\x_i, y$ as well as the signature $s$ of model parameters that have previously been certified as fair.
  	\item Within the secure computation, the agreed upon protocol computes the hash of the provided model parameters $\btheta$ and checks whether it matches the signature $s$. If so, it verifies that applying the provided model to the provided inputs $\x_i$ results in the provided outcome $y$.
  	\item $\reg$ receives indicators for whether both checks passed the verification as output of the joint computation.
  	$\dc$ receives no output from the computation.
  \end{enumerate}
\end{description}

\section{Technical challenges of fair training}
\label{sec:icml:tc}

We now present our custom tailored approaches for learning and evaluating fair models with encrypted sensitive attributes.
We do this via the following contributions:
\begin{itemize}
\item We argue that current optimization techniques for fair learning algorithms are unstable for fixed-point data, which is required by our MPC techniques.
\item We describe optimization schemes that are well-suited for learning over fixed-point number representations.
\item We combine tricks to approximate non-linear functions with specialized operations to make fixed-point arithmetic feasible and avoid over- and under-flows.
\end{itemize}

The optimization problem at hand is to learn a classifier $\btheta$ subject to a (convex relaxation of a) fairness constraint~$\F(\btheta)$:
\begin{equation}\label{eq:optim}
\min_\btheta\;\; \sum_{i=1}^{n}\ell_{\btheta}(\x_i, y_i)\qquad
\text{subject to}\;\; \F(\btheta) \le \B{0} \eqc
\end{equation}
where $\ell_{\btheta}$ is a loss term (the logistic loss in our applications).
We collect user data from~$\user_1, \ldots, \user_n$ into matrices~$\X \in \bR^{n\times d}, \Z\in\{0,1\}^{n\times p}$, and a label vector~$\y \in \{0,1\}^n$.

We will focus on a convex approximation of the $p\%$-rule, see eq.~\eqref{eq:p_percent}, for linear classifiers following the derivation of \citet{Zafar2017}.
To this end, we measure unfairness as the covariance between $\Z$ and the signed distance of the feature vectors $\X$ to the decision boundary implied by a logistic  model with parameters $\btheta$, i.e., these distances are given by $\X \btheta$.
We then compute the covariance as
\begin{equation*}
 \cov(\z, \x^{\top} \btheta) = \E[ (\z - \bar{\z})\, (\x^{\top} \btheta - \bar{\x}^{\top} \btheta) ] = \E[(\z - \bar{\z}) \x^{\top} \btheta] - \usub{= 0}{\E[(\z - \bar{\z})]} \bar{\x}^{\top} \btheta = \E[(\z - \bar{\z}) \x^{\top} \btheta] \eqc
\end{equation*}
where $\bar{\z}$ and $\bar{\x}$ are the means of all inputs $\z_i$ and $\x_i$ respectively.
By defining $\hat{\Z}$ to be the matrix of all $\hat{\z}_i : = \z_i - \bar{\z}$, we can now write the approximate fairness constraint as
\begin{equation}\label{eq:ppercent}
\F(\btheta) = \frac{1}{n} |\hat{\Z}^{\top} \X \btheta| - \const \eqc
\end{equation}
where $\const \in \bR^d$ is a constant vector with non-negative entries corresponding to the tightness of the fairness constraint.
Note that the entries of $\const$ take the role of slack variables---the smaller the entries of $\const$, the tighter the constraint.
Unlike the original $p\%$-rule, eq.~\eqref{eq:ppercent} is convex in $\btheta$.
The correspondence to the $p\%$-rule can be understood as follows:
Because of $\prob(y = 1 \given \z) = \prob(\sign(\x^{\top} \btheta) = 1 \given \z)$, whenever the $p\%$-rule is satisfied for $p=100$, we have that $\prob(\x^{\top} \btheta > 0 \given \z) = \prob(\x^{\top} \btheta > 0 \given \z')$ for all $\z, \z'$.
In that situation the covariance will also be approximately zero at least asymptotically in the large sample limit.
Hence, a sensible relaxation of the 100\%-rule, which corresponds to the covariance being exactly zero, we can allow the covariance to deviate from zero in a controlled fashion.
In this sense, the slack or tightness variable $\const$ roughly corresponds to the $p$ value in the $p\%$-rule.
With $\A := \nicefrac{1}{n}\, \hat{\Z}^{\top} \X$, our final \emph{$p\%$ constraint} thus reads~$\F(\btheta) = |\A \btheta | - \const$, where the absolute value is taken element-wise.

\paragraph{Current techniques.}
To solve the optimization problem in eq.~\eqref{eq:optim}, with the fairness function~$\F{}$ in eq.~\eqref{eq:ppercent}, \citet{Zafar2017} use Sequential Least Squares Programming (SLSQP).
This technique works by reformulating eq.~\eqref{eq:optim} as a sequence of Quadratic Programs (QPs).
After solving each QP, their algorithm uses the Han-Powell method, a quasi-Newton method that iteratively approximates the Hessian $\B{H}$ of the objective function via the update
\begin{equation*}
\B{H}_{t+1} = \B{H}_t + \frac{\B{l}_{\Delta} \B{l}_{\Delta}^\top}{\btheta_{\Delta}^\top \B{l}_{\Delta}} - \frac{\B{H}_t \btheta_{\Delta} \btheta_{\Delta}^\top \B{H}_t}{\btheta_{\Delta}^\top \B{H}_t \btheta_{\Delta}} \eqc
\end{equation*}
where
\begin{equation*}
\B{l}_{\Delta} = \B{l}(\btheta_{t+1},\blambda_{t+1}) - \B{l}(\btheta_{t},\blambda_{t}) \qtxtq{and} \btheta_{\Delta} = \btheta_{t+1} - \btheta_t \eqp
\end{equation*}
The Lagrangian is given by
\begin{equation*}
\B{l}(\btheta_t,\blambda_t) = \sum_{i=1}^n \ell_{\btheta_t}(\x_i,y_i) + \blambda_t^\top \F(\btheta_t) \eqp
\end{equation*}

There are two issues with this approach from an MPC perspective.
First, solving a sequence of QPs is prohibitively time-consuming in MPC.
Second, while the above Han-Powell update performs well on floating-point data, the two divisions by non-constant, non-integer numbers easily underflow or overflow with fixed-point numbers.

\paragraph{Fixed-point-friendly optimization techniques.}
Instead, to solve the optimization problem in eq.~\eqref{eq:optim}, we perform stochastic gradient descent and experiment with the following techniques to incorporate the constraints.

\begin{itemize}
\item \emph{Lagrangian multipliers.}
Here we minimize
\begin{equation*}
\mathcal{L} := \frac{1}{n} \sum_{i=1}^{n} \ell_{\btheta}^{\mathrm{BCE}}(\x_i, y_i) + \blambda^{\top} \max\{ \F(\btheta), \B{0} \} \eqc
\end{equation*}
using stochastic gradient descent, i.e., alternating updates
\begin{equation*}
\btheta \gets \btheta - \eta_{\btheta} \nabla_{\btheta} \mathcal{L} \qtxtq{and} \blambda \gets \max\{\blambda + \eta_{\blambda} \nabla_{\blambda} \mathcal{L}, \B{0} \}\eqc
\end{equation*}
where $\eta_{\btheta}, \eta_{\blambda}$ are the learning rates.

\item \emph{Projected gradient descent.}
For this method, consider specifically the~$p\%$-rule based notion~$\F(\btheta) = |\A \btheta | - \const$. We first define~$\hat{\A}$ as the matrix consisting of the rows of~$\A$ for which~$\F(\btheta) > \B{0}$, i.e., where the constraint is active.
In each step, we project the computed gradient of the binary-cross-entropy loss~$\mathcal{L}^{\mathrm{BCE}}$---either of a single example or averaged over a minibatch---back into the constraint set, i.e.,
\begin{equation}\label{eq:projected}
\btheta \gets \btheta - \eta_{\btheta} (\mathrm{Id}_d - \hat{\A}^{\top} (\hat{\A} \hat{\A}^{\top})^{-1} \hat{\A}) \nabla_{\btheta} \ell_{\btheta}^{\mathrm{BCE}} \eqp
\end{equation}

\item \emph{Interior point log barrier \citep{boydsbook}.}
We can approximate eq.~\eqref{eq:optim} for the~$p\%$-rule constraint~$\F(\btheta) = |\A \btheta | - \const$ by:
\begin{equation*}
\mathrm{minimize}\; \sum_{i=1}^{n} \ell_{\btheta}^{\mathrm{BCE}}(\x_i, y_i)
- \frac{1}{t} \sum_{j=1}^{p}
\bigl(\log(\avec_j^{\top} \btheta + c_j) + \log(-\avec_j^{\top} \btheta + c_j)\bigr)\eqc
\end{equation*}
where~$\avec_j$ is the~$j$th row of~$\A$.
The parameter~$t$ trades off the approximation of the true objective and the smoothness of the objective function.
Throughout training,~$t$ is increased, allowing the solution to move closer to the boundary.
As the gradient of the objective has a simple closed form representation, we can perform regular (stochastic) gradient descent.
\end{itemize}

After extensive experiments, described in detail in \secref{sec:icml:experiments}, we found the Lagrangian multipliers technique to work best.
It yields high accuracies, reliably stays within the constraints and is robust to hyperparameter changes such as learning rates or the batch size.
For a proof of concept, in \secref{sec:icml:experiments} we focus on the $p\%$-rule, i.e.,~eq.~\eqref{eq:ppercent}.
We note that the gradients for equalized odds or equal opportunity criteria take a similarly simple form, i.e., balancing the true positive or true negative rates is simple to implement for the Lagrangian multiplier technique, but harder for projected gradient descent.
However, these fairness notions are more expensive as we have to compute~$\Z^{\top} \X$ for each update step, instead of pre-computing it once at the beginning of training, see Algorithm~\ref{algo:lagrange}.
We could speed up the computation again by evaluating the constraint only on the current minibatch for each update, in which case we risk violating the fairness constraint.

\begin{figure}
\centering
\includegraphics{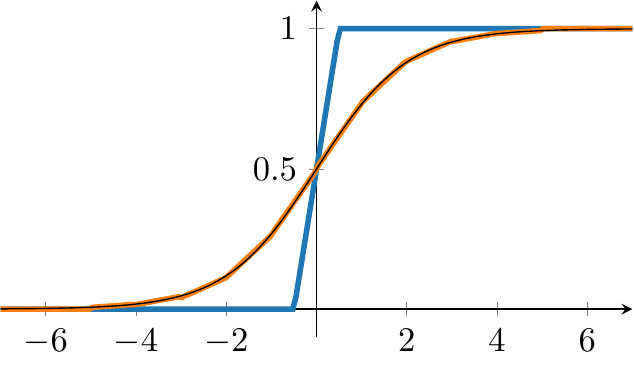}
\caption[Piecewise linear approximations for the non-linear sigmoid function]{Piecewise linear approximations for the non-linear sigmoid function (in black) from \citet{mohassel2017secureml} in blue and from \citet{sigmoid-approx} in orange.}
\label{fig:approximations}
\end{figure}

\paragraph{MPC-friendliness.}
For eq.~\eqref{eq:ppercent}, we can compute the gradient updates in all three methods with elementary linear algebra operations (matrix multiplications) and a single evaluation of the logistic function.
While MPC is well suited for linear operations, most nonlinear functions are prohibitively expensive to evaluate in an MPC framework.
Hence we tried two piecewise linear approximations for~$\sigma(x)$.
The first was recently suggested for machine learning in an MPC context \citep{mohassel2017secureml} and is simply constant~$0$ and~$1$ for $x< -0.5$ and $x>0.5$ respectively, and linear in between.
The second uses the optimal first order Chebychev polynomial on each interval~$[x, x+1]$ for~$x \in \{-5, -4, \ldots, 4\}$, and is constant~$0$ or~$1$ outside of~$[-5,5]$ \citep{sigmoid-approx}.
While it is more accurate, we only report results for the simpler first approximation, as it yielded equal or better results in all our experiments.
Both approximations are shown in Figure~\ref{fig:approximations}.

As the largest number that can be represented in fixed-point format with~$m$ integer and~$m$ fractional bits is roughly~$2^m + 1$, overflow becomes a common problem.
Since we whiten the features~$\X$ column-wise, we need to be careful whenever we add roughly~$2^m$ numbers or more, because we cannot even represent numbers greater than $2^m$.
In particular, the minibatch size has to be smaller than this limit.
For large~$n$, the multiplication~$\Z^{\top} \X$ in the fairness function~$\F$ for the $p\%$-rule is particularly problematic.

Hence, we split both factors into blocks of size $b\times b$ with~$b < 2^m$ and normalize the result of each blocked matrix multiplication by~$b$ before adding up the blocks.
We then multiply the sum by~$\nicefrac{b}{n} > 2^{-m}$.
As long as~$b$, $\nicefrac{b}{n}$ (and thus also~$\nicefrac{n}{b}$) can be represented with sufficient precision, which is the case in all our experiments, this procedure avoids under- and overflow.
Note that we require the sample size~$n$ to be a multiple of~$b$.
In practice, we have to either discard or duplicate part of the data.
Since the latter may introduce bias, we recommend subsampling.
Once we have (an approximation of) $\A \in \bR^{p \times d}$, we resort to normal matrix multiplication, as typically~$p, d \lesssim 100$, see Table~\ref{table.timing}.

Division is prohibitively expensive in MPC. Because there are no simple protocols such as the ones described for addition and multiplication, divisions have to approximated by a sequence of more basic computations that are faster to perform in MPC.
Hence, we set the minibatch size to a power of two, which allows us to use fast bit shifts for divisions when averaging over minibatches.
Unlike general division, bit shifts can be implemented efficiently in MPC.
To exploit the same trick when averaging over/across blocks in the blocked matrix multiplication, we choose~$n$ as the largest possible power of two, see Table~\ref{table.timing}.

\begin{algorithm}[htb]
  \caption{Fair model training with private sensitive values using Lagrangian multipliers for~$\F(\btheta) = \nicefrac{1}{n} |\Z^{\top} X| - \const$ with two parties: the modeler $\dc$ and the regulator $\reg$.}
  \begin{algorithmic}[1]
    \Inputx \textbf{from $\dc$:} $\B{\share{Z}{1}} \in \bZ_q^{n\times p}$
    \Inputx \textbf{from $\reg$:} $\X \in \bZ_q^{n\times d}$, $\B{y} \in \bZ_q^n$, $\B{\share{Z}{2}} \in \bZ_q^{n\times p}$
    \Inputx \textbf{publicly known:} learning rates $\eta_{\btheta}, \eta_{\blambda}$, number of training examples~$n$, minibatch size $2^s$, constraints $\B{c}\in \bZ_q^p$, and number of epochs $N_e$

    \State $\btheta \gets \B{0}$, $\blambda \gets \B{0}$

    \State $\A \gets \textsc{BlockedMultShiftAvg}(\Z^{\top}, \X)$

    \For{$j$ from $1$ to $N_e$}

    \For{$i$ from $1$ to $\nicefrac{n}{2^s}$}

    \State $(\X_{i},\y_{i}) \gets \textsc{SampleMinibatch}(\X, \y)$

    \State $\F \gets |\A \btheta| - c$

    \State $\nabla_{\blambda} \gets \max\{\F, \B{0} \}$

    \State $\sigma \gets \textsc{SigmoidApprox}(\X_i \btheta)$

    \State $\nabla_{\btheta}^{\mathrm{BCE}} \gets \textsc{ShiftDivide}(\X_i^{\top}(\sigma - \y_i), 2^s)$

    \State $\nabla_{\btheta}^{\mathrm{CON}} \gets
      \begin{cases}
      \A^{\top} \blambda, &\text{if } \A > \B{0} \wedge \F > 0 \\
      - \A^{\top} \blambda, &\text{if } \A < \B{0} \wedge \F > 0\\
      \B{0}, &\text{if } \F \le 0
      \end{cases}$

    \State $\btheta \gets \btheta - \eta_{\btheta} (\xi^{\mathrm{BCE}}_j \nabla_{\btheta}^{\mathrm{BCE}} + \xi^{\mathrm{CON}}_j \nabla_{\btheta}^{\mathrm{CON}})$

    \State $\blambda \gets \max\{\blambda + \eta_{\blambda} \nabla_{\blambda}, \B{0}\}$
    \EndFor
    \EndFor
    \State \Return Parameters~$\btheta$
  \end{algorithmic}
  \label{algo:lagrange}
\end{algorithm}

\paragraph{The fair training algorithm.}
Algorithm~\ref{algo:lagrange} describes the computations $\dc$ and $\reg$ have to perform for fair model training using the Lagrangian multiplier technique and the $p\%$-rule from eq.~\eqref{eq:ppercent}.
The hyperparameters introduced in the algorithm, such as $\xi^{\mathrm{BCE}}, \xi^{\mathrm{CON}}$, will be explained in the next section.
We implicitly assume all computations are performed jointly on additively shared secrets by $\dc$ and $\reg$ as described in \secref{sec:icml:mpc}.
This means that $\dc$ and $\reg$ each receive a secret share of the protected attributes~$\Z$.
Following the protocols outlined in \secref{sec:icml:mpc}, they can then jointly evaluate the steps in Algorithm~\ref{algo:lagrange}.
This allows them to operate on the sensitive values within the MPC computation, while preventing unilateral access to them by $\dc$ and $\reg$.
The result of these computations is the same as evaluating the algorithm as described with data in the clear.

\textsc{BlockedMultShiftAvg} stands for the blocked matrix multiplication to avoid overflow for fixed-point numbers described towards the end of \secref{sec:icml:tc}.
Note that it already contains the division by~$n$.
The averaging within the blocked matrix multiplications as well as over the results thereof are done by fast bit shifts instead of slow MPC division circuits.
This is possible, because we chose all parameters such that divisions are always by powers of two.

We found the piecewise linear approximation of the sigmoid function introduced by \citet{mohassel2017secureml}
\begin{equation*}
\textsc{SigmoidApprox}(x) :=
\begin{cases}
0 &\text{if } x \le -\frac{1}{2} \eqc \\
x + \frac{1}{2} &\text{if } -\frac{1}{2} < x < \frac{1}{2} \eqc \\
1 &\text{if } x \ge \frac{1}{2} \eqp
\end{cases}
\end{equation*}
to work well in practice.

\section{Experiments}
\label{sec:icml:experiments}

\begin{table}
\centering
\caption[Dataset sizes and online timing results of MPC certification and training over 10 epochs with batch size~64]{Dataset sizes and online timing results of MPC certification and training over 10 epochs with batch size~64.}
\label{table.timing}

\begin{tabular}{lrrrrr}
\toprule
 & Adult & Bank & COMPAS & German & SQF \\
  \midrule
$n$ training examples & $2^{14}$ & $2^{15}$ & $2^{12}$ & $2^9$ & $2^{16}$
\\
$d$ features & 51 & 62 & 7 & 24 & 23
\\
$p$ sensitive attr.& 1 & 1 & 7 & 1 & 1
\\
certification & 802~ms & 827~ms & 288~ms & 250~ms & 765~ms
\\
training & 43~min & 51~min & 7~min & 1~min & 111~min \\
\bottomrule
\end{tabular}
\end{table}

The root cause for most technical difficulties pointed out in the previous section is the necessity to work with fixed-point numbers and the high computational cost of MPC.
Hence, major concerns are loss of precision and infeasible running times.
In this section, we show how to overcome both doubts and that fair training, certification and verification are feasible for realistic datasets.

\paragraph{Experimental setup and datasets.}
We work with two separate code bases.
Our Python code does not implement MPC, but allows to flexibly switch between floating and fixed-point numbers as well as exact non-linear functions and their approximations.
We used it mostly for validation and empirical guidance in our design choices.
The full MPC protocol is implemented in C++ on top of the Obliv-C garbled circuits framework \citep{zahur2015obliv} and the Absentminded Crypto Kit \citep{liback}. This is done as described in \secref{sec:icml:mpc} for the Lagrangian multiplier technique.
It accurately mirrors the computations performed by the first implementation on encrypted data.
Except for the timing results in Table~\ref{table.timing}, all comparisons with floating-point numbers or non-linearities were done with the versatile Python implementation.

All our experiments use a batch size of~64, a fixed number of epochs scaling inversely with dataset size~$n$ (such that we always perform roughly 15,000 gradient updates), fixed learning rates of~$\eta_{\btheta} = 10^{-4}, \eta_{\blambda} = 0.05$, and an annealing schedule for~$\nicefrac{1}{t}$ in the interior point logarithmic barrier method as described by \citet{boydsbook}.
The weights for the gradients of the regular binary cross entropy loss (BCE) and the loss from the constraint terms (CON) follow the schedules
\begin{equation*}
\xi^{\mathrm{BCE}}_j = \frac{N_e}{N_e + j} \eqc
\qquad
\xi^{\mathrm{CON}}_j = \frac{N_e + 10 j}{N_e} \eqp
\end{equation*}
Weight decay, adaptive learning rate schedules, and momentum neither consistently improved nor impaired training.
Therefore, all reported numbers were achieved with vanilla SGD, for fixed learning rates, and without any regularization.
After extensive testing on all datasets, we converged to a fixed-point representation with 16 bits for the integer and fractional part respectively.
The smaller the number of bits, the faster the MPC implementation and the higher the risk of loss of precision or over- and underflows.
We found~16 bits to be the minimally needed precision for all our experiments to work.

We consider~5 real world datasets, namely the adult (\emph{Adult}), German credit (\emph{German}), and bank market (\emph{Bank}) datasets from the UCI machine learning repository \citep{uci}, the stop, question and frisk 2012 dataset (\emph{SQF}),\footnote{\url{https://perma.cc/6CSM-N7AQ}} and the COMPAS dataset \citep{Angwin2016} (\emph{COMPAS}).
For practical purposes (see \secref{sec:icml:tc}), we subsample~$2^i$ examples from each dataset with the largest possible~$i$, see
Table~\ref{table.timing}.
Moreover, we also run on synthetic data, generated as described by \citet[Section 4.1]{Zafar2017}, as it allows us to control the correlation between the sensitive attributes and the class labels.
It is thus well suited to observe how different optimization techniques handle the fairness accuracy trade-off.
For comparison we use the SLSQP approach described in \secref{sec:icml:tc} as a baseline.
We run all methods for a range of constraint values in~$[10^{-4}, 10^0]$ and a corresponding range for SLSQP.

In the plots in this section, discontinuations of lines indicate failed experiments.
The most common reasons are overflow and underflow for fixed-point numbers, and instability due to exploding gradients.
Plots and analyses for the remaining datasets can be found in Appendix~\ref{chap:icml:appendix}.

\paragraph{Comparing optimization techniques.}
First we evaluate which of the three optimization techniques works best in practice.
Figure~\ref{fig:ppercent_acc} shows the test set accuracy over the constraint value.
By design, the synthetic dataset exhibits a clear trade-off between accuracy and fairness.
The {\color{py-3-3}Lagrange} technique closely follows the (dotted) baseline from \citet{Zafar2017}, whereas {\color{py-3-1}iplb} performs slightly worse (and fails for small $c$).
Even though the {\color{py-3-2}projected} gradient method formally satisfies the proxy constraint for the $p\%$ rule, it does so by merely shrinking the parameter vector $\btheta$, which is why it also fails for small $c$.
We analyze this behavior in more detail in Appendix~\ref{chap:icml:appendix}.

The COMPAS dataset is the most challenging as it contains 7 sensitive attributes, one of which has only 10 positive instances in the training set.
Since we enforce the fairness constraint individually for each sensitive attribute (we randomly picked one for visualization), the classifier tends to collapse to negative predictions.
All three methods maintain close to optimal accuracy in the unconstrained region, but collapse more quickly than SLSQP.
This example shows that the $p\%$-rule proxy itself needs careful interpretation when applied to multiple sensitive attributes simultaneously and that our SGD based approach seems particularly prone to collapse in such a scenario.
On the Bank dataset, accuracy increases for {\color{py-3-1}iplb} and {\color{py-3-3}Lagrange} when the constraint becomes active as $c$ decreases until they match the baseline.
Determining the cause of this---perhaps unintuitive---behavior requires further investigation.
We currently suspect the constraint to act as a regularizer.
The {\color{py-3-2}projected} gradient method is unreliable on the Bank dataset.

\begin{figure}
\centering
\includegraphics[width=\textwidth]{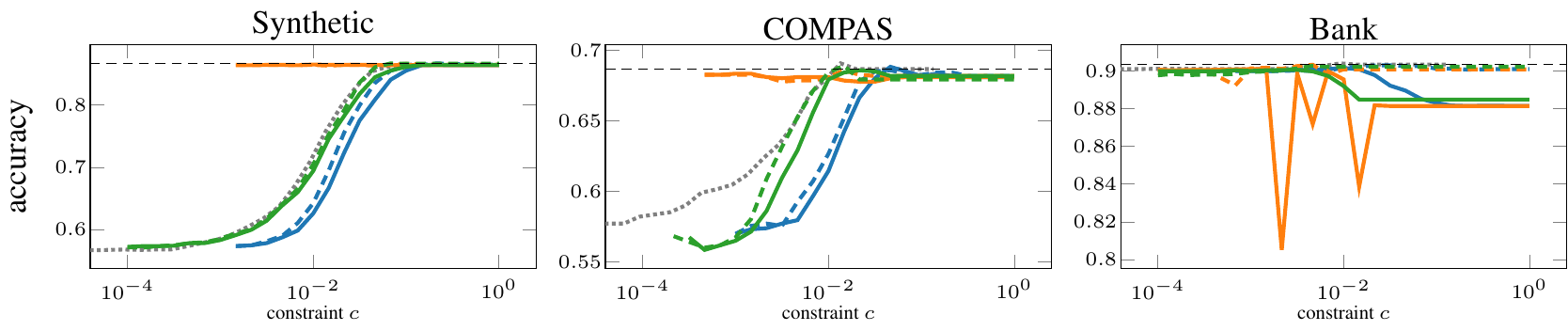}
\caption[Test set accuracy over the $p\%$ value for different optimization methods and datasets]{Test set accuracy over the $p\%$ value for different optimization methods ({\color{py-3-1} blue: iplb}, {\color{py-3-2} orange: projected}, {\color{py-3-3} green: Lagrange}) and either no approximation (\emph{continuous}) or a piecewise linear approximation (\emph{dashed}) of the sigmoid using floating-point numbers. The gray dotted line is the baseline (see \secref{sec:icml:tc}) and the black dashed line is unconstrained logistic regression (from scikit-learn).}
\label{fig:ppercent_acc}
\end{figure}

\begin{figure}
\centering
\includegraphics[width=\textwidth]{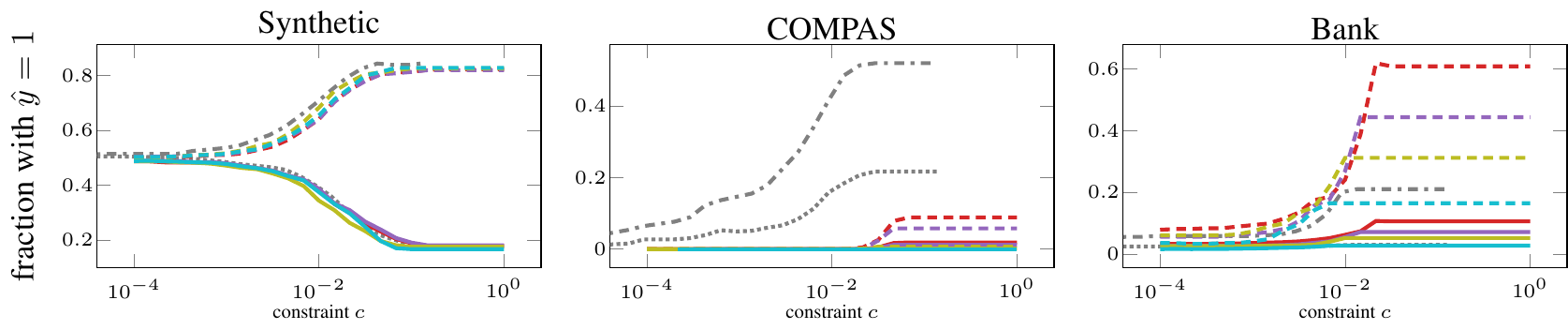}
\caption[The fractions of individuals in both groups receiving positive outcomes for different optimization settings]{The fraction of people with $z=0$ (\emph{continuous/dotted}) and $z=1$ (\emph{dashed/dash-dotted}) who get assigned positive outcomes
({\color{py-4-1}red: no approx. + float},
{\color{py-4-2}purple: no approx. + fixed},
{\color{py-4-3}yellow: pw linear + float},
{\color{py-4-4}turquoise: pw linear + fixed},
{gray: baseline}).
As the slack $c$ decreases, the fairness constraint is tightened and we observe dashed and continuous lines of the same color moving closer together.
This indicates that as the fairness constraint is enforced, equal fractions of both populations receive positive outcomes.}
\label{fig:ppercent_fair}
\end{figure}

Empirically, the Lagrangian multiplier technique is most robust with maximal deviations of accuracy from SLSQP of $<4$\% across the 6 datasets and all constraint values.
We substantiate this claim in Appendix~\ref{chap:icml:appendix}.
For the rest of this section we only report results for Lagrangian multipliers.
Figure~\ref{fig:ppercent_acc} also shows that using a piecewise linear approximation as described in \secref{sec:icml:tc} for the logistic function does not spoil performance.

\paragraph{Fair training, certification and verification.}
Figure~\ref{fig:ppercent_fair} shows how the fractions of users with positive outcomes in the two groups ($z=0$ is continuous and $z=1$ is dashed) are gradually balanced as we decrease the fairness constraint~$c$.
These plots can be interpreted as the degree to which disparate impact is mitigated as the constraint is tightened.
The effect is most pronounced for the synthetic dataset by construction.
As discussed above, the collapse for the COMPAS dataset occurs faster than for SLSQP due to the constraints from multiple sensitive attributes.
In the Bank dataset, for large $c$---before the constraint becomes active---the fractions of positive outcomes for $z=1$ differ, which is related to the slightly suboptimal accuracy at large $c$ that needs further investigation.
However, as the constraint becomes active, the fractions are balanced at a similar rate as the baseline.
Overall, our Lagrangian multiplier technique with fixed point numbers and piecewise linear approximations of non-linearities robustly manages to satisfy the p\%-rule proxy at similar rates as the baseline with only minor losses in accuracy on all but the challenging COMPAS dataset.

In Table~\ref{table.timing} we show the online running times of 10 training epochs on a laptop computer.
While training takes orders of magnitudes longer than a non-MPC implementation, our approach still remains feasible and realistic.
We use the one time offline precomputation of multiplication triples described and timed in \citet[Table 2]{mohassel2017secureml}.
As pointed out in \secref{sec:icml:mpc}, certification of a trained model requires checking whether $\F(\btheta) > 0$. We already perform this check at least once for each gradient update during training.
It only takes a negligible fraction of the computation time, see Table~\ref{table.timing}.
Similarly, the operations required for certification stay well below one second.

\section{Conclusion}
\label{sec:icml:conclusion}

Real world fair learning has suffered from a dilemma:
in order to enforce fairness, sensitive attributes must be examined;
yet in many situations, users may feel uncomfortable in revealing these attributes, or modelers may be legally restricted in collecting and utilizing them.
By introducing recent methods from MPC, and extending them to handle linear constraints as required for various notions of fairness, we have demonstrated that it is practical on real-world datasets to:
\begin{enumerate*}[label=\alph*)]
\item certify and sign a model as fair;
\item learn a fair model; and
\item verify that a fair-certified model has indeed been used;
\end{enumerate*}
all while maintaining cryptographic privacy of all users' sensitive attributes.
Connecting concerns in privacy, algorithmic fairness and accountability, our proposal empowers regulators to provide better oversight, modelers to develop fair and private models, and users to retain control over data they consider highly sensitive.

We have demonstrated the practicability of private and fair model training, certification and verification using MPC as described in Figure~\ref{figure.MPC}.
Using the methods and tricks introduced in \secref{sec:icml:tc}, we can overcome accuracy as well as over- and underflow concerns due to fixed-point numbers.
Offline precomputation combined with a fast C++ implementation yield viable running times for reasonably large datasets on a laptop computer even though we acknowledge that these are still considerably slower than times obtained without protection via secure MPC.
Such trade-offs and as well as potential tensions in the interplay of different notions of fairness, privacy and security pose important challenges for real-world applications, opening up  new fields of research that go beyond the traditional assumptions of fair machine learning.

\chapter{Decisions versus predictions}
\label{chap:decisions}

\graphicspath{{figs/chap7/}}

In this chapter, we go beyond the assumption that labels are always available in the training data and show that naive predictive modeling using only the data at hand is therefore not the optimal way to arrive at decisions.
To consistently learn accurate predictive models, one needs access to ground truth labels.
Unfortunately, in practice, labels may only exist conditional on certain decisions---if a loan is denied, there is not even an option for the individual to pay back the loan.
Hence, the observed data distribution depends on how decisions are being made.
In this chapter, we show that in this \emph{selective labels} setting, learning a predictor directly only from available labeled data is suboptimal in terms of both fairness and utility.
To avoid this undesirable behavior, we propose to directly learn decision policies that maximize utility under fairness constraints and thereby take into account how decisions affect which data is observed in the future.
Our results suggest the need for a paradigm shift in the context of fair machine learning from the currently prevalent idea of simply building predictive models from a single static dataset via risk minimization, to a more interactive notion of ``\emph{learning to decide}''.
In particular, such policies should not entirely neglect part of the input space, drawing connections to explore/exploit tradeoffs in reinforcement learning, data missingness, and potential outcomes in causal inference.
Experiments on synthetic and real-world data illustrate the favorable properties of learning to decide in terms of utility and fairness.

The main content of this chapter has been published in the following paper:

\fbox{\parbox{\textwidth}{
\pubitem{Fair Decisions Despite Imperfect Predictions}
{Niki Kilbertus, Manuel Gomez-Rodriguez, Bernhard Schölkopf, Krikamol Muandet, Isabel Valera}
{International Conference on Artificial Intelligence and Statistics (AISTATS), 2020}
{https://arxiv.org/abs/1902.02979}
{https://github.com/nikikilbertus/fair-decisions}
{}
}}

\section{Introduction}
\label{sec:intro}

We start by revisiting some common settings of consequential decisions that may be (partially) automated.
In pretrial release decisions, a judge may consult a learned model of the probability of recidivism to decide whether to grant bail or not.
In loan decisions, a bank may decide whether or not to offer a loan based on learned estimates of the credit default probability.
In fraud detection, an insurance company may flag suspicious claims based on a machine learning model'{}s predicted probability that the claim is fraudulent.
In all these scenarios, the goal of the decision maker (bank, law court, or insurance company) may be to take decisions that maximize a given utility function.
Since such a utility is a function of the entire policy, rather than an individual loss term for each prediction in isolation, it may encode preferences that go beyond merely ``correct or incorrect'' for a predictive classification task.
For example, the utility may accommodate various fairness and diversity considerations, or enforce actions that could improve the wellbeing of certain individuals and groups in the long run---even when such actions are not warranted by the directly observed outcomes.
In contrast, the goal of a supervised predictive machine learning model is solely to provide accurate predictions given the available training set, typically under the implicit assumption that the training data is an i.i.d.\ sample from the distribution encountered during test time.
Such a simplistic approach, which we refer to as \emph{learning to predict}, ignores that once decisions are based on these predictions, they may interact with the data collection or have direct impact on the underlying distribution relevant during deployment \citep{perdomo2020performative}.

Nevertheless, most work on fair machine learning does not distinguish between decisions and label predictions, which also leads to a perceived trade-off between fairness and accuracy or performance.
This stems from the fact that viewing ``incorrect'' predictions as positively desirable, e.g., because they may promote fairness, is incompatible with the standard goal of minimizing a predictive loss in supervised learning.
From that viewpoint, an accurate prediction is equivalent---or at least directly translates---to a good decision.
Only recently has the distinction been made explicit, typically emphasizing that \emph{not predicting historically recorded labels correctly} can often be part of the goal when it comes to fairness \citep{corbett2017algorithmic,kleinberg2017human,mitchell2018prediction,valera2018enhancing}.
We also remind the reader of the discussion of some fundamental challenges, in particular long- versus short-term goals, in \secref{sec:goals}.
This recent line of work has shown that if a predictive model achieves perfect prediction accuracy, \emph{deterministic threshold rules}, which derive decisions deterministically from the predictive model by thresholding, indeed achieve maximum utility under various fairness constraints.
At first, this lends support to focusing on deterministic threshold rules and seemingly justifies using predictions and decisions interchangeably.

However, in many practical scenarios including the ones described in the beginning of this section, the decision determines whether a label is realized or not---if bail (a loan) is denied, there is not even an option for the individual to reoffend (pay back the loan).
This problem has been referred to by \citet{lakkaraju2017selective} as \emph{selective labels}.
As a consequence, the labeled data used to train predictive models often depend on the decisions taken, which likely leads to suboptimal performance.
For example, a racist initial policy may categorically reject applicants from a certain demographic group.
Therefore, no data about these individuals is collected and there are no guarantees for how a predictive model may extrapolate into this unseen region of the input space.
Indeed, deterministic threshold rules using even slightly imperfect predictive models can be far from optimal \citep{woodworth2017learning}.
This negative result raises the following question: \textit{Can we do better if we learn directly to decide rather than to predict?}\footnote{We remark that our notion of learning to decide is not immediately related to (Bayesian) decision theory. Moreover, for this thesis, we restrict the notion of learning to predict to simple point estimates (categorical predictions) from risk minimization based on available data. Crafting ``predictions'' more carefully, for example by regarding a data-missingness model or proper uncertainty estimates, may not be prone to the same issues. However, most works on fairness in machine learning have been staged within our simplistic ``learning to predict'' framework.}
Here, by ``learning to decide'' we mean learning to maximize a utility of a decision policy (rather than predictions) when deployed under test time conditions (rather than merely trained on observed data).

In the present chapter, we first articulate how the ``learning to predict'' approach fails in a utility maximization setting (with fairness constraints) that accommodates a variety of real-world applications, including those mentioned previously.
We show that label data gathered under deterministic rules (e.g., prediction based threshold rules) are neither sufficient to improve the accuracy of the underlying predictive model, nor the utility of the decision making process.
We then demonstrate how to overcome this undesirable behavior using a particular family of stochastic decision rules and introduce a simple gradient-based algorithm to learn them from data.
Experiments on synthetic and real-world data illustrate our theoretical results and show that, under imperfect predictions, \emph{learning to predict} is inferior to \emph{learning to decide}.

\xhdr{Related work}
The work most closely related to ours analyzes the long-term effects of consequential decisions informed by data-driven predictive models on underrepresented groups \citep{hucheng2018,liu18c,mouzannar2019fair}.
However, this line of work focuses mainly on the evolution of several measures of well-being under a perfect predictive model and neglecting the data collection phase.
In contrast, we focus on analyzing how to improve a suboptimal decision process when labels exist only for positive decisions.
Potential issues arising from neglecting the data collection process have also been highlighted in a survey of what machine learning practitioners in industry actually need to enforce fairness \citep{holstein2018improving}.
\citet{dimitrakakis2019bayesian} similarly point out how attempts of fair machine learning may fail when there is uncertainty about the underlying probabilistic model of the world.
They achieve fairness in such settings by deploying a Bayesian approach that aims at enforcing fairness in all possible models weighted by their probability given the current information.
In contrast, we focus on analyzing how to improve a suboptimal decision process when labels exist only for positive decisions.
We also build on previous work on counterfactual inference and policy learning \citep{athey2017efficient,ensign2017decision,gillen2018online,heidari2018preventing,joseph2016fairness,jung2018algorithmic,kallus2018balanced,kallus2018residual,lakkaraju2017learning}.
In these settings, the decision typically determines which of the potential outcomes is observed and the focus is on confounders that affect both the decision and the outcome \citep{Rubin05:PO}.
In contrast, in our approach the decision determines whether there will be an outcome at all, but there is no unobserved confounding.
Two notable exceptions are by \citet{kallus2018residual} and \citet{ensign2017decision}, which also consider limited feedback.
However, \citet{kallus2018residual} focus on designing unbiased estimates for fairness measures, rather than learning how to decide.
\citet{ensign2017decision} assume a deterministic mapping between features and labels, which allows them to reduce the problem to the apple tasting problem \citep{helmbold2000apple}.
Remarkably, in their deterministic setting, they also conclude that the optimal decisions should be stochastic.

Unlike in the fairness literature, where deterministic policies dominate \citep{corbett2017algorithmic,valera2018enhancing,Meyer2018objecting}, stochastic policies are often necessary to ensure adequate exploration \citep{Silver14:DPG} in contextual bandits \citep{Dudik11:DoublyRobust,Langford08:ES,Agarwalb14:Monster} and reinforcement learning \citep{jabbari2016fairness,Sutton98:RL}.
However, the typical problem setting there differs fundamentally from ours and typically neither fairness constraints nor selective labels are taken into account.
A recent notable exception is \citet{joseph2016fairness}, initiating the study of fairness in multi-armed bandits, however, using a fairness notion orthogonal to the observational group matching criteria we consider in our work, and ignoring the selective labels problem.

\section{Decisions from imperfect predictive models}
\label{sec:formulation}

First, we briefly recap notation.
Let $\xspace \subseteq \bR^d$ be the feature domain, $\sspace = \{0,1\}$ the range of sensitive attributes, and $\outspace = \{0,1\}$ the set of ground truth labels.
A \emph{decision rule} or \emph{policy}\footnote{We use the terms \emph{decision rule}, \emph{decision making process} and \emph{policy} interchangeably in this chapter.}
is a mapping $\pi: \xspace \times \sspace \to \mathcal{P}(\{0,1\})$ that maps an individual'{}s feature vector and sensitive attribute to a probability distribution over \emph{decisions} $d \in \{0,1\}$.
Note that instead of a deterministic predictor $\hY$ we are now considering a stochastic policy for decisions instead of predictions.
We sample $x, z$ and $y$ from a ground truth distribution $\dP(X,Z,Y) = \dP(Y\given X,Z) \dP(X,Z)$.
Decisions $d$ are sampled from a policy $d \sim \pi(D\given x, z)$, where we often write $\pi(x,z)$ for $\pi(D\given x, z)$ and $\pi(D=1 \given x, z)$ for the probability of a positive decision given features~$x, z$.
The decision determines whether the label $y \sim \dP(Y\given X,Z)$ comes into existence.
In loan decisions, the feature vector $x$ may include salary, education, or credit history;
the sensitive attribute $z$ may indicate sex;
a loan can be granted ($d = 1$) or denied ($d = 0$);
and the label $y$ indicates repayment ($y = 1$) or default ($y=0$) \emph{upon receiving a loan}.

Inspired by \citet{corbett2017algorithmic}, we measure the \emph{utility} as the expected overall profit provided by the policy with respect to the distribution $\dP$, i.e.,
\begin{equation}\label{eq:utility}
  u_{\dP}(\pi) := \E_{x,z,y \sim \dP,\, d\sim \pi(x, z)} \left[y\, d - c\, d \right]
  = \E_{x,z \sim \dP} \left[\pi(D = 1 \given x, z) (\prob(Y = 1 \given x, z) - c) \right] \eqc
\end{equation}
where $c \in (0, 1)$ reflects economic considerations of the decision maker.
For example, in a loan scenario, the utility gain is $(1-c)$ if a loan is granted and repaid, $-c$ if a loan is granted but the individual defaults, and zero if the loan is not granted.
One could think of adding a term for negative decisions of the form $g(y) (1-d)$ for some given definition of $g$, however, we would not be able to compute such a term due to the selective labels, except for constant $g$.
Therefore, without loss of generality, we assume that $g(y)= 0$ for all $y$, because any non-zero constant $g$ can easily be absorbed in our framework.

For fairness considerations, we define the \emph{$f$-benefit for group $z \in \{0,1\}$} with respect to the distribution $\dP$ by
\begin{equation*}
  b_{\dP}^z(\pi) := \E_{x, y \sim \dP(X, Y \given z),\, d\sim \pi(x, z)} [f(d, y)] \eqc
\end{equation*}
with $f: \{0, 1\} \times \{0, 1\} \to \bR$.
Note that common observational group matching fairness criteria can be expressed as $b_{\dP}^0(\pi) = b_{\dP}^1(\pi)$ for different choices of $f$ (and perhaps conditioning on specific values of $y$ or $d$ corresponding to criteria of separability or sufficiency respectively).
For simplicity, we will focus on demographic parity (or no disparate impact), which simply amounts to $f(d, y) = d$.

Under perfect knowledge of $\dP(Y \given x, z)$, the policy maximizing the above utility subject to the group benefit fairness constraint $b_{\dP}^0(\pi) = b_{\dP}^1(\pi)$ is a deterministic threshold rule \citep{corbett2017algorithmic}\footnote{Here, $\B{1}[\bullet]$ is $1$ if the predicate $\bullet$ is true and $0$ otherwise.}
\begin{equation}\label{eq:detthresh}
  \pi^*(D = 1 \given x, z) = \B{1}[\prob(Y = 1 \given x, z) \geq c_z] \eqc
\end{equation}
where we allow for group specific cost factors $c_0, c_1$ such that $b_{\dP}^0(\pi) = b_{\dP}^1(\pi)$.
Without fairness constraints, we simply have $c_0 = c_1 = c$.
However, as discussed by \citet{woodworth2017learning}, in practice we typically do not have access to the true conditional distribution $\dP(Y \given x, z)$, but instead to an imperfect predictive model $Q(Y \given x, z)$ trained on a finite training set.
Such a predictive model can similarly be used to implement a deterministic threshold rule as
\begin{equation}\label{eq:detthreshQ}
  {\pi}_{Q}(D = 1 \given x, z) = \B{1}[Q(Y = 1 \given x, z) \geq c] \eqp
\end{equation}
Here, the predictor $Q(Y = 1 \given x, z) \approx \dP(Y = 1 \given x, z) - \delta_z$, with $\delta_z = c_z - c$, directly incorporates the fairness constraint, i.e., it is trained to maximize predictive power subject to the fairness constraint.
In this context, \citet{woodworth2017learning} have shown that this approach often leads to better performance than post-processing a potentially unfair predictor as proposed by \citet{Hardt2016}.
Unfortunately, they have also shown that, because of the mismatch between $Q(Y = 1 \given x, z)$ and $\dP(Y = 1 \given x, z) - \delta_z$, the resulting policy $\pi_Q$ will usually still be suboptimal in terms of both utility and fairness.
To make things worse, due to the selective labeling, the data points $x, z, y$ observed under a given policy $\pi_0$ are not i.i.d.\ samples from the ground truth distribution $\dP(X, Z, Y)$, but instead from the weighted distribution
\begin{equation}\label{eq:imperfect-p}
  \dP_{\pi_0}(X,Z,Y) \propto \dP(Y \given X,Z)\, {\pi_0}(D=1\given X,Z)\, \dP(X,Z) \eqp
\end{equation}
Consequently, if $\pi_0$ is not optimal, i.e., $\pi_0 \ne \pi^*$, the necessary i.i.d.\ assumption for consistency results of empirical risk minimization is violated, which may also be one reason for a common observation in fairness, namely that predictive errors are often systematically larger for minority groups \citep{Angwin2016}.
In the remainder, we will say that the distributions $\dP_{\pi_0}(X,Z,Y)$ and $\dP_{\pi_0}(X,Z)$ are \emph{induced} by the policy $\pi_0$.
In the next section, we study how to learn the optimal policy, potentially subject to fairness constraints, if the data is collected from an initial faulty policy $\pi_0$.

\section{From deterministic to stochastic policies}
\label{sec:sequential}

Consider a class of policies $\Pi$, within which we want to maximize utility, as defined in eq.~\eqref{eq:utility} subject to the group benefit fairness constraint $b_{\dP}^0(\pi) = b_{\dP}^1(\pi)$.
We formulate this as an unconstrained optimization with an additional penalty term, namely to maximize
\begin{equation}\label{eq:policy_learning}
  v_{\dP}(\pi) := u_{\dP}(\pi) - \frac{\lambda}{2} (b_{\dP}^0(\pi) - b_{\dP}^1(\pi))^2
\end{equation}
over $\pi \in \Pi$ under the assumption that we do not have access to samples from the ground truth distribution $\dP(X, Z, Y)$, which $u_{\dP}(\pi)$ and $b_{\dP}^z(\pi)$ depend on.
Instead, we only have access to samples from a distribution $\dP_{\pi_0}(X, Z, Y)$ induced by a given initial policy $\pi_0$ as in eq.~\eqref{eq:imperfect-p}.
We first analyze this problem for deterministic threshold rules, before considering general deterministic policies, and finally also general stochastic policies.

\subsection{Deterministic policies}

Assume the initial policy $\pi_0$ is a given deterministic threshold rule and $\Pi$ is the set of all deterministic threshold rules, which means that each $\pi \in \Pi$ (and $\pi_0$) is of the form eq.~\eqref{eq:detthreshQ} for some predictive model $Q(Y \given x, z)$.
Given a hypothesis class of predictive models $\Qcal$, we reformulate eq.~\eqref{eq:policy_learning} to maximize
\begin{equation}\label{eq:policy_learning-p}
  v_{\dP}(\pi_Q) := u_{\dP}(\pi_Q) - \frac{\lambda}{2} (b_{\dP}^0(\pi_Q) - b_{\dP}^1(\pi_Q))^2
\end{equation}
over $Q \in \Qcal$, where the utility and the benefits for $z\in \{0,1\}$ are simply $u_{\dP}(\pi_Q) = \E_{x, z, y \sim \dP} [ \B{1}[Q(Y = 1\given x, z) \ge c] (y - c)]$ and $b_{\dP}^z(\pi_Q) = \E_{x, z, y \sim \dP} [f(\B{1}[Q(Y = 1\given x, z) \ge c], y)]$.
Note that eq.~\eqref{eq:policy_learning} has a unique optimum $\pi^*$ (up to differences on sets of measure zero).
Therefore, if $\pi^* \in \Pi$ (the set of all deterministic threshold rules), eq.~\eqref{eq:policy_learning-p} will also reach this optimum if $\Qcal$ is rich enough.
However, the optimal predictor $Q^*$ may not be unique, because the utility and the benefits are not sensitive to the precise values of $Q(Y = 1 \given x, z)$ above or below the threshold $c$.

If we only have access to samples from the distribution $\dP_{\pi_0}$ induced by some $\pi_0 \ne \pi^*$, we may choose to simply learn a predictive model $Q_0^* \in \Qcal$ that empirically maximizes the objective $v_{\dP_{\pi_0}}(\pi_Q)$, where the utility and the benefits are computed with respect to the induced distribution $\dP_{\pi_0}$.
However, the following negative result shows that, under mild conditions, $Q^*_0$ leads to a suboptimal deterministic threshold rule.
\begin{proposition}\label{prop:limitation}
If there exists a subset $\mathcal{V} \subset \xspace \times \sspace$ of positive measure under $\dP$ such that $\dP(Y=1 \given \mathcal{V}) \ge c$ and $\dP_{\pi_0} (Y = 1\given \mathcal{V}) < c$, then there exists a maximum $Q_0^* \in \mathcal{Q}$ of $v_{\dP_{\pi_0}}$ such that $v_{\dP}(\pi_{Q_0^*}) < v_{\dP}(\pi_{Q^*})$.
\end{proposition}
\begin{proof}
First, note that any deterministic policy $\pi$ is fully characterized (up to differences of measure zero) by the sets $W_d(\pi) = \{(x,z) \given \pi(D = 1 \given x,z) = d\}$ for $d \in \{0, 1\}$.
For a deterministic threshold rule $\pi_Q$, we write $W_d(Q) = \{(x,z) \given \B{1}[Q(Y = 1\given x,z) > c] = d\} = W_d(\pi_Q)$.
By definition, we have that $v(\pi_{Q}) \le v(\pi_{Q^*})$.
We note that whenever the symmetric difference between the sets $W_d(Q)$ and $W_d(Q^*$), $W_d(Q) \Delta W_d(Q^*)$, has positive inner measure (induced by $\dP$) for $d \in \{0,1\}$ and a $Q \in \Qcal$, we have $v(\pi_{Q}) \ne v(\pi_{Q^*})$ and thus $v(\pi_{Q}) < v(\pi_{Q^*})$.
Thus it only remains to show that $W_d(Q^*) \Delta W_d(Q_0^*)$ has positive inner measure for $d \in \{0,1\}$.
Since $\dP(Y = 1 \given \mathcal{V}) \ge c$ by assumption, we have $\mathcal{V} \subset W_1(Q^*)$.
At the same time, because of $\dP_{\pi_0}(Y = 1 \given \mathcal{V}) < c$ by assumption, we have $\mathcal{V} \cap W_1(\pi_0) = \emptyset$.
Finally, we note that for any $Q\in \mathcal{Q}$, we have that $v_{\dP_{\pi_0}}(Q) = v_{\dP_{\pi_0}}(Q \cdot \chi_{W_1(\pi_0)})$, where $\chi_{\bullet}$ is the indicator function on the set $\bullet$.
Therefore, we can choose a $Q^*_0$ maximizing $v_{\dP_{\pi_0}}$ such that $W_1(Q^*_0) \subset W_1(\pi_0)$ and thus $\mathcal{V} \cap W_1(Q^*_0) = \emptyset$.
Therefore $\mathcal{V} \subset W_1(Q_0^*) \Delta W_1(Q^*)$ and $\mathcal{V}$ has positive measure under $\dP$ by assumption.
Thus $W_d(Q_0^*) \Delta W_d(Q^*)$ has positive inner measure and we conclude $v_{\dP}(\pi_{Q_0^*}) < v_{\dP}(\pi_{Q^*})$.
\end{proof}
\begin{figure}
\centering
\includegraphics[width=0.32\columnwidth]{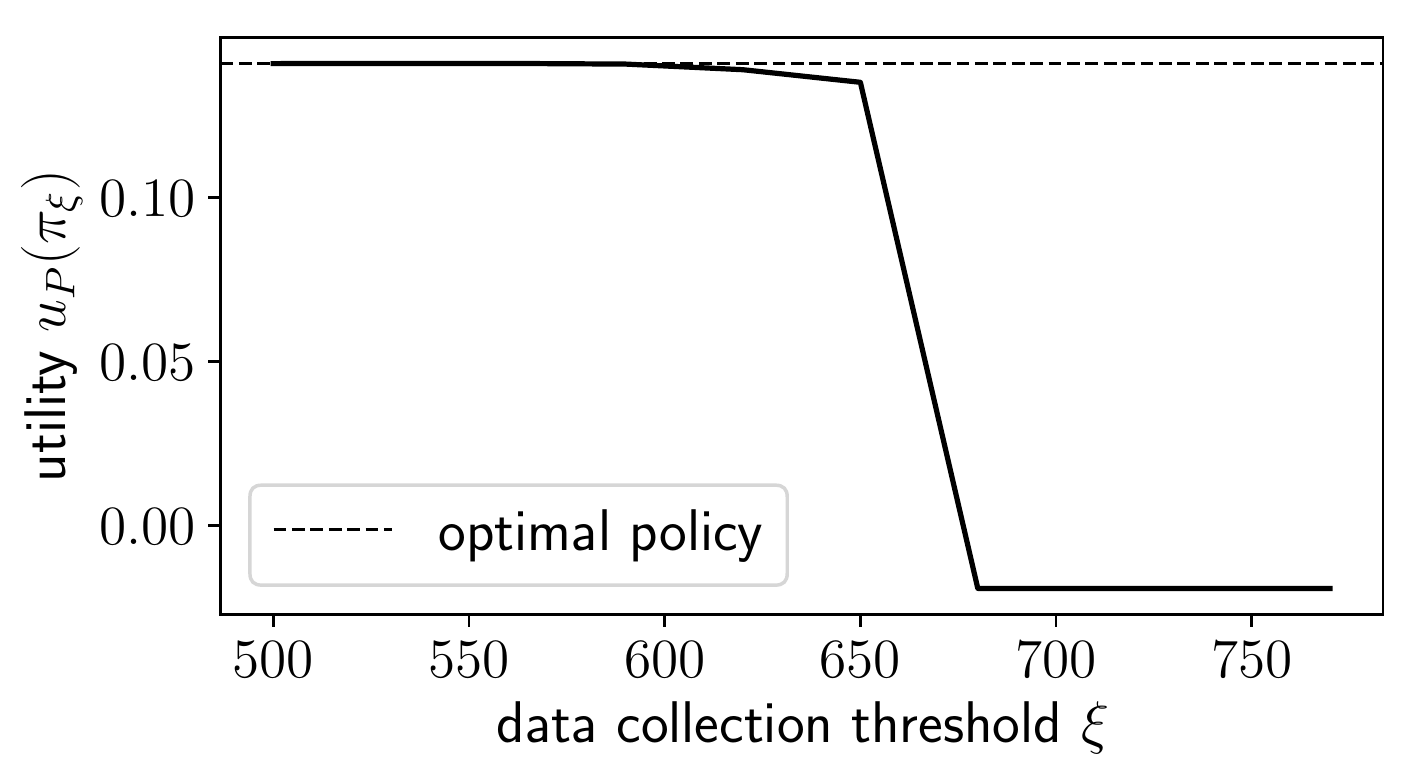}%
\includegraphics[width=0.32\columnwidth]{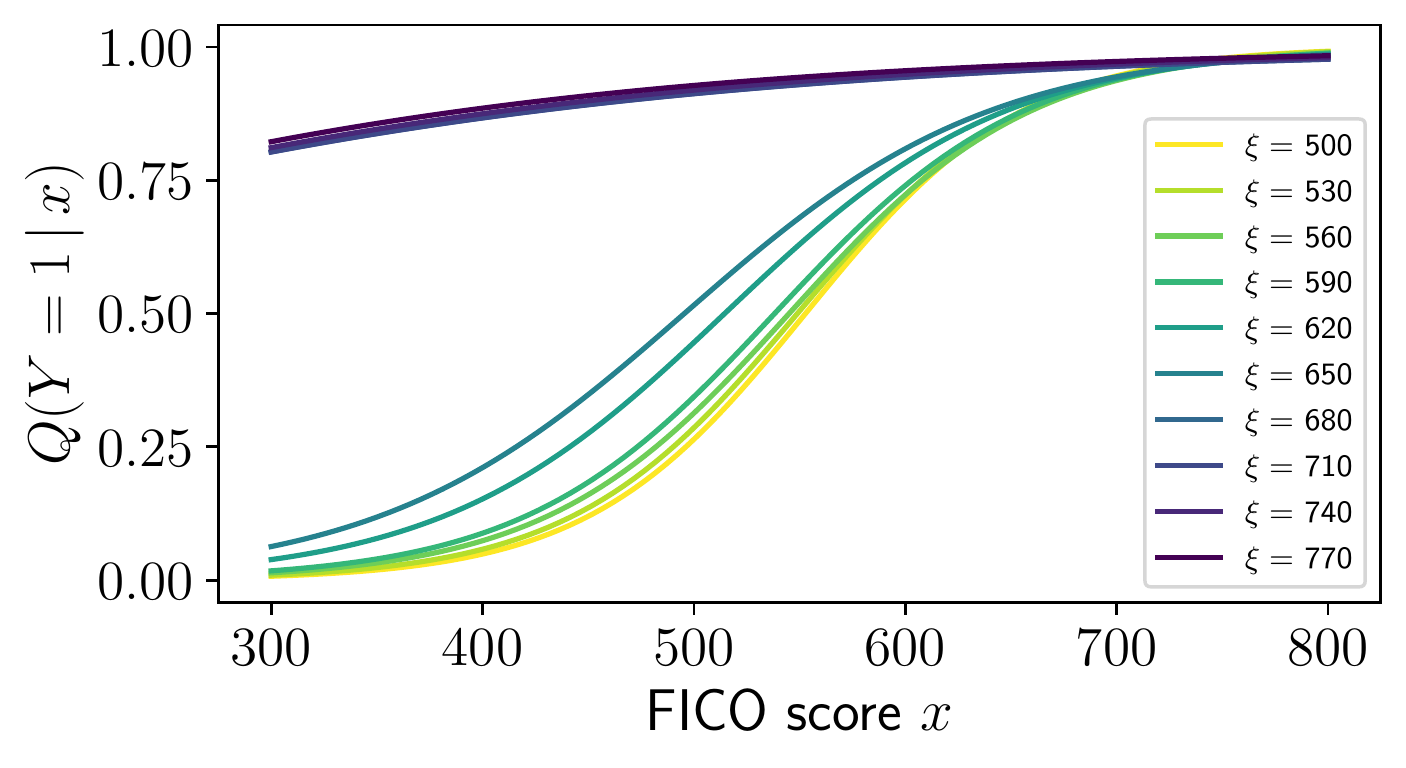}%
\includegraphics[width=0.32\columnwidth]{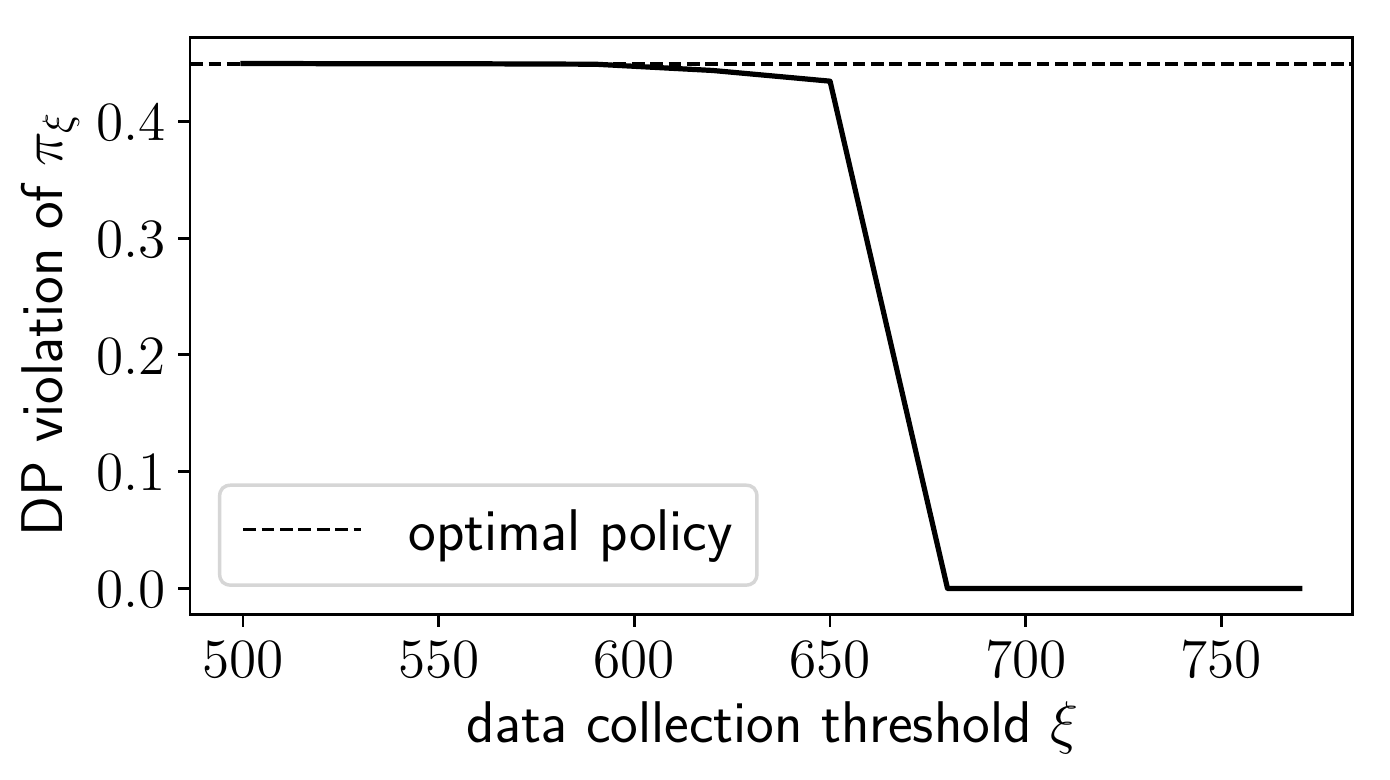}%
\caption[Empirical results of utility and fairness of the introductory lending example in a selective labels setting]{%
We show the utility (left) and the predictive models $Q_{\xi}$ learned from data collected with an initial threshold of $\xi$ (middle).
Finally, we present the violation of demographic parity (right) of threshold decision rules $\pi_{\xi}$ learned from data collected with an initial threshold of $\xi$.
Harsh data collection policies (i.e., large $\xi$)---while achieving demographic parity---render the learned policies useless in terms of utility.
}
\label{fig:fico-example}
\end{figure}
\textbf{Lending example.} We briefly illustrate this result in a lending example based on FICO credit score data as described in \citet{Hardt2016}.
Such single feature scenarios are highly relevant for score-based decision support systems where full training data and the functional form of the score are often not available (e.g., also for pretrial risk assessment).
For any score that is strictly monotonic in the true success rate, the optimal policy is simply to threshold the score.
This lends additional support to score-based systems.

Here, we can generate new scores for a given group via inverse transform sampling from the known cumulative distribution functions.
We consider $80$\% white and $20$\% black applicants.
A hypothetical new bank that has access to FICO scores $x \in \xspace := \{300, \ldots, 820\}$, but not to the corresponding repayment probabilities may expect to be profitable if at least $70$\% of granted loans are repaid, i.e., $c = 0.7$.
A risk-averse lender may initially choose a high score threshold $\xi \in \xspace$ and employ the decision rule $\B{1}[x > \xi]$.
After collecting repayment data $\cD^{(\xi)} := \{(x_i, y_i)\}_{i = 1}^n$ with this initial threshold, they learn a model $Q_{\xi}(Y = 1 \given x)$ and then decide based on $\pi_{\xi}(D = 1 \given x) = \B{1}[Q_{\xi}(Y = 1 \given x) > c]$.

For a range of initial data collection score thresholds $\xi \in [500, 800]$, we sample 10,000 scores from the specified population ($80$\% white, $20$\% black) via inverse transform sampling given the cumulative distributions functions over scores of the two groups.
The relatively large number of examples is chosen to illustrate that the negative result is not a consequence of insufficient data.
We then fit an L2 regularized logistic regression model to each of these datasets using 5-fold cross validation to select the regularization parameter.
This results in a predictive model $Q_{\xi}$ for each initial data collection threshold $\xi$.
For each of these models we construct the decision rule $\pi_{\xi}(D = 1 \given x) = \B{1}[Q_{\xi}(Y = 1 \given x) > c]$, with $c=0.7$.
We then estimate utility and fairness violation of demographic parity on a large sample from the entire population (one million examples).

In Figure~\ref{fig:fico-example} we show how the initial data collection threshold $\xi$ affects utility and fairness of the resulting predictive model-based decision rule.
Conservatively high initial thresholds of $\xi \ge 650$ lead to essentially useless decisions $\pi_{\xi}$, because of imperfect prediction models regardless of how much data was collected.
More lenient initial policies can result in near optimal decisions with improved fairness compared to the maximum utility policy for the given cost $c$ (dashed).

A simple fix seems to present itself: Do not start with high thresholds.
However, Proposition~\ref{prop:limitation} tells us that it does not matter how low we set the initial threshold, if we use it to derive deterministic decisions.
By deterministically thresholding the score, we inevitable reject a subset of the population (with positive measure) and thus can never learn whether some of them may actually repay.
This will be highlighted in the next paragraphs, showing practical impossibility results for recovering from a bad initial policy even in a sequential training setting when using deterministic decision rules.
The only way to overcome this issue in our example is not to draw a hard threshold, bat to accept applicants with some non-zero probability for \emph{every possible score}.
We will formalize this idea and its advantages in \secref{sec:exploring}.

\textbf{Impossibility results.} Supplementing the result in Proposition~\ref{prop:limitation}, we will now prove that---in certain situations---a sequence of deterministic threshold rules, fails to recover the optimal policy despite it being in the hypothesis class.
We assume that each threshold rule is of the form of eq.~\eqref{eq:detthreshQ} and its associated predictive model is trained using the data gathered through the deployment of previous threshold rules.
To this end, we consider a \emph{sequential policy learning task}, which is given by a tuple $(\pi_0, \Pi', \mathcal{A})$, where:
\begin{enumerate}[label=\alph*)]
  \item $\Pi' \subset \Pi$ is the hypothesis class of policies,
  \item $\pi_0 \in \Pi'$ is the initial policy, and
  \item $\mathcal{A}: \Pi' \times \bigcup_{i=1}^{\infty} (\xspace \times \sspace \times \outspace)^i \to \Pi'$ is an update rule.
\end{enumerate}
The update rule $\mathcal{A}$ takes an existing policy $\pi_t$ and a dataset $\cD \in (\xspace \times \sspace \times \outspace)^n$ and produces an updated policy $\pi_{t+1}$, which typically aims to improve the policy in terms of the objective function $v_{\dP}(\pi)$ in eq.~\eqref{eq:policy_learning}.
In our setting, the dataset $\cD$ is collected by deploying previous policies, i.e., from a mixture of the distributions $\dP_{\pi_{\tau}}(X, Z, Y)$ with $\tau \le t$.

Recall that for deterministic threshold policies we can partition the space $\xspace \times \sspace = W_0(\pi)\cup W_1(\pi)$ into regions of negative and positive decisions.
Then, we say an update rule is \emph{non-exploring on $\cD$} if and only if $W_0(\mathcal{A}(\pi, \cD)) \subset W_0(\pi)$.
Intuitively, this means that no individual who has received a negative decision under the old policy $\pi$ would receive a positive decision under the new policy $\mathcal{A}(\pi, \cD)$.
Common learning algorithms for classification, such as gradient boosted trees are \emph{error based}, i.e., they only change the decision function when they make errors on the training set.
As a result, they lead to non-exploring update rules on $\cD$ whenever they achieve zero error.
\begin{proposition}\label{prop:limits}
Let $(\pi_0, \Pi', \mathcal{A})$ be a sequential policy learning task, where $\Pi' \subset \Pi$ are deterministic threshold policies based on a class of predictive models, and let the initial policy be more strict than the optimal one, i.e., $W_0(\pi_0) \supsetneq W_0(\pi^*)$.
If $\mathcal{A}$ is non-exploring on any i.i.d.\ sample $\cD \sim \dP_{\pi_t}(X,Z,Y)$ with probability at least $1 - \delta_t$ for all $t \in \bN$, then
\begin{equation}
\Pr[\pi_T \neq \pi^{*}] > 1 - \sum_{t=0}^T \delta_t \quad \text{for any } T \in \bN \eqp
\end{equation}
\end{proposition}
\begin{proof}
  At each step we have
  \begin{equation*}
  \Pr[\pi_t = \pi^*]
    = \Pr[W_0(\pi_t) = W_0(\pi^*)]
   \le \Pr[W_0(\pi_t) \supset W_0(\pi^*)]
   \le \delta_t + \Pr[\pi_{t-1} = \pi^*] \eqp
 \end{equation*}
 By the assumption that $\pi_0 \ne \pi^*$, we recursively get $\prob[\pi_t = \pi^*] \le \sum_{i=0}^t \delta_i$ which concludes the proof.
\end{proof}
We can thus conclude that, for error based learning algorithms under no fairness constraints, learning within deterministic threshold policies is guaranteed to fail.
Even though the optimal policy lies within the set of deterministic threshold policies, it cannot easily be approximated within this set starting from a suboptimal predictive model.
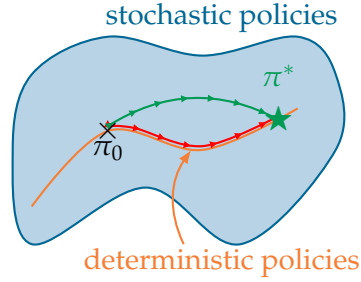
\begin{figure}
\centering
\begin{tikzpicture}[scale=0.5,decoration={triangles, shape size=0.5mm}]
\draw[thick, MidnightBlue, fill=MidnightBlue!20] plot[smooth cycle, tension=.6] coordinates {(-3.5,0.5) (-3,2.5) (-1,3.5) (1.5,3) (4,3.5) (5,2.5) (5,0.5) (2.5,-2) (0,-0.5) (-3,-2)};
\draw[thick, Orange] plot[smooth, tension=.6] coordinates {(-3, -1) (-1, 1) (1.5, 0.5) (4, 1.6)};
\draw[thick, red, postaction={draw,decorate}] plot[smooth, tension=.6] coordinates {(-1, 1.13) (-0.3, 1.08) (1.5, 0.6) (3.5, 1.4)};
\draw[thick, ForestGreen, postaction={draw,decorate}] (-1, 1.13) to[bend left] (3.5, 1.3);
\node[label={\color{ForestGreen}$\pi^*$}] at (3.5, 1.3) {{\color{ForestGreen}$\bigstar$}};
\node at (2, 4) {\color{MidnightBlue}stochastic policies};
\node at (2, -2.5) {\color{Orange}deterministic policies};
\draw[thick, Orange, ->, >=latex] (1, -2) to[bend left] (1.2, 0.5);
\node[label={[yshift=-0.8cm]$\pi_0$}] at (-1, 1) {{\large $\times$}};
\end{tikzpicture}
\caption[Illustration of the impossibility of finding the optimal policy when restricted to deterministic decision rules]{This figure illustrates how it can be impossible to find the optimal policy when the allowed set of policies is restricted to deterministic decision rules.}
\label{fig:impossibility}
\end{figure}

Figure~\ref{fig:impossibility} illustrates that, even though the optimal policy $\pi^*$ is deterministic, when starting from a deterministic initial policy $\pi_0$ (black cross), we cannot iteratively reach $\pi^*$ (green star) when updating solely within deterministic policies (red line following the orange line of deterministic policies).
It is necessary to deploy stochastic policies (blue area) along the way to then be able to converge to the optimal policy (along the green line).
We will introduce such ``exploring policies'' after our final impossibility result for error based learning algorithms.

\begin{corollary}\label{cor:errorbased}
A deterministic threshold policy $\pi \ne \pi^*$ with $\Pr[\pi(x, z) \ne y] = 0$ under $\dP$ will fail to converge to $\pi^*$ under an error based learning algorithm for the underlying predictive model with probability $1$.
\end{corollary}
\begin{proof}
Since error based learning algorithms lead to non-exploring policies whenever
\begin{equation*}
\sum_{(x,z,y) \in \cD} \B{1}[\pi(x,z) \ne y] = 0 \eqc
\end{equation*}
using the assumption $\Pr[\pi(x,z) \ne y] = 0$, we can use Proposition~\ref{prop:limits}
with $\delta_t = 0$ for all $t \in \bN$.
\end{proof}
While we have focused on deterministic threshold rules, our results readily generalize to \emph{all} deterministic policies.
An arbitrary deterministic policy $\pi$ can always be written as a threshold rule $\pi_Q$ as in eq.~\eqref{eq:detthreshQ} with $Q(Y = 1 \given x, z) = \B{1}[\pi(D = 1 \given x, z) = 1]$.
To conclude, if we can only observe the outcomes of previous decisions taken by a deterministic initial policy $\pi_0$, these outcomes may be insufficient to find the (fair) deterministic decision rule that maximizes utility.

\subsection{Stochastic policies}
\label{sec:exploring}

A naive but instructive way to overcome the undesirable behavior exhibited by deterministic policies discussed in the previous section, is to fully randomize initial decisions, i.e., ${\pi_0}(D=1 \given x,z) = \nicefrac{1}{2}$ for all $x,z$.
It readily follows from eq.~\eqref{eq:imperfect-p} that then $\dP_{\pi_0} = \dP$.
Hence, if the hypothesis class of predictive models $\Qcal$ is rich enough, we could learn the optimal policy $\pi^{*}$ from data gathered under $\pi_0$.
In practice, fully randomized initial policies are unacceptable in terms of utility or unethical---it would entail releasing defendants by a coin flip.
Fortunately, we will show next that full randomization is not required to learn the optimal policy.
We only need to choose an initial policy $\pi_0$ such that $\pi_0(D = 1 \given x, z) > 0$ on any measurable subset of $\xspace \times \sspace$ with positive probability under $\dP$, a requirement that is more acceptable for the decision maker in terms of initial utility.
We refer to any policy with this property as an \emph{exploring} policy.
A policy $\pi$ is exploring, if and only if the true distribution $\dP$ is absolutely continuous with respect to the induced distribution $\dP_{\pi}$.
This means the data collection distribution must not ignore regions where the true distribution puts mass.
We note that this condition does not strictly require randomness, but could be achieved by a pre-determined process, e.g., ``$d=1$ for every $n$-th decision''.
For an exploring policy $\pi_0$, we can compute the utility in eq.~\eqref{eq:utility} and the group benefits for $z\in \{0,1\}$ via inverse propensity score weighting
\begin{align}
\begin{split}\label{eq:UtilityStoch}
  u_{\dP_{\pi_0}}(\pi, \pi_0)
  &:= \E_{\substack{x, z, y \sim \dP_{\pi_0} \\ d \sim \pi(x, z)}}
  \Bigl[\frac{d (y-c)}{\pi_0(D = 1 \given x, z)}\Bigr] \eqc \\
  b_{\dP_{\pi_0}}^z(\pi, \pi_0)
  &:= \E_{\substack{x, z, y \sim \dP_{\pi_0} \\ d \sim \pi(x, z)}}
  \Bigl[ \frac{f(d, y)}{\pi_0(D = 1 \given x, z)} \Bigr] \eqp
\end{split}
\end{align}
Crucially, even though $u_{\dP}(\pi) = u_{\dP_{\pi_0}}(\pi, \pi_0)$ and $b_{\dP}^z(\pi) = b_{\dP_{\pi_0}}^z(\pi, \pi_0)$, the expectations are with respect to the induced distribution $\dP_{\pi_0}(X, Z, Y)$, yielding the following positive result.
\begin{proposition}\label{prop:positive}
Let $\Pi$ be the set of exploring policies and let $\pi_0 \in \Pi \setminus \{\pi^*\}$.
Then, the optimal objective value is
\begin{equation*}
   v(\pi^*) = \sup_{\pi \in \Pi \setminus \{\pi^*\}} \Bigl\{u_{\dP_{\pi_0}}(\pi, \pi_0)
  - \frac{\lambda}{2} (b_{\dP_{\pi_0}}^0 (\pi, \pi_0) - b_{\dP_{\pi_0}}^1(\pi, \pi_0))^2 \Bigr\} \eqp
\end{equation*}
\end{proposition}
\begin{proof}
We already know that the supremum is upper bounded by $v(\pi^*)$, i.e., it suffices to construct a sequence of policies $\{\pi_n\}_{n\in \bN_{>0}} \subset \Pi \setminus \{\pi^*\}$ such that $v(\pi_n) \to v(\pi^*)$ for $n \to \infty$.
Using notation from the proof of Proposition~\ref{prop:limitation}, we define
\begin{equation*}
\pi_n(D = 1 \given x,z) :=
\begin{cases}
1 &\text{if } (x, z) \in W_1(\pi^*)\eqc\\
\frac{1}{n} &\text{otherwise}\eqp
\end{cases}
\end{equation*}
It is clear that $\pi_n$ is exploring, i.e., $\pi_n \in \Pi$, for all $n \in \bN_{>0}$ as well as that $\pi_n \ne \pi^*$.
To compute
\begin{equation*}
\lim_{n \to \infty} v_{\dP_{\pi_0}}(\pi_n, \pi_0)
= \lim_{n \to \infty} \Bigl( u_{\dP_{\pi_0}}(\pi_n, \pi_0)
- \frac{\lambda}{2} \bigl(b_{\dP_{\pi_0}}^0(\pi_n, \pi_0) - b_{\dP_{\pi_0}}^1(\pi_n, \pi_0)\bigr)^2 \Bigr)
\end{equation*}
we look at the individual limits.
For the utility we have
\begin{align*}
\lim_{n \to \infty} u_{\dP_{\pi_0}}(\pi_n, \pi_0)
&=
\lim_{n \to \infty}
    \E_{x, z, y \sim \dP_{\pi_0}(X, Z, Y)}
      \left[
        \frac{\pi_n(D = 1 \given x, z)}{\pi_0(D = 1 \given x, z)} (y - c)
      \right] \\
&=
\int_{W_1(\pi^*)} \frac{\dP(Y=1 \given x, z) - c}{\pi_0(D=1 \given x, z)}\, d\dP_{\pi_0}(x,z)\; +
\\
&\quad\;
\lim_{n\to \infty} \frac{1}{n}
\usub{=: C_1 \text{ with } |C_1| < \infty \text{ for any given exploring } \pi_0 \in \Pi}{\int_{W_1(\pi^*)^{\complement}} \frac{\dP(Y=1 \given x, z) - c}{\pi_0(D=1 \given x, z)}\, d\dP_{\pi_0}(x,z)}\\
&= \int_{W_1(\pi^*)} (y-c)\, d\dP(x,z,y) + \lim_{n\to \infty} \frac{C_1}{n}\\
&= u_{\dP}(\pi^*) \eqp
\end{align*}
Similarly, for the benefit terms that are linear in both arguments, such as $f(d, y) = d$, we have for $z \in \{0,1\}$
\begin{align*}
\lim_{n \to \infty} b_{\dP_{\pi_0}}^z(\pi_n, \pi_0)
&= \E_{x, y \sim \dP_{\pi_0}(X, Y \given z)}
          \left[
            \frac{f(\pi_n(D = 1 \given x, z), y)}{\pi_0(D = 1 \given x, z)}
          \right]
\\
&= \int_{W_1(\pi^*)} \frac{f(1, \dP(Y=1 \given x, z))}{\pi_0(D=1 \given x, z)}\, d\dP_{\pi_0}(x \given z)\; +
\\
&\quad\;
\lim_{n\to \infty} \frac{1}{n}
\usub{=: C_2^z \text{ with } |C_2^z| < \infty \text{ for any given exploring } \pi_0 \in \Pi}{\int_{W_1(\pi^*)^{\complement}} \frac{f(1, \dP(Y=1 \given x, z))}{\pi_0(D=1 \given x, z)}\, d\dP_{\pi_0}(x \given z)}\\
&= \int_{W_1(\pi^*)} f(1, y)\, d\dP(x, y \given z) + \lim_{n\to \infty} \frac{C_2^z}{n}\\
&= b_{\dP}^z(\pi^*) \eqp
\end{align*}
Because all the limits are finite, via the rules for sums and products of limits we get
\begin{align*}
  \lim_{n \to \infty} v_{\dP_{\pi_0}}(\pi_n, \pi_0)
 &= \lim_{n \to \infty} u_{\dP_{\pi_0}}(\pi_n, \pi_0) - \frac{\lambda}{2} (\lim_{n \to \infty} b_{\dP_{\pi_0}}^0(\pi_n, \pi_0) - \lim_{n \to \infty} b_{\dP_{\pi_0}}^1(\pi_n, \pi_0))^2 \\
 &= u_{\dP}(\pi^*) - \frac{\lambda}{2} (b_{\dP}^0(\pi^*) - b_{\dP}^1(\pi^*))^2 \\
 &= v_{\dP}(\pi^*) \eqp
\end{align*}
\end{proof}
This shows that---unlike within deterministic threshold models---within exploring policies we can learn the optimal policy using only data from an induced distribution.
Finally, we would like to highlight that not all exploring policies may be (equally) acceptable to society.
For example, in lending scenarios without fairness constraints (i.e., $\lambda=0)$, it may appear wasteful to deny a loan with probability greater than zero to individuals who are believed to repay by the current model.
In those cases, one may like to consider exploring policies that, given sufficient evidence, decide $d=1$ deterministically, i.e., $\pi_0(D = 1 \given x, z) = 1$ for some values of $x, z$.
We will operationalize this notion in \secref{sec:algorithm} as what we call the \emph{semi-logistic policy}.
Other settings, like the criminal justice system, call for a more general discussion about the ethics of non-deterministic decision making.

\section{How to learn exploring policies}
\label{sec:algorithm}

\begin{algorithm}[t!]
\caption{\textsc{ConsequentialLearning}: train a sequence of policies $\pi_{\btheta_t}$ of increasing $v_{\dP}(\pi_{\btheta_t})$.}
\label{algo:SGDpolicylearning}
  \begin{algorithmic}[1]
    \Input cost $c$, time steps $T$, decisions $N$, iterations $M$, minibatch size $B$, penalty $\lambda$, learning rate $\alpha$
        \State $\btheta_0 \gets \Call{InitializePolicy}{{}}$
        \For{$t = 0, \ldots, T-1$} \Comment{time steps}
            \State $\cD^t \gets \Call{CollectData}{{\btheta_t, N}}$
            \State $\btheta_{t+1} \gets \Call{UpdatePolicy}{{\btheta_t, \cD^t, M, B, \alpha}}$
        \EndFor
        \State \Return $\{\pi_{\btheta_{t}}\}_{t=0}^{T}$
    \Statex
    \Function{CollectData}{$\btheta$, $N$}
        \State $\cD \gets \emptyset$
        \For{$i = 1, \ldots, N$} \Comment{$N$ decisions}
            \State $(x_i, z_i) \sim \dP(X,Z)$ and $d_i \sim \pi_{\btheta}(x_i,z_i)$
            \If{$d_i=1$} \Comment{positive decision}
                \State $\cD \gets \cD \cup \{(x_i,z_i,y_i)\}$ with $y_i \sim \dP(Y\given x_i, z_i)$
            \EndIf
        \EndFor
        \State \Return $\cD$ \Comment{data observed under $\pi_{\btheta}$}
    \EndFunction
    \Statex
    \Function{UpdatePolicy}{$\btheta'$, $\cD$, $M$, $B$, $\alpha$}
        \State $\btheta^{(0)} \gets \btheta'$
        \For{$j=1,\ldots, M$} \Comment{iterations}
        \State $\cD^{(j)} \gets \Call{Minibatch}{{\cD}, B}$ \Comment{sample minibatch}
        \State $\nabla \gets 0$, $n_j \gets 0$
        \For{$(x, z, y) \in \cD^{(j)}$} \Comment{accumulate gradients}
                \State $d \sim \pi_{\btheta^{(j)}}(x, z)$
                \If{$d = 1$}
                  \State $n_j \gets n_j + 1$
                  \State $%
                    \nabla \gets \nabla +
                    \nabla_{\btheta} v(\pi_{\btheta}, \pi_{\btheta'}) |_{\btheta=\btheta^{(j)}}$
                \EndIf
         \EndFor
         \State $\btheta^{(j+1)} \gets \btheta^{(j)} + \alpha \, \frac{\nabla}{n_j}$
        \EndFor
        \State \Return $\btheta^{M}$
    \EndFunction
  \end{algorithmic}
\end{algorithm}

In this section, we exemplify Proposition~\ref{prop:positive} via a simple, yet practical, gradient-based algorithm to find the solution to eq.~\eqref{eq:policy_learning} within a (differentiable) parameterized class of exploring policies $\Pi(\Theta)$ using data gathered by a given, already deployed, exploring policy $\pi_0$.
To this end, we consider a class of parameterized exploring policies $\Pi(\Theta)$ and we aim to find the policy $\pi_{\btheta^*} \in \Pi(\Theta)$ that solves the optimization problem in eq.~\eqref{eq:policy_learning}.

We use stochastic gradient ascent (SGA) \citep{kiefer1952stochastic} to learn the parameters of the new policy, i.e.,
\begin{equation*}
  \btheta_{i + 1} = \btheta_{i} + \alpha_{i} \nabla_{\btheta} v_{\dP}(\pi_{\btheta}) |_{\btheta = \btheta_{i}}\eqc
\end{equation*}
where
\begin{equation*}
\nabla_{\btheta} v_{\dP}(\pi_{\btheta}) = \nabla_{\btheta} u_{\dP}(\pi_{\btheta}) - \lambda (b_{\dP}^0(\pi_{\btheta}) - b_{\dP}^1(\pi_{\btheta})) (\nabla_{\btheta} b_{\dP}^0(\pi_{\btheta}) - \nabla_{\btheta}b_{\dP}^1(\pi_{\btheta}))\eqc
\end{equation*}
and $\alpha_i > 0$ is the learning rate at step $i \in \bN$.
With the reweighting from eq.~\eqref{eq:UtilityStoch} and the log-derivative trick \citep{williams1992simple}, we can compute the gradient of the utility and the benefits as
\begin{align}
\begin{split}\label{eq:gradient-pi0}
  \nabla_{\btheta} u_{\dP}(\pi_{\btheta})
  &= \E_{\substack{x, z, y \sim \dP_{\pi_0} \\ d \sim \pi_{\btheta}(x, z)}}
  \Big[ \frac{d\, (y - c) \nabla_{\btheta}\log \pi_{\btheta} }{\pi_0(D = 1 \given x, z)} \Big] \eqc \\
  \nabla_{\btheta} b_{\dP}^z(\pi_{\btheta})
  &= \E_{\substack{x, z, y \sim \dP_{\pi_0} \\ d \sim \pi_{\btheta}(x, z)}}
  \Big[ \frac{f(d,y) \nabla_{\btheta}\log \pi_{\btheta}}{\pi_0(D = 1 \given x, z)} \Big] \eqc
\end{split}
\end{align}
where $\nabla_{\btheta} \log \pi_{\btheta} := \nabla_{\btheta} \log \pi_{\btheta} (D \given x, z)$ is the score function \citep{Hyvarinen05:ScoreMatching}.
Thus, our implementation resembles a REINFORCE algorithm with horizon one.

Note that we can obtain an expression for $\nabla_{\btheta_t} v_{\dP}(\pi_{\btheta_t})$ by simply replacing $\pi_0$ with $\pi_{\btheta_{t-1}}$ in eq.~\eqref{eq:gradient-pi0}.
Thus we can estimate the gradient with samples $(x_i, z_i, y_i)$ from the distribution $\dP_{\pi_{t-1}}$ induced by the previous policy $\pi_{t-1}$, and sample the decisions from the policy under consideration $d_i \sim \pi_{\btheta_t}$.
This yields an unbiased finite sample Monte-Carlo estimator for the gradients
\begin{align}
\begin{split}\label{eq:gradutilestim}
  & \nabla_{\btheta_t} u(\pi_{\btheta_t}, \pi_{\btheta_{t-1}}) \approx \frac{1}{n_{t-1}}
  \sum_{i = 1}^{n_{t-1}} \frac{d_i (y_i - c)}{\pi_{\btheta_{t-1}}(D = 1 \given x_i, z_i)}\, \nabla_{\btheta_t} \log \pi_{\btheta_t}(D = d_i \given x_i, z_i) \eqc
  \\
  & \nabla_{\btheta_t} b^z(\pi_{\btheta_t}, \pi_{\btheta_{t-1}}) \approx \frac{1}{n_{t-1}}
  \sum_{i = 1}^{n_{t-1}} \frac{f(d_i, y_i)}{\pi_{\btheta_{t-1}}(D = 1 \given x_i, z_i)}\, \nabla_{\btheta_t} \log \pi_{\btheta_t}(D = d_i \given x_i, z_i) \eqc
\end{split}
\end{align}
where $n_{t-1}$ is the number of positive decisions taken by $\pi_{\btheta_{t-1}}$.
Here, it is important to notice that, while the decisions by $\pi_{\btheta_{t-1}}$ were actually taken and, as a result, (feature and label) data was gathered under $\pi_{\btheta_{t-1}}$, the decisions $d_i \sim \pi_{\btheta_{t}}$ are just sampled to implement SGA.
The overall policy learning process is summarized in Algorithm~\ref{algo:SGDpolicylearning}, where \textsc{Minibatch}$(\cD, B)$ samples a minibatch of size $B$ from the dataset $\cD$ and \textsc{InitializePolicy}$()$ initializes the policy parameters.

Unfortunately, the above procedure has two main drawbacks.
First, it may require an abundance of data from $\dP_{\pi_0}$, which can be unacceptable in terms of utility if $\pi_0$ is far from optimal.
Second, if $\pi_0(D = 1 \given x, z)$ is small in a region where $\pi_{\btheta}$ often takes positive decisions, one may expect that an empirical estimate of the above gradient will have high variance, due to similar arguments as in weighted inverse propensity scoring \citep{Sutton98:RL}.
On the other hand, in most practical applications updating the model after every single decision is impractical.
Typically, a fixed model will be deployed for a certain period, before it is updated using the data collected within this period.
This is also a natural mode of operation for predictive models in real-world applications.

\begin{figure}\label{fig:syntsetting}
\centering
\begin{minipage}{0.5\textwidth}
\centering
\begin{tikzpicture}%
    \node[anchor=south west,inner sep=0] (image) at (0,0) {%
      \includegraphics[width=\columnwidth]{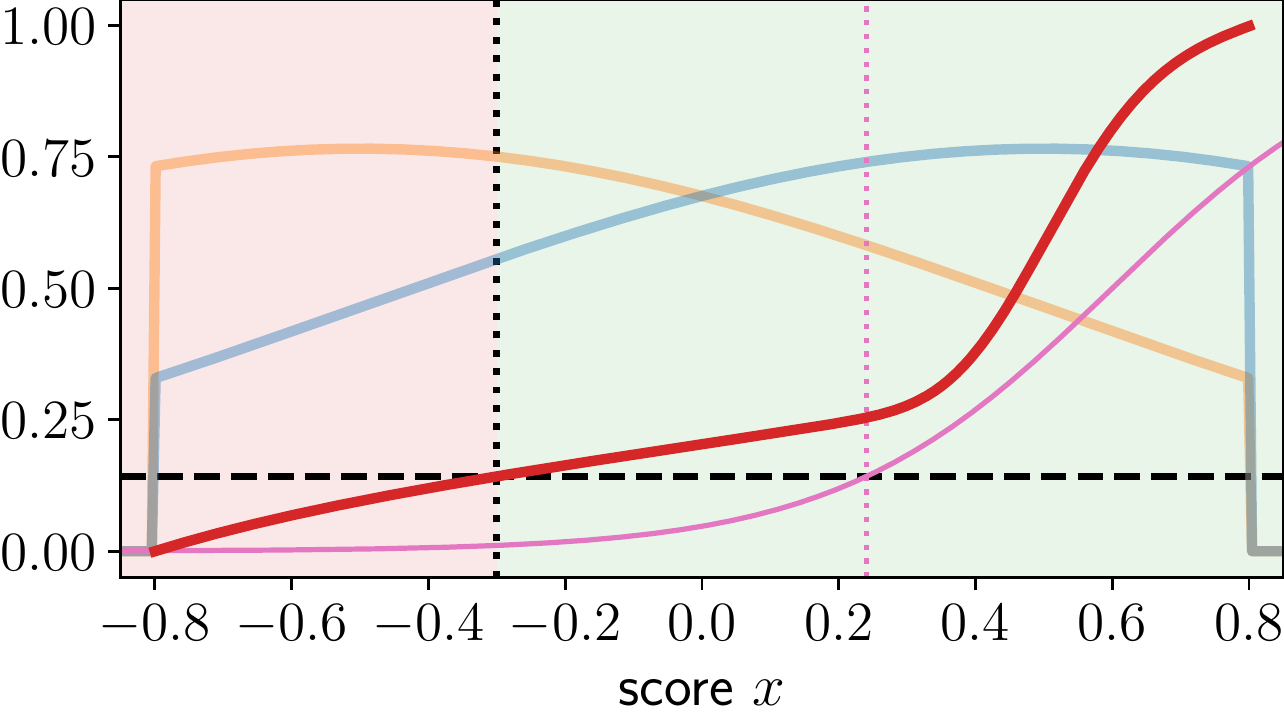}};
    \begin{scope}[x={(image.south east)}, y={(image.north west)}]
        \node at (1.25,0.6) {{\small \textbf{First Setting}}};
    \end{scope}
\end{tikzpicture}\\
\begin{tikzpicture}%
    \node[anchor=south west,inner sep=0] (image) at (0,0) {%
      \includegraphics[width=\columnwidth]{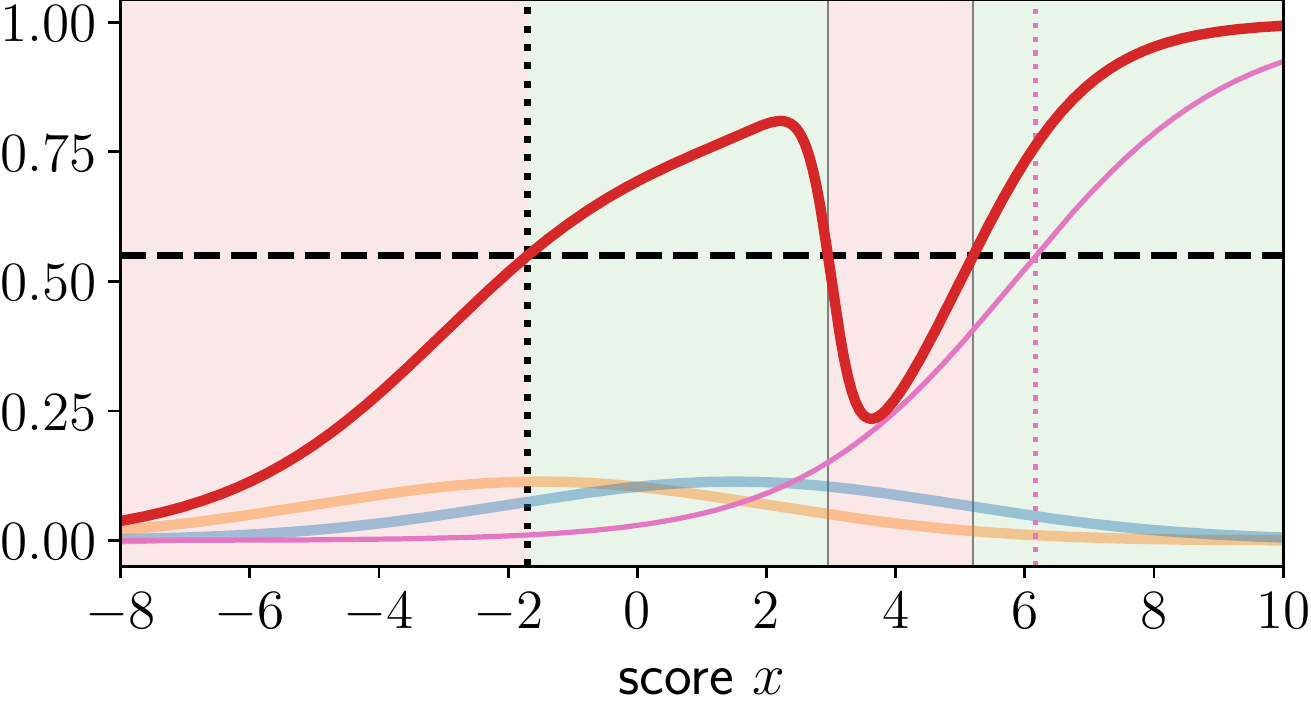}};
    \begin{scope}[x={(image.south east)}, y={(image.north west)}]
        \node at (1.25,0.4) {{\small \textbf{Second Setting}}};
    \end{scope}
\end{tikzpicture}%
\end{minipage}
\hspace{0.5cm}
\begin{minipage}{0.3\columnwidth}
\includegraphics[width=\columnwidth]{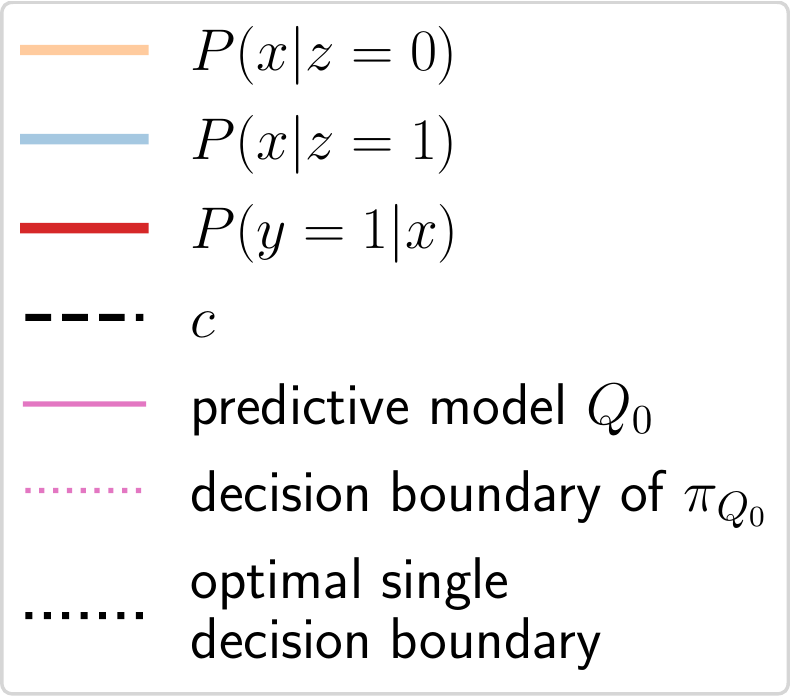}%
\end{minipage}
\caption[Illustration and description of two synthetic settings to learn fair decisions]{%
Two synthetic settings.
In red, we show $\dP(Y = 1\given x)$, where the score $x$ is drawn from different distributions for the two groups (blue/orange).
For given $c$ (black, dashed), the optimal policy decides $d=1$ ($d=0$) in the shaded green (red) regions.
The vertical black, dotted line shows the best policy achievable with a single threshold on $x$.
In pink, we show a possible imperfect logistic predictive model and its corresponding (suboptimal) threshold in $x$.}
\label{fig:synthetic-setting}
\end{figure}

To overcome these drawbacks, we build two types of sequences of policies $\{ \pi_{\btheta_t} \}_{t=0}^{T}$:
\begin{enumerate}[label=\alph*)]
  \item the \emph{iterative sequence} $\pi_{t+1} := \mathcal{A}(\pi_t, \cD^t)$ with $\cD^t \sim \dP_{\pi_t}(X,Z,Y)$, where only the data gathered by the immediately previous policy are used to update the current policy;
  \item the \emph{aggregated sequence} $\pi_{t+1} := \mathcal{A}(\pi_t, \bigcup_{i=0}^t \cD^i)$ with $\cD^i \sim \dP_{\pi_i}(X,Z,Y)$, where the data gathered by all previous policies are used to update the current policy.
\end{enumerate}

\xhdr{Remarks}
Note that in Algorithm~\ref{algo:SGDpolicylearning} we learn each policy $\pi_t$ only using data from the previous policy $\pi_{t-1}$.
This may readily be generalized to a mix of multiple previous policies $\pi_{t'}$ in eq.~\eqref{eq:gradutilestim}.
Averaging multiple gradient estimators for several $t' < t$ is again an unbiased gradient estimator.
To reduce variance, in practice one may consider recent policies $\pi_{t'}$ most similar to $\pi_t$.

The way in which we use weighted sampling to estimate the above gradients closely relates to the concept of weighted inverse propensity scoring (wIPS), commonly used in counterfactual learning \citep{Bottou13:Counterfactual,Swaminathan15:CRM}, off-policy reinforcement learning \citep{Sutton98:RL}, and contextual bandits \citep{Langford08:ES}.
However, a key difference is that, in wIPS, the labels $y$ are always observed.
As an example, in the case of counterfactual learning one may interpret $\pi_0(x,z)$ in eq.~\eqref{eq:imperfect-p} as a treatment assignment mechanism in a randomized controlled trial.
Under this interpretation, the two most prominent differences with respect to the literature become
apparent. First, we do not observe outcomes in the control group. Second, in observational studies for treatment effect estimation \citep{Rubin05:PO}, one usually estimates
the direct causal effect of $d$ on $y$, i.e., $\dP(Y \given do(D=d'), x, z)$, in the presence of confounders
$x, z$ that affect both $d$ and $y$. This could be evaluated in a (partially) randomized controlled trial, where wIPS also comes in naturally \citep{Pearl2009}.
In contrast, in our setting, the true label $y$ is independent of the decision $d$ and we estimate the conditional $\dP(Y \given x,z)$ using data from the induced distribution $\dP_{\pi_0}(X,Z) \propto \dP(X,Z) \pi_0(D = 1\given x,z)$.
With exploring policies, we obtain indirect access to the true data distribution $\dP(x,z)$ (positivity), and thus to an unbiased estimator of the conditional distribution $\dP(Y\given x,z)$ (consistency).

Despite this difference, we believe that recent advances to reduce the variance of the gradients in weighted inverse propensity scoring, such as
clipped-wIPS \citep{Bottou13:Counterfactual}, self-normalized estimator \citep{Swaminathan15:SEC}, or doubly robust estimators \citep{Dudik11:DoublyRobust},
may also be applicable to our setting.

Finally, we opt for the simple SGA approach on \mbox{(semi-)logistic} policies over, e.g., contextual bandits algorithms, because it provides a direct and fairer comparison with commonly used prediction based decision policies (e.g., logistic regression), also often trained via SGA.

While our algorithm works for any differentiable class of exploring policies, here we consider two examples of exploring policy classes in particular.

\xhdr{Logistic policy}
Our first concrete parameterization of $\pi_{\btheta}$, a \emph{logistic policy} is given by
\begin{equation*}
    \pi_{\btheta}(D = 1\given x, z) = \sigma(\bphi(x,z)^{\top} \btheta) \in (0,1) \eqc
\end{equation*}
where $\sigma(a) := \frac{1}{1+\exp(-a)}$ is the logistic function, $\btheta \in \Theta \subset \bR^m$ are the model parameters, and
$\bphi: \bR^d \times \{0,1\} \to \bR^m$ is a fixed feature map.
Note that any logistic policy is an exploring policy and we can analytically compute its score function $\nabla_{\btheta_t} \log \pi_{\btheta_t} (D=1\given x, z)$ as
\begin{equation*}
 \nabla_{\btheta_t} \log(\sigma(\bphi_i^{\top} \btheta_t))
 = \frac{\bphi_i}{1 + e^{\bphi_i^{\top} \btheta_t}} \in \bR^m \eqc
\end{equation*}
where $\bphi_i := \bphi(x_i, z_i)$.
Using this expression, we can rewrite the empirical estimator for the gradient in eq.~\eqref{eq:gradutilestim}
\begin{align*}
  &\nabla_{\btheta_t} u(\pi_{\btheta_t}, \pi_{\btheta_{t-1}}) \approx \frac{1}{n_{t-1}}
  \sum_{i = 1}^{n_{t-1}} \frac{1 + e^{-\bphi_i^{\top} \btheta_{t-1}}}{1 + e^{\bphi_i^{\top} \btheta_t}}\, d_i\, (y_i - c)\, \bphi_i \eqc
  \\
  & \nabla_{\btheta_t} b^z(\pi_{\btheta_t}, \pi_{\btheta_{t-1}}) \approx \frac{1}{n_{t-1}}
  \sum_{i = 1}^{n_{t-1}} \frac{1 + e^{-\bphi_i^{\top} \btheta_{t-1}}}{1 + e^{\bphi_i^{\top} \btheta_t}}\, f(d_i, y_i)\, \bphi_i \eqp
\end{align*}
Given the above expression, we have all the necessary ingredients to implement Algorithm~\ref{algo:SGDpolicylearning}.

\xhdr{Semi-logistic policy}
As discussed in the previous section, randomizing decisions may be questionable in certain practical scenarios.
For example, in loan decisions, it may appear wasteful for the bank and contestable for the applicant to deny a loan with probability greater than zero to individuals who are believed to repay by the current model.
In those cases, one may consider the following modification of the logistic policy, which we refer to as \emph{semi-logistic policy}:
\begin{equation*}
    \pitil_{\btheta}(D = 1\given x, z) =
    \begin{cases}
    1 & \text{ if } \bphi(x, z)^{\top} \btheta \ge 0 \eqc \\
    \sigma(\bphi(x, z)^{\top} \btheta) &\text{ if } \bphi(x, z)^{\top} \btheta < 0 \eqp
    \end{cases}
\end{equation*}
Similarly as in the logistic policy, we can compute the score function analytically as:
\begin{equation*}
    \nabla_{\btheta} \log \pitil_{\btheta}(D = 1 \given x, z) =
    \frac{\bphi(x, z)}{1 + e^{\bphi(x, z)^{\top}\btheta}} \, \B{1}[\bphi(x, z)^{\top}\btheta < 0] \eqc
\end{equation*}
and use this expression to compute an unbiased estimator for the gradient in eq.~\eqref{eq:gradutilestim} as:
\begin{align*}
  \nabla_{\btheta_t} u(\pi_{\btheta_t}, \pi_{\btheta_{t-1}})
  &\approx \frac{1}{n_{t-1}}
  \sum_{\substack{i=1 \\ \bphi_i^{\top}\btheta_t < 0}}^{n_{t-1}} \frac{d_i\, (y_i - c)\, \bphi_i} {1 + e^{\bphi_i^{\top} \btheta_t}}
  \times
  \begin{cases}
      1 &\text{ if } \bphi_i^{\top} \btheta_{t-1} \ge 0 \eqc \\
      (1 + e^{-\bphi_i^{\top}\btheta_{t-1}}) &\text{ if } \bphi_i^{\top} \btheta_{t-1} < 0 \eqp
  \end{cases}
  \\
  \nabla_{\btheta_t} b^z(\pi_{\btheta_t}, \pi_{\btheta_{t-1}})
  &\approx \frac{1}{n_{t-1}}
  \sum_{\substack{i=1 \\ \bphi_i^{\top}\btheta_t < 0}}^{n_{t-1}} \frac{f(d_i, y_i)\, \bphi_i} {1 + e^{\bphi_i^{\top} \btheta_t}}
  \times
  \begin{cases}
      1 &\text{ if } \bphi_i^{\top} \btheta_{t-1} \ge 0 \eqc \\
      (1 + e^{-\bphi_i^{\top}\btheta_{t-1}}) &\text{ if } \bphi_i^{\top} \btheta_{t-1} < 0 \eqp
  \end{cases}
\end{align*}
Note that the semi-logistic policy is an exploring policy and thus satisfies the assumptions of Proposition~\ref{prop:positive}. Finally, in all our experiments, we directly work with the available features $x$ as inputs and add a constant offset, i.e., $\phi(x, z) = (1, x)$.

\section{Experiments}
\label{sec:experiments}

We learn a sequence of policies $\{ \pi_{\btheta_t} \}_{t=1}^{T}$ using the following strategies:

\begin{description}
\item[Optimal:] decisions are taken by the optimal deterministic threshold rule $\pi^{*}$ given by eq.~\eqref{eq:detthresh}, i.e., $\pi_t = \pi^{*}$ for all $t$.
It can only be computed when the ground truth conditional $\dP(Y \given x, z)$ is known.
\item[Deterministic:] decisions are taken by deterministic threshold policies $\pi_t = \pi_{Q_t}$, where $Q_t$ are logistic predictive models maximizing label likelihood trained either in an iterative or aggregate sequence.
\item[Logistic:] decisions are taken by logistic policies $\pi_t = \pi_{\btheta_t}$ trained via Algorithm~\ref{algo:SGDpolicylearning} either in an iterative or aggregate sequence.
\item[Semi-logistic:] decisions are taken by semi-logistic policies $\pitil_t = \pitil_{\btheta_t}$ trained via Algorithm~\ref{algo:SGDpolicylearning} either in an iterative or aggregate sequence.
\end{description}

It is crucial that while each of the above methods decides over the same set of proposed $\{(x_i, z_i)\}_{i=1}^N$ at each time step $t$, depending on their decisions, they may collect labels for differing subsets and thus receive different amounts of new training data.
During learning, we record the following metrics:\footnote{For readability we only show medians over 30 runs. Figures with 25 and 75 percentiles are in Appendix~\ref{chap:aistats:additional}.}

\begin{description}
\item[Utility:] the utility $u_{\dP}(\pi_t)$ achieved by the current policy $\pi_t$ estimated empirically on a held-out dataset, the \emph{test set}, sampled i.i.d.\ from the ground truth distribution $\dP(X,Z,Y)$.
This is the utility that the decision maker would obtain if they deployed the current policy $\pi_t$ at large in the population.
\item[Effective utility:] the utility realized during the learning process up to time $t$, i.e.,
\begin{equation*}
\hat{u}(t) = \frac{1}{N\cdot t} \sum_{t' \leq t} \sum_{(x_i, z_i, y_i) \in \cD^{t'}} (y_i - c)\eqc
\end{equation*}
where $\cD^{t'}$ are the data in which the policy $\pi_{t'}(x_i, z_i)$ took positive decisions $d_i = 1$ and $N$ is the number of considered examples at each time step $t$.
This is the utility accumulated by the decision maker while learning better policies.
\item[Fairness:] the difference in group benefits between sensitive groups $\Delta b_{\dP}(\pi) := b_{\dP}^0(\pi) - b_{\dP}^1(\pi)$ for disparate impact: $f(d, y) = d$.
A policy fully satisfies the chosen criterion if and only if $\Delta b_{\dP}(\pi) = 0$.
Again, we estimate fairness empirically on the test set and thus measure the level of fairness $\pi_t$ would achieve in the entire population.
\end{description}

Detailed parameter settings for the experiments are explained in \secref{sec:app:parameters} in Appendix~\ref{chap:aistats:additional}.

\subsection{Experiments on synthetic data}

\begin{figure}
\centering
\def\figheight{4cm}
\includegraphics[width=\textwidth]{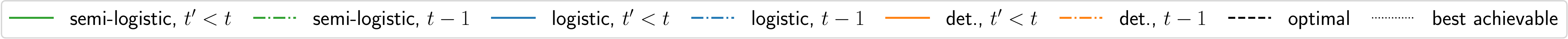}\\
\textbf{First Setting}\\
\includegraphics[height=\figheight]{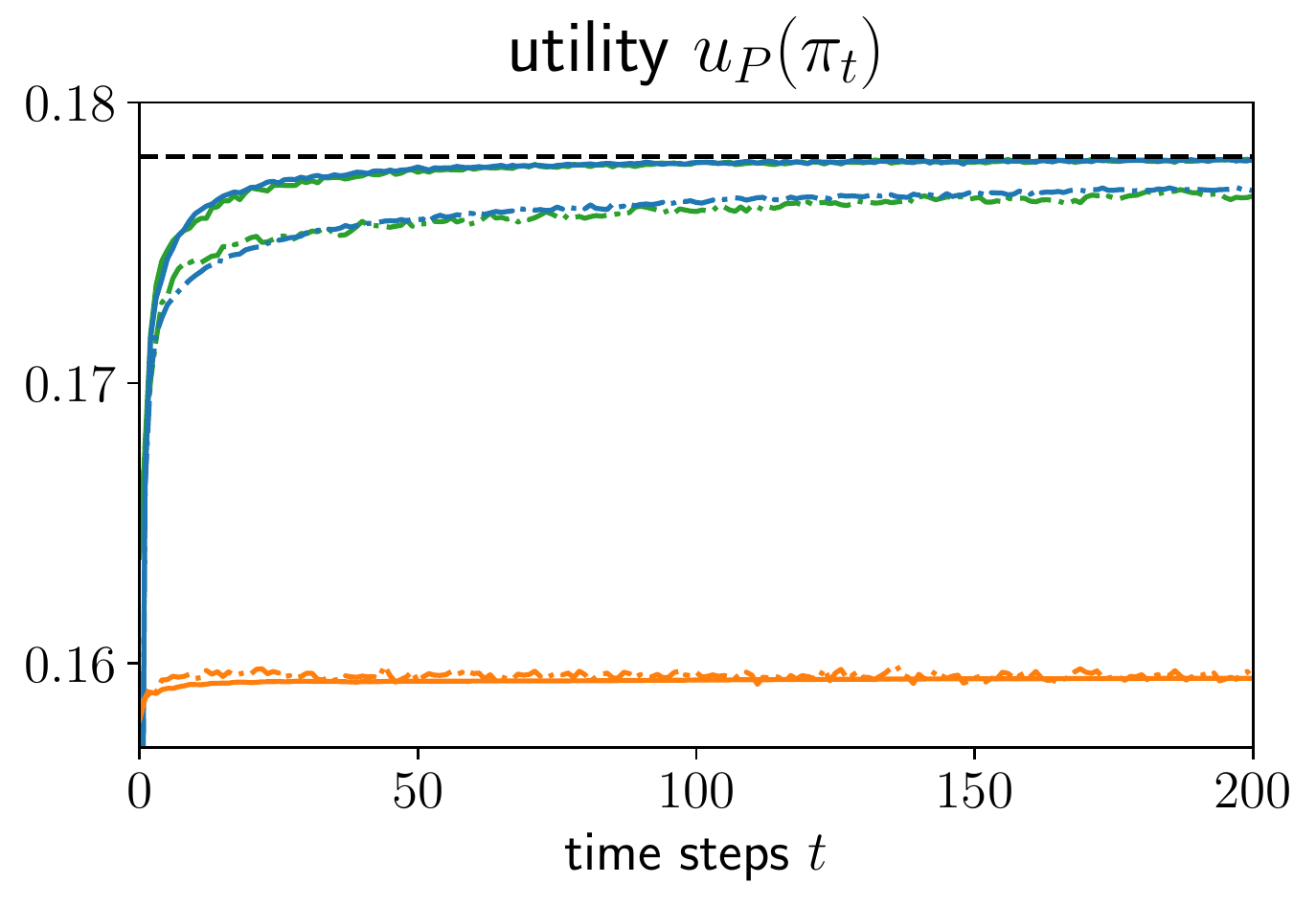}%
\hspace{1cm}
\includegraphics[height=\figheight]{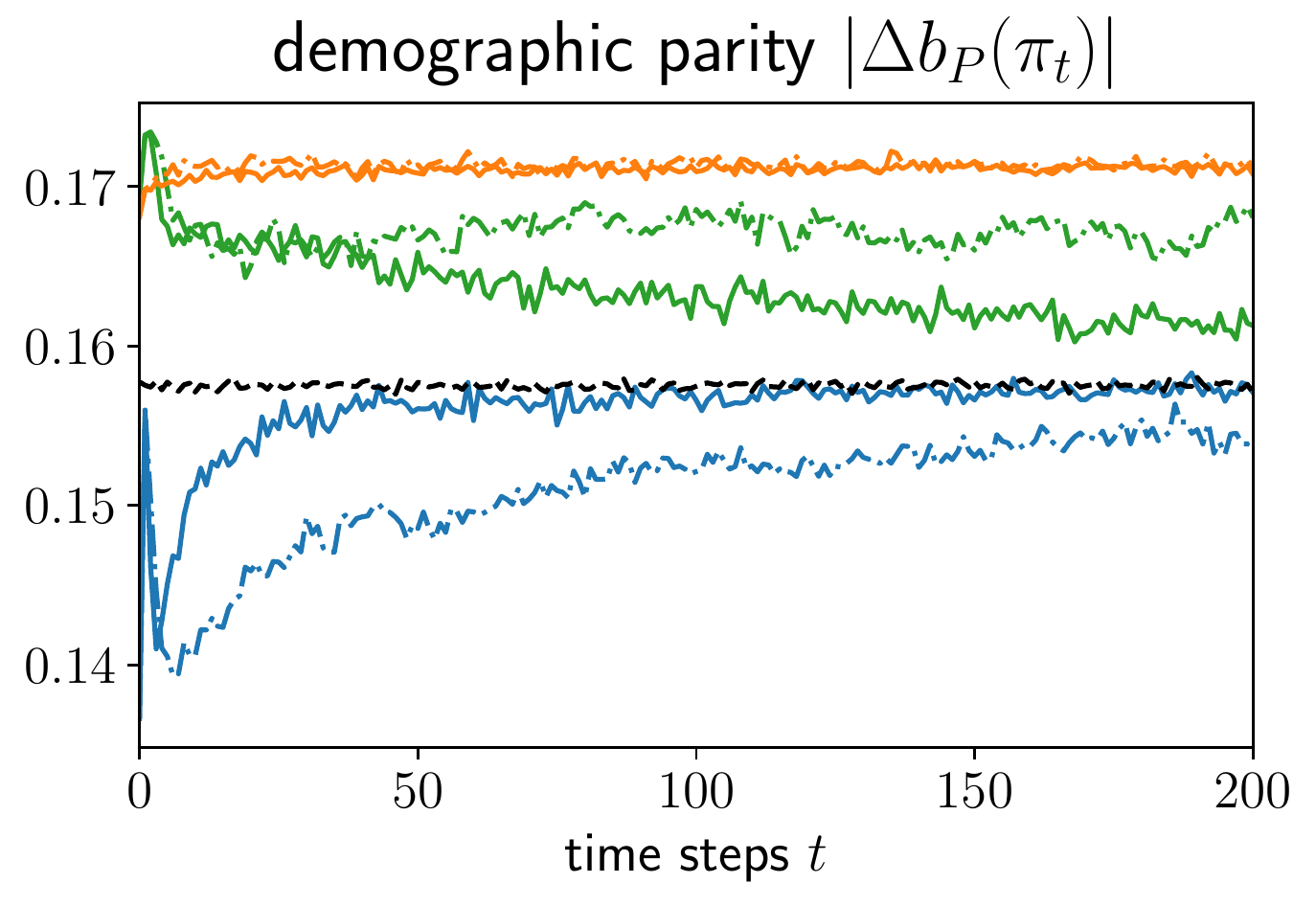}\\%
\textbf{Second Setting}\\
\includegraphics[height=\figheight]{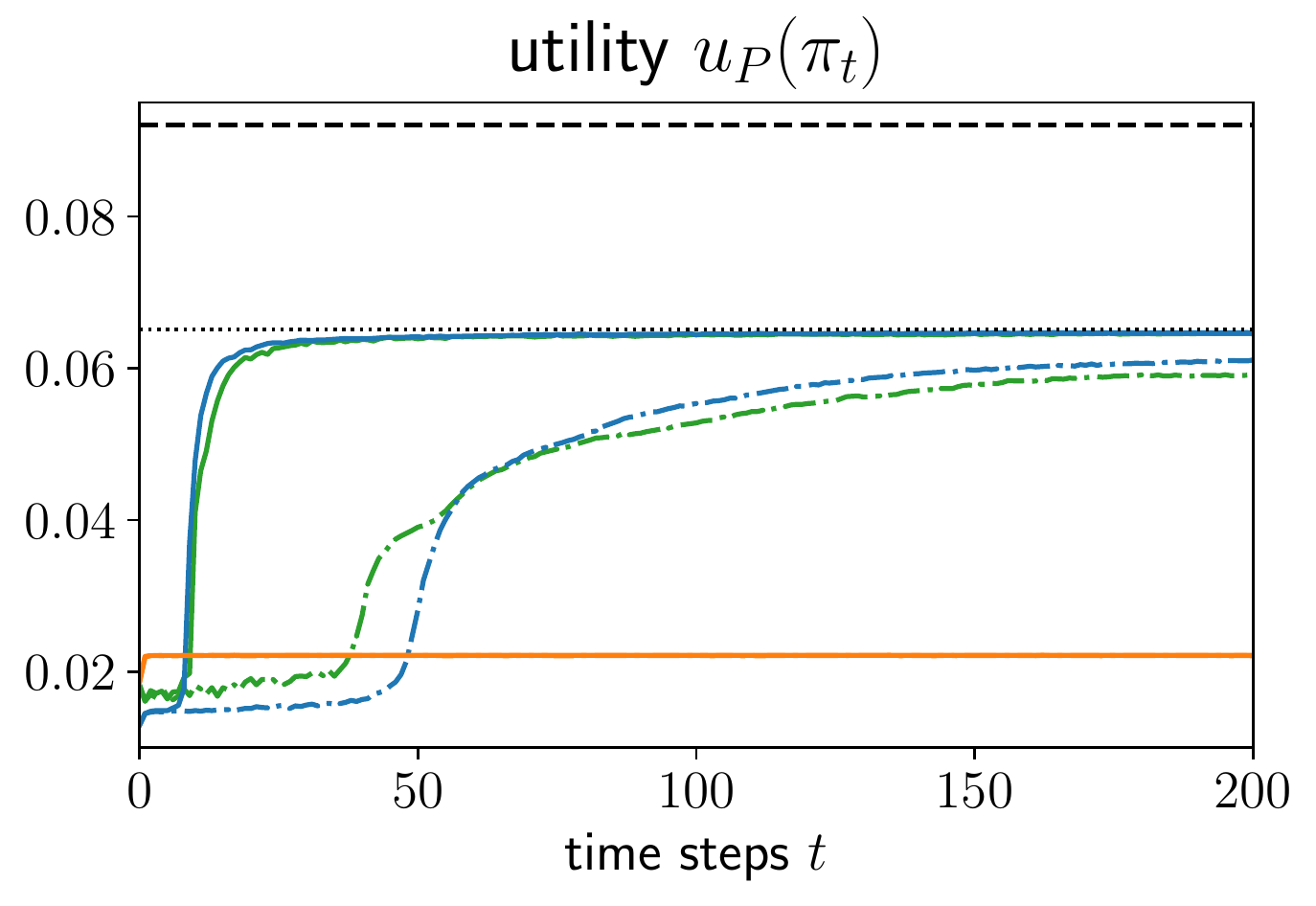}%
\hspace{1cm}
\includegraphics[height=\figheight]{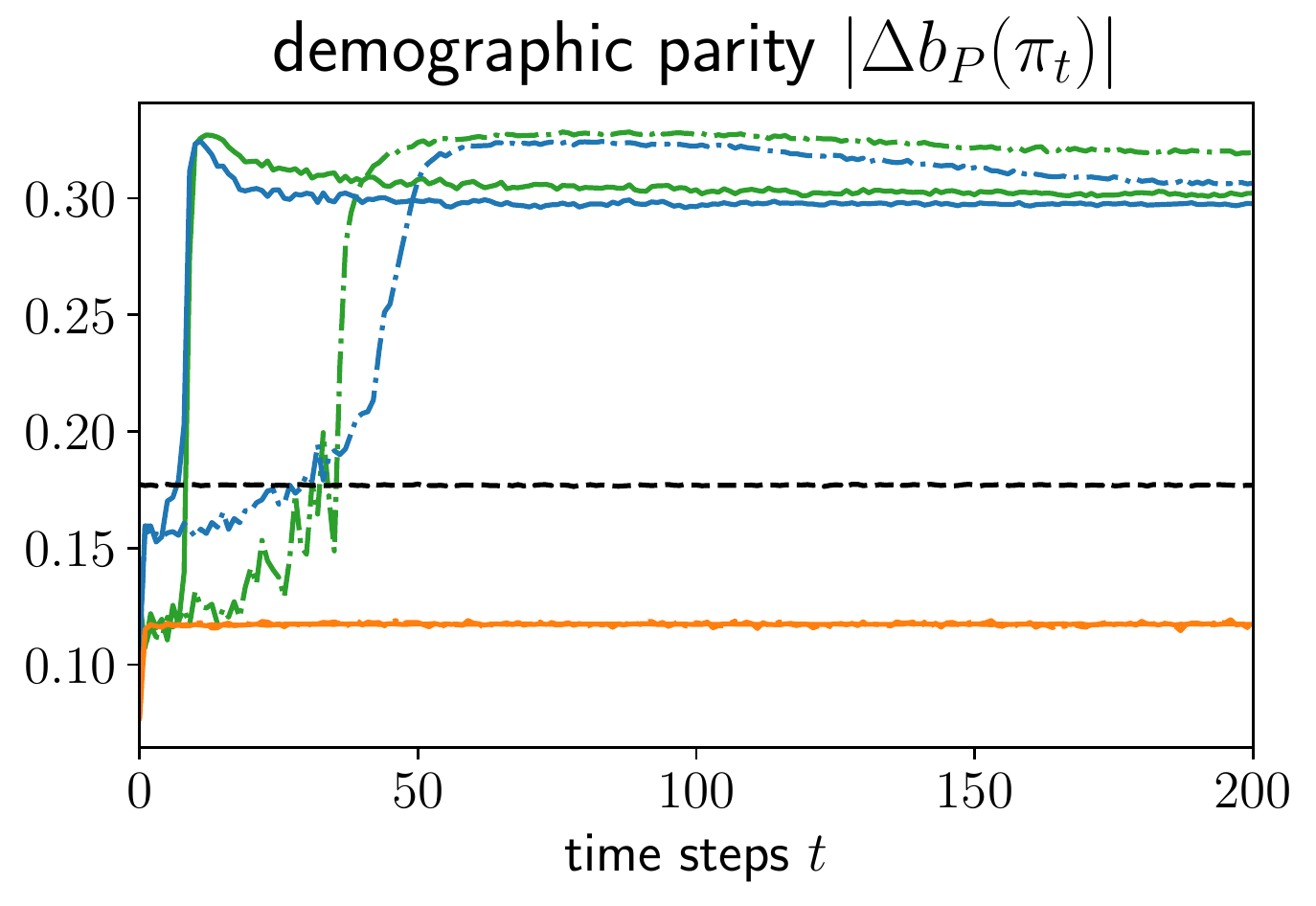}%
\caption[Empirical results of utility and fairness when learning decisions with exploring policies in two synthetic settings]{Utility and demographic parity in the synthetic settings of Figure~\ref{fig:synthetic-setting} without enforcing fairness constraints, i.e., $\lambda=0$.}
\label{fig:results-synthetic}
\end{figure}

We assume that there is a single non-sensitive feature $x \in \bR$ per individual---similar to the lending example in \secref{sec:sequential}---and a sensitive attribute $z \in \{0,1\}$.
While $\dP(X \given z=0) \ne \dP(X \given z=1)$, in our experiments the policies only take $x$ as input, and \emph{not} the sensitive attribute, which is only used for the fairness constraints.
We consider two different settings, illustrated in Figure~\ref{fig:synthetic-setting}, where $z \sim \mathrm{Ber}(0.5)$ and the distributions over $x$ differ for the two groups, see Appendix~\ref{chap:aistats:additional}.
In the first setting, the conditional probability $\dP(Y = 1 \given x)$ is strictly monotonic in the score and does not depend on $z$, but is not well calibrated, i.e., not directly proportional to $x$.
In the second setting, the conditional probability $\dP(Y = 1\given x)$ crosses the cost threshold $c$ multiple times resulting in two disjoint intervals for which the optimal decision is $d=1$ (green areas).

Figure~\ref{fig:results-synthetic} summarizes the results for $\lambda = 0$, i.e., without explicitly enforcing fairness constraints.
Our method outperforms prediction based deterministic threshold rules in terms of utility in both settings.
This can be easily understood from the evolution of policies illustrated in Figure~\ref{fig:results_synthetic_evolution} in Appendix~\ref{chap:aistats:additional}.
In the first setting, exploring policies locate the optimal decision boundary, whereas the deterministic threshold rules get stuck, even though $\dP(Y = 1\given x)$ is monotonic in $x$.
In the second setting, our methods explore more and eventually identify the best single threshold at the black vertical dotted line in Figure~\ref{fig:synthetic-setting}.
In contrast, non-exploring deterministic threshold rules converge to a suboptimal threshold at $x\approx 5$, ignoring the left green region.

In the first setting, we also observe that the suboptimal predictive models amplify unfairness beyond the levels exhibited by the optimal policy.
For our approach, levels of unfairness are comparable to or even below those of the optimal policy.
The second setting shows that depending on the ground truth distribution, higher utility can be directly linked to larger fairness violations.
In such cases, our approach allows to explicitly control for fairness.
Results on utility and demographic parity under fairness constraints with different $\lambda$ are shown in Figure~\ref{fig:results-synthetic-final-dp} in Appendix~\ref{chap:aistats:additional}.
In essence, $\lambda$ trades off utility and fairness violations to the point of perfect fairness in the ground truth distribution.

\subsection{Experiments on real data}

\begin{figure}
\centering
\includegraphics[width=\textwidth]{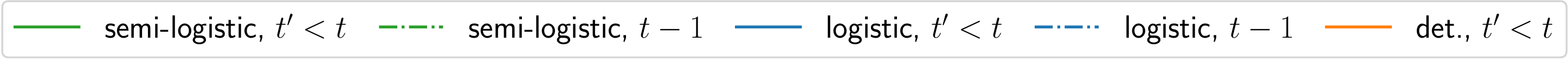}\\
\includegraphics[width=0.32\textwidth]{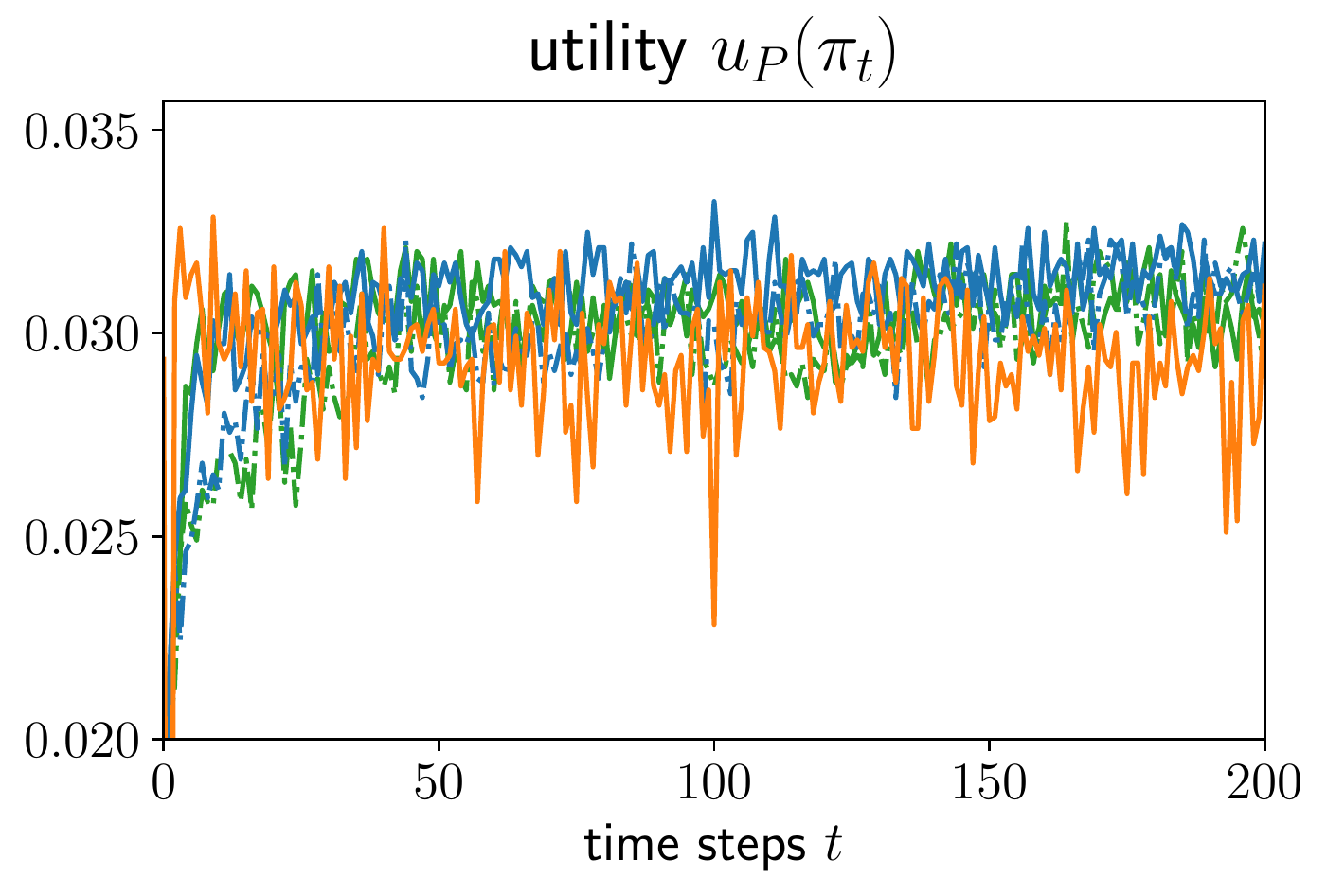}
\includegraphics[width=0.32\textwidth]{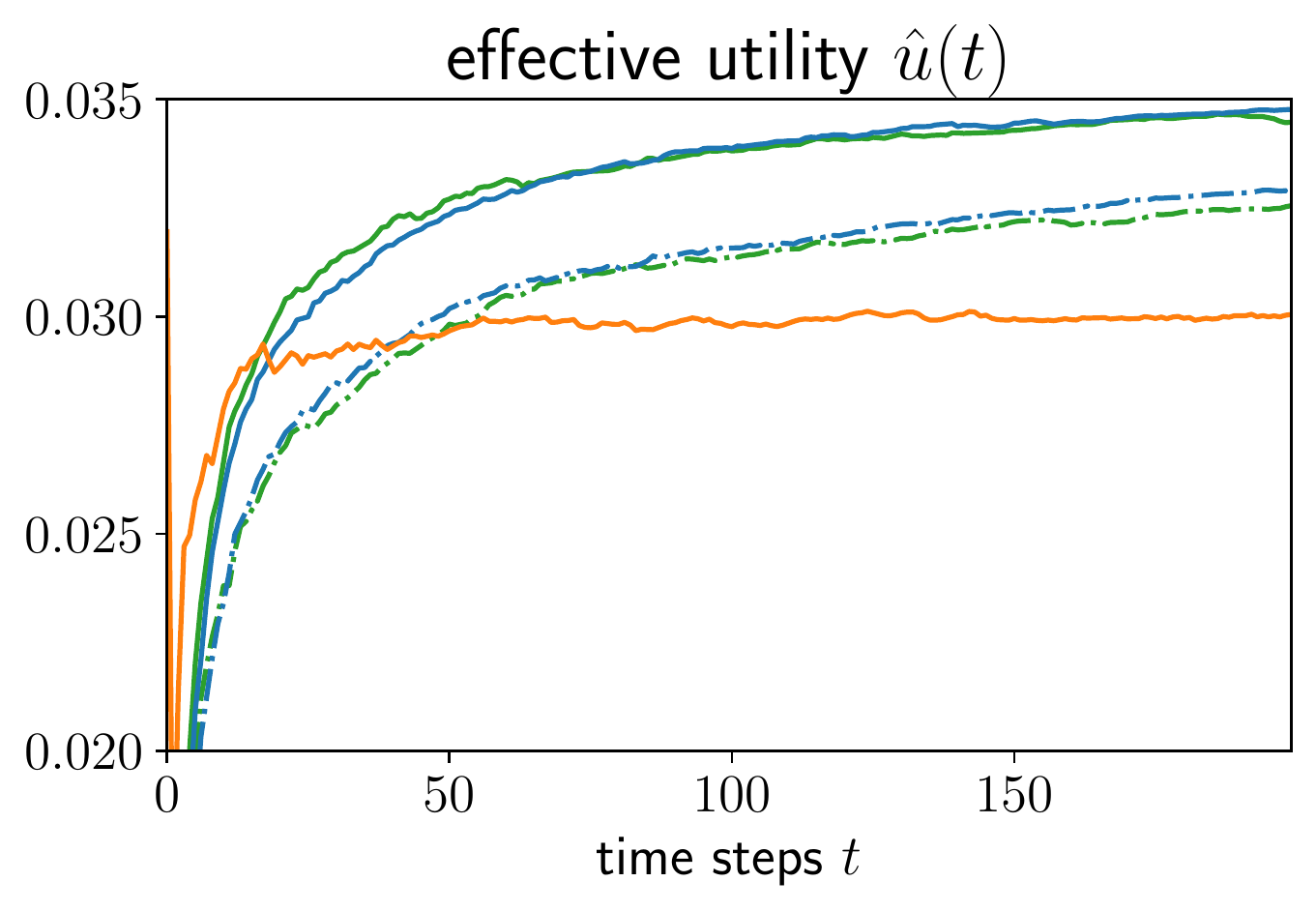}
\includegraphics[width=0.32\textwidth]{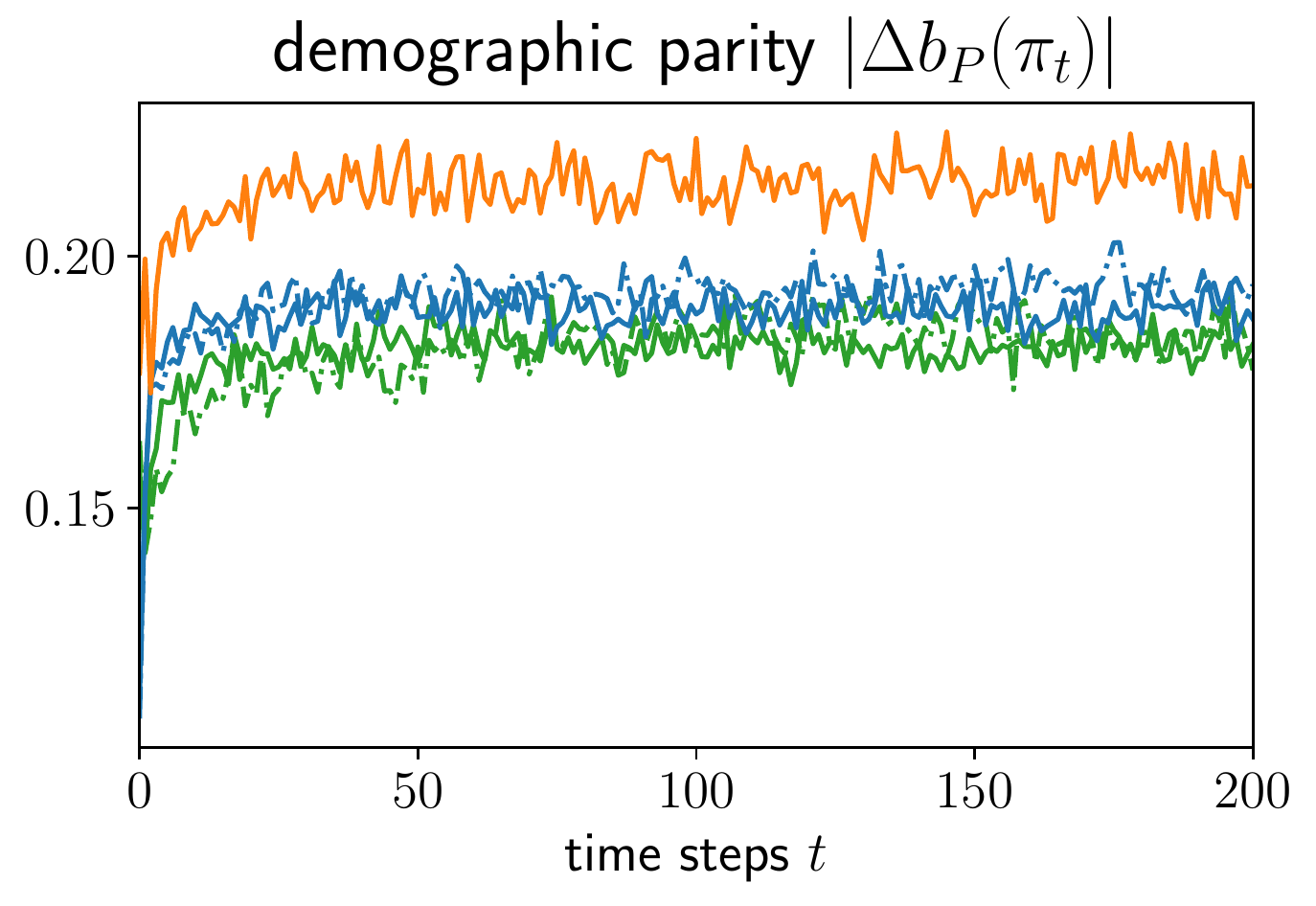}
\caption[Training progress of learning decisions on COMPAS data without fairness constraints]{Training progress on COMPAS data for $\lambda = 0$, i.e., without fairness constraints.}
\label{fig:results-real-time}
\end{figure}
\begin{figure}
\centering
\def\figheight{4cm}
\def\figspacing{1}
\includegraphics[width=\textwidth]{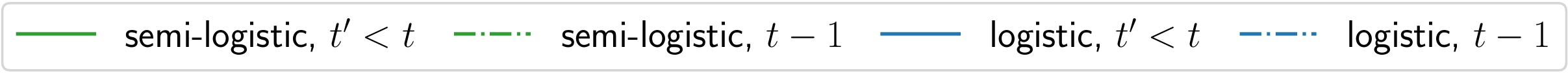}\\
\includegraphics[height=\figheight]{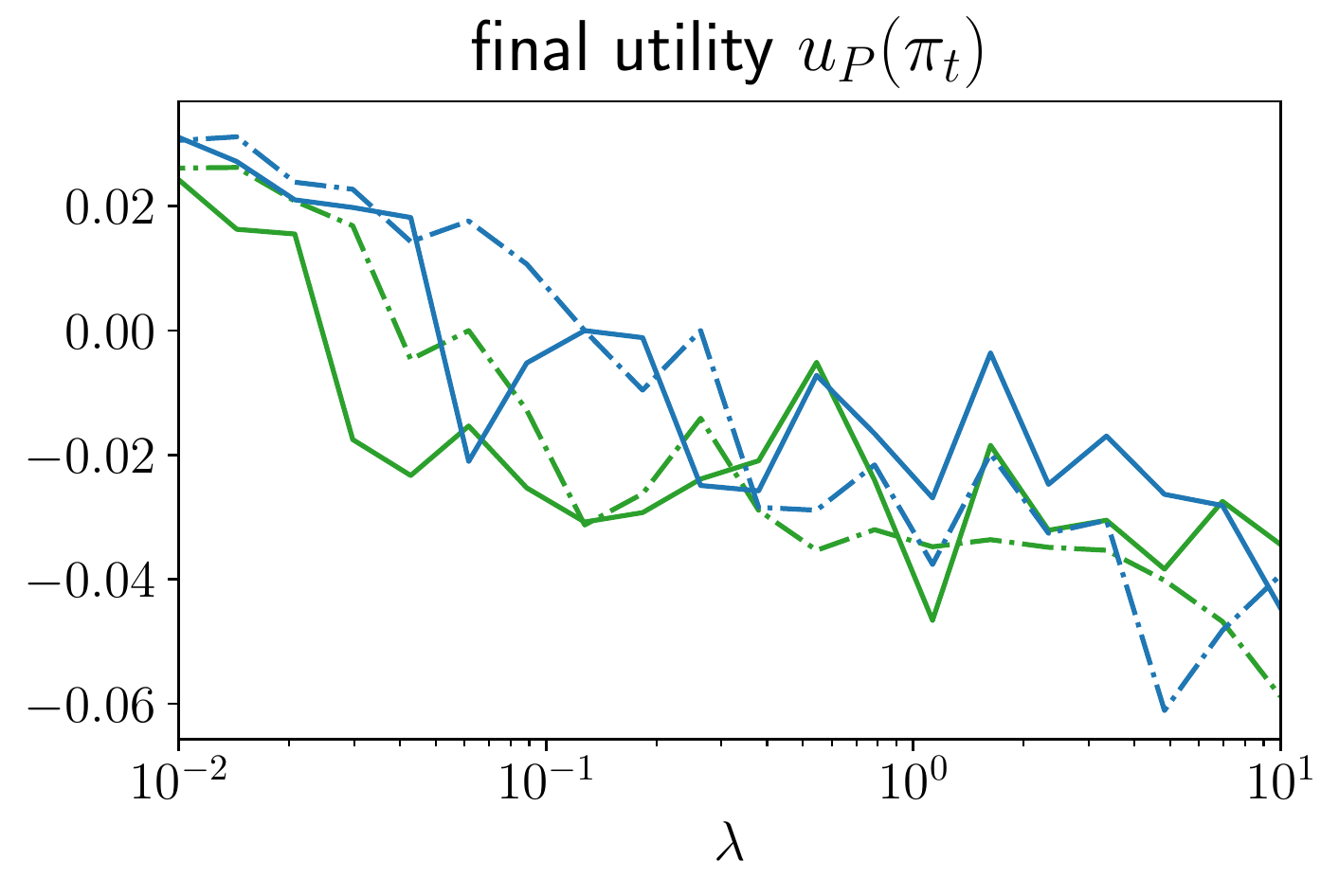}
\hspace{\figspacing cm}
\includegraphics[height=\figheight]{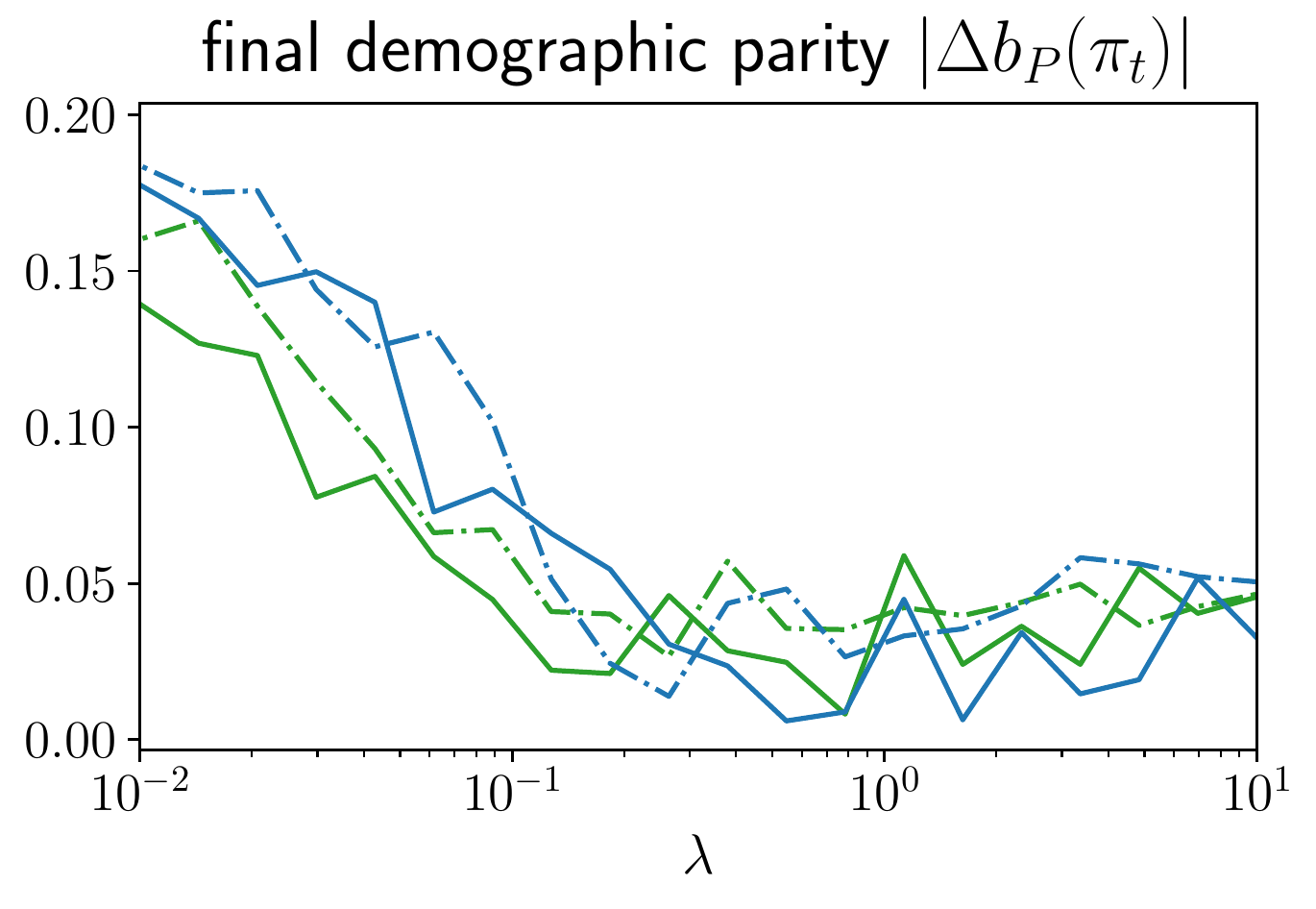}
\caption[Fairness evaluation of learning decisions on COMPAS data]{Fairness evaluation on COMPAS data for the final ($t=200$) policy as a function of $\lambda$ for demographic parity.
All quantities are estimated on the held-out set.}
\label{fig:results-real-lambda}
\end{figure}

Here, we use the COMPAS recidivism dataset compiled by ProPublica \citep{Angwin2016}, which comprises of information about criminal offenders screened through the COMPAS tool in Broward County, Florida during 2013-2014.
For each offender, the dataset contains a set of demographic features, the criminal history, and the risk score assigned by COMPAS.
Moreover, ProPublica collected whether or not these individuals were rearrested within two years of the screening.
In our experiments, $z \in \{0, 1\}$ indicates whether individuals were identified ``white'', $y$ indicates rearrest, and $d \sim \pi(x, z)$ determines whether an individual is let out on parole.
Again, $z$ is not used as an input.
We use 80\% of the data for training, where at each step $t$, we sample (with replacement) $N$ individuals, and the remaining 20\% as a held-out set to evaluate each learned policy in the population of interest.

We first summarize the results for $\lambda=0$, i.e., without fairness constraints in Figure~\ref{fig:results-real-time}.
A slight initial utility advantage of the deterministic threshold rule is quickly overcome by our exploring policies.
This is best seen when looking at \emph{effective utility}, the average utility accumulated by the decision maker on training data up to time $t$, for which our strategies dominate after $t=100$.
Hence, early exploration not only pays off to eventually be able to take better decisions, but also reaps higher profit during training.
Moreover, all strategies based on exploring policies consistently achieve lower violations of demographic parity than the deterministic threshold rules.
In summary, even without fairness constraints, i.e., in a pure utility maximization setting, exploring policies achieve higher utility and simultaneously reduce unfairness compared to deterministic threshold rules.

In Figure~\ref{fig:results-real-lambda}, we show how utility and demographic parity of the final policy $\pi_{t=200}$ changes as a function of $\lambda$ when constraining demographic parity.
As expected, while we are able to achieve perfect demographic parity, this comes with a drop in utility.
All remaining metrics under fairness constraints are shown in Figure~\ref{fig:results-real-final} in Appendix~\ref{chap:aistats:additional}.
Finally, two remarks are in order.
First, for real-world data we cannot evaluate the optimal policy and do not expect it to reside in our model class.
However, even when logistic models do not perfectly capture the conditional $\dP(Y = 1 \given x)$, our comparisons here are ``fair'' in that all strategies have equal modeling capacity.
Second, we take the COMPAS dataset as our (empirical) ground truth distribution even though it likely also suffered from selective labels.
To learn about the real distribution underlying the dataset, we would need to actually deploy our strategy.

\section{Conclusion}
\label{sec:conclusions}

In this chapter, we have analyzed consequential decision making using imperfect predictive models, which are learned from data gathered by potentially biased historical decisions.
First, we have articulated how this approach fails to optimize utility when starting with a non-optimal deterministic policy.
Next, we have presented how directly learning to decide with exploring policies avoids this failure mode while respecting a common fairness constraint.
Finally, we have introduced and evaluated a simple, yet practical gradient-based algorithm to learn fair exploring policies.

Unlike most previous work on fairness in machine learning, which phrases decision making directly as a prediction problem, we argue for a shift from ``learning to predict'' to ``learning to decide''.
In particular, we propose to not simply equate decisions with predictions obtained directly from limited available data, but to remain cognizant of how made decisions can affect and interfere with future data collection by continued exploration.
Not only does this lead to improved fairness in this context, but it also establishes connections to other areas such as counterfactual inference, reinforcement learning and contextual bandits.
Within reinforcement learning, it would be interesting to move beyond a static distribution $\dP$ by incorporating feedback from decisions or non-static externalities.
Moreover, since we have shown how shifting focus from learning predictions to learning decisions requires exploration, we hope to stimulate future research on how to explore ethically in different domains.

The crucial difference between mere predictions and actual decisions has been further highlighted in concurrent and later work \citep{kleinberg2018algorithmic,rambachan2020economic}.
There, the authors argue for a social planner taking decisions to maximize a social welfare function that may also include fairness preferences.
If the social planner has access to predictions from machine learning systems, it is optimal to keep the decision and prediction stage completely separate and train the machine learning pipeline to maximize predictive performance without any additional constraints.
This suggests that machine learning systems should indeed be used for prediction alone, but there is a separate optimization problem in deriving decisions from these predictions.
It is this second stage in which we must account for fairness considerations.
However, these approaches do not consider selective labels.
In general, as soon as the observed data depends on the chosen decision rule, the prediction and decision stage may not be easily decoupled anymore.
We believe the importance to clearly distinguish between predictions and decisions on a technical level reflects our arguments in \chapref{chap:intro} that discrimination is usually a systemic cultural issue.
Merely adjusting or improving predictive models does not suffice to make beneficial decisions.

\chapter{Conclusion}
\label{chap:conclusion}

In this thesis we developed methodology and tools to overcome certain restrictive assumptions commonly made in fair machine learning for consequential decision making. 
Thereby we broaden the applicability of such systems and bring them closer to application.
This chapter summarizes these contributions, draws conclusions, and suggests directions for future work.
At the end, we again put our contributions into the broader context of socially beneficial machine learning.

\section{Summary of contributions}

In \chapref{chap:causal} we applied a causal lens to see beyond observational group matching criteria and put the causal data generating mechanisms in the center of the discussion.
This allowed us to overcome fundamental limitations of observational fairness criteria and revealed some subtle---yet crucial---aspects relating to the meaning we assign to protected attributes.
The main constituents of our conceptual framework are resolving variables and proxy variables, which play a dual role in defining types of discrimination from a skeptic or benevolent viewpoint.
We developed a practical procedure to remove proxy discrimination given
the structural equation model is of a certain functional form, and analyzed a similar approach for unresolved discrimination.
Whilst not always feasible, the causal approach naturally
creates an incentive to scrutinize the data more closely and work out plausible
assumptions to be discussed alongside any conclusions regarding fairness.
In particular, our framework assumes that we have access to the true underlying causal graph.

Since this is a strong assumption in itself, in \chapref{chap:sensitivity} we also developed tools to analyze the sensitivity of counterfactually fair classifiers to potential unmeasured confounding.
A grid-based approach was introduced for confounding---modeled as unobserved correlation between error terms---between two variables, and an optimization-based approach for confounding in the general multivariate case.
These methods for sensitivity analysis are a step towards extending the applicability of causal methods to assess and mitigate discrimination in real-world settings.

In \chapref{chap:blindjustice} we addressed a dilemma of traditional approaches to fairness:
in order to enforce fairness, sensitive attributes must be examined;
yet in many situations, users may refuse to provide them, or modelers may be prohibited from collecting them.
We adopted and improved recent methods in secure multi-party computation and demonstrated that it is practical on real-world datasets to:
\begin{enumerate*}[label=\alph*)]
\item certify and sign a model as fair;
\item learn a fair model; and
\item verify that a fair-certified model has indeed been used;
\end{enumerate*}
all without users ever having to disclose their sensitive attributes in the clear to anyone and without the modeler having to disclose their model to any other party.
These techniques empower users to retain control over data and modelers to preserve their intellectual property rights while still being able to train fair models.
Thereby, our contributions are an important step towards jointly addressing concerns in privacy, algorithmic fairness and accountability.

\Chapref{chap:decisions} scrutinized the traditional framework of predictive modeling for consequential decisions.
Under certain conditions, predictive models can lead to optimal decisions, even in terms of fairness.
However, we argued that one of the required assumptions, namely that we have access to i.i.d.\ labeled data, is commonly violated in practice.
Often, outcomes only exist when a certain decision is made.
We have shown how deterministic threshold rules from predictive models fail in this scenario.
Further, we introduced an approach of directly learning decisions with exploring policies instead and demonstrated its efficacy on synthetic and real-world data via a simple, gradient-based implementation.
Our results strongly point towards the importance of distinguishing predictive modeling and decision making in consequential settings.
It thereby contributes both conceptually as well as methodologically to our understanding of the applicability of algorithmic fairness techniques.

\section{Conclusions and directions for future work}

This thesis is a first step towards being more cognizant of the confines of unrealistic technical assumptions underlying traditional approaches to fair machine learning and revealing paths to overcome these limitations.
\emph{One of the key conclusions from our findings is that as machine learning researchers, we must expand our horizon in terms of what we factor into the modeling process.}

The causal framework introduced in \chapref{chap:causal} suggests deeper scrutiny of how the data came about, forcing us to bring the data collection phase to the center of our attention.
It requires us to hypothesize or ideally even validate assumptions about why the data are what the data are.
Making such assumptions is difficult and typically requires a much deeper understanding of the subject matter than machine learning researchers or engineers can be expected to bring to the table.
However, from our findings we conclude that instead of avoiding the discomfort of making and communicating assumptions all together, to build socially beneficial systems we must start to become comfortable with potentially untestable modeling assumptions.
As a positive side effect, we hope that such discomfort will encourage researchers and engineers to initiate a dialog with domain experts to seek further backing or refutation of their assumptions.
From a technical angle, \chapref{chap:sensitivity} also develops methodological tools to alleviate the uncertainty around potentially wrong assumptions.

Looking forward, we believe that causal modeling of societal interactions between humans and machines will play a crucial part in building fair, accountable and explainable systems.
In future work, concrete technical advances could be to extend our procedures for avoiding causal measures of discrimination to larger function classes and allowing for a more fine-grained division into fair and unfair pathways.
More broadly, we will also need to improve our understanding of the ontological stability of various social constructs that are relevant for fair machine learning systems, such as protected attributes.
Identifying ontologically stable concepts as well as robust causal mechanisms in socioalgorithmic systems seems to be a challenge that can only be tackled effectively in an interdisciplinary endeavor.
Since causal insights can only be obtained when we are willing to make assumptions (or provide existing knowledge), an interesting and important direction for future work is also to develop more flexible tools for sensitivity analysis.
Concretely, how can we scale our method for unobserved confounding efficiently to large systems, how can we formulate restrictions on observed confounding more flexibly, or how can we extend the formalism to also allow for structural misspecifications of edges in the causal graph?
These tools can uncover when there is ``wiggle room'' for causal assumptions, which alleviates the cognitive burden of making them confidently, which may otherwise lead to inaction.

In \chapref{chap:blindjustice}, it became apparent that the modes of practical data availability and transportability must also be taken into account by machine learning researchers and data scientists.
As the collection, storage and transport of data become a key driver of economic success as well as subject to data privacy regulations, we must move away from the common idea that ``all data is always available everywhere''.
Instead, data availability must become a first-class citizen in the modeling process, especially for fairness sensitive applications.
Therefore, we believe that further adaptations of existing as well as developments of new cryptographic techniques for privacy preserving machine learning will be fruitful directions for future work.
Specifically, it would be interesting to combine our methods from \chapref{chap:blindjustice}, which ensure data security in transport and at rest, with privacy guarantees from differential privacy.
More broadly, can we leverage ideas from (partially homomorphic encryption), secure multi-party computation, or zero-knowledge proofs to hold modelers accountable by proving publicly that their models satisfy given fairness properties without having to disclose their intellectual property?
Can we practically scale such applications to large models and datasets?

Finally, just like modeling should take into account data collection, it must also consider the future impact of decisions.
\Chapref{chap:decisions} demonstrated the importance of including downstream consequences in the selective labels setting, where future observations depend on the current model.
As pointed out in the introductory chapters, this is but one possibility of a broader phenomenon sometimes referred to as \emph{performativity} \citep{perdomo2020performative}.
While in our setup we assumed the ground truth distribution to be static, an interesting direction for future work would be to also allow the ground truth distribution to change over time depending on how decisions are taken.
We currently lack reliable models of how humans react to and adjust their behavior when facing automated decision systems.
An important question for the future will be how to navigate the fine line between building models that are robust to individual attempts of cheating the system while broadly incentivizing desirable behavior.
This hints at the intersection between mechanism design, game theory, and machine learning to be a promising area for further advances in socially beneficial automated decision making.

\emph{Tied together, our contributions on relaxing specific assumptions at different stages of the data science loop all highlight the importance of taking into account the entire interacting socioalgorithmic system from the get-go.}

As a second insight, we note that our solutions predominantly rely upon combining and advancing recent ideas and concepts from \emph{different fields of research}.
\Chapref{chap:causal} raises questions in the intersection of sociology, philosophy and causal inference regarding ontological stability of socially constructed quantities.
In \chapref{chap:sensitivity}, we combined advances in efficient gradient-based optimization and modern, highly tuned auto differentiation frameworks to perform sensitivity analysis without one-size-fits-all identifiability assumptions.
In \chapref{chap:blindjustice}, we exploited recent theoretical and implementation advances in cryptography, particularly secure multi-party computation, to reconcile fairness with privacy and accountability concerns.
Finally, in \chapref{chap:decisions}, we borrowed ideas from bandit and reinforcement learning to uncover inadequate assumptions about the interaction between predictions and consequential decisions.
Adapting concepts around exploration versus exploitation allowed us to accommodate downstream effects of decisions in our models, and raised further interesting questions about the relation between exploration, stochasticity and ethical decision-making.

\emph{Hence, we conclude that when it comes to fair machine learning systems, bridging the gap between theory and practice appears to require a truly interdisciplinary approach that will benefit from considering advances in adjacent as well as more distant fields of research.}

\section{Final remarks}

Let us end by circling back to the introduction, reflecting on the goals we set ourselves and to what extent we achieved them.
There are myriads of ways in which injustice and discrimination manifest themselves in our society.
We have briefly elaborated on some of the broader challenges of achieving fairness when making consequential decisions in \chapref{chap:intro}.
While we repeatedly emphasized that machine learning certainly cannot solve these problems single-handedly, we ended \secref{sec:takeaway} on an optimistic note.
Hopefully, the subsequently developed concepts and techniques support this optimistic view in showing that there are paths to extend the narrow, unrealistic formalization of ``fair consequential decisions'', perhaps even to the point where their application in the real world does more good than harm.
That said, let us now re-emerge from these formally defined settings and acknowledge that racism, gender discrimination, xenophobia, and stereotyping are still deeply ingrained in our society.
While data-driven decision systems play a role, I appeal to all of us to listen to and learn from the ones suffering the hardship so we can give them agency.
Smarter algorithms are no substitute for empathy, mutual respect, and activism.

\begin{spacing}{0.9}
\bibliographystyle{refstyle}
\cleardoublepage
\bibliography{references}
\end{spacing}

\begin{appendices}
\chapter{Additional experiments for \chapref{chap:blindjustice}}
\label{chap:icml:appendix}

\graphicspath{{figs/chap6/}}

\section{Results on remaining datasets}

Analogously to Figures~\ref{fig:ppercent_acc} and Figure~\ref{fig:ppercent_fair} we report the results on test accuracy as well as the mitigation of disparate impact for the Lagrangian multiplier method in Figure~\ref{fig:ppercent_rest_accs}.
In the Adult dataset we are able to mitigate disparate impact with slightly worse accuracy as compared to the baseline.
Note that the German dataset contains only 512 training and 200 test examples, which explains the discrete jumps in accuracy in minimal steps of $\nicefrac{1}{200} = 0.005$.
Hence, even though the Lagrangian multiplier technique here consistently removes disparate impact to a similar extent as the baseline, interpretations of results on such small datasets require great care.
For the much larger stop, question and frisk dataset we again observe the curious initial increase in accuracy similar to our observations for the Bank dataset.
In this dataset about $93\%$ of all samples have positive labels, which explains the near optimal accuracy when collapsing to always predict 1, which happens for the baseline as well for our method at a similar rate as $c$ decreases.

\begin{figure}
\centering
\includegraphics[width=\textwidth]{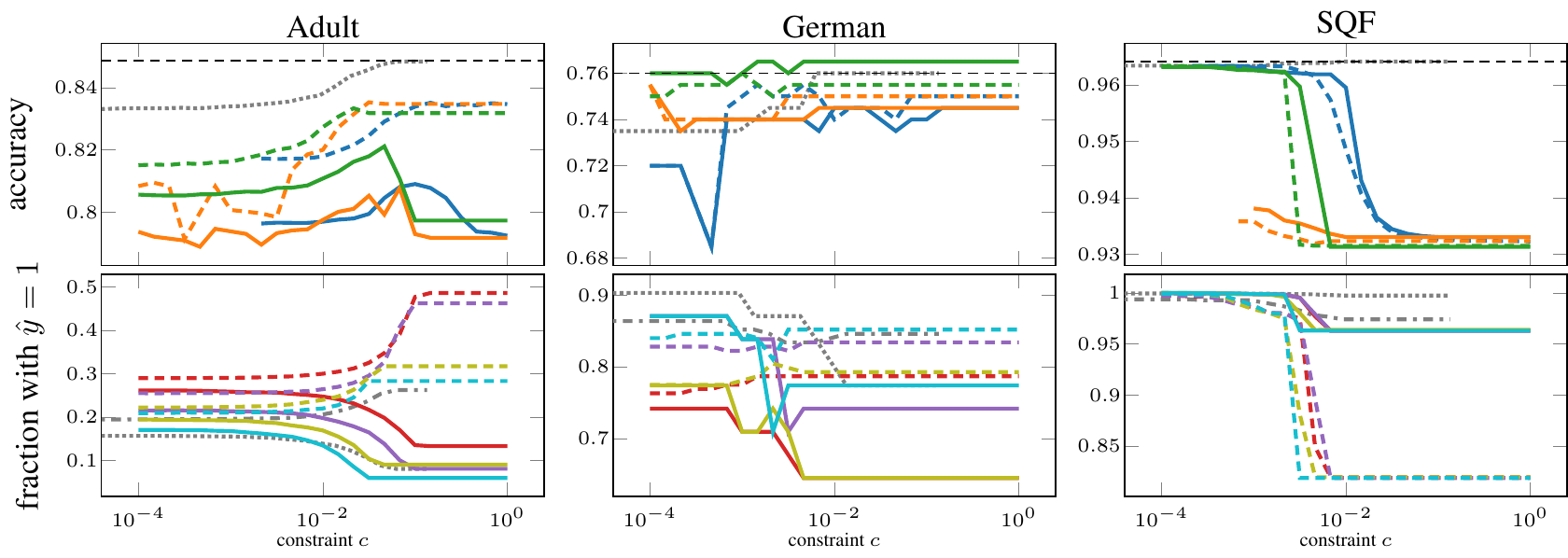}
\caption[Accuracy and fractions of positive outcomes across sensitive groups for different optimization settings and datasets]{\textbf{First row:} The color code is {\color{py-3-1} blue: iplb}, {\color{py-3-2} orange: projected}, {\color{py-3-3} green: Lagrange} with \emph{continuous} lines for no approximation and \emph{dashed} lines for piecewise linear approximation.
The gray dotted line is the baseline and the dashed black line marks unconstrained logistic regression.
\textbf{Second row:} \emph{Continuous/dotted} lines correspond to $z=0$ and \emph{dashed/dash-dotted} lines to $z=1$. The color code is
({\color{py-4-1}red: no approx. + float},
{\color{py-4-2}purple: no approx. + fixed},
{\color{py-4-3}yellow: pw linear + float},
{\color{py-4-4}turquoise: pw linear + fixed},
gray: baseline).}
\label{fig:ppercent_rest_accs}
\end{figure}

\begin{figure}
\centering
\includegraphics[width=\textwidth]{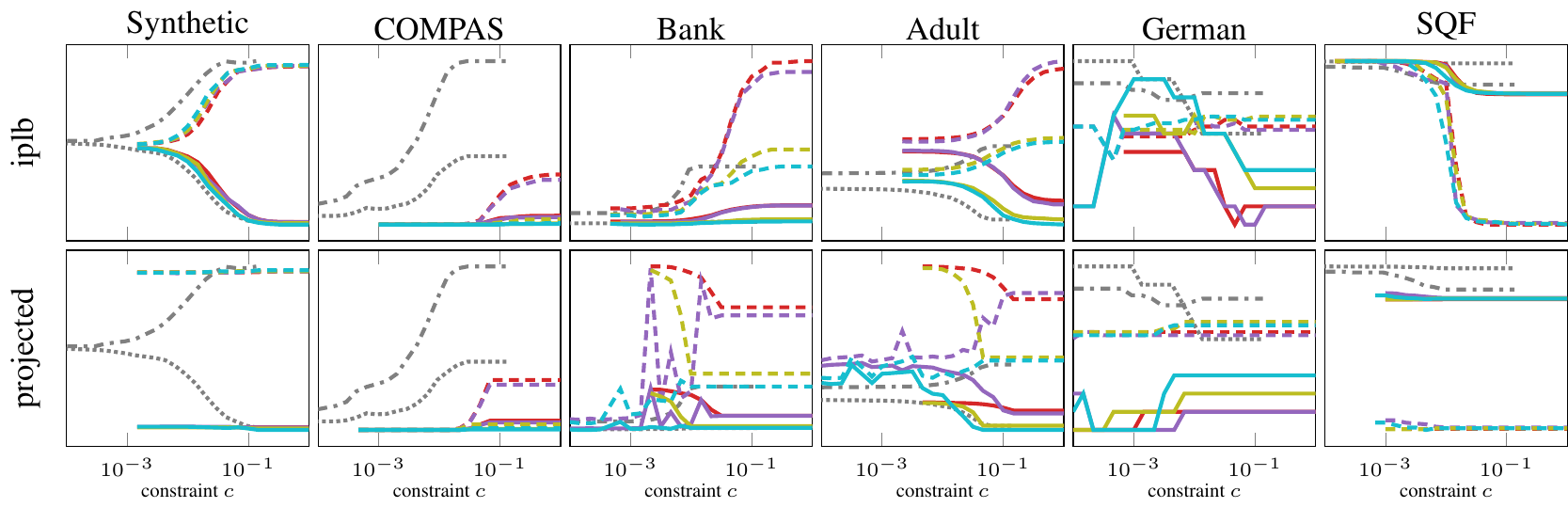}
\caption[Fractions of positive outcomes across sensitive groups for different optimization settings and datasets]{We plot the fraction of people with $z=0$ (\emph{continuous/dotted}) and with $z=1$ (\emph{dashed/dash-dotted}) who get assigned positive outcomes over the constraint $c$ for 6 different datasets. The different colors correspond to
({\color{py-4-1}red: no approximation + floats},
{\color{py-4-2}purple: no approximation + fixed-point},
{\color{py-4-3}yellow: piecewise linear + floats},
{\color{py-4-4}turquoise: piecewise linear + fixed-point},
gray: baseline).}
\label{fig:iplb_and_projected}
\end{figure}

\section{Disadvantages of other optimization methods}

In \secref{sec:icml:experiments} we suggest the Lagrangian multiplier technique for fair model training using fixed-point numbers.
Here we substantiate this suggestion with further empirical evidence.
Figure~\ref{fig:iplb_and_projected} shows analogous results to Figure~\ref{fig:ppercent_fair} and the second row of Figure~\ref{fig:ppercent_rest_accs}.
These plots reveal the shortcomings of the interior point logarithmic barrier and the projected gradient methods.

\paragraph{Interior Point Logarithmic Barrier method.}
While the interior point logarithmic barrier method does balance the fractions of people being assigned positive outcomes between the two different demographic groups when the constraint is tightened, it soon breaks down entirely due to overflow and underflow errors.
The number of failed runs was substantially higher than for the Lagrangian multiplier technique.
As explained by \citet{boydsbook}, when we increase the parameter~$t$ of the interior point logarithmic barrier method during training, the barrier becomes steeper, approaching the function
\begin{equation*}
I_{-}(x) =
\begin{cases}
    0 & \text{for } x \le 0 \eqc \\
    \infty & \text{for } x > 0 \eqp
  \end{cases}
\end{equation*}
From this it becomes obvious that when facing tight constraints, the gradients might change from almost zero to large values within a single update of the parameters~$\btheta$.
Moreover, iplb requires careful tuning and scheduling of~$t$.
Hence, the interior point logarithmic barrier method, while achieving good results over some domains, is not well suited for MPC.

\paragraph{Projected gradient method.}
In Figure~\ref{fig:iplb_and_projected}, we observe that the projected gradient method seems to fail in most cases, since it does not actually balance the fractions of positive outcomes across the sensitive groups.
There is a simple explanation why it can satisfy the constraint~$\F(\btheta) \le 0$ for the $p\%$-rule even with small~$\const$ and still retain near optimal accuracy.
Note that the accuracy only depends on the direction of~$\btheta$, i.e., it is invariant to arbitrary rescaling of~$\btheta$.
Since the constraint~$\F(\btheta) = |\A \btheta | - \const \le 0$ is always satisfied for $\btheta = 0$, dividing any $\btheta$ by a large enough factor will result in a classifier that achieves equal accuracy and satisfies the constraint (by continuity).
However, minimizing the loss in the original logistic regression optimization problem (or equivalently maximizing the likelihood), which is not invariant under rescaling of $\btheta$, counteracts shrinking $\btheta$ as it enforces high confidence of decisions, i.e., large~$\btheta$.
The projection method produces high accuracy classifiers with small weights that formally fulfill the fairness constraint, but do not properly mitigate disparate impact as measured by the true $p$\%-rule instead of the computational proxy.
It also often fails for small constraint values, as the projection matrix in eq.~\eqref{eq:projected} turns out to become near singular producing over- and underflow errors.

\chapter{Additional experiments for \chapref{chap:decisions}}
\label{chap:aistats:additional}

\graphicspath{{figs/chap7/}}

\begin{figure}
\centering
\def\figheight{3.5cm}
\includegraphics[width=\textwidth]{legend-synthetic-time.pdf}\\
\includegraphics[height=\figheight]{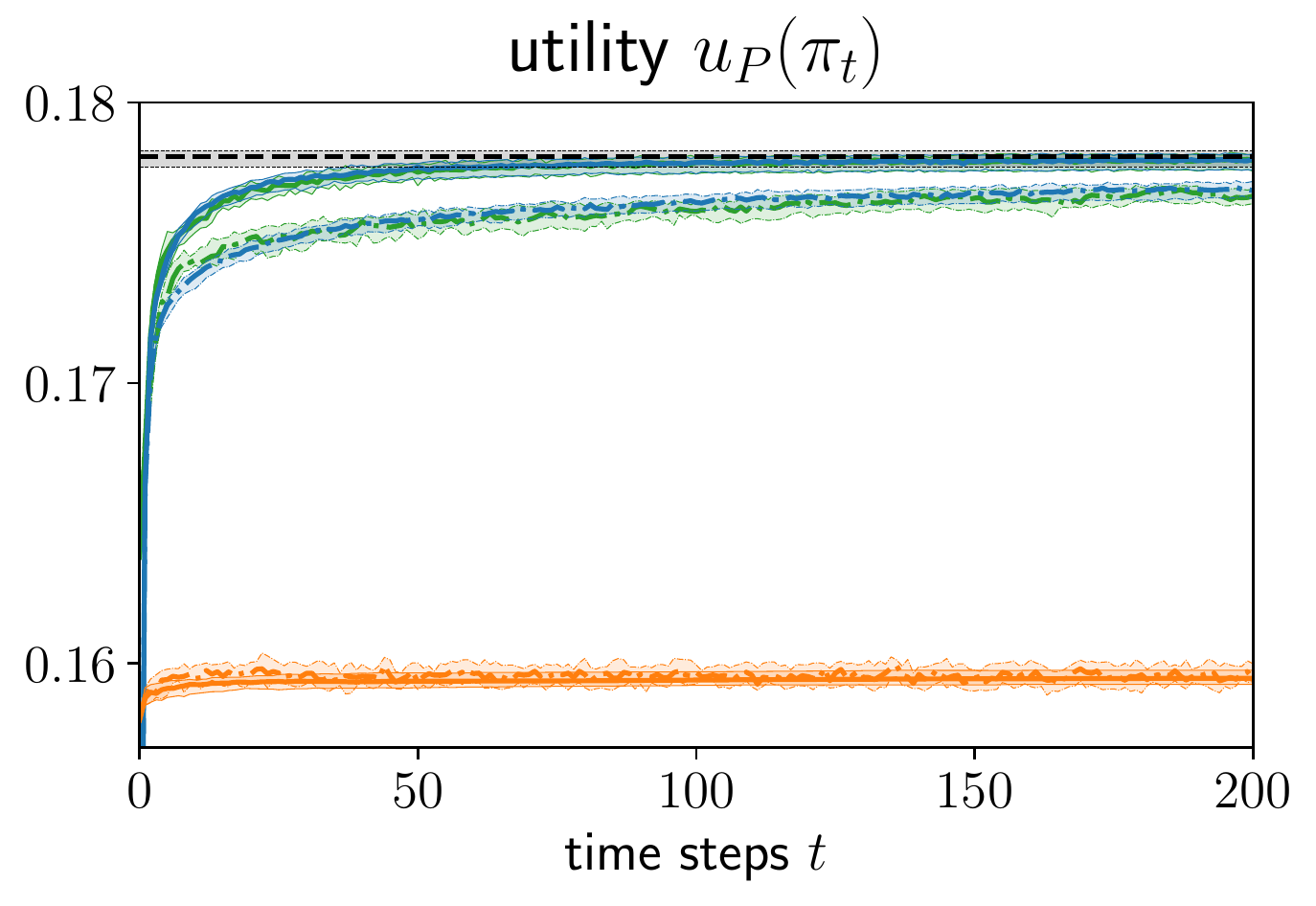}%
\hfill
\includegraphics[height=\figheight]{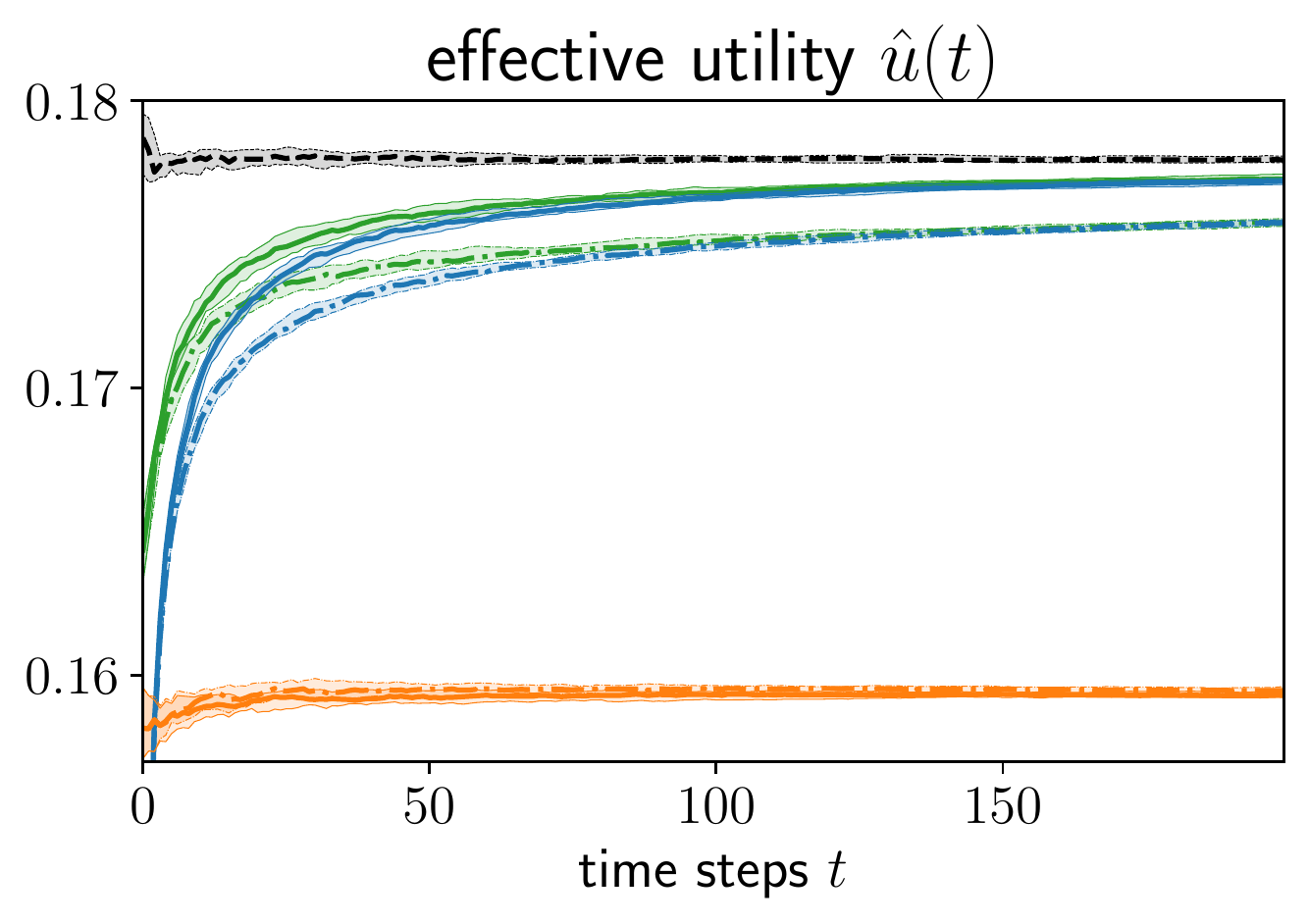}
\hfill
\includegraphics[height=\figheight]{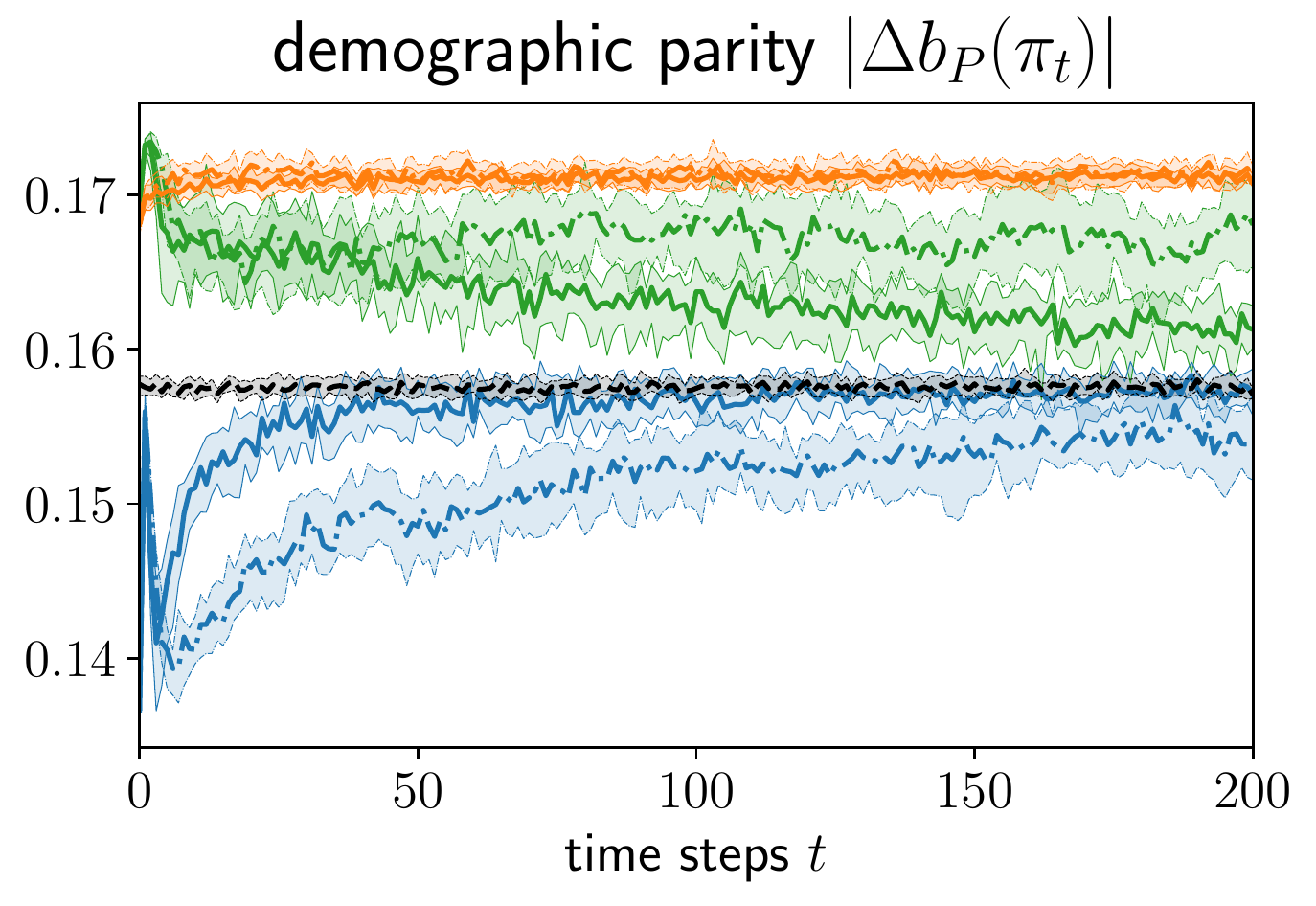}%
\\
\includegraphics[height=\figheight]{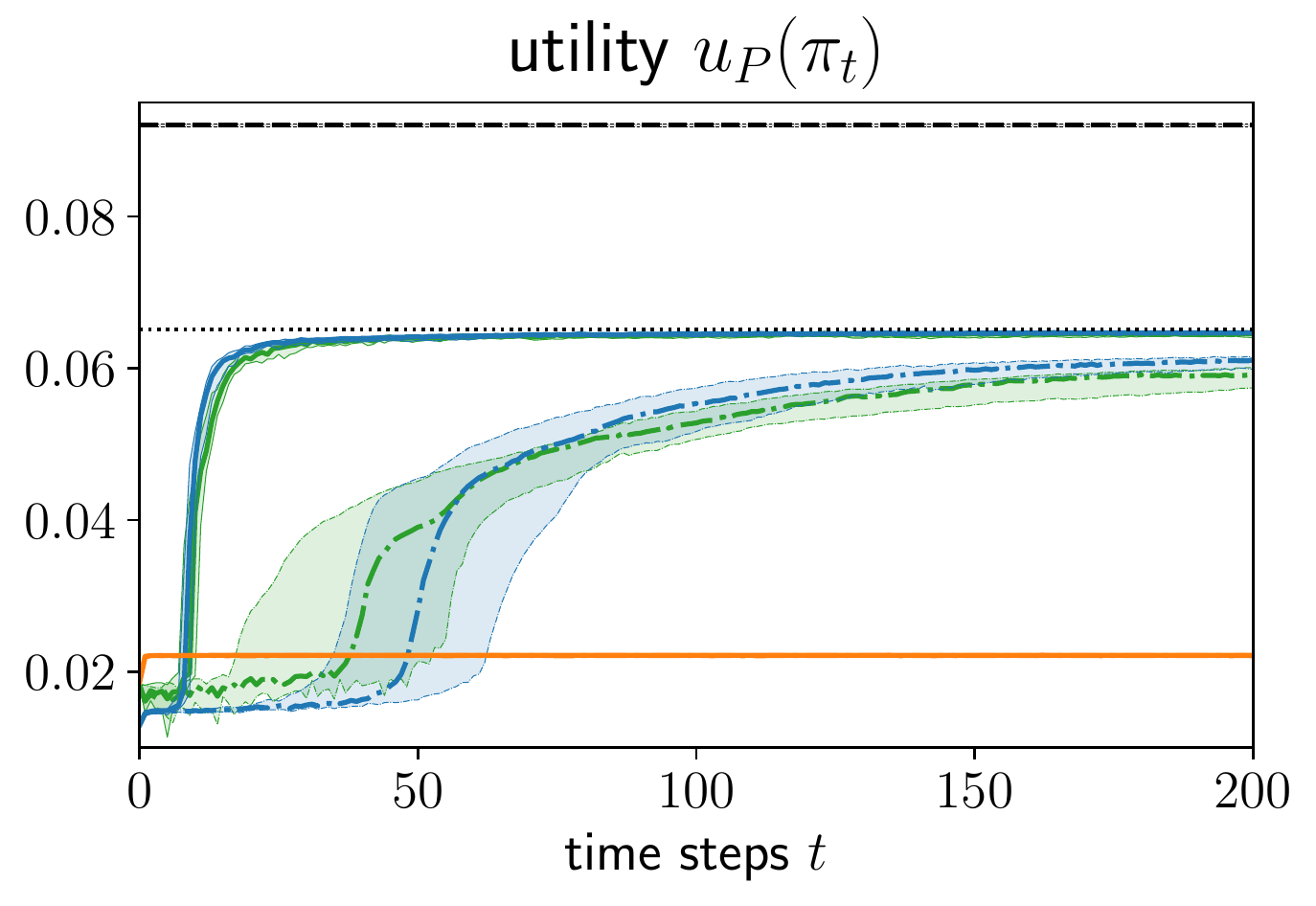}%
\hfill
\includegraphics[height=\figheight]{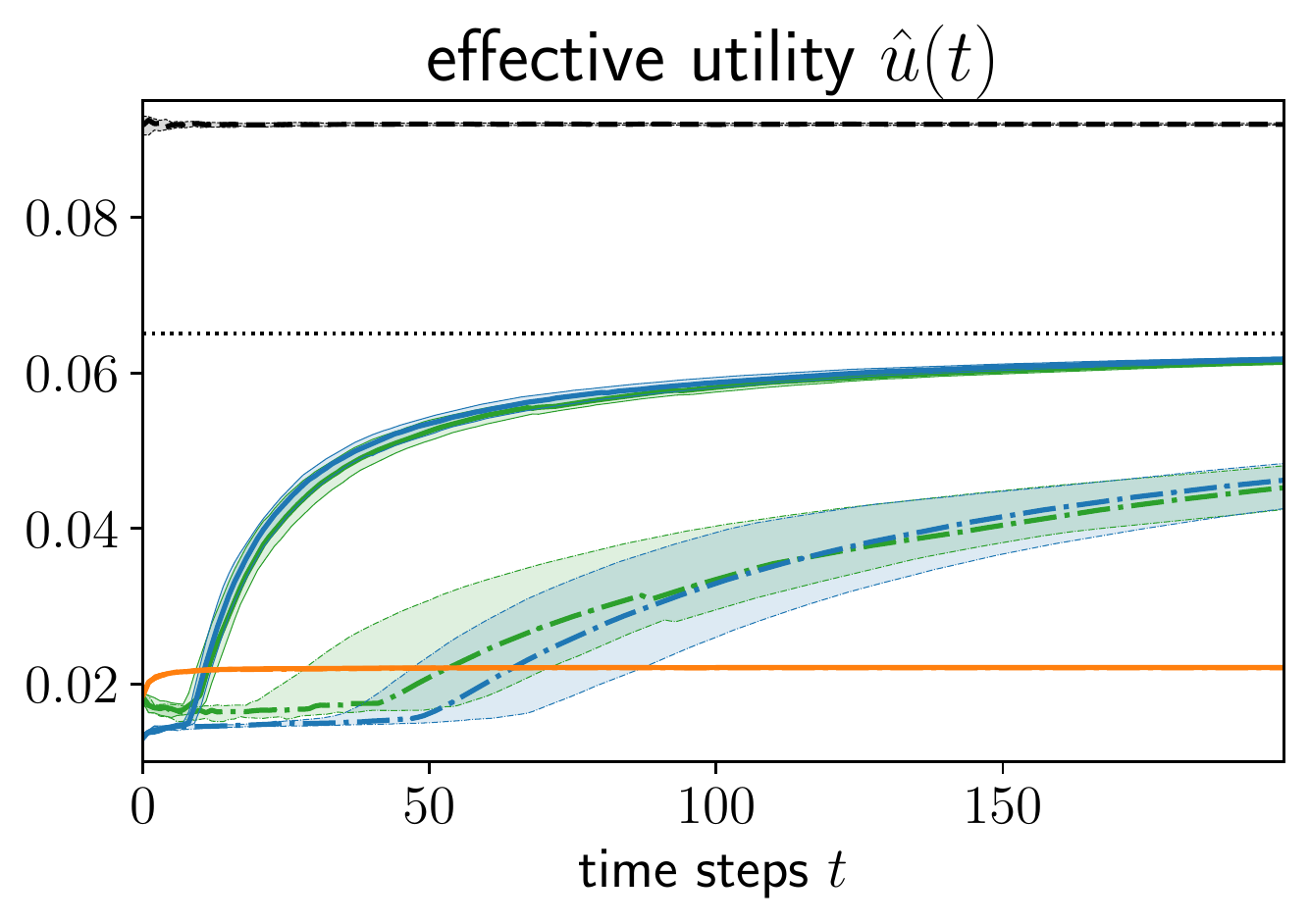}
\hfill
\includegraphics[height=\figheight]{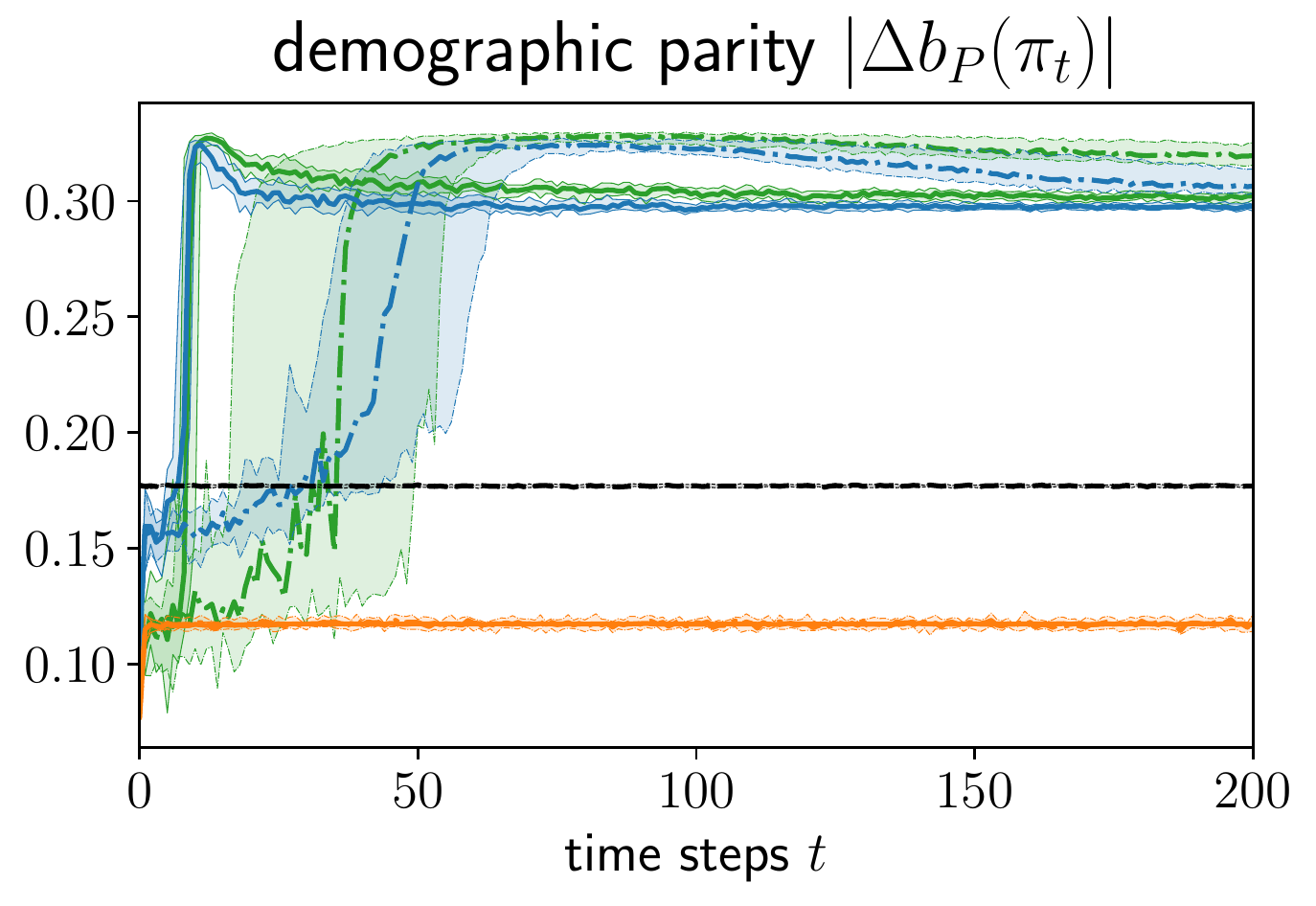}%
\caption[Utility and fairness of learning decisions with exploring policies in two synthetic settings]{Utility, effective utility, and demographic parity in the synthetic settings of Figure~\ref{fig:synthetic-setting}.}
\label{fig:results-synthetic-errs}
\end{figure}

\begin{figure}
\captionsetup[subfigure]{labelformat=empty}
\captionsetup[subfloat]{farskip=-1mm,captionskip=-0.5mm}
\centering
\def\figwid{0.27}
\subfloat[\tiny\textbf{deterministic, $t'<t$}]{%
\includegraphics[width=\figwid \textwidth]{%
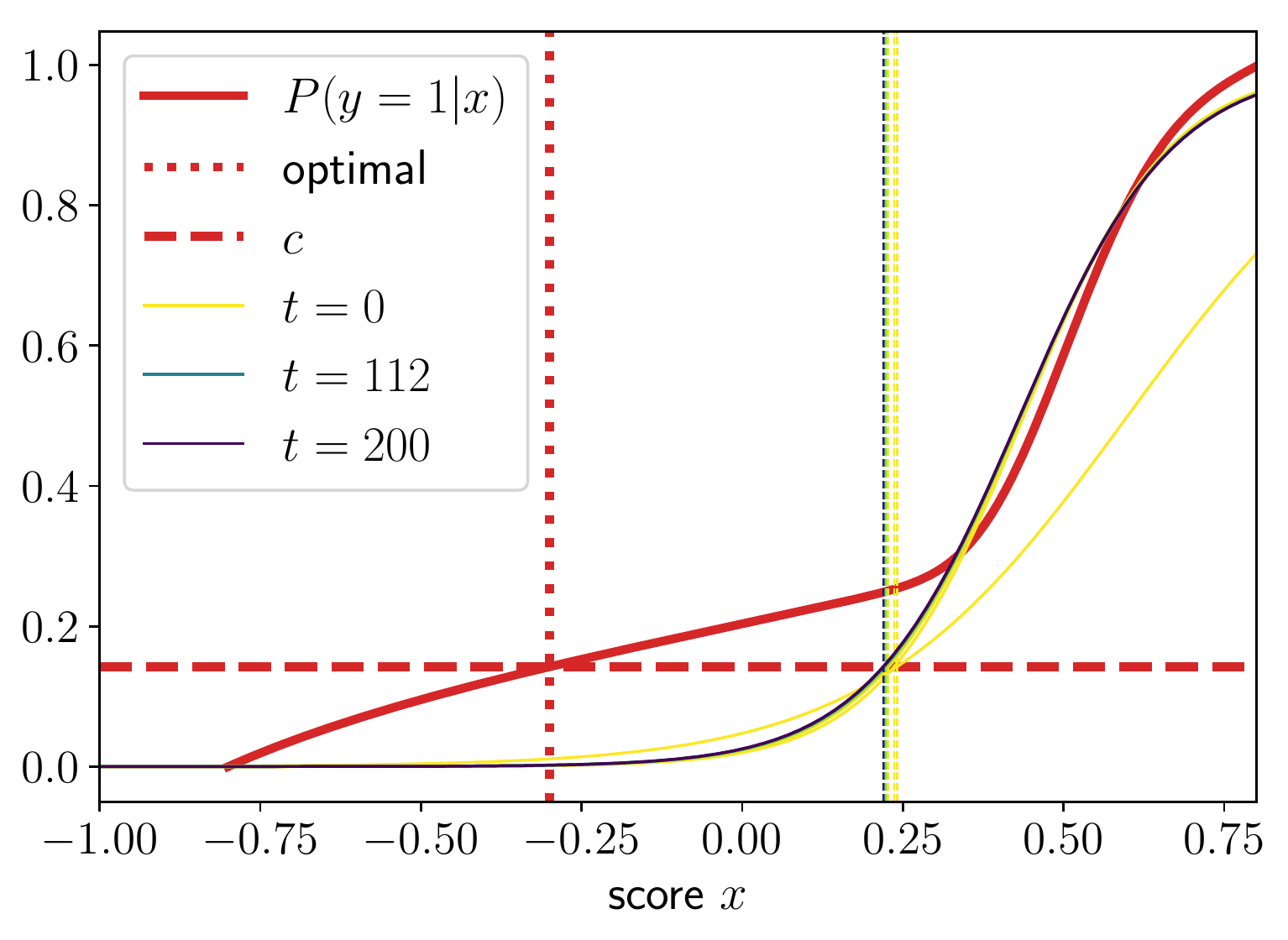}}
\hspace{0.5cm}
\subfloat[\tiny\textbf{deterministic, $t'<t$}]{%
\includegraphics[width=\figwid \textwidth]{%
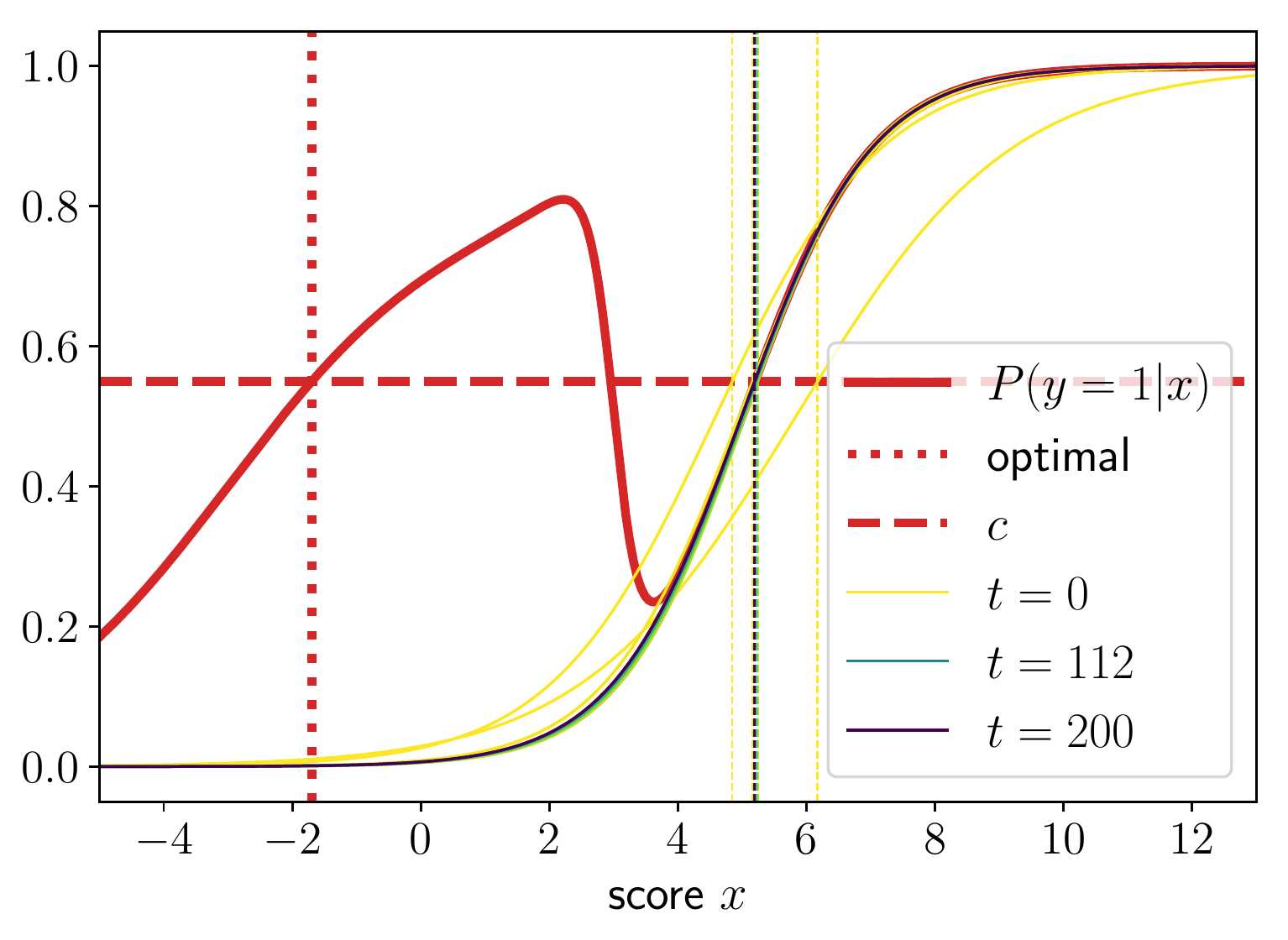}}
\\
\subfloat[\tiny\textbf{deterministic, $t-1$}]{%
\includegraphics[width=\figwid \textwidth]{%
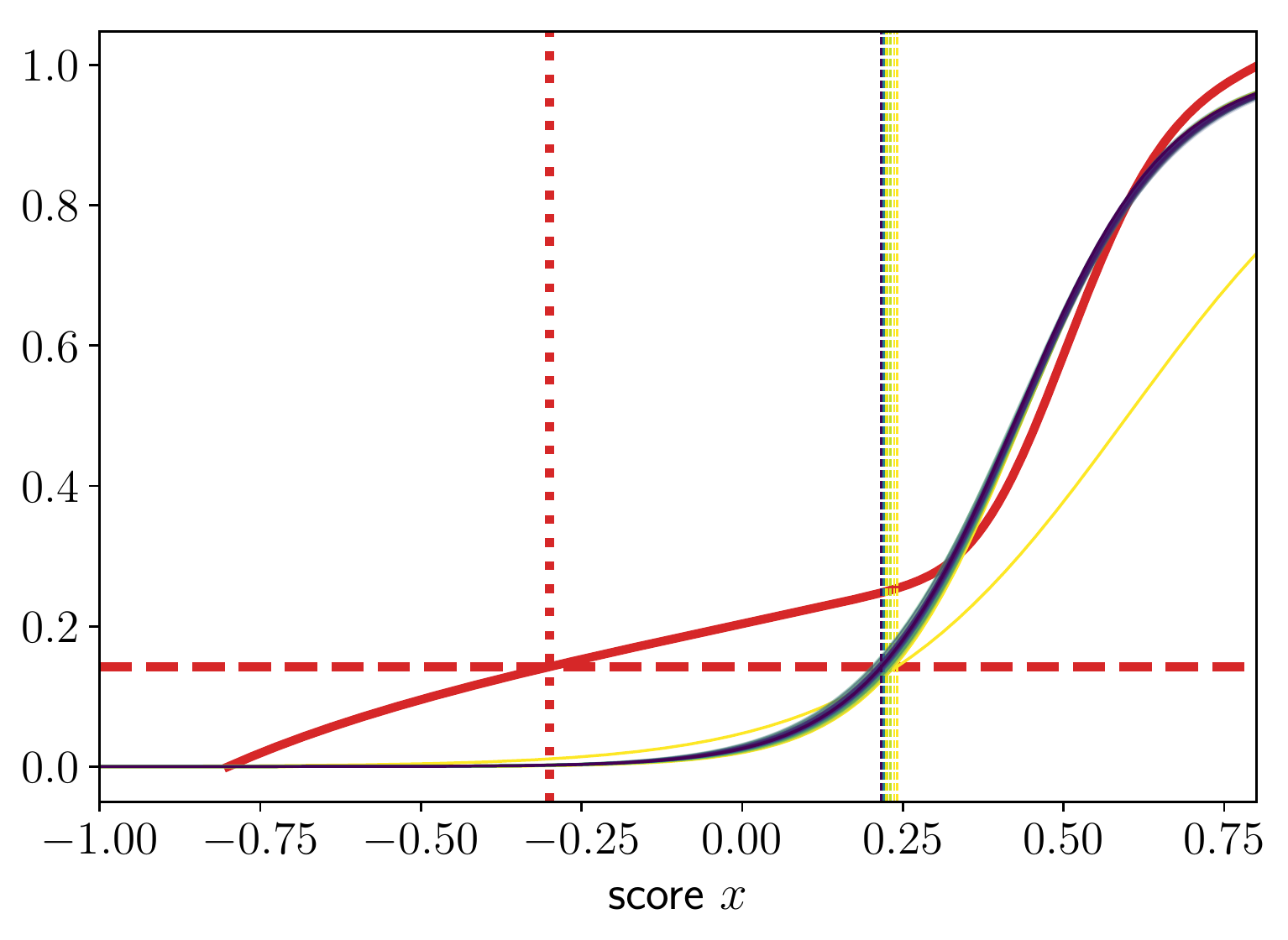}}
\hspace{0.5cm}
\subfloat[\tiny\textbf{deterministic, $t-1$}]{%
\includegraphics[width=\figwid \textwidth]{%
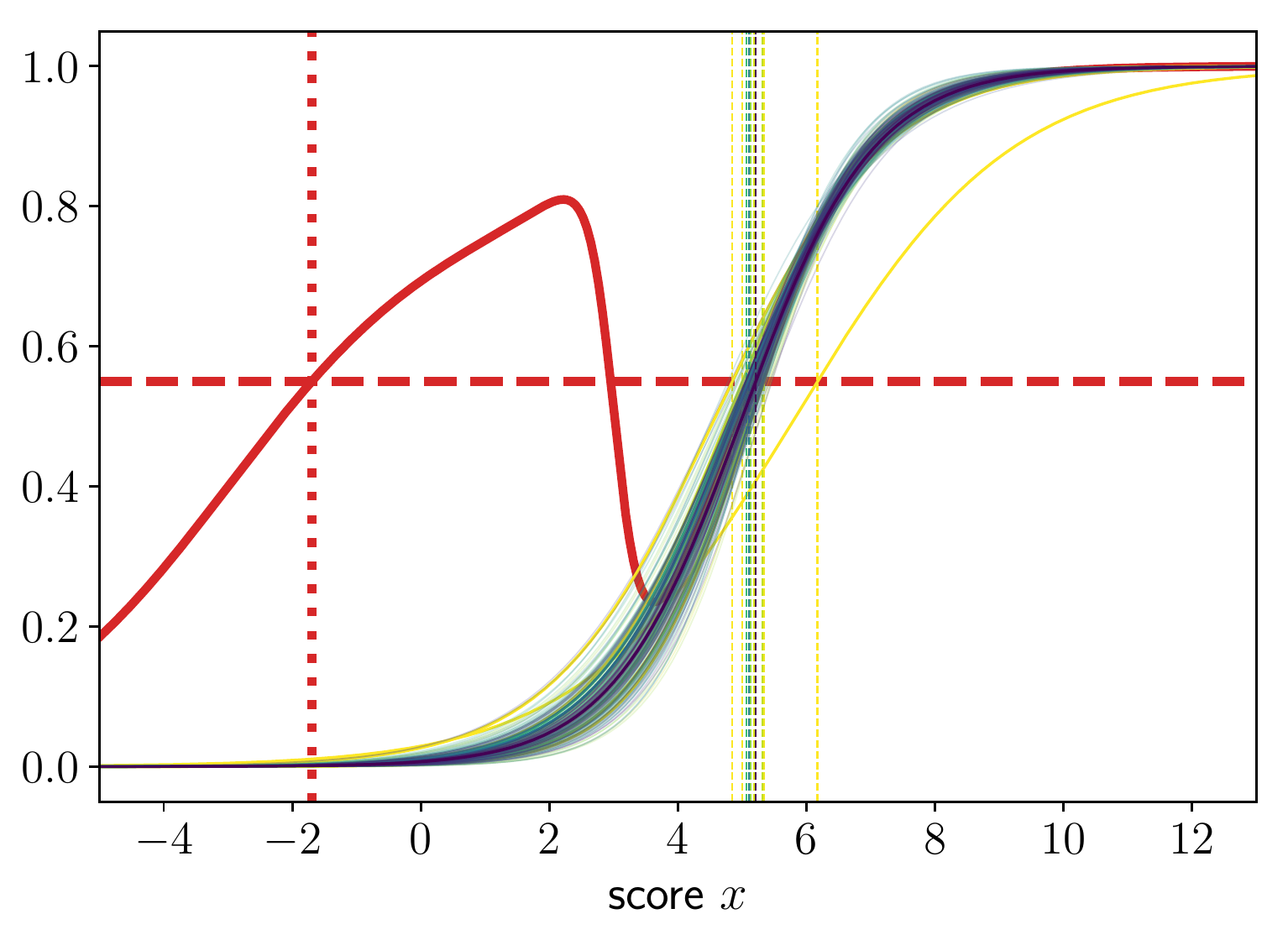}}
\\
\subfloat[\tiny\textbf{logistic, $t' < t$}]{%
\includegraphics[width=\figwid \textwidth]{%
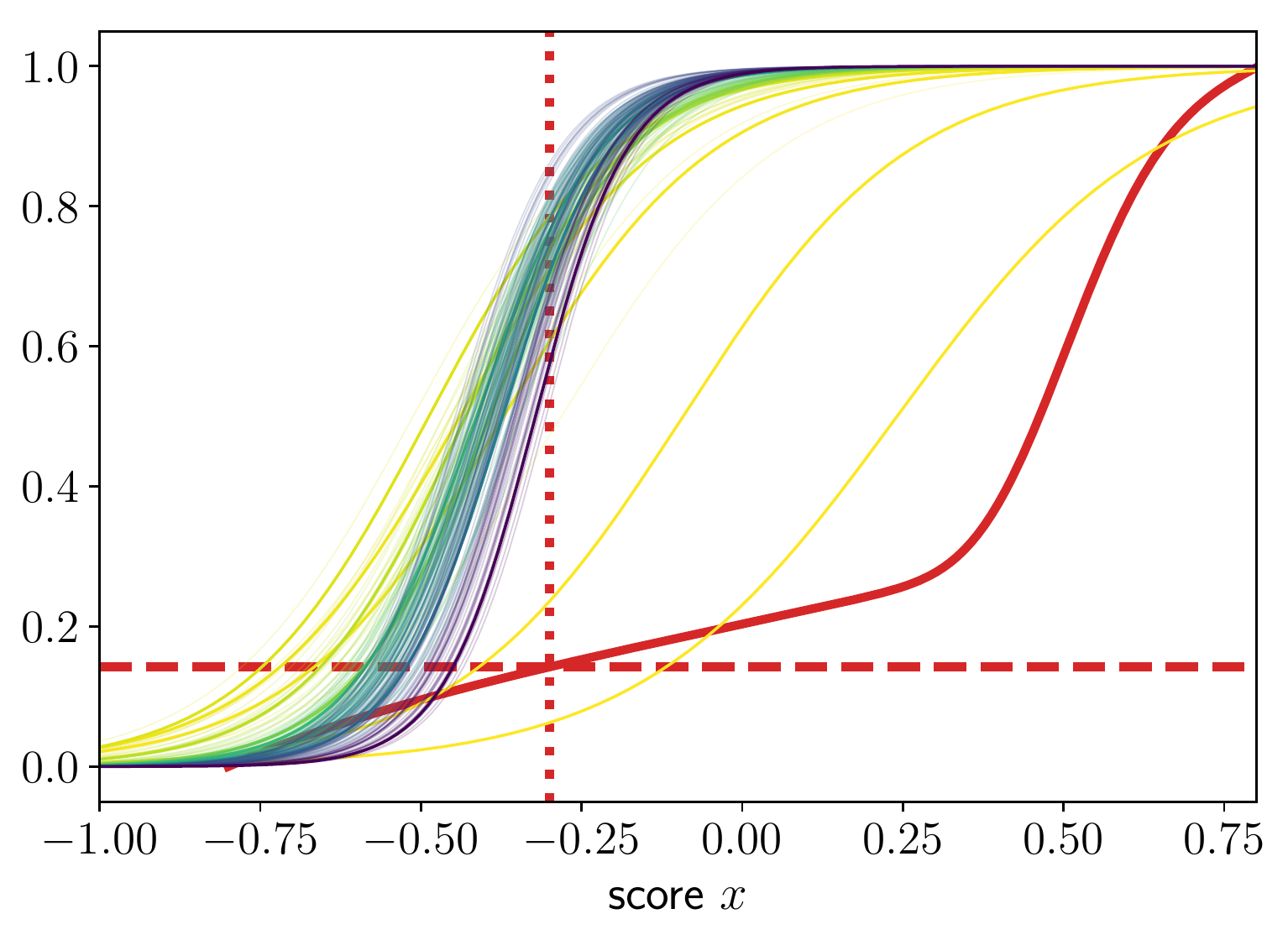}}
\hspace{0.5cm}
\subfloat[\tiny\textbf{logistic, $t' < t$}]{%
\includegraphics[width=\figwid \textwidth]{%
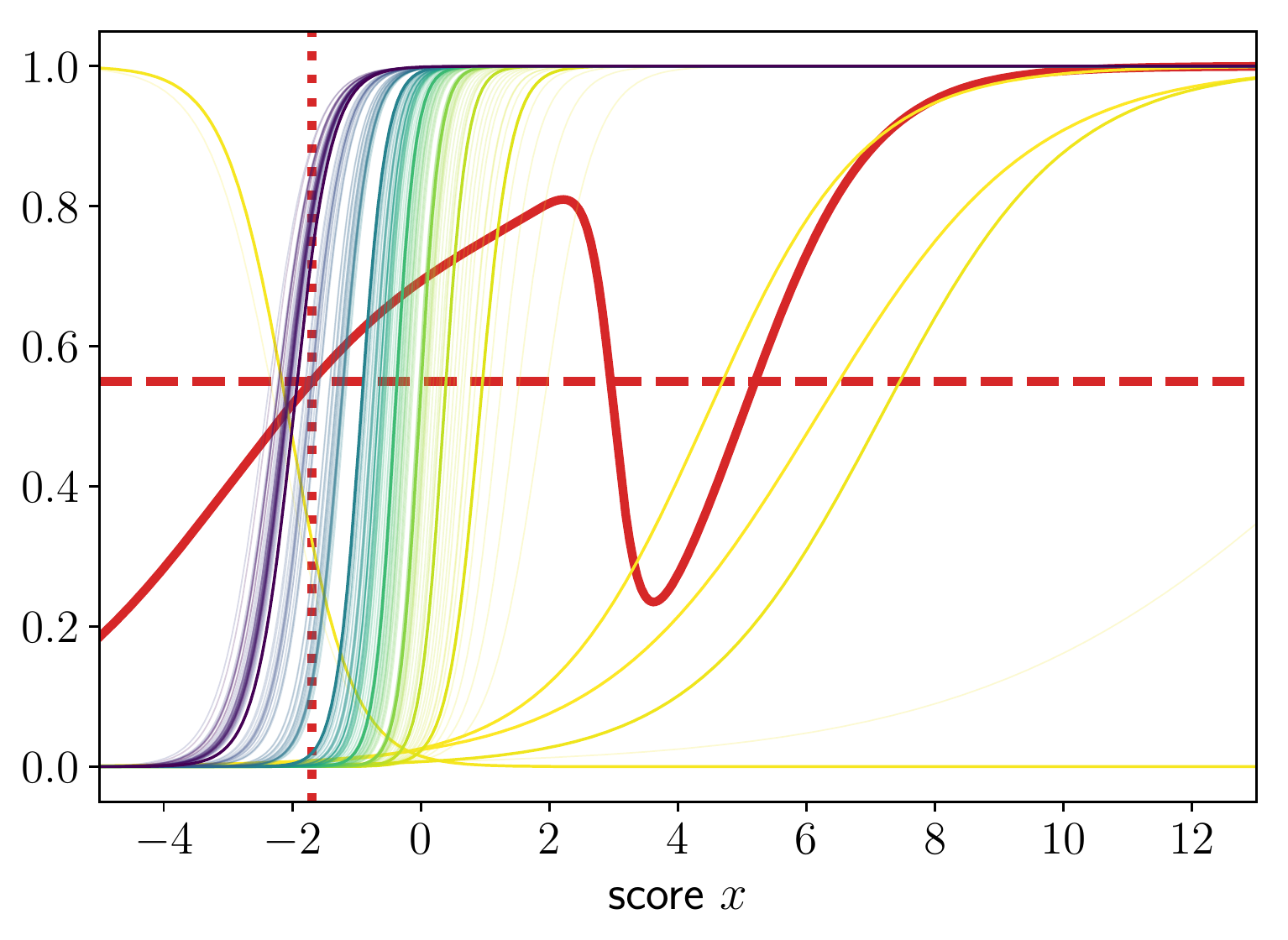}}
\\
\subfloat[\tiny\textbf{logistic, $t-1$}]{%
\includegraphics[width=\figwid \textwidth]{%
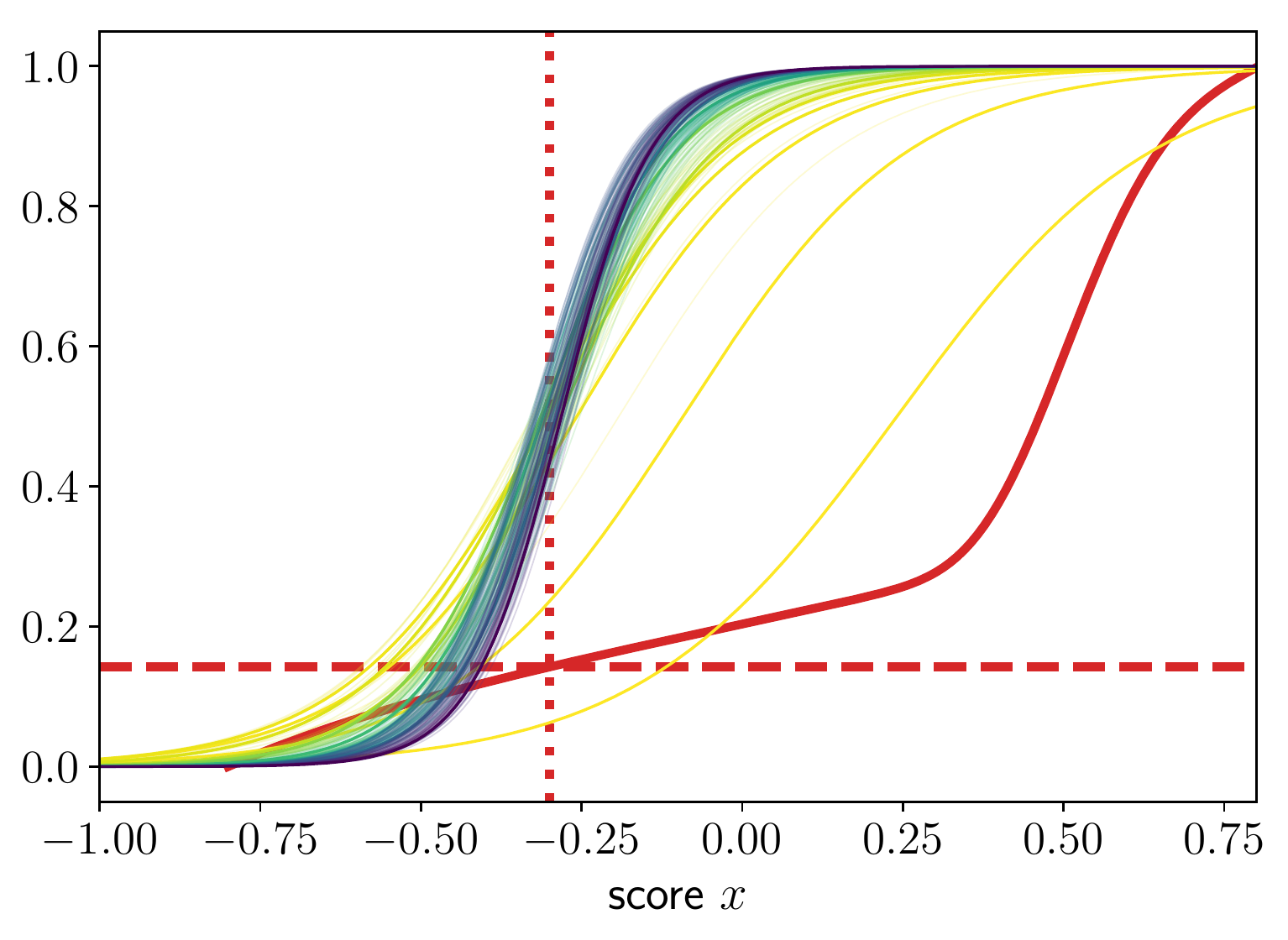}}
\hspace{0.5cm}
\subfloat[\tiny\textbf{logistic, $t-1$}]{%
\includegraphics[width=\figwid \textwidth]{%
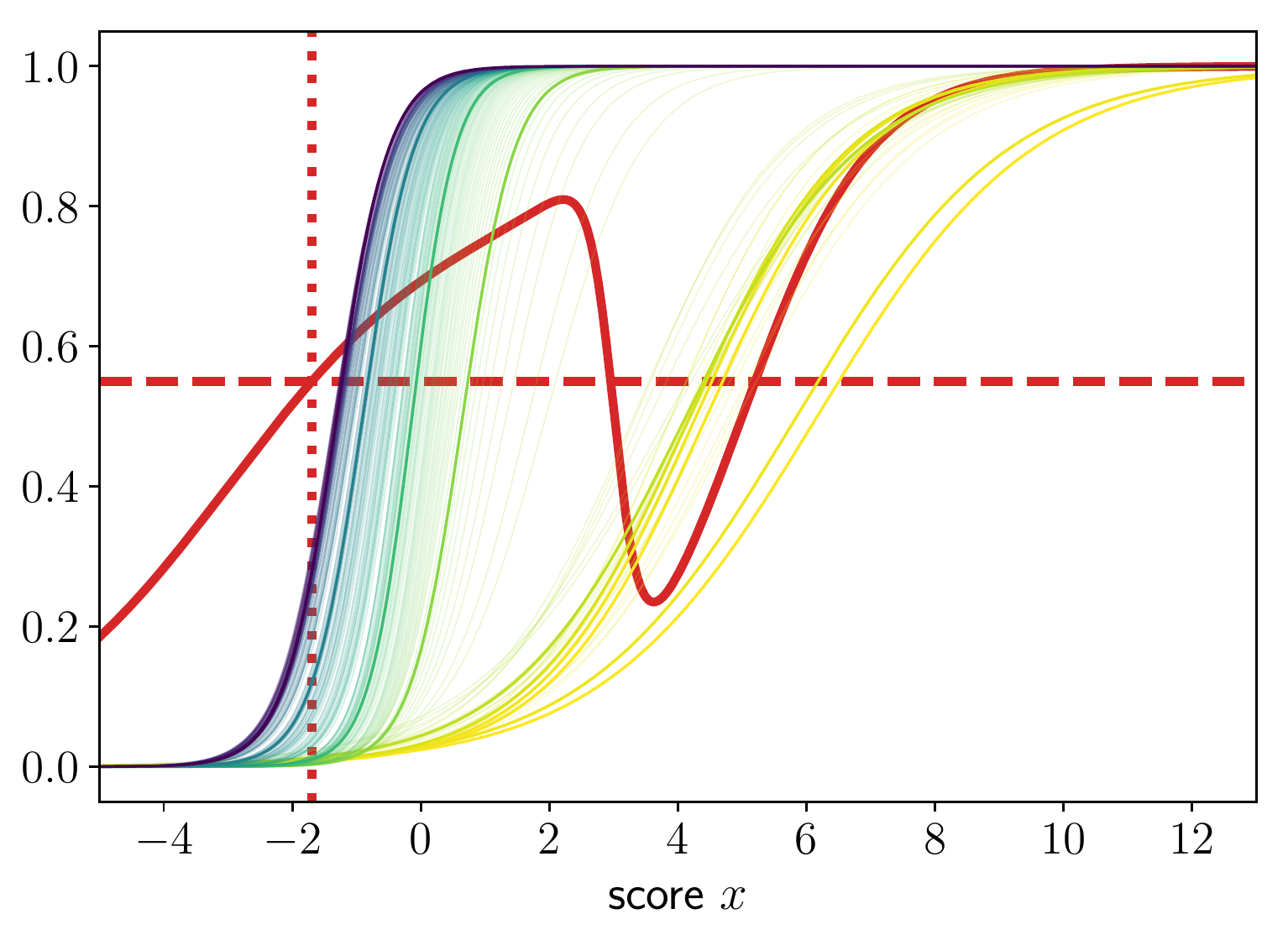}}
\\
\subfloat[\tiny\textbf{semi-logistic, $t'<t$}]{%
\includegraphics[width=\figwid \textwidth]{%
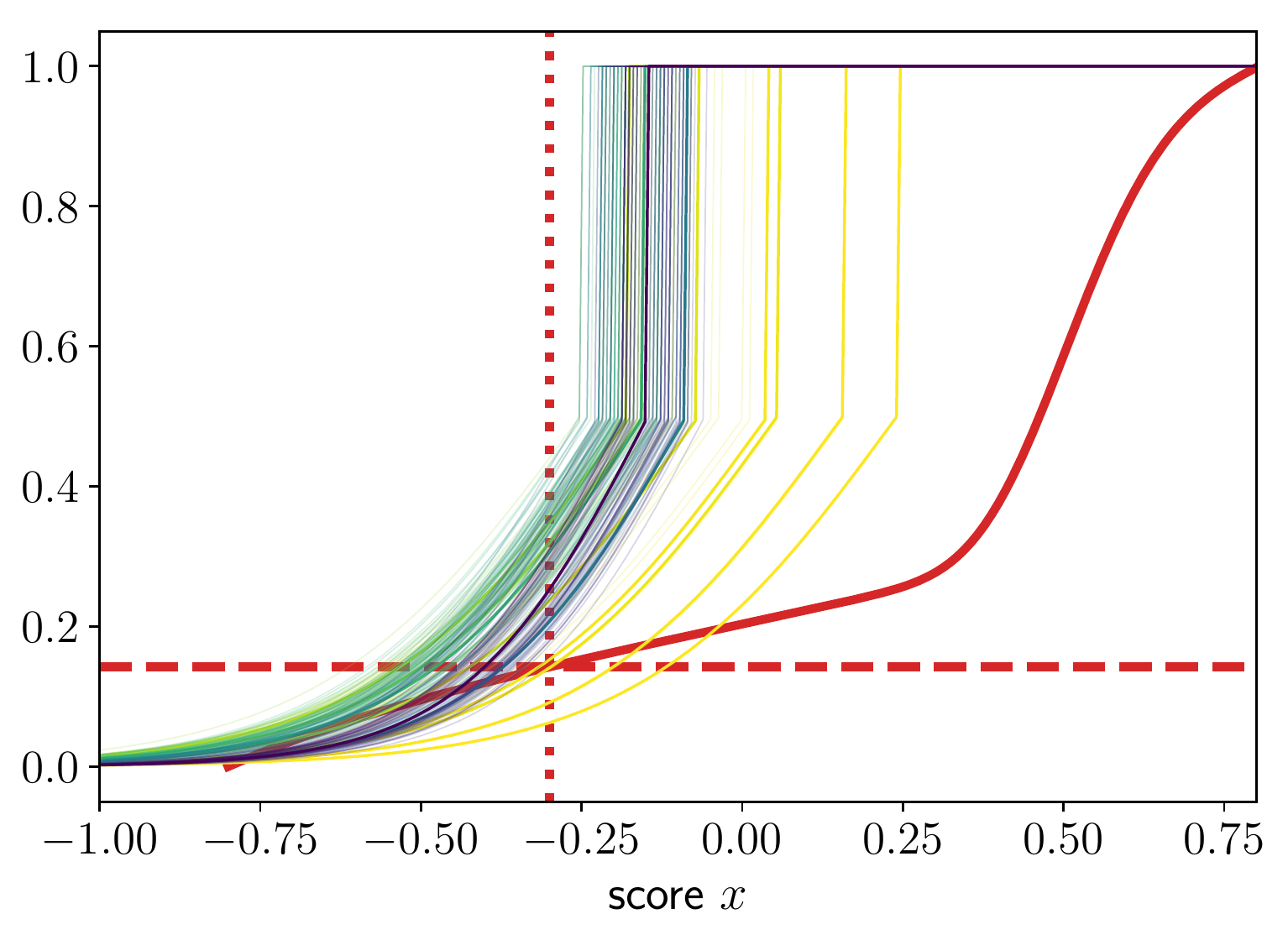}}
\hspace{0.5cm}
\subfloat[\tiny\textbf{semi-logistic, $t'<t$}]{%
\includegraphics[width=\figwid \textwidth]{%
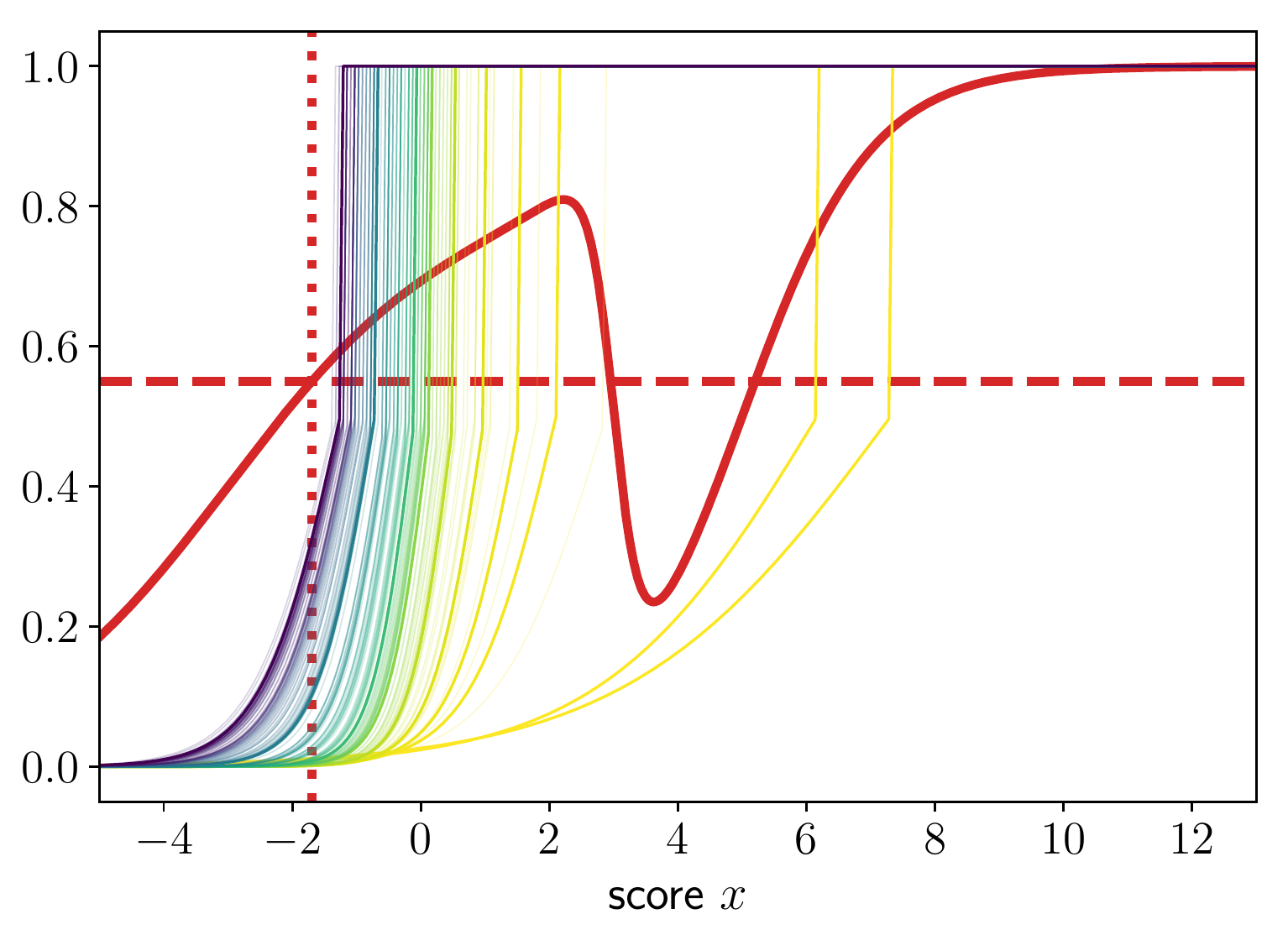}}
\\
\subfloat[\tiny\textbf{semi-logistic, $t-1$}]{%
\includegraphics[width=\figwid \textwidth]{%
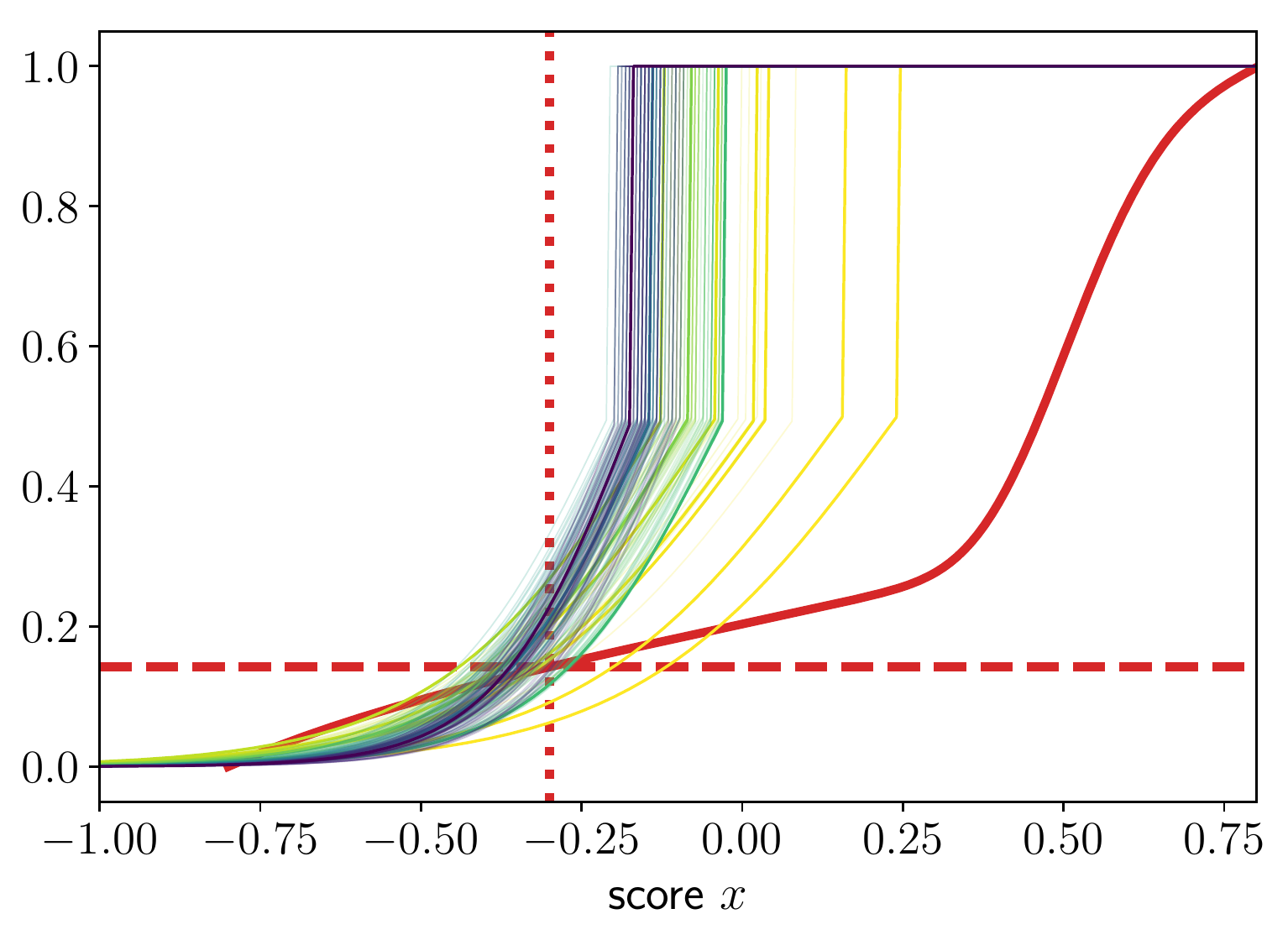}}
\hspace{0.5cm}
\subfloat[\tiny\textbf{semi-logistic, $t-1$}]{%
\includegraphics[width=\figwid \textwidth]{%
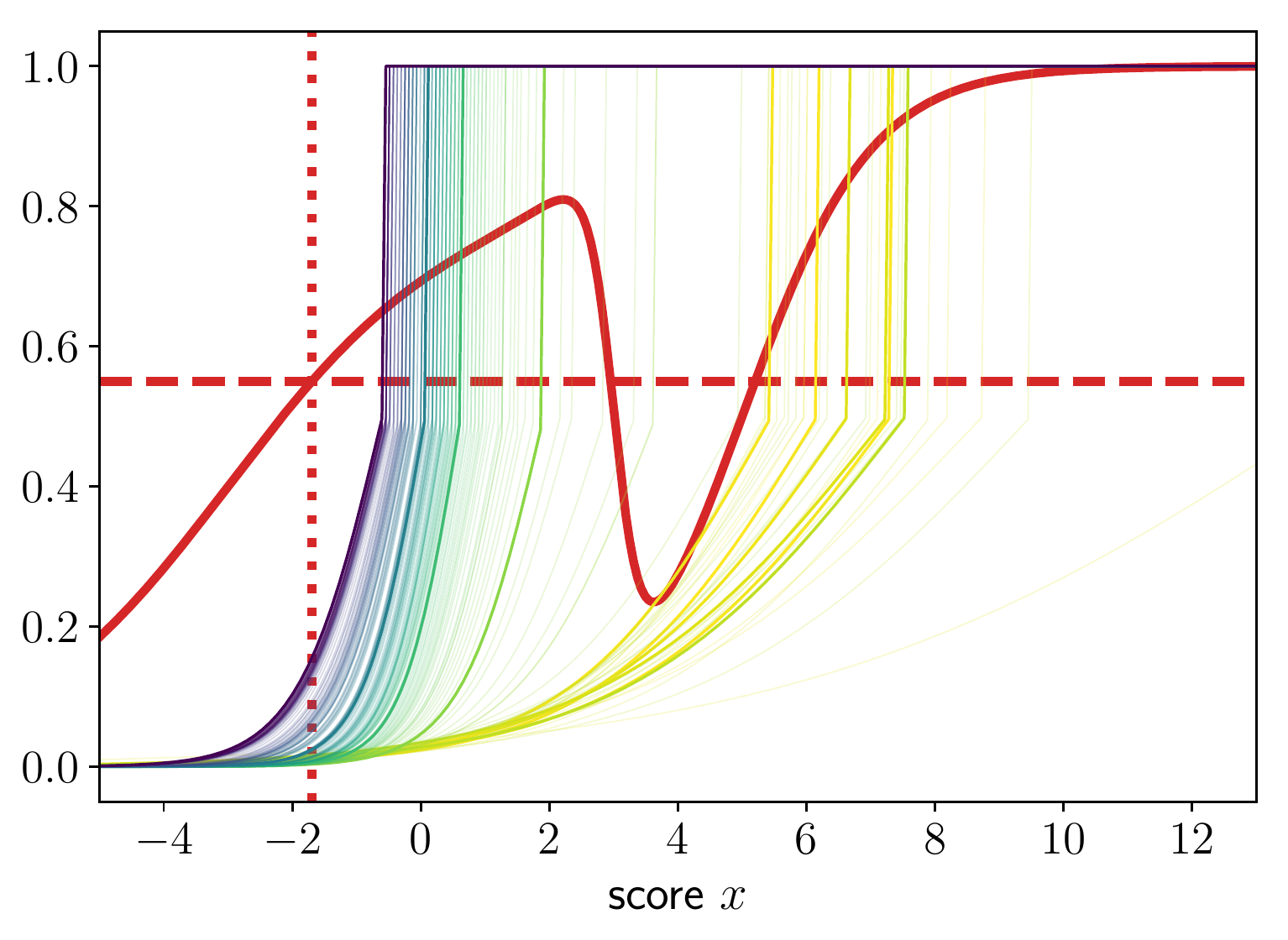}}
\caption[Learned predictive models for deterministic threshold rules and learned decision rules for the (semi-)logistic policies]{Learned predictive models for deterministic threshold rules and learned decision rules for the (semi-)logistic policies.
The columns correspond to the two synthetic settings.
We overlay the ground truth distribution $\dP(Y=1\given x)$ (red line), cost parameter $c$ (dashed, red), and optimal single decision boundary in $x$ within our model class (dotted, red).
We describe the plots in detail in the text.}
\label{fig:results_synthetic_evolution}
\end{figure}

\section{Experiments on synthetic data}

\xhdr{Setup}
The precise setup for the two different synthetic settings, illustrated in Figure~\ref{fig:synthetic-setting}, is as follows.
The only feature $x$ is a scalar score and $z \sim \mathrm{Ber}(0.5)$.
In the first setting, $x$ is sampled from a normal distribution $\mathcal{N}(\mu = 0.5 - z, \sigma = 1)$ truncated to $x \in [-0.8, 0.8]$, and the conditional probability $\dP(Y \given x)$ is strictly monotonic in the score and does not explicitly depend on $s$.
As a result, for any $c$, there exists a single decision boundary for the score that results in the optimal policy, which is contained in the class of logistic policies.
Note, however, that the score is not well calibrated, i.e., $\dP(Y \given x)$ is not directly proportional to~$x$.

In the second setting, $x \sim \mathcal{N}(\mu = 3 (0.5 - z), \sigma = 3.5)$.
Here, the conditional probability $\dP(Y \given x)$ crosses the cost threshold $c$ multiple times, resulting in two disjoint intervals of scores for which the optimal decision is $d=1$ (green areas).
Consequently, the optimal policy cannot be implemented by a deterministic threshold rule based on a logistic predictive model.
We show the best achievable single decision threshold in Figure~\ref{fig:synthetic-setting}.

\xhdr{Repeated figure}
First, in Figure~\ref{fig:results-synthetic-errs} we again show the contents of Figure~\ref{fig:results-synthetic} in the main text, but added effective utility and also show shaded regions for the 25th and 75th percentile over 30 runs.

\xhdr{Evolution of policies}
In Figure~\ref{fig:results_synthetic_evolution} we show for a representative run at $\lambda = 0$ how the different policies evolve in the two synthetic settings over time.
The two columns correspond to the two different synthetic settings.
For all policies, we show snapshots at a fixed number of logarithmically spaced time steps between $t=0$ and $t=200$.
For deterministic threshold rules, we show the logistic function of the underlying predictive model.
The vertical dashed line corresponds to the decision boundary in $x$.
For the logistic and semi-logistic policies, the lines correspond to $\pi_t(D=1 \given x)$, i.e., to the probability of giving a positive decision for a given input $x$.
Note that the semi-logistic policies have a discontinuity, because we do not randomize when the model believes $d = 1$ is a favorable decision with more than 50\% certainty.
For reference, we also show the true conditional distribution, the cost parameter, as well as the best achievable single decision boundary.

In the first setting, the exploring policies locate the optimal decision boundary, whereas the deterministic threshold rules, which are based on learned predictive models, do not, even though $\dP(Y = 1\given x)$ is monotonic in $x$ and has a sigmoidal shape.
The predictive models focus on fitting the rightmost part of the conditional well, but ignore the left region, from which they never receive data.

In the second setting, our methods explore more and eventually take positive decisions for $x$ right of the vertical dotted line in Figure~\ref{fig:synthetic-setting}, which is indeed the best achievable single threshold policy.
In contrast, non-exploring deterministic threshold rules again suffer from the same issue as in the first setting and converge to a suboptimal threshold at $x\approx 5$.
They ignore the left green region in Figure~\ref{fig:synthetic-setting} and do not overcome the dip of $\dP(Y=1\given x)$ below $c$, because they never receive data for $x \le 4$.

\begin{figure}
\centering
\def\figheight{3.5cm}
\includegraphics[width=\textwidth]{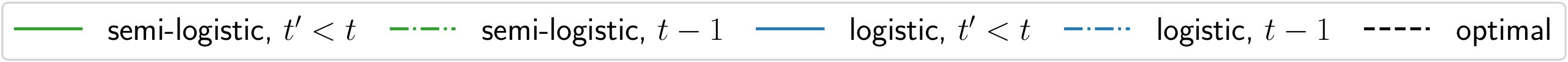}\\
\includegraphics[height=\figheight]{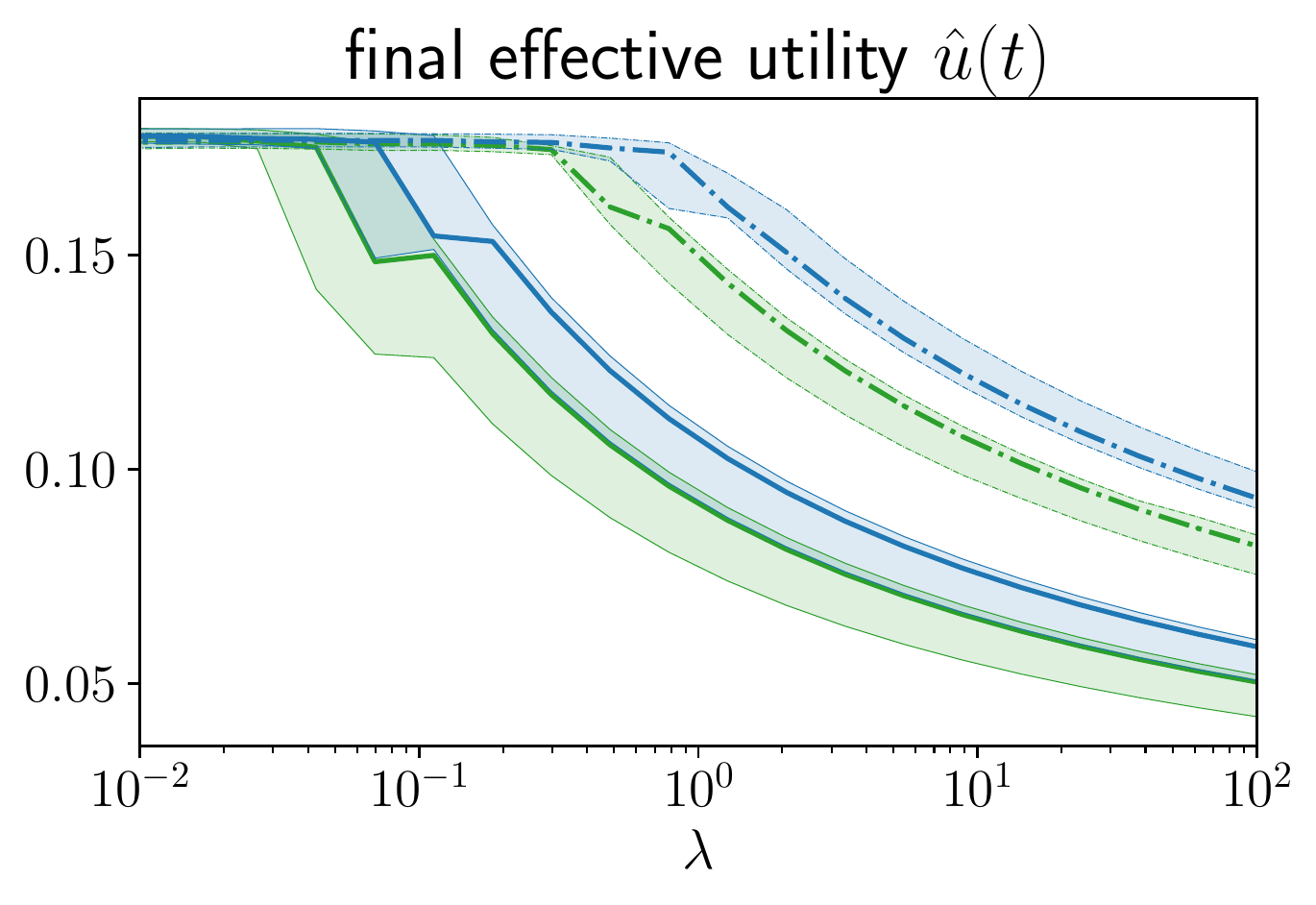}%
\hfill
\includegraphics[height=\figheight]{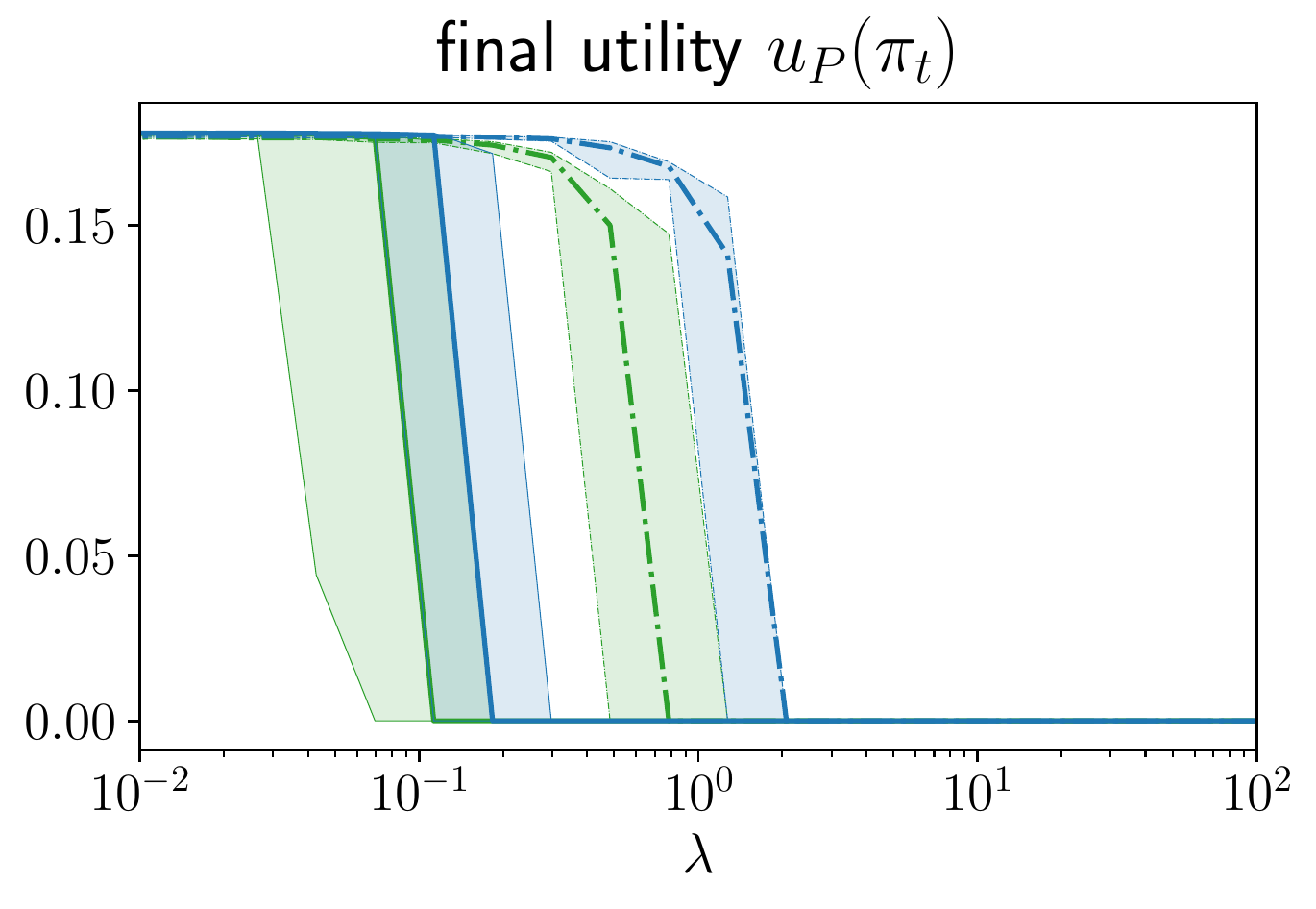}%
\hfill
\includegraphics[height=\figheight]{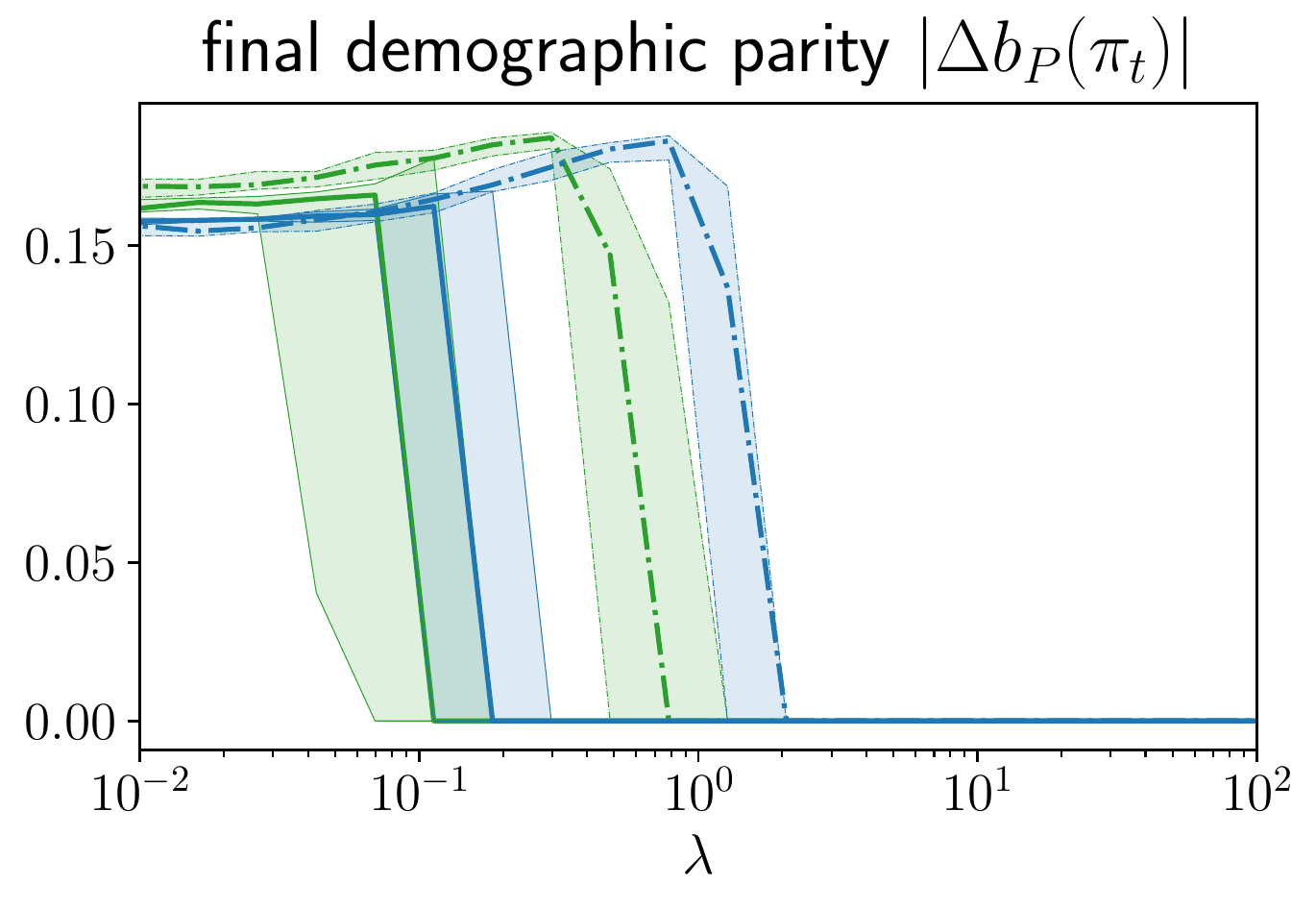}%
\\
\includegraphics[height=\figheight]{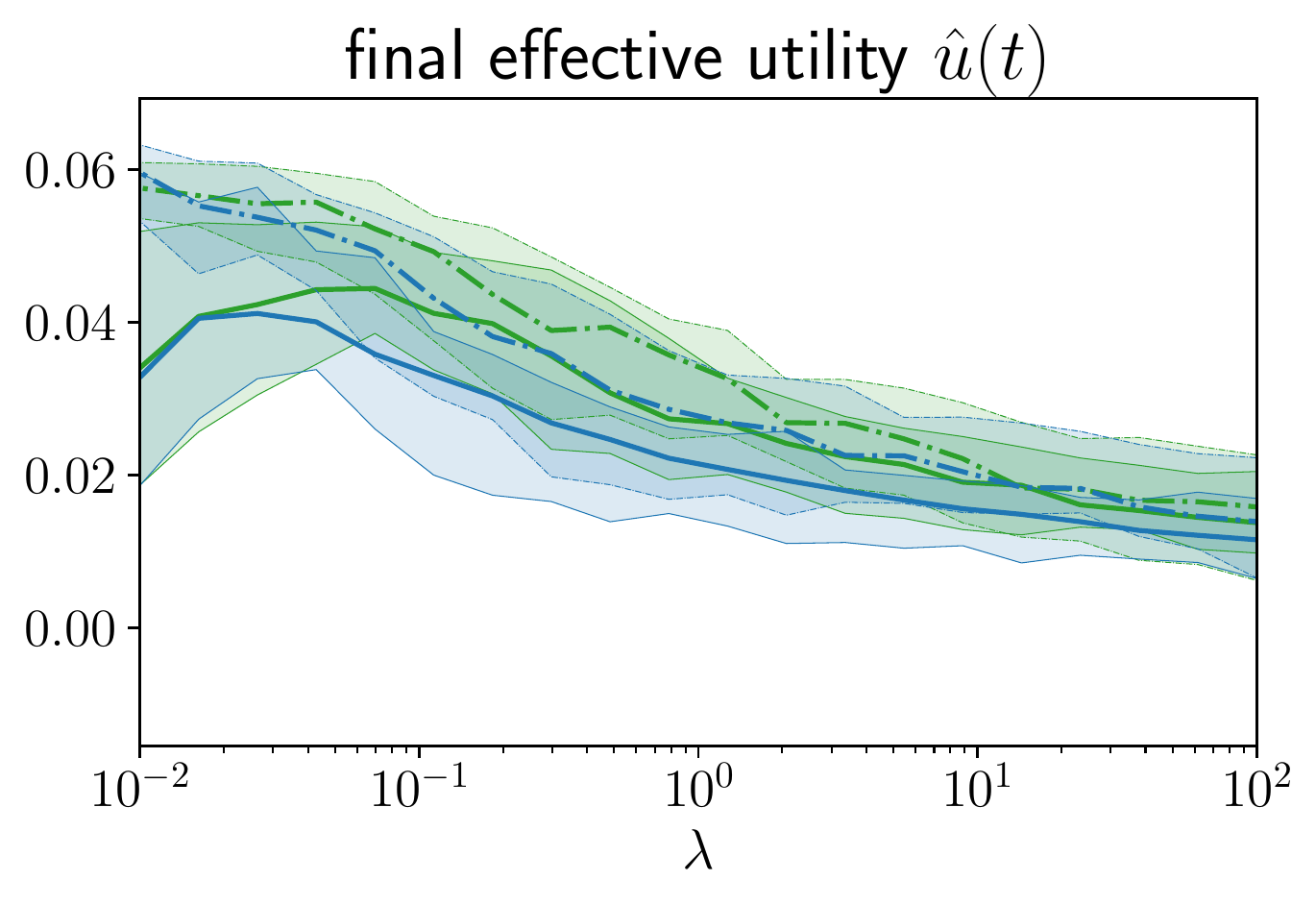}%
\hfill
\includegraphics[height=\figheight]{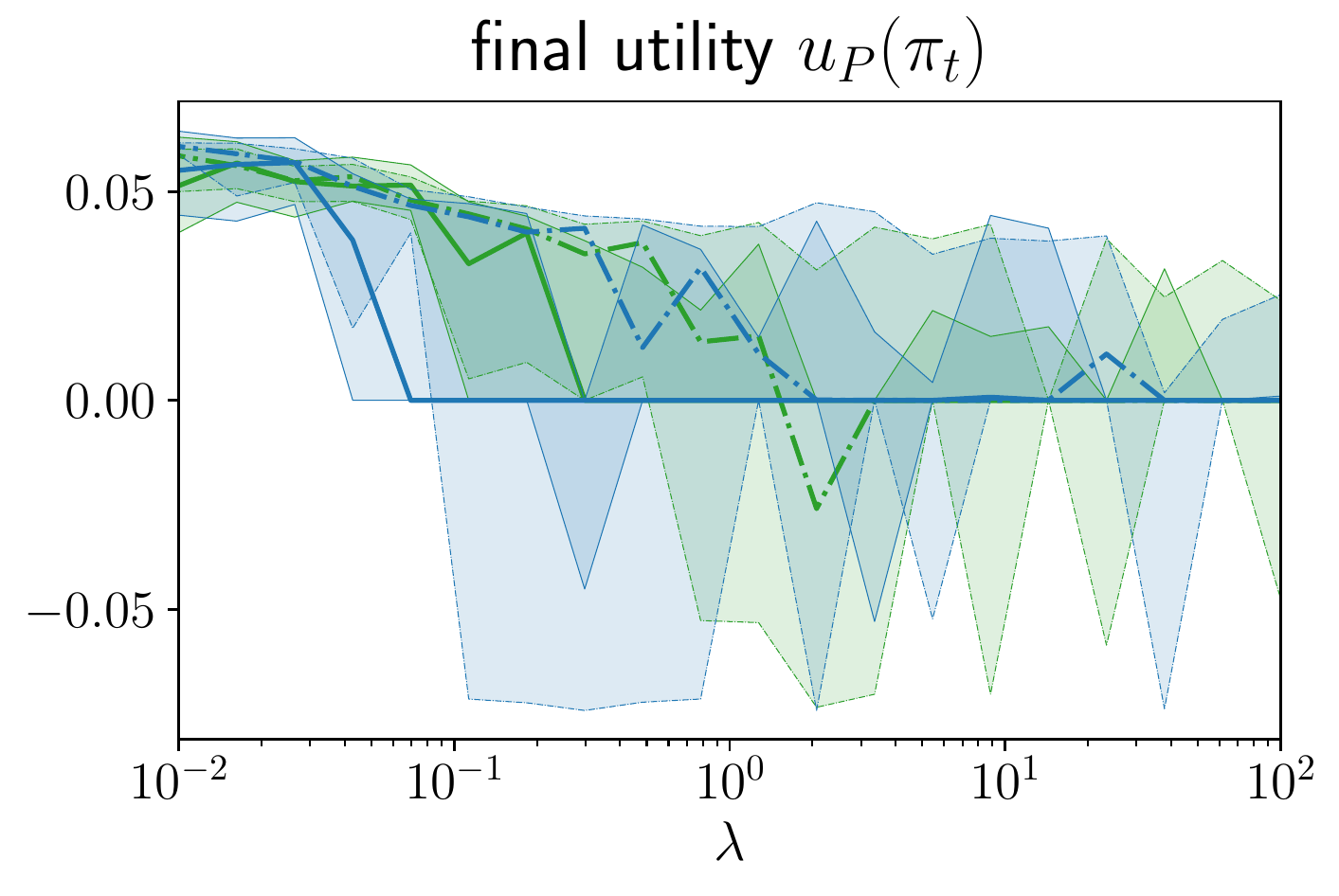}%
\hfill
\includegraphics[height=\figheight]{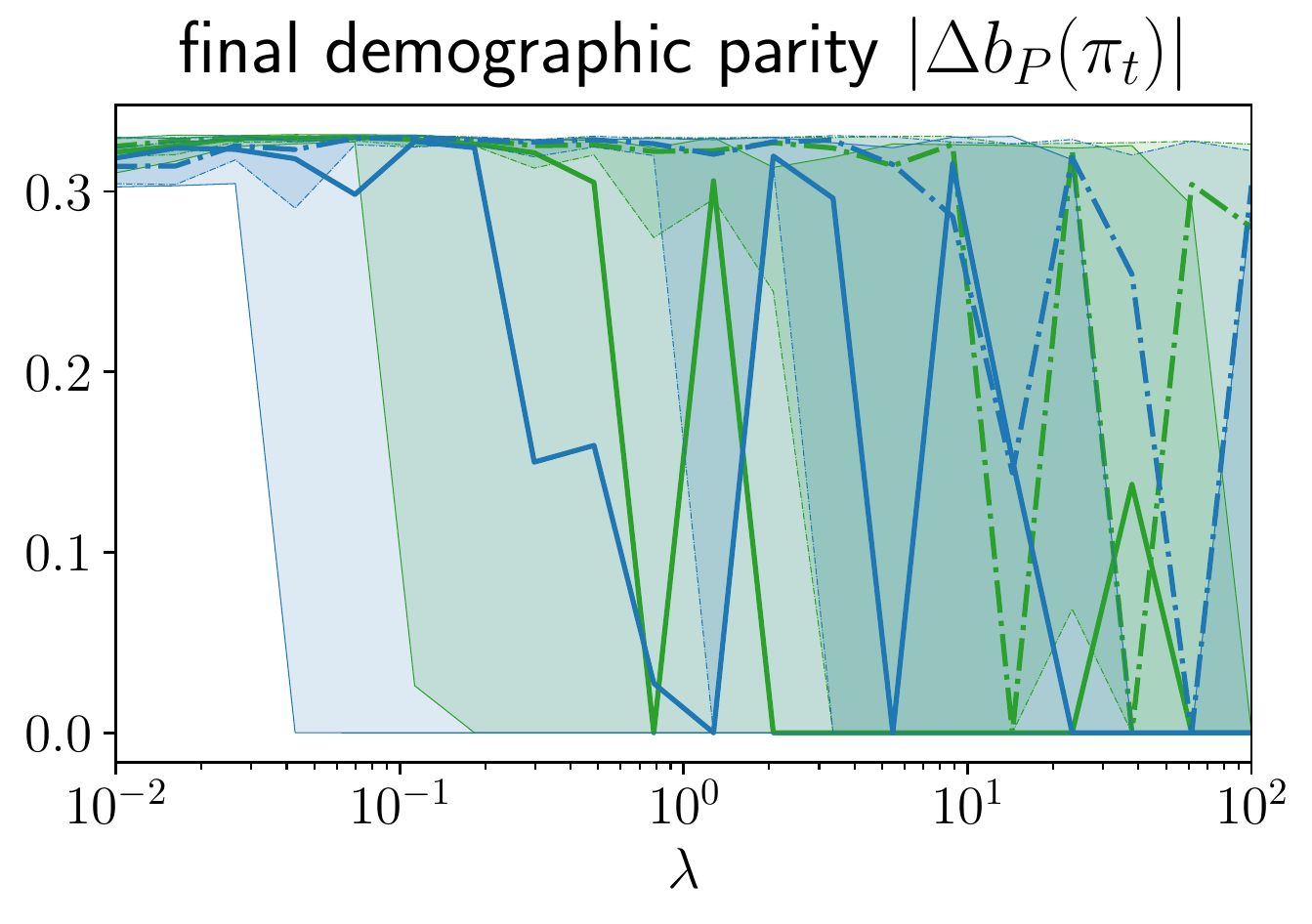}%
\caption[Utility and fairness of learning decisions with exploring policies in two synthetic settings when constraining demographic parity]{We show utility, effective utility, and demographic parity (columns) at the final time step $t=200$ as a function of $\lambda$ where we constrain demographic parity ($f(d, y) = d$).
The first row corresponds to the first setting and the second row corresponds to the second setting.}
\label{fig:results-synthetic-final-dp}
\end{figure}

\xhdr{Adding fairness constraints}
Figure~\ref{fig:results-synthetic-final-dp} shows how all metrics at the final time step $t = 200$ evolve as $\lambda$ is increased over the range $[10^{-0.5}, 10^4]$.
We use the benefit function for demographic parity in the fairness constraint, i.e., $f(d, y) = d$.
The first row corresponds to the first setting and the second row corresponds to the second setting.
In both cases, our approach achieves perfect fairness for sufficiently large $\lambda$ at the expected cost of a drop in (effective) utility.

\section{Experiments on real data}

\begin{figure}
\centering
\def\figheight{3.4cm}
\includegraphics[width=\textwidth]{legend-compas-time.pdf}\\
\includegraphics[height=\figheight]{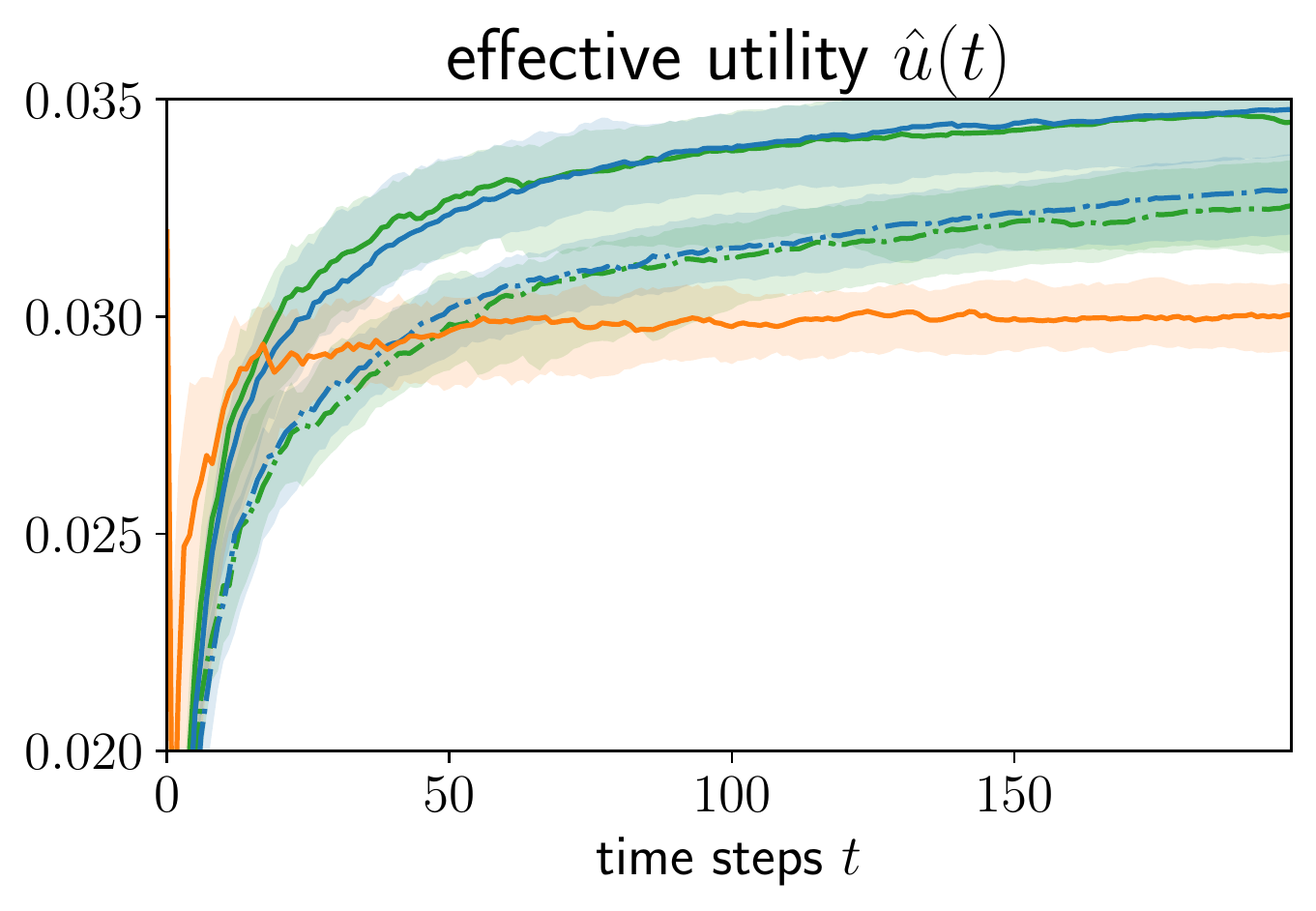}
\hfill
\includegraphics[height=\figheight]{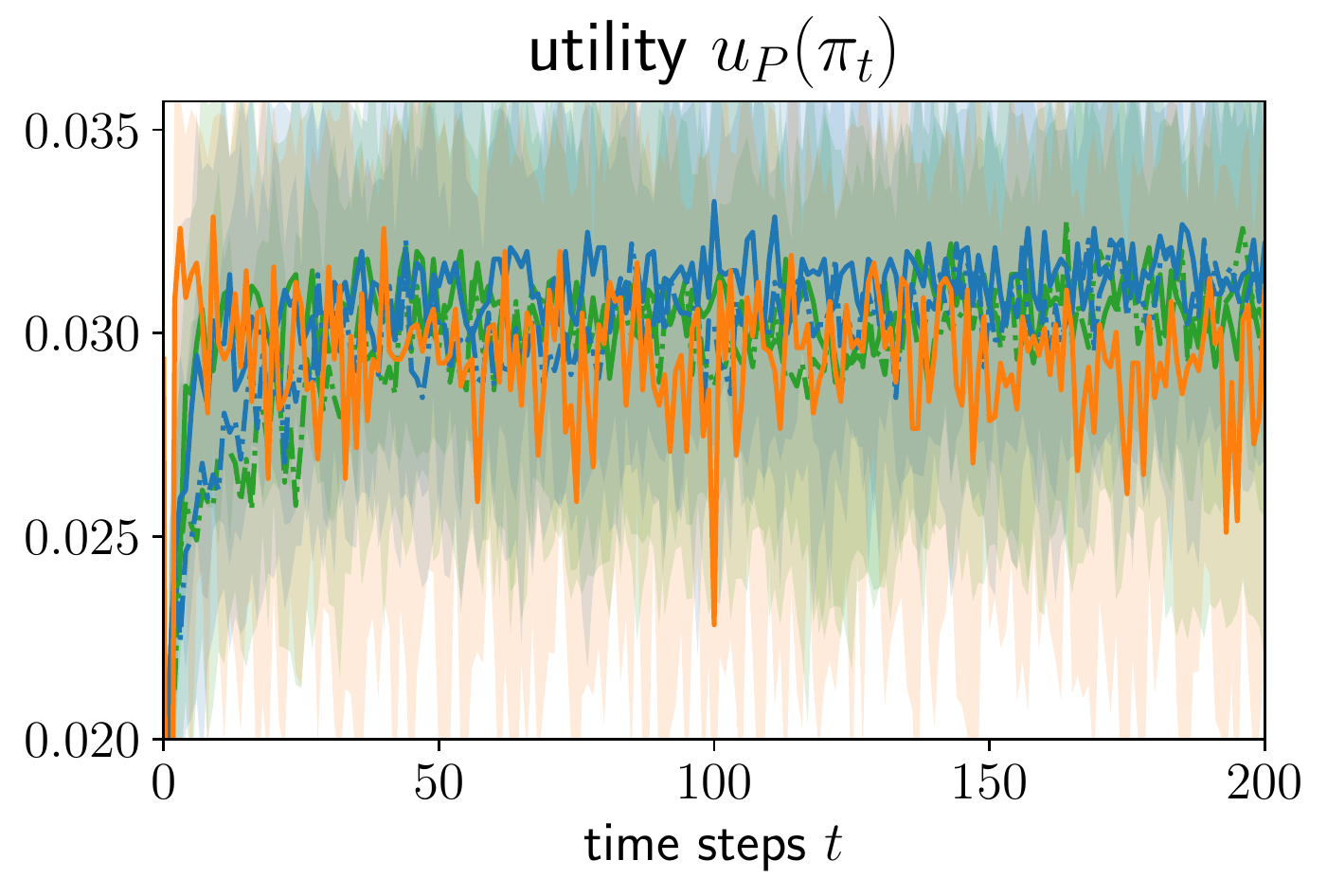}
\hfill
\includegraphics[height=\figheight]{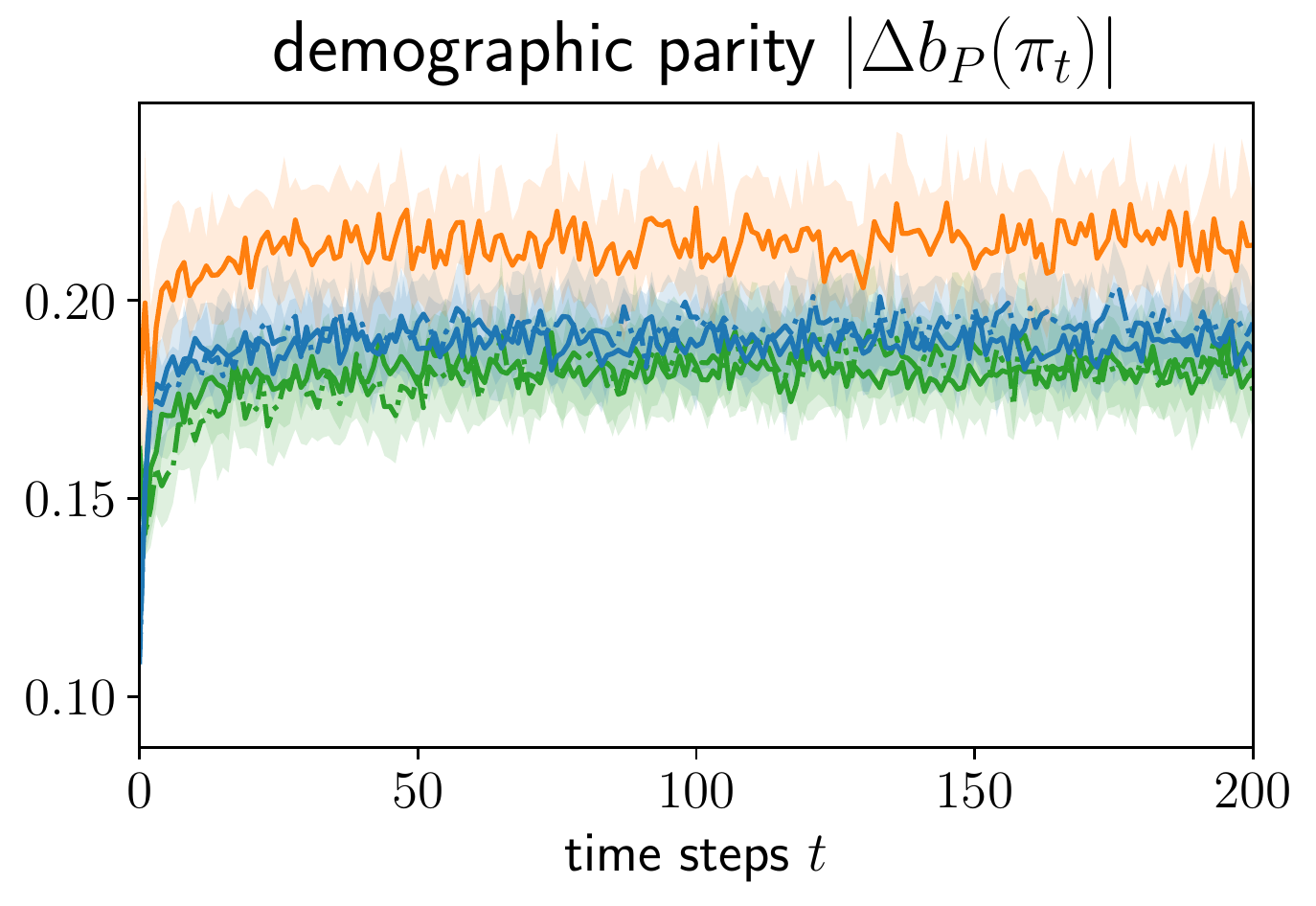}
\caption[Utility and fairness of learning decisions with exploring policies on COMPAS data]{Performance on COMPAS data.
We show the training progress for $\lambda = 0$, where all metrics are estimated on the held-out dataset.}
\label{fig:results-real-errs}
\end{figure}

First, in Figure~\ref{fig:results-real-errs} we again show the contents of Figure~\ref{fig:results-real-time} in the main text with shaded regions for the 25th and 75th percentile over 30 runs.
Analogously to Figure~\ref{fig:results-synthetic-final-dp}, we show the effect of enforcing fairness constraints in the COMPAS dataset in Figure~\ref{fig:results-real-final}.
The overall trends are similar to the results we have observed in the synthetic settings, reinforcing the applicability of our approach on real-world data.

\begin{figure}
\centering
\def\figheight{3.4cm}
\includegraphics[width=\textwidth]{legend-compas-lambda.pdf}\\
\includegraphics[height=\figheight]{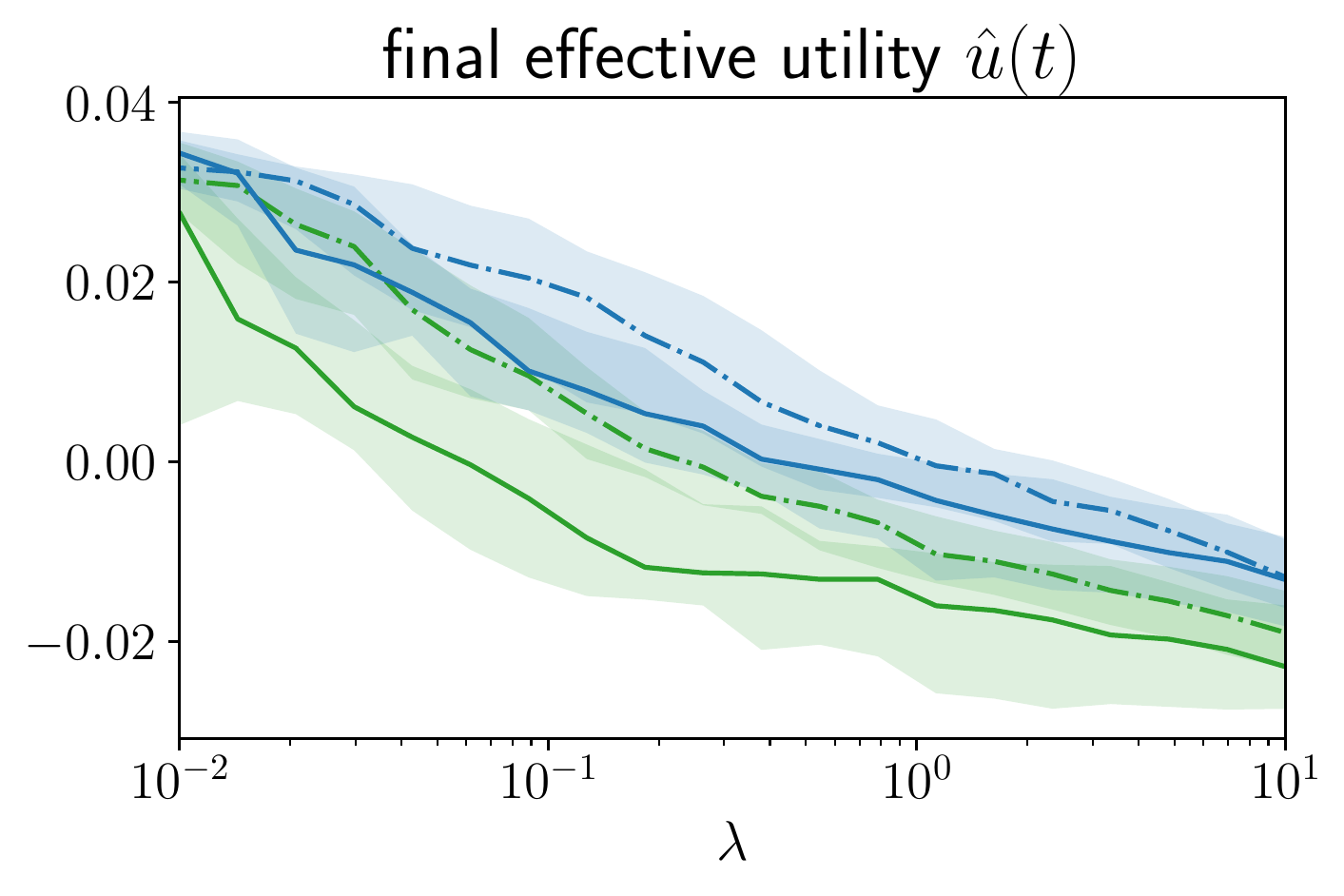}%
\hfill
\includegraphics[height=\figheight]{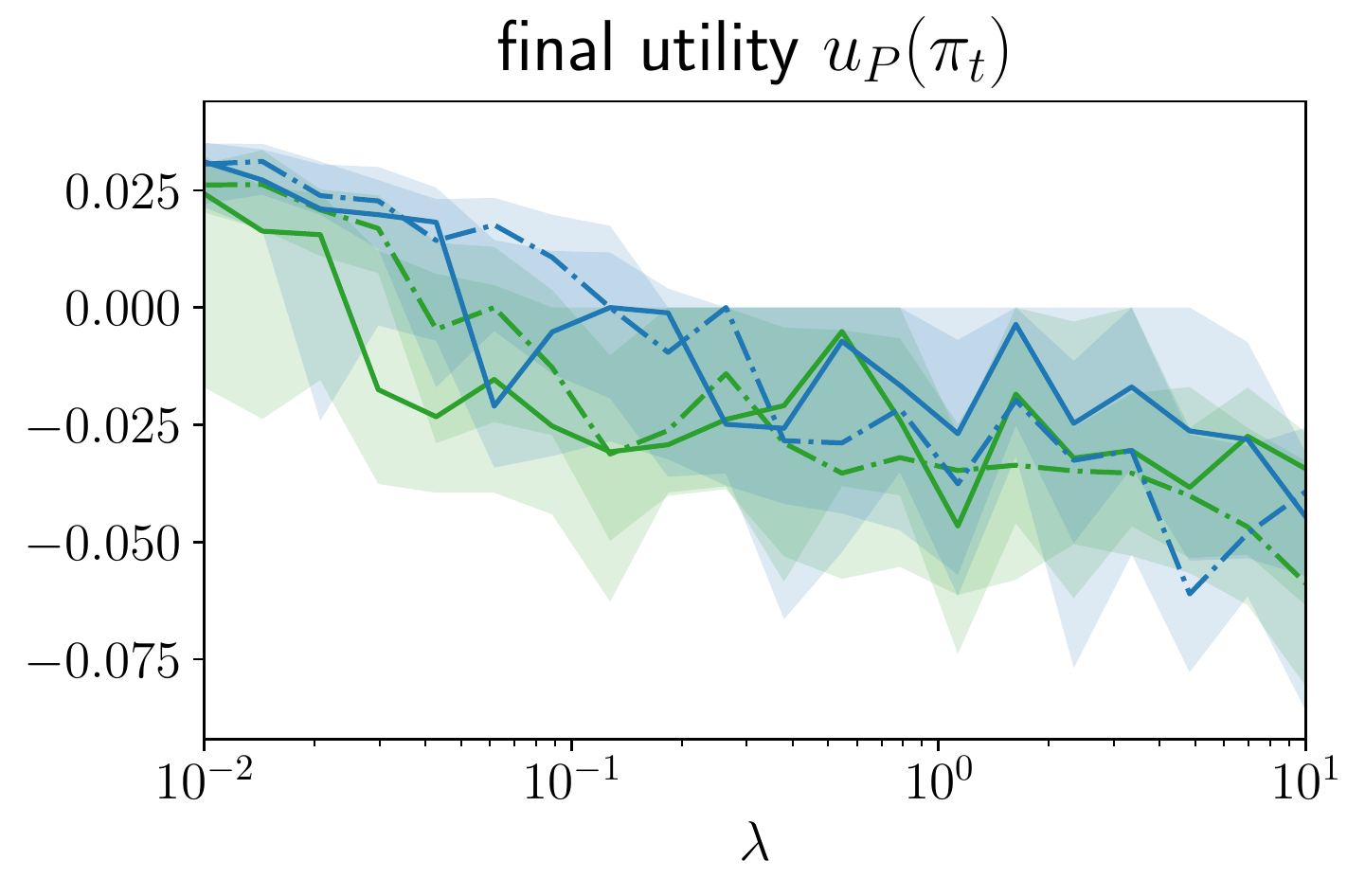}%
\hfill
\includegraphics[height=\figheight]{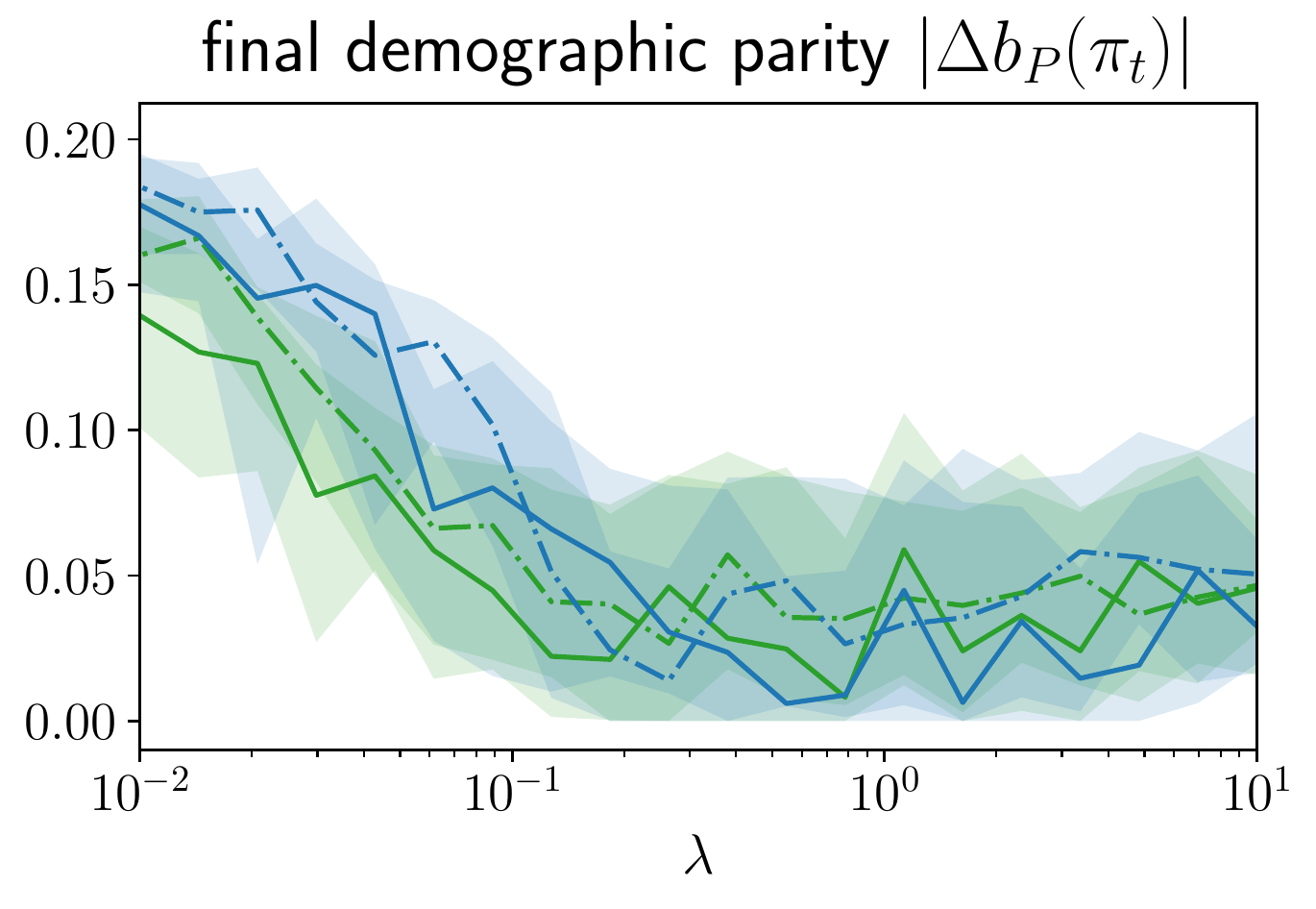}%
\caption[Utility and fairness of learning decisions with exploring policies on COMPAS data when enforcing demographic parity]{We show (effective) utility and demographic parity (columns) for the COMPAS dataset at the final time step $t=200$ estimated on the held-out dataset as a function of $\lambda$.}
\label{fig:results-real-final}
\end{figure}

\section{Parameter settings}
\label{sec:app:parameters}

The parameters used for the different experiments have been found by few iterations of manual trial.
The number of time steps is $T = 200$ for all datasets.
For the first synthetic setting we used $\alpha = 1$, $B = 256$, $M = 128$, $N = B \cdot M$, and $c \approx 0.142$ (chosen such that the optimal decision boundary is at $x = -0.3$).
For the second synthetic setting we used $\alpha = 0.5$, $B = 128$, $M = 32$, $N = B \cdot M$, and $c = 0.55$.
Here we also decay the learning rate by a factor of $0.8$ every 30 time steps.
For the COMPAS dataset we used $\alpha = 0.1$, $B = 64$, $M = 40\cdot B$, $N = B^2$, and $c = 0.6$.
While the initialization for the synthetic settings can be seen in Figure~\ref{fig:results_synthetic_evolution}, for COMPAS we trained a logistic predictive model on 500 i.i.d.\ examples for initializing policies and predictive models.

\end{appendices}

\end{document}